\documentclass[10pt]{article} 
\usepackage[preprint]{tmlr}

\usepackage[pagebackref=true,backref=page]{hyperref}
\usepackage[nohints]{minitoc}
\usepackage{graphicx}
\usepackage{array}
\usepackage{booktabs}
\usepackage{placeins}
\usepackage{capt-of}
\usepackage{my_equations}
\usepackage{amsmath,amsfonts,amssymb,bbm,bm}
\usepackage{dsfont} 
\usepackage{upgreek} 
\usepackage{xcolor}
\usepackage{amsthm}
\usepackage{thmtools}
\usepackage{mathtools}

\definecolor{niceRed}{RGB}{190,38,38}
\definecolor{niceYellow}{HTML}{f5b400}
\definecolor{blueGrotto}{HTML}{059DC0}
\definecolor{royalBlue}{HTML}{057DCD}
\definecolor{navyBlue}{HTML}{0B579C}
\definecolor{limeGreen}{HTML}{81B622}
\definecolor{nicePurple}{HTML}{9c27b0}
\definecolor{lightRoyalBlue}{HTML}{def2ff}
\definecolor{gold}{HTML}{ffa300}
\definecolor{frameblue}{RGB}{0,123,255}     
\definecolor{framebg}{RGB}{241,244,247}     
\definecolor{shadecolor}{gray}{0.90}
\definecolor{plum}{HTML}{9c27b0}

\usepackage{enumitem}
\usepackage{mdframed} 
\usepackage{amsthm}
\usepackage{thmtools}
\definecolor{shadecolor}{gray}{0.90}
\declaretheoremstyle[
headfont=\normalfont\bfseries,
notefont=\mdseries, notebraces={(}{)},
bodyfont=\normalfont,
postheadspace=0.5em,
spaceabove=0.5em,
spacebelow=0.5em,
mdframed={
 skipabove=8pt,
 skipbelow=8pt,
 hidealllines=true,
 backgroundcolor={shadecolor},
 innerleftmargin=4pt,
 innerrightmargin=4pt}
]{shaded}
\declaretheoremstyle[
  bodyfont=\normalfont  
]{nonitalic}
\newmdenv[
    backgroundcolor=framebg,
    skipabove=8pt,
    skipbelow=8pt,
    roundcorner=5pt,
    linewidth=0pt,
    innertopmargin=4pt
]{myframe}

\declaretheorem[within=section]{style}

\declaretheorem[name=Main Theorem,sibling=style]{main theorem}

\declaretheorem[sibling=style]{corollary}
\declaretheorem[style=nonitalic,sibling=style]{lemma}

\declaretheorem[sibling=style]{proposition}

\declaretheorem[style=nonitalic,sibling=style, numbered=no]{remark}
\declaretheorem[style=nonitalic,sibling=style, numbered=no]{proof sketch}

\usepackage[capitalize,nameinlink]{cleveref}
\crefname{assumption}{assumption}{assumptions}
\Crefname{assumption}{Assumption}{Assumptions}
\crefname{lemma}{lemma}{lemmas}
\Crefname{lemma}{Lemma}{Lemmas}
\crefname{theorem}{theorem}{theorems}
\Crefname{theorem}{Theorem}{Theorems}
\crefname{main theorem}{Main Theorem}{Main Theorems}
\Crefname{main theorem}{Main Theorem}{Main Theorems}
\crefname{definition}{definition}{definitions}
\Crefname{definition}{Definition}{Definitions}
\crefname{corollary}{corollary}{corollaries}
\Crefname{corollary}{Corollary}{Corollaries}
\crefname{proposition}{proposition}{propositions}
\Crefname{proposition}{Proposition}{Propositions}
\crefname{fact}{fact}{facts}
\Crefname{fact}{Fact}{Facts}
\crefname{identity}{identity}{identities}
\Crefname{identity}{Identity}{Identities}
\crefname{function}{function}{functions}
\Crefname{function}{Function}{Functions}

\DeclareMathOperator*{\argmin}{arg\,min}

\DeclareMathOperator{\Tr}{Tr}
\DeclareMathOperator{\sech}{sech}

\newcommand{\R}{\mathbb{R}}

\newcommand{\p}{\uppi} 
\renewcommand{\Pr}{\mathbb{P}}
\newcommand{\E}{\mathbb{E}}

\newcommand{\Var}{\mathrm{Var}}

\DeclareMathOperator{\KL}{\mathsf{KL}}
\DeclareMathOperator{\TV}{\mathsf{TV}}
\DeclareMathOperator{\FI}{\mathsf{FI}}

\newcommand{\calA}{\mathcal{A}}

\newcommand{\calC}{\mathcal{C}}
\newcommand{\calD}{\mathcal{D}}
\newcommand{\calE}{\mathcal{E}}

\newcommand{\calH}{\mathcal{H}}

\newcommand{\calJ}{\mathcal{J}}
\newcommand{\calK}{\mathcal{K}}
\newcommand{\calL}{\mathcal{L}}

\newcommand{\calN}{\mathcal{N}}
\newcommand{\calO}{\mathcal{O}}
\newcommand{\calP}{\mathcal{P}}

\newcommand{\calR}{\mathcal{R}}
\newcommand{\calS}{\mathcal{S}}

\newcommand{\calW}{\mathcal{W}}

\newcommand{\sgn}{\textnormal{sgn}}

\def\<{\langle}
\def\>{\rangle}

\renewcommand{\d}{\mathrm{d}}

\newcommand{\bu}{\mathbf{u}} 
\newcommand{\bx}{\mathbf{x}} 

\newcommand{\mathe}{\mathrm{e}}

\newcommand{\ind}{\mathrel{\perp\!\!\!\perp}} 

\hypersetup{colorlinks=true,linkcolor=royalBlue,citecolor=navyBlue,urlcolor=navyBlue}

\title{\raggedright
Connecting Score Matching, Maximum Likelihood, and\\
Expectation-Maximization in Mixed Linear Regression}

\author{\name Zhankun Luo \email luo333@purdue.edu \\
      \name Abolfazl Hashemi \email abolfazl@purdue.edu \\
      \addr School of Electrical and Computer Engineering\\
      Purdue University, West Lafayette, IN, USA}

\begin{document}

\maketitle

\begingroup
\renewcommand{\thefootnote}{}
\footnotetext{LLMs were used as a general-purpose assistive tool during manuscript preparation
for language editing and code
assistance. The authors independently verified theoretical and experimental results and take full
responsibility for the manuscript.}
\endgroup
\vspace{-1.5em} 
\begin{abstract}
We study variance-preserving diffusion of the response in mixed linear regression (MLR) with unknown mixing weights.
Our analysis separates the statistical guarantees of score matching from the loss geometry and optimization signal at a fixed diffusion noise level. 
The KL divergence links the denoising score matching
objective integrated over the diffusion path with the likelihood and a terminal
discrepancy. Under a vanishing terminal retained-signal schedule, the
score matching estimator converges to the ground truth MLR parameters and
is asymptotically normal with covariance equal to the inverse Fisher
information, as is the maximum-likelihood estimator (MLE).
At a fixed scale of the diffusion noise level, we derive a decomposition linking the score matching loss to cross-entropy and Expectation-Maximization (EM) operators. 
This decomposition yields an EM-related low-noise gradient expansion with additional correction terms of latent variance. 
For the score matching loss, we derive exact preconditioned gradient descent updates involving only EM operators.
In the high-noise limit, we further characterize the convergence rate of gradient descent on this limiting loss. Along fixed high signal-to-noise ratio rays, the score matching imbalance gradient and the latent-variance term tend to zero pointwise.
Numerical experiments illustrate our theoretical findings and statistical guarantees.
\end{abstract}

\doparttoc[n] 
\faketableofcontents 
\vspace{-1.5em} 
\section{Introduction}\label{sec:intro}

For a response law with several modes, its mean can fall in a region of low
probability and fail to describe a plausible outcome
\citep{bishop1994mixture,han2022card}. Classically, structured latent-variable models represent
the modes parametrically and the parameters can be estimated via maximum likelihood or Expectation-Maximization
(EM). Recently, diffusion models learn scores of Gaussian-smoothed
response laws~\citep{ddpm,song2021scorebased}, retaining the covariates while
adding noise to the response. 
\citet{batzolis2021conditional} justified conditional denoising score
estimation, while \citet{fu2024unveil} and \citet{tang2025conditional}
established conditional-score and distribution-estimation guarantees under
their respective assumptions.
Connecting these two perspectives raises two natural questions: when score
matching is applied to a structured latent-variable model, can it achieve
statistically efficient recovery of the underlying parameters from finite
samples, and how does its optimization relate to maximum likelihood and EM?

To study these questions, we consider symmetric two-component mixed linear regression (MLR)
with unknown regression parameters and mixing weights~\citep{quandt1972switching,de1989mixtures,zilber2023imbalanced}.
The regression vectors are negatives of one another, while the mixing weights
may differ. For a covariate vector \(\bx\in\R^d\), the response \(y_0\in\R\)
has conditional law
\begin{equation}
    y_0\mid\bx\sim\pi^\ast(1)\,\calN\!\left(\langle\theta^\ast,\bx\rangle,1\right)+\pi^\ast(2)\,\calN\!\left(-\langle\theta^\ast,\bx\rangle,1\right),
\end{equation}
where the ground truth regression vector \(\theta^\ast\in\R^d\) determines
the covariate-dependent component means \(\pm\langle\theta^\ast,\bx\rangle\).
The parameter \(\nu^\ast\in\R\) characterizes the imbalance of the ground truth
mixing weights \(\pi^\ast=(\pi^\ast(1),\pi^\ast(2)):=((1+\tanh\nu^\ast)/2,(1-\tanh\nu^\ast)/2)\).
From observed covariate--response pairs with unobserved component labels,
we seek to recover \(\theta^\ast\) and \(\nu^\ast\) using score matching,
up to a sign change that exchanges the component labels.
Under Gaussian diffusion of the response, the conditional score and the
posterior of the latent component label remain explicit at every noise scale.
This tractability allows us to analyze both the statistical guarantees of
path-integrated score matching across noise scales, as well as the optimization behaviors
of both the path-integrated and fixed-noise-scale score matching objectives. 
In particular, we naturally ask the following questions.

\begin{myframe}[nobreak=true]
\textbf{Questions:} 
\textit{For this MLR model, under what conditions does the path-integrated score matching estimator recover the regression parameters and mixing weights? \\
How is the score matching loss related to Expectation-Maximization, and how do the noise level and signal strength affect the optimization dynamics and parameter information retained by score matching?}
\end{myframe}

\paragraph{Score matching and mixture learning.}

Score matching was linked early to
likelihood derivatives under Gaussian smoothing and to denoising
objectives~\citep{hyvarinen2005estimation,lyu2009interpretation,vincent2011connection}.
Subsequent variational and likelihood analyses related diffusion score
objectives to log-likelihood under their respective reverse models and
weightings~\citep{huang2021variational,song2021maximum}, while path-integrated score objectives for density and parameter
estimation were studied in~\citet{li2024density,chewi2025arxiv}. In particular,
\citet{chewi2025arxiv} proved asymptotic normality with inverse
Fisher-information covariance via a likelihood-bridge identity under assumed
consistency and local likelihood regularity. For Gaussian mixtures,
\citet{chen2024learning,gatmiry2024learning} provided distribution-learning
guarantees, \citet{shah2023learning} analyzed parameter recovery in balanced
spherical mixtures via a high-noise spectral stage and a low-noise comparison to
the M-step of EM, and \citet{ozkara2025spire} studied mixing-weight estimation at
a single diffusion noise level with fixed, known component means.

In parallel, classical and modern analyses of MLR addressed identifiability,
spectral and moment methods, convex relaxations, alternating minimization,
approximate message passing, and
EM~\citep{hennig2000identifiability,chaganty2013spectral,zhong2016mixed,yi2014alternating,yudong2018trans,tan2023mixed,balakrishnan2017statistical,dana2019estimate2mix,kwon2021minimax,kwon2024global}.
Within imbalanced MLR specifically, \citet{zilber2023imbalanced} treated unknown
mixing weights, and \citet{luo2025structural,luo2026characterizing}
characterized the coupled EM operators governing the regression parameters and
mixing-weight imbalance.

Bridging these two lines of work to answer the questions above exposes two gaps.
First, on the statistical side, \citet{ozkara2025spire} studied
mixing-weight estimation in a two-component Gaussian mixture, fixing both the
diffusion noise scale and the component means in their theoretical analysis, 
whereas the analysis for path-integrated score matching~\citep{chewi2025arxiv}
assumed consistency and local likelihood regularity a priori. 
This leaves open whether, when both the
regression parameters and mixing weights in MLR are unknown, the
path-integrated score matching estimator achieves consistency and asymptotic
normality with inverse Fisher-information covariance under weaker assumptions,
without imposing consistency or local likelihood regularity a priori.
Second, on the optimization side, \citet{shah2023learning} studied Gaussian
mixtures with known equal mixing weights and isotropic component covariances,
expressing the two-component score matching update as a relaxed EM step plus a
correction bounded to establish low-noise contraction under separation and
initialization conditions. Extending this connection to MLR with unknown regression parameters and mixing weights calls for an exact
relationship between the score matching loss and its gradient and
the coupled EM operators for unknown regression parameters and mixing weights.

\paragraph{Score blindness and the role of noise.}
For score matching at any fixed diffusion noise
level, increasing component separation can weaken the score's sensitivity to
mixing weights. \citet{wenliang2020blindness} demonstrated this score blindness
in one-dimensional Gaussian mixtures whose Fisher divergence vanished as
separation grew despite distinct mixing weights. \citet{zhang2022healing}
proposed a mixture Fisher divergence to mitigate this failure, while
\citet{koehler2023statistical} related the statistical efficiency of score
matching to isoperimetry, which degrades across separated modes. By contrast,
for population EM in noiseless MLR, \citet{luo2024unveiling} showed that the
mixing-weight update was independent of the current mixing-weight imbalance, yet
still carried information about the ground truth mixing weights through the
alignment between the current and ground truth regression parameters.

Recent multi-scale and tempering analyses sought to overcome separation-induced
insensitivity across noise levels. Using a continuously tempered generalized
score matching objective related to annealed score matching, \citet{qin2024fit}
proved a polynomial asymptotic bound on the statistical complexity of fitting
component means in finite Gaussian mixtures with known mixing weights and a
known shared covariance. For mixing weights of a two-Gaussian family with known
component parameters, \citet{schwienhorst2026diffusiondenoising} used
time-integrated denoising and showed how horizon tuning preserved mixing-weight
information when the components were well separated, while
\citet{dennehy2026mixtureweights} related diffusion-score accuracy across
mode-overlapping intermediate scales to mixing-weight recovery from generated
samples.

None of these works addresses joint estimation when both the component means
and mixing weights are unknown for score matching, as in MLR with covariate-dependent means. For
MLR with jointly unknown regression parameters and mixing weights, we track the
pointwise behavior of the mixing-weight gradient as signal strength grows at a
fixed diffusion scale, characterize the low- and high-noise limits of the
fixed-scale loss and the parameter information they retain, and analyze the
global gradient-descent dynamics induced by multiplicative coupling between the
regression and imbalance parameters at high noise.

\subsection{Technical Overview and Contributions}\label{subsec:technical_overview}

We establish the statistical guarantees by combining the relative-entropy identity
\citep{stam1959some}, the diffusion--likelihood bridge
\citep{chewi2025arxiv,chewi2025ddpm}, and standard M-estimation theory
\citep{van2000asymptotic}.
Also, we analyze the optimization dynamics of path-integrated and fixed-scale score matching losses through the coupled EM operators
\citep{dempster1977maximum,weinberger2022algorithm,luo2025structural}, discrete Lyapunov bounds, Kurdyka--\L{}ojasiewicz gradient inequalities~\citep{josz2026computing}, and center-stable-manifold argument~\citep{lee2016gradient}.
Our contributions are the following results and theoretical arguments.
\begin{itemize}
    \item \textbf{Statistical Guarantees for Path-Integrated Score Matching: Consistency and Asymptotic Normality}
    (Section~\ref{sec:asymptotic}).
    Using path-integrated score matching, we jointly estimate the regression
    parameters and mixing weights in MLR. By combining the empirical likelihood
    bridge with M-estimation, we prove consistency up to exchange of the two
    component labels and, after this alignment, asymptotic normality with the same
    inverse Fisher-information covariance as maximum likelihood
    (Main Theorem~\ref{thm:consistency_asymptotic_normality}). These guarantees
    hold for any Gaussian or non-Gaussian covariate law with an almost everywhere
    positive density, finite second moments for consistency, and finite fourth
    moments for asymptotic normality, provided the ground truth regression
    parameters are nonzero, both ground truth mixing weights are positive, and the
    signal remaining at the final diffusion time tends to zero as the sample size
    grows. We further bound the difference between the score matching and
    maximum-likelihood estimators fitted to the same sample in terms of this
    terminal remaining signal (Proposition~\ref{prop:paired_sm_mle}).

    Whereas \citet[Theorem~4.1]{chewi2025arxiv} assumed consistency and local
    likelihood regularity a priori, we establish consistency and identifiability
    up to a joint sign change, verify 
    nonsingularity of the joint Fisher information without assuming the local likelihood regularity, and prove existence of an
    empirical score matching minimizer with probability tending to one.

    To relax the stronger diffusion-horizon decay requirement in the objective-value argument of
    \citet{chewi2025arxiv}, we control the correction term in the empirical
    likelihood bridge with respect to the model parameters and its derivatives, yielding asymptotic
    normality with inverse Fisher-information covariance whenever the signal
    remaining at the final diffusion time tends to zero at any rate.

    In contrast to the weight-estimation analyses of
    \citet{ozkara2025spire,schwienhorst2026diffusiondenoising}, which fixed the
    component parameters, component separation in MLR varies with the covariates
    and depends on the unknown regression parameters. Although our objective
    integrates over diffusion time like that of
    \citet{schwienhorst2026diffusiondenoising}, their horizon was tuned to
    component separation to control mixing-weight estimation error, whereas our
    horizon grows with sample size to achieve joint asymptotic normality matching
    the maximum-likelihood estimator for both the regression parameters and mixing
    weights.\par

    \item \textbf{Optimization Analysis of Fixed-Scale and Path-Integrated Score Matching: Cross-Entropy and EM Decompositions}
    (Section~\ref{sec:connect}).
    \citet{shah2023learning} studied Gaussian mixtures with known equal mixing
    weights and isotropic component covariances, expressing the two-component score
    matching update as a relaxed EM step plus a bounded correction to establish low-noise
    contraction under separation and initialization conditions. In contrast, we
    derive exact score matching loss and joint-gradient identities for MLR with general covariates
    and unknown, potentially unequal mixing weights at every fixed diffusion noise
    scale. After identifying the EM operators of the diffused MLR model and
    expressing the cross-entropy gradients as EM residuals
    (Propositions~\ref{prop:em_update} and~\ref{prop:cross_entropy}), we decompose the fixed-scale loss through cross-entropy and
    expected squared activation, or equivalently through EM operators and latent
    variance (Main Theorem~\ref{thm:information_theoretical_decomposition}).
    Differentiating these decompositions yields joint-gradient identities
    (Corollaries~\ref{cor:sm_gradients_entropy_activation}
    and~\ref{cor:sm_gradients_em_variance} and
    Proposition~\ref{prop:sm_gradients}) that retain the EM-operator derivatives
    and latent-variance terms distinguishing score matching descent from an EM
    update.

    For the regression parameters, we express this entire gradient correction
    through the derivative of the EM operator with respect to the
    retained-signal level (Eq.~\eqref{eq:sm_em_scale_correction}). This identity
    recasts preconditioned gradient descent on the fixed-scale loss entirely in
    terms of the EM operator and its scale derivative
    (Eq.~\eqref{eq:preconditioned_sm_regression_update}), and integrating it across
    noise scales yields a preconditioned gradient descent update for the
    path-integrated loss involving only EM operators at the original and terminal
    scales (Eq.~\eqref{eq:path_integrated_em_update}).

    From these loss and gradient identities, we derive both the low- and high-noise
    limits (Eq.~\eqref{eq:low_noise_loss_limit} and
    Propositions~\ref{prop:low_noise_em}
    and~\ref{prop:high_noise_effective_coefficient}). In the low-noise limit, the
    gradient retains the correction terms that distinguish score matching from EM.
    In the high-noise limit, the objective reduces to a squared-error criterion for
    the conditional-mean regression coefficient, which couples the regression
    parameters with the mixing-weight imbalance so that parameter pairs sharing
    this coefficient have the same limiting objective.\par

    \item \textbf{High-Noise Gradient Dynamics and Score Matching Blindness}
    (Section~\ref{sec:theory}).
    Since the high-noise limiting objective couples the regression parameters and
    mixing-weight imbalance multiplicatively, optimizing it via gradient descent
    differs fundamentally from the high-noise spectral iteration for balanced
    Gaussian mixtures in \citet{shah2023learning}. Under admissible step sizes and initializations, we combine a discrete Lyapunov bound
    with an analysis of the coupled regression and mixing-weight updates to
    establish global convergence rates for the loss
    (Main Theorem~\ref{thm:global_convergence_high_noise}). When the ground truth
    conditional-mean regression coefficient is zero, we establish an upper bound on
    the loss decaying as the inverse square of the iteration count. When this
    coefficient is nonzero, we prove geometric loss decay after a finite transient
    from generic initializations, using a center-stable-manifold
    argument~\citep{lee2016gradient} to show that the exceptional set of
    initializations attracted to the saddle manifold has measure zero.

    Finally, we analyze how mixing-weight sensitivity degrades at a fixed diffusion
    scale as component separation grows. For Gaussian mixtures with common
    component distributions and distinct mixing weights,
    \citet{wenliang2020blindness} showed that the Fisher divergence vanished as the
    components separated. Subsequent analyses addressed this insensitivity under
    fixed component distributions either by integrating denoising over time with a
    separation-tuned horizon when Gaussian component means are
    known~\citep{schwienhorst2026diffusiondenoising}, or by quantifying
    intermediate-scale score sensitivity to control squared mixing-weight error
    through the integrated score loss~\citep{dennehy2026mixtureweights}. Using our
    EM and latent-variance decomposition for MLR, we establish an analogous
    fixed-scale vanishing-gradient limit
    (Proposition~\ref{prop:pointwise_sm_blindness}). Specifically, for any fixed
    candidate parameters and fixed positive retained-signal factor, the gradient
    with respect to the mixing-weight imbalance vanishes as the ground truth
    regression parameters diverge along a fixed ray, under finite covariate second
    moments and a nonzero-projection condition.\par
\end{itemize}

Table~\ref{tab:contribution-comparison} compares our main results,
assumptions, and proof techniques with closely related work.
\begin{table}[!htbp]
    \centering
    \small
    \setlength{\tabcolsep}{2.5pt}
    \renewcommand{\arraystretch}{1.00}
    \caption{\textbf{Comparison of results, assumptions, and proof techniques with related work.}}
    \label{tab:contribution-comparison}
    \begin{tabular}{@{}>{\raggedright\arraybackslash}p{0.12\linewidth}
        >{\raggedright\arraybackslash}p{0.10\linewidth}
        >{\raggedright\arraybackslash}p{0.105\linewidth}
        >{\raggedright\arraybackslash}p{0.18\linewidth}
        >{\raggedright\arraybackslash}p{0.24\linewidth}
        >{\raggedright\arraybackslash}p{0.20\linewidth}@{}}
        \toprule
        & Model & Unknown parameters
        & Main results & Key proof techniques & Main assumptions \\
        \midrule
        \multicolumn{6}{@{}l}{\textbf{Statistical Guarantees} for Path-Integrated Score Matching:} \\
        \multicolumn{6}{@{}l}{Consistency and Asymptotic Normality (Section~\ref{sec:asymptotic})} \\
        \addlinespace[2pt]
        Chewi et al.\newline\citeyearpar{chewi2025arxiv}
        & Parametric family
        & Model parameters
        & Asymptotic normality
        & Likelihood identity; Kullback--Leibler decay bounds; M-estimation
        & Assumed consistency; local regularity; 
          moment bounds;
          retained signal \(\bar{\alpha}_{T_n}=\mathe^{-2T_n}=o(n^{-1})\) at the final diffusion time \(T_n\) \\
        \addlinespace
        L\"utke\newline Schwienhorst\newline et al.~\citeyearpar{schwienhorst2026diffusiondenoising}
        & GMM
        & Mixing weights
        & Asymptotic normality;
        Finite-sample mixing-weight \(\mathcal{O}(n^{-1/2})\)  error
        & Empirical process theory; Lipschitz bounds; population loss curvature
        & Known component means \(\mu\); 
        isotropic Gaussian \(\mathcal{N}(\mu, I_d)\); 
        mixing weights away from zero \\
        \addlinespace
        Our work
        & MLR
        & Regression parameters; mixing weights
        & Consistency; 

          joint asymptotic normality;

          \(\mathcal{O}(\bar{\alpha}_{T_n}/\sqrt{n})\) bound on the difference between SM estimators and MLE
        & Mixture identifiability; quadratic inequality bounding the regression estimator;

          mean-zero gradient identities;
          
          Hessian bounds for the empirical conditional terminal discrepancy
        & Consistency: \(\E\|\bx\|^2<\infty\).\newline
          Efficiency and SM--MLE comparison: \(\E\|\bx\|^4<\infty\);\newline
          retained signal \(\bar{\alpha}_{T_n}=o(1)\) at the final diffusion time \(T_n\) vanishes at any rate \\
        \midrule
        \multicolumn{6}{@{}l}{\textbf{Optimization Analysis} of Fixed-Scale and Path-Integrated Score Matching:} \\
        \multicolumn{6}{@{}l}{Cross-Entropy and EM Decompositions (Section~\ref{sec:connect});} \\
        \multicolumn{6}{@{}l}{High-Noise Gradient Dynamics and Score Matching Blindness (Section~\ref{sec:theory})} \\
        \addlinespace[2pt]
        Shah et al.\newline\citeyearpar{shah2023learning}
        & GMM
        & Component means
        & Connection between gradient descent of SM loss and EM in low noise regime;

        finite-sample mean recovery by gradient descent and its projected variant
        & Approximation of gradient updates by EM and power iteration 
        & Known equal mixing weights; 
        isotropic Gaussians \(\mathcal{N}(\mu, I_d)\);
          lower bound on distances between component means;
          random or warm starts \\
        \addlinespace
        Our work
        & MLR
        & Regression parameters; mixing weights
        & Preconditioned gradient descent on path-integrated and fixed-scale SM, with preconditioner \(\Sigma^{-1}\) by EM operators; 
        gradient descent rates for the high-noise limit;
          vanishing gradient in mixing imbalance at high SNR
        & EM and posterior-variance identities; discrete Lyapunov bounds;
          Kurdyka--\L{}ojasiewicz gradient inequalities;
          center-stable-manifold argument
        & Unknown mixing weights;
          covariate second-moment matrix \(\Sigma\succ0\);
          arbitrary starts with admissible step sizes \\
        \bottomrule
    \end{tabular}
    \par\smallskip
    \begin{minipage}{\linewidth}
    \footnotesize
    \raggedright
    \textit{Abbreviations:} MLR, mixed linear regression; GMM, Gaussian mixture
    model; 
    
    SM, score matching; MLE, maximum-likelihood estimator; EM, Expectation-Maximization; SNR, signal-to-noise ratio.
    Here \(\Sigma=\E[\bx\bx^\top]\) is the covariate second-moment matrix.
    \par\smallskip
    \textit{Shared tools:} All efficiency analyses of asymptotic normality use the likelihood
    bridge~\citep{chewi2025arxiv} and M-estimation~\citep[Chapter~5]{van2000asymptotic}.
    \end{minipage}
\end{table}

\paragraph{Organization.}
In Section~\ref{sec:setup}, we formulate the MLR model, score matching
objectives, and associated EM operators. In Section~\ref{sec:asymptotic},
we use the likelihood bridge to establish asymptotic normality with inverse Fisher-information covariance
for the path-integrated score matching estimator, matching the limiting distribution
of the maximum-likelihood estimator (MLE). In Section~\ref{sec:connect},
we derive cross-entropy and EM decompositions that characterize
the fixed-scale and path-integrated score matching loss and gradients across the low- and high-noise regimes.
In Section~\ref{sec:theory},
we analyze high-noise gradient dynamics and pointwise score matching blindness.
In Section~\ref{sec:experiments}, we summarize numerical illustrations of these results.
In Appendix~\ref{sup:lemma}, we collect auxiliary identities and bounds as the supporting lemmas used in the proofs.
In Appendices~\ref{sup:setup}--\ref{sup:theory}, we prove the theoretical results stated in
Sections~\ref{sec:asymptotic}--\ref{sec:theory}.
In Appendix~\ref{sup:experiment}, we describe the experimental settings, and in
Appendix~\ref{sup:additional_experiments}, we report additional numerical studies to support our findings.

\section{Problem Setup}\label{sec:setup}
In this section, we specify the symmetric two-component mixed linear regression (MLR) model
with unknown regression parameters and mixing weights. The two component
regression vectors are negatives of one another, and the mixing weights
may be unequal. Our goal is to recover the regression parameters and
mixing weights from observed covariate--response pairs using path-integrated
score matching (Section~\ref{sec:asymptotic}). We later relate the fixed-scale
score matching loss to the classical EM algorithm and study its gradient
dynamics (Sections~\ref{sec:connect} and~\ref{sec:theory}). Here, we describe
variance-preserving diffusion of the response
and define the population and empirical score matching objectives.
We also define the cross-entropy, EM operators, and notation used in the
statistical and optimization analyses.
The detailed derivations of supporting lemmas are provided in Appendix~\ref{sup:lemma}.
\subsection{Mixed Linear Regression and Variance-Preserving Diffusion}\label{subsec:2mlr}
Let \(\bx\in\R^d\) denote the covariate vector and \(y_0\in\R\) the response.
The model has regression parameters \(\theta\in\R^d\) and
\(-\theta\), and the response is generated according to
\begin{equation}\label{eq:2mlr}
    y_0 = (-1)^{z+1}\langle \theta, \bx \rangle + \varepsilon_0,
\end{equation}
where \(\E[\bx \bx^\top]=\Sigma\succ0\) and
the observation noise \(\varepsilon_0 \sim \mathcal{N}(0, 1)\) has fixed unit
variance. The latent variable
\(z \in \{1,2\}\) has probabilities \(\Pr(z=1\mid \pi) = \pi(1)\) and
\(\Pr(z=2\mid \pi) = \pi(2)\). The mixing weights
\(\pi = (\pi(1), \pi(2))\) are positive and sum to one.
We impose the independence assumptions \((z; \pi)\ind \theta\), \(\bx\ind (z; \theta, \pi)\), and \(\varepsilon_0 \ind (\bx, z; \theta, \pi)\).
The half log-odds parameter \(\nu := (\ln\pi(1)-\ln\pi(2))/2\) characterizes the imbalance of the mixing weights and satisfies \(\tanh\nu=\pi(1)-\pi(2)\), \(\pi(1)=(1+\tanh\nu)/2\), and \(\pi(2)=(1-\tanh\nu)/2\).

Let \(\theta^\ast\in\R^d\) and
\(\pi^\ast=(\pi^\ast(1),\pi^\ast(2))\) denote the ground truth regression
parameters and mixing weights, with
\(\nu^\ast:=(\ln\pi^\ast(1)-\ln\pi^\ast(2))/2\).
We observe \(n\) independent and identically distributed pairs
\(\{(y_0^{(i)},\bx^{(i)})\}_{i=1}^n
\stackrel{\mathrm{i.i.d.}}{\sim}p(y_0,\bx\mid\theta^\ast,\pi^\ast)\),
with unobserved component labels.
The estimation target is \((\theta^\ast,\pi^\ast)\), equivalently
\((\theta^\ast,\nu^\ast)\).
Exchanging the component labels changes \((\theta,\nu)\) to
\((-\theta,-\nu)\) and leaves the conditional response law invariant.
We therefore formulate parameter recovery up to this joint sign change,
with the identifiability conditions stated in Section~\ref{sec:asymptotic}.
To construct the score matching objectives, we diffuse the response while
keeping the covariate vector fixed. The variance-preserving diffusion process
satisfies the stochastic differential equation (SDE)
\begin{equation}\label{eq:sde}
    \d y_t = -\frac{\beta_t}{2}y_t \d t + \sqrt{\beta_t}\d B_t,
\end{equation}
where \(\beta_t>0\) is the noise schedule and \(B_t\) is a standard Brownian motion satisfying \(\E[B_t]=0\) and \(\E[B_tB_s]=\min(t,s)\).
Noise schedules used or analyzed in the literature include the constant choice
\(\beta_t=2\)~\citep{chen2023probability}, the linear choice
\(\beta_t=a+bt\)~\citep{ddpm}, and the power-law family
\(\beta_t=(a+bt)^\rho\) with \(\rho\geq1\)~\citep{gao2025wasserstein,gao2025convergence}.
The retained-signal factor
\(
\bar{\alpha}_t := \exp\left(-\int_0^t \beta_\tau \d \tau \right)
\)
encodes cumulative attenuation along the diffusion path, and its square root
\(\sqrt{\bar{\alpha}_t}\) controls the decay of the signal from \(y_0\).
The corresponding forward transition density is
\(
p(y_t|y_0) = \mathcal{N}\left(y_t; y_0\sqrt{\bar{\alpha}_t}, 1-\bar{\alpha}_t\right)
\), so the response variable \(y_t\) at time \(t\) admits the representation
\begin{equation}\label{eq:yt}
    y_t = y_0\sqrt{\bar{\alpha}_t} + \xi_t \sqrt{1-\bar{\alpha}_t} = (-1)^{z+1} \langle \theta_t, \bx \rangle + \varepsilon_t, 
\end{equation}
where \(\xi_t \sim \mathcal{N}(0, 1)\) and \(\varepsilon_t \sim \mathcal{N}(0, 1)\) are additive Gaussian noises independent of \(\bx\) and \(z\), while \(\theta_t := \theta \sqrt{\bar{\alpha}_t}\) and \(\mu_t := \langle \theta_t, \bx \rangle\) provide concise notation for the scaled parameters at time \(t\).
At \(t = 0\), we have \(\theta_0 \equiv \theta\) and \(\mu_0 \equiv \langle \theta, \bx \rangle\). These identities imply \(\mu_t = \mu_0 \sqrt{\bar{\alpha}_t}\) and \(\d \mu_t/\d t= -(\beta_t/2)\mu_t\).
For the ground truth parameters, we write
\(\theta_t^\ast:=\theta^\ast\sqrt{\bar\alpha_t}\) and
\(\mu_t^\ast:=\langle\theta_t^\ast,\bx\rangle\).

\subsection{Score Function and Score Matching Loss}\label{subsec:score}
With the scaled regression parameters \(\theta_t \equiv \theta \sqrt{\bar{\alpha}_t}\), mixing weights \(\pi\), and shorthand \(\mu_t \equiv \mu_0 \sqrt{\bar{\alpha}_t} = \langle \theta, \bx \rangle \sqrt{\bar{\alpha}_t}\), the Stein score function with respect to \(y_t\) is
\begin{equation}\label{eq:score}
    s_{\theta_t, \nu}(y_t,\bx) := \nabla \log p(y_t\mid \bx; \theta_t, \pi) = -y_t + \tanh(\mu_t y_t + \nu)\mu_t.
\end{equation}

The score matching loss \(\calL_t(\theta,\nu)\) is
the covariate expectation of
the relative Fisher information 
\(\FI(p(y_t\mid\bx;\theta_t^\ast,\pi^\ast)\parallel
p(y_t\mid\bx;\theta_t,\pi))
:=\E_{y_t\mid\bx;\theta_t^\ast,\pi^\ast}
\|s_{\theta_t,\nu}(y_t,\bx)-s_{\theta_t^\ast,\nu^\ast}(y_t,\bx)\|^2\) between the score at model parameters \(\theta,\pi\) and the score at ground truth
parameters \(\theta^\ast,\pi^\ast\). The normalization factor is \(1/(2\bar{\alpha}_t)\).
\begin{equation}\label{eq:loss}
    \calL_t(\theta,\nu) := \frac{1}{2\bar{\alpha}_t}\E_{y_t,\bx|\theta^*,\pi^*} \|s_{\theta_t, \nu}(y_t,\bx) - s_{\theta^*_t, \nu^*}(y_t,\bx)\|^2
    = \frac{1}{2 \bar{\alpha}_t} \E_{\bx} \FI(p(y_t\mid \bx; \theta^*_t, \pi^*)\parallel p(y_t\mid \bx; \theta_t, \pi)).
\end{equation}


The auxiliary denoising score matching loss \(\calJ_t(\theta,\nu)\) compares the model score with the
score of the forward transition. The Gaussian transition is
\(p(y_t\mid y_0) = \mathcal{N}(y_t; y_0\sqrt{\bar{\alpha}_t},
1-\bar{\alpha}_t)\). Its score satisfies
\(\nabla \log p(y_t\mid y_0) =
-\frac{y_t-y_0\sqrt{\bar{\alpha}_t}}{1-\bar{\alpha}_t}
\stackrel{\mathsf{d}}{=} -\frac{\xi}{\sqrt{1-\bar{\alpha}_t}}\), where
\(\xi \stackrel{\mathsf{d}}{=} \xi_t \sim \mathcal{N}(0, 1)\) is independent
of \((y_0,\bx)\).
To express the loss in terms of the original response and independent Gaussian
noise, we substitute the representation
\(y_t \stackrel{\mathsf{d}}{=} y_0\sqrt{\bar{\alpha}_t} + \xi
\sqrt{1-\bar{\alpha}_t}\) and write
\begin{equation}\label{eq:auxiliary_loss}
\begin{aligned}
    \calJ_t(\theta,\nu) :=& \frac{1}{2\bar{\alpha}_t}\E_{y_t,y_0,\bx|\theta^*,\pi^*} 
    \left[\|s_{\theta_t, \nu}(y_t,\bx)\|^2 - 2 \langle s_{\theta_t, \nu}(y_t,\bx), \nabla \log p(y_t\mid y_0) \rangle\right]\\
    =& \frac{1}{2\bar{\alpha}_t}\E_{y_0,\bx|\theta^*,\pi^*} \E_{\xi}
    \left[
    \Big\|s_{\theta_t, \nu}(y_0\sqrt{\bar{\alpha}_t} + \xi \sqrt{1-\bar{\alpha}_t},\bx) \Big\|^2
    + 2 \left\langle s_{\theta_t, \nu}(y_0\sqrt{\bar{\alpha}_t} + \xi \sqrt{1-\bar{\alpha}_t},\bx), \frac{\xi}{\sqrt{1-\bar{\alpha}_t}} \right\rangle
    \right].
\end{aligned}
\end{equation}

For the observed pairs introduced above, we define
\(\Pr_n:=n^{-1}\sum_{i=1}^n\delta_{(y_0^{(i)},\bx^{(i)})}\) and
\(\E_{\Pr_n}[f(y_0,\bx)]:=n^{-1}\sum_{i=1}^n
f(y_0^{(i)},\bx^{(i)})\). Thus \(\E_{\Pr_n}\) averages over the observed
pairs. The inner expectation in \(\E_{\Pr_n}\E_{y_t\mid y_0}[\cdot]\)
averages over forward-diffusion noise. The
finite-sample loss is
\begin{equation}\label{eq:auxiliary_loss_finite}
\begin{aligned}
    \calJ_{t}^n(\theta,\nu) 
    :=& \frac{1}{2\bar{\alpha}_t}\E_{\Pr_n} \E_{y_t\mid y_0}
    \left[
        \|s_{\theta_t, \nu}(y_t,\bx)\|^2 - 2 \langle s_{\theta_t, \nu}(y_t,\bx), \nabla \log p(y_t\mid y_0) \rangle
    \right]\\
    =& 
    \frac{1}{2\bar{\alpha}_t}\E_{\Pr_n} \E_{\xi}
    \left[
    \Big\|s_{\theta_t, \nu}(y_0\sqrt{\bar{\alpha}_t} + \xi \sqrt{1-\bar{\alpha}_t},\bx) \Big\|^2
    + 2 \left\langle s_{\theta_t, \nu}(y_0\sqrt{\bar{\alpha}_t} + \xi \sqrt{1-\bar{\alpha}_t},\bx), \frac{\xi}{\sqrt{1-\bar{\alpha}_t}} \right\rangle
    \right].
\end{aligned}
\end{equation}

\subsection{Negative Log-Likelihood and Expectation-Maximization Operators}\label{subsec:em_update}
We now introduce the unit-variance negative log-likelihood and the
associated population EM operators. See detailed derivations in Appendix~\ref{sup:lemma} and in Proposition~3 and Appendix~A of~\citet{luo2025structural}. The negative log-likelihood is
\begin{equation}
F(y_t, \mu_t, \nu) := -\ln p(y_t\mid \bx; \theta_t, \pi) = \frac{1}{2}(y_t^2 + \mu_t^2) - \ln \frac{\cosh(\mu_t y_t+\nu)}{\cosh(\nu)} + \frac{1}{2} \ln(2\p).
\end{equation}
We denote expectation under the ground truth law by
\(\E_{\ast}[\cdot]:=\E_{y_t,\bx\mid\theta^\ast,\pi^\ast}[\cdot]\).
The cross-entropy
\(\calH(\theta_t,\nu):=\E_{\ast}[-\ln
p(y_t\mid\bx;\theta_t,\pi)]=\E_{\ast}[F(y_t,\mu_t,\nu)]\) is the expected
negative log-likelihood. 
The population Expectation-Maximization (EM) operators for
\(\theta_t\) and the imbalance of mixing weights \(\tanh\nu \equiv \pi(1)-\pi(2)\) are
\begin{equation}\label{eq:em_updates}
    M(\theta_t, \nu) = \Sigma^{-1}\E_{\ast}\big[\bx y_t \tanh(\mu_t y_t + \nu)\big],
    \qquad N(\theta_t, \nu) = \E_{\ast}\big[\tanh(\mu_t y_t + \nu)\big].
\end{equation}

\subsection{Notations on Metrics, Asymptotics, and Estimation Objectives}\label{subsec:notations}

For probability measures \(P,Q\) on a common measurable space, we define
\(\KL(P\parallel Q):=\int\ln(\d P/\d Q)\,\d P\) when \(P\ll Q\) and
\(\TV(P,Q):=\sup_A|P(A)-Q(A)|\) over measurable sets.
For probability measures \(P,Q\) with finite second moments, the Wasserstein distance is defined as
\(W_2^2(P,Q):=\inf_{\gamma\in\Gamma(P,Q)}
\int\|u-v\|_2^2\,\gamma(\d u,\d v)\) over their set of couplings \(\Gamma(P,Q)\).
For a deterministic positive sequence \(r_n\), we write \(X_n=\calO_P(r_n)\) when \(X_n/r_n\) is stochastically bounded and \(X_n=o_P(r_n)\) when \(X_n/r_n\xrightarrow{\Pr}0\).
The model parameters \(\vartheta:=(\theta,\nu)\) and the ground truth values \(\vartheta^\ast:=(\theta^\ast,\nu^\ast)\)
are introduced for the convenience of notation.
\(T_n\) denotes a deterministic terminal horizon for the total number of samples \(n\), 
and \(\bar{\alpha}_{T_n}\) is the retained-signal factor at the terminal horizon.
Using \(\Pr_n\) and \(\calJ_t^n\) defined above, the empirical negative
log-likelihood is defined as \(\widehat{\calR}_n^{\,\mathsf{MLE}}(\vartheta)
:=-\E_{\Pr_n}[\ln p(y_0\mid\bx;\theta,\pi)]\), and the path-integrated score
matching objective is denoted by \(\widehat{\calR}_{n,T_n}^{\,\mathsf{SM}}(\vartheta)
:=\int_{\bar\alpha_{T_n}}^{\bar\alpha_0}\calJ_t^n(\theta,\nu)\,
\d\bar\alpha_t\). On the event that this objective has a nonempty argmin, let
\(\widehat\vartheta_{n,T_n}^{\,\mathsf{SM}}
:=(\widehat\theta_{n,T_n}^{\,\mathsf{SM}},
\widehat\nu_{n,T_n}^{\,\mathsf{SM}})\) be a sample-measurable exact minimizer.
On the complementary event, set
\(\widehat\vartheta_{n,T_n}^{\,\mathsf{SM}}:=(\mathbf0,0)\). This convention
defines the score matching estimator on the entire sample space, and
Lemma~\ref{lemma:sm_minimum_attainment} justifies the measurable selection used
under the conditions of Main Theorem~\ref{thm:consistency_asymptotic_normality}.
The empirical conditional terminal discrepancy is defined as
\(\Delta_n(\vartheta):=\E_{\Pr_n}\!\left[\KL\!\left(
p(y_{T_n}\mid y_0)\parallel
p(y_{T_n}\mid\bx;\theta_{T_n},\pi)\right)\right]\geq0\).
The constant term \(\mathsf{c}(\bar{\alpha}_{T_n}) := \frac{1}{2}(\ln\frac{1-\bar{\alpha}_{T_n}}{\bar{\alpha}_{T_n}} + \ln (2\p) + 1)\) only depends on the retained-signal factor \(\bar{\alpha}_{T_n}\) at the terminal horizon.

\section{Statistical Guarantees for Path-Integrated Score Matching: Consistency and Asymptotic Normality}\label{sec:asymptotic}
In this section, we establish the statistical guarantees of the path-integrated score matching estimator by 
establishing a bridge between the score matching loss and the log-likelihood and 
analyzing the empirical conditional terminal discrepancy \(\Delta_n(\vartheta)\).
After showing that the terminal population law at any positive retained-signal level
identifies the ground truth parameters up to a joint sign change, we further establish consistency and asymptotic normality
with inverse Fisher-information covariance whenever the terminal retained
signal vanishes along a deterministic schedule \(\bar{\alpha}_{T_n}\to 0\). Proofs for the results in this
section are deferred to Appendices~\ref{sup:lemma} and~\ref{sup:setup}.

At the population level, the decrease in KL divergence along the diffusion path is the
integral of the population score matching loss \(\calL_t(\theta,\nu)\) by the relative de Bruijn identity~\citep{stam1959some}.
\begin{equation}\label{eq:population_sm_mle_bridge}
\int_{\bar{\alpha}_T}^{\bar{\alpha}_0} \calL_t(\theta, \nu) \d \bar{\alpha}_t
= \KL(p(y_0, \bx\mid\theta^*,\pi^*) \parallel p(y_0, \bx\mid\theta,\pi)) - \KL(p(y_T, \bx\mid\theta_T^*,\pi^*) \parallel p(y_T, \bx\mid\theta_T,\pi)).
\end{equation}
For the path-integrated empirical objective \(\widehat{\calR}_{n,T_n}^{\,\mathsf{SM}}(\vartheta)
:=\int_{\bar\alpha_{T_n}}^{\bar\alpha_0}\calJ_t^n(\theta,\nu)\,
\d\bar\alpha_t\),
we similarly obtain the corresponding empirical bridge to the empirical negative log-likelihood \(\widehat{\calR}_n^{\,\mathsf{MLE}}(\vartheta)
:=-\E_{\Pr_n}[\ln p(y_0\mid\bx;\vartheta)]\) for \(\vartheta\equiv(\theta, \nu)\) as follows:
\begin{equation}\label{eq:empirical_sm_mle_bridge}
    \widehat{\calR}_{n,T_n}^{\,\mathsf{SM}}(\vartheta)
    = \widehat{\calR}_n^{\,\mathsf{MLE}}(\vartheta)
- \Delta_n(\vartheta)
- \mathsf{c}(\bar{\alpha}_{T_n}).
\end{equation}
Here \(\mathsf c(\bar\alpha_{T_n})\) doesn't depend on the parameters \(\vartheta\), while
the empirical discrepancy \(\Delta_n(\vartheta)\) averages a transition-kernel KL
divergence over the observed sample, satisfying the following relation to the terminal population KL divergence 
by taking expectations under the ground truth law 
(see Lemmas~\ref{lemma:kl_divergence_SM_loss},~\ref{lemma:log_likelihood_SM_loss_finite}). 
\begin{equation}\label{eq:expected_terminal_discrepancy}
    \E_*[\Delta_n(\vartheta)-\Delta_n(\vartheta^\ast)]
    =\KL\bigl(p(y_{T_n},\bx\mid\theta_{T_n}^\ast,\pi^\ast)
       \parallel p(y_{T_n},\bx\mid\theta_{T_n},\pi)\bigr).
\end{equation}

Let \(D=\pi(1)\delta_{\langle\theta,\bx\rangle}
+\pi(2)\delta_{-\langle\theta,\bx\rangle}\) denote the conditional distribution
of the noiseless response, and let \(D^\ast\) denote its counterpart at the
ground truth parameters \(\theta^\ast\) and \(\pi^\ast\), then the terminal KL divergence is bounded by the expected discrete Wasserstein cost between \(D\) and \(D^\ast\):
\begin{equation}
\KL(p(y_{T_n},\bx\mid\theta_{T_n}^\ast,\pi^\ast)
\parallel p(y_{T_n},\bx\mid\theta_{T_n},\pi))
\leq \frac{\bar\alpha_{T_n}}{2(1-\bar\alpha_{T_n})}\E_{\bx}W_2^2(D,D^\ast).
\end{equation}
Consequently, the terminal population laws become indistinguishable as the
retained signal factor \(\bar\alpha_{T_n}\) vanishes. Similarly, an analogous transition-kernel bound controls
\(\Delta_n\) over bounded regression parameters
(Lemmas~\ref{lemma:decaying_kl_bound}
and~\ref{lemma:decaying_kl_bound_transition}). 
Notably, the terminal population law still identifies
the parameters up to exchange of the component labels for at every retained-signal level.
\begin{proposition}[Global Minimizers of the Terminal KL Divergence]\label{prop:terminal_kl_pinsker_zero_set}
    The retained-signal level satisfies \(0<\bar{\alpha}_T\leq 1\), the
    ground truth parameters satisfy \(\theta^\ast\neq\mathbf{0}\) and
    \(\nu^\ast\in\R\), and \(\bx\) has a density that is positive
    almost everywhere on \(\R^d\) and finite second moments.
    For the terminal population KL divergence
    \(
        \KL_T(\theta,\nu)
        := \KL\left(
        p(y_T,\bx\mid\theta_T^\ast,\pi^\ast)
        \parallel
        p(y_T,\bx\mid\theta_T,\pi)
        \right),
    \)
    the set of global minimizers is
    \(
        \argmin_{(\theta,\nu)\in\R^d\times\R}\KL_T(\theta,\nu)
        =\{(\theta^\ast,\nu^\ast),(-\theta^\ast,-\nu^\ast)\}.
    \)
\end{proposition}

\begin{proof sketch}
    To characterize zero terminal KL divergence, we apply
    Pinsker's inequality~\citep{cover2006elements}:
    \[
        \KL_T(\theta,\nu)\geq 2\E_{\bx}\TV^2\left(
            p(\,\cdot\mid\bx;\theta_T^\ast,\pi^\ast),
            p(\,\cdot\mid\bx;\theta_T,\pi)
        \right),
    \]
    where \(\TV\) denotes total variation distance. Zero terminal KL divergence
    therefore implies that the conditional response laws agree for almost every
    covariate vector. Since \(\bar\alpha_T>0\), their matching second moments imply
    \((\theta^\top\bx)^2=((\theta^\ast)^\top\bx)^2\). Because the covariate
    density is positive almost everywhere, this identity implies
    \(\theta\theta^\top=\theta^\ast(\theta^\ast)^\top\) and hence
    \(\theta=\pm\theta^\ast\). Matching the first moments then fixes
    \(\tanh\nu\), and therefore \(\nu\), under the corresponding sign choice.
    Conversely, both joint sign choices preserve the conditional law and
    attain zero KL divergence, establishing the stated minimizer set. Moreover,
    by Lemma~\ref{lemma:wasserstein_distance_discrete} and the same
    positive-density argument, \(\E_{\bx}W_2^2(D,D^\ast)=0\) holds at the same
    two solutions, so the terminal KL divergence and the expected discrete
    Wasserstein cost share the same global zero set. The full proof of
    Proposition~\ref{prop:terminal_kl_pinsker_zero_set} is deferred to
    Appendix~\ref{sup:lemma}.
\end{proof sketch}

Building on the empirical bridge in Eq.~\eqref{eq:empirical_sm_mle_bridge}
which is similar to the diffusion--likelihood argument of
\citet{chewi2025arxiv}, the next theorem establishes the asymptotic guarantees
of the path-integrated estimator. Because the bridge isolates the terminal
discrepancy from the empirical likelihood, consistency requires only that the
discrepancy value vanish, whereas asymptotic normality requires controlling
its parameter gradient. Evaluated at the ground truth parameters, this
gradient has expectation zero under the data distribution, and its
fluctuations vanish with the terminal retained signal. This additional
cancellation allows any deterministic terminal schedule satisfying
\(\bar{\alpha}_{T_n}\to 0\), more relaxed than the requirement of \(\bar{\alpha}_{T_n} = o(1/n)\) for the
 objective-value argument in \citet{chewi2025arxiv}.
\begin{myframe}
\begin{main theorem}[Consistency and Asymptotic Normality of Score Matching Estimator]\label{thm:consistency_asymptotic_normality}
    The ground truth parameters
    \((\theta^\ast,\nu^\ast)\in\R^d\times\R\) satisfy
    \(\theta^\ast\ne\mathbf0\). The deterministic terminal-horizon sequence
    \(\{T_n\}_{n\geq1}\) satisfies
    \(0<\bar\alpha_{T_n}<1\) and \(\bar\alpha_{T_n}\to0\),
    and the covariate vector \(\bx\) has a density that is positive almost everywhere
    and satisfies \(\E\|\bx\|^2<\infty\). The estimator \(\widehat{\vartheta}_{n,T_n}^{\,\mathsf{SM}} \equiv (\widehat{\theta}_{n,T_n}^{\,\mathsf{SM}},\widehat{\nu}_{n,T_n}^{\,\mathsf{SM}})\) is \textsf{consistent} up
    to a joint sign change. For the alignment map
    \(
        s_n \in \argmin_{s\in \{ +1, -1\}} \left\| s\,\left(\widehat{\theta}_{n,T_n}^{\,\mathsf{SM}},\widehat{\nu}_{n,T_n}^{\,\mathsf{SM}}\right) - (\theta^\ast,\nu^\ast) \right\|
    \),
    as \(n \to \infty\),
    \[
        s_n \left(\widehat{\theta}_{n,T_n}^{\,\mathsf{SM}},\widehat{\nu}_{n,T_n}^{\,\mathsf{SM}}\right) \xrightarrow{\Pr} (\theta^\ast,\nu^\ast).
    \]
    If \(\E[\|\bx\|^4] < \infty\), then the aligned score matching estimator
    is \textsf{asymptotically normal}. As \(n\to\infty\),
    \[
        \sqrt{n} \left[ s_n \left(\widehat{\theta}_{n,T_n}^{\,\mathsf{SM}},\widehat{\nu}_{n,T_n}^{\,\mathsf{SM}}\right) - (\theta^\ast,\nu^\ast) \right]
        \xrightarrow{\mathsf{d}} \mathcal{N}(\mathbf{0}, I(\theta^\ast,\nu^\ast)^{-1}),
    \]
    where \(I\) denotes the Fisher information matrix.
\end{main theorem}
\end{myframe}

\begin{proof sketch}[Bridge between Score Matching and MLE]
    We write \(\widetilde\vartheta_{n,T_n}^{\,\mathsf{SM}}
    := s_n\widehat\vartheta_{n,T_n}^{\,\mathsf{SM}}
    = s_n(\widehat\theta_{n,T_n}^{\,\mathsf{SM}},\widehat\nu_{n,T_n}^{\,\mathsf{SM}})\)
    for the aligned score matching estimator, and notice that 
    the empirical score
    matching objective \(\widehat{\calR}_{n,T_n}^{\,\mathsf{SM}}(\vartheta)\)
    attains a global minimum with probability tending to one and admits, on
    that event, a measurable minimizer satisfying the first-order condition
    \(\nabla\widehat{\calR}_{n,T_n}^{\,\mathsf{SM}}
    (\pm\widehat\vartheta_{n,T_n}^{\,\mathsf{SM}})=\mathbf0\) and therefore \(\nabla\widehat{\calR}_{n,T_n}^{\,\mathsf{SM}}
    (\widetilde\vartheta_{n,T_n}^{\,\mathsf{SM}})=\mathbf0\) (see Lemma~\ref{lemma:sm_minimum_attainment}).
    We proceed in the following three steps, and the proof of Main Theorem~\ref{thm:consistency_asymptotic_normality}
    is deferred to Appendix~\ref{sup:setup}.\\
    \textbf{Step 1. Boundedness and excess-risk control.}
    On the attainment event, the empirical bridge in
    Eq.~\eqref{eq:empirical_sm_mle_bridge} bounds the empirical likelihood
    excess risk through the terminal discrepancy,
    \(\widehat{\calR}_n^{\,\mathsf{MLE}}(\widehat\vartheta_{n,T_n}^{\,\mathsf{SM}})-\widehat{\calR}_n^{\,\mathsf{MLE}}(\vartheta^\ast)
    = \widehat{\calR}_{n,T_n}^{\,\mathsf{SM}}(\widehat\vartheta_{n,T_n}^{\,\mathsf{SM}})-\widehat{\calR}_{n,T_n}^{\,\mathsf{SM}}(\vartheta^\ast) + \Delta_n(\widehat\vartheta_{n,T_n}^{\,\mathsf{SM}}) - \Delta_n(\vartheta^\ast) \leq \Delta_n(\widehat\vartheta_{n,T_n}^{\,\mathsf{SM}})\).
    To bound the estimator, we compare the transition-kernel KL upper bound
    with a quadratic lower bound on the likelihood excess risk. This constrains the norm of 
    \(\widehat\theta_{n,T_n}^{\,\mathsf{SM}}\) induced by the empirical covariate second-moment
    matrix through a quadratic inequality, whose leading coefficient \(\frac{1}{2} - \frac{\bar{\alpha}_{T_n}}{2 (1- \bar{\alpha}_{T_n})}\) tends to \(1/2\) while the remaining
    coefficients stay \(\calO_P(1)\) under finite second moments. Solving this
    inequality and using convergence of the empirical second-moment matrix to
    \(\Sigma\succ0\), we conclude that
    \(\|\widehat\theta_{n,T_n}^{\,\mathsf{SM}}\|=\calO_P(1)\), and the
    vanishing terminal-signal condition then implies
    \(\Delta_n(\widehat\vartheta_{n,T_n}^{\,\mathsf{SM}})=\calO_P(\bar{\alpha}_{T_n})=o_P(1)\).\\
    \textbf{Step 2. Consistency up to joint sign.}
    The population likelihood excess risk is a conditional KL divergence with
    zero set \(\{\vartheta^\ast,-\vartheta^\ast\}\). Together with minimum
    attainment and the empirical bridge, Step~1 makes
    \(\widehat\vartheta_{n,T_n}^{\,\mathsf{SM}}\) an \(o_P(1)\)-near maximizer
    of the empirical likelihood. We therefore apply Theorem~5.14 of
    \citet[p.~48]{van2000asymptotic} to establish consistency up to a joint sign
    change, \(\widetilde\vartheta_{n,T_n}^{\,\mathsf{SM}}\xrightarrow{\Pr}\vartheta^\ast\),
    and \(\|\widehat\vartheta_{n,T_n}^{\,\mathsf{SM}}\|=\calO_P(1)\).\\
    \textbf{Step 3. Asymptotic normality.}
    With consistency of \(\widetilde\vartheta_{n,T_n}^{\,\mathsf{SM}}\)
    established through the boundedness and vanishing-discrepancy bounds in
    Steps~1 and~2, the distance between \(\vartheta^\ast\) and
    \(\widetilde\vartheta_{n,T_n}^{\,\mathsf{SM}}\) is bounded by
    \(\left\|\widetilde\vartheta_{n,T_n}^{\,\mathsf{SM}}
        -\vartheta^\ast\right\|=o_P(1)
    \).
    For a fixed compact convex neighborhood \(U\) of \(\vartheta^\ast\), this
    entire segment connecting \(\vartheta^\ast\) and
    \(\widetilde\vartheta_{n,T_n}^{\,\mathsf{SM}}\) therefore lies in \(U\) with probability tending to one.
    For the local gradients and curvatures, we differentiate the
    empirical bridge in Eq.~\eqref{eq:empirical_sm_mle_bridge}:
    \[
        \nabla\widehat{\calR}_{n,T_n}^{\,\mathsf{SM}}
        =\nabla\widehat{\calR}_n^{\,\mathsf{MLE}}-\nabla\Delta_n,
        \qquad
        \nabla^2\widehat{\calR}_{n,T_n}^{\,\mathsf{SM}}
        =\nabla^2\widehat{\calR}_n^{\,\mathsf{MLE}}-\nabla^2\Delta_n.
    \]
    To control the derivatives of the terminal discrepancy \(\Delta_n\), we have 
    the following observation:
    At the ground truth parameters, this averaged score has expectation zero under the data
    distribution, while the independence across observations and the
    fourth-moment assumption control its second moment, and the Hessian can be bounded uniformly:
    \[
       \E[\nabla\Delta_n(\vartheta^\ast)]=0,\qquad
       \E\|\nabla\Delta_n(\vartheta^\ast)\|^2
       =\calO(\bar\alpha_{T_n}^2/n),\qquad
       \sup_{\vartheta\in U}\|\nabla^2\Delta_n(\vartheta)\|
       =\calO_P(\bar\alpha_{T_n}).
    \]
    Under \(\bar\alpha_{T_n}=o(1)\), it follows that
    \(\sqrt n\,\nabla\Delta_n(\vartheta^\ast)
    =\calO_P(\bar\alpha_{T_n})=o_P(1)\) and
    \(\sup_{\vartheta\in U}\|\nabla^2\Delta_n(\vartheta)\|=o_P(1)\).
    We next control the empirical likelihood curvature along the segment connecting \(\vartheta^\ast\) and
    \(\widetilde\vartheta_{n,T_n}^{\,\mathsf{SM}}\).
    Under the fourth-moment assumption, the per-observation
    negative-log-likelihood Hessian is continuous in the parameters and admits
    an integrable envelope. Consequently, the information identity and the
    weak law of large numbers
    \citep[Theorem~2.2.14]{durrett2019probability} imply
    \(\nabla^2\widehat{\calR}_n^{\,\mathsf{MLE}}(\vartheta^\ast)
    =I(\vartheta^\ast)+o_P(1)\), while continuity and this integrable envelope
    on the compact set \(U\) justify the uniform law of large numbers:
    \[
        \sup_{\vartheta\in U}
        \left\|\nabla^2\widehat{\calR}_n^{\,\mathsf{MLE}}(\vartheta)
        -\E\!\left[\nabla^2\widehat{\calR}_n^{\,\mathsf{MLE}}(\vartheta)\right]
        \right\|=o_P(1).
    \]
    Here the expectation denotes the population Hessian evaluated at a fixed
    candidate parameter under the ground truth law, which is continuous by
    dominated convergence. Because the segment shrinks in probability by
    Step~2, the variation of this population Hessian from
    \(I(\vartheta^\ast)\) is uniformly \(o_P(1)\) along the segment. 
    The terminal-Hessian correction vanishes uniformly, so the
    twice-differentiated bridge transfers this curvature limit to score matching:
    \[
        \int_0^1\nabla^2\widehat{\calR}_{n,T_n}^{\,\mathsf{SM}}
        \bigl((1-\lambda)\vartheta^\ast+
        \lambda\widetilde\vartheta_{n,T_n}^{\,\mathsf{SM}}\bigr)\,\mathrm{d}\lambda
        =I(\vartheta^\ast)+o_P(1).
    \]
    On the attainment event, joint-sign invariance makes the aligned
    estimator \(\widetilde\vartheta_{n,T_n}^{\,\mathsf{SM}}\) a minimizer as
    well, so
    \(
        \nabla\widehat{\calR}_{n,T_n}^{\,\mathsf{SM}}
        \bigl(\widetilde\vartheta_{n,T_n}^{\,\mathsf{SM}}\bigr)=\mathbf0.
    \)
    Expanding this zero-gradient condition around \(\vartheta^\ast\) via
    Taylor's formula and substituting both the differentiated bridge and the
    integrated Hessian limit \(I(\vartheta^\ast)+o_P(1)\), we obtain
    \[
        0=\nabla\widehat{\calR}_{n,T_n}^{\,\mathsf{SM}}
        \bigl(\widetilde\vartheta_{n,T_n}^{\,\mathsf{SM}}\bigr)
        =\nabla\widehat{\calR}_n^{\,\mathsf{MLE}}(\vartheta^\ast)
        -\nabla\Delta_n(\vartheta^\ast)
        +[I(\vartheta^\ast)+o_P(1)]
        \bigl(\widetilde\vartheta_{n,T_n}^{\,\mathsf{SM}}-\vartheta^\ast\bigr).
    \]

    Rearranging the gradient terms and multiplying by \(\sqrt n\), we apply
    the multivariate central limit theorem
    \citep[Theorem~3.10.7]{durrett2019probability} to the i.i.d. likelihood
    scores with mean zero and covariance \(I(\vartheta^\ast)\) to obtain
    \(-\sqrt n\,\nabla\widehat{\calR}_n^{\,\mathsf{MLE}}(\vartheta^\ast)
    \xrightarrow{\mathsf d}\mathcal N(\mathbf0,I(\vartheta^\ast))\),
    while \(\sqrt n\,\nabla\Delta_n(\vartheta^\ast)=o_P(1)\).
    Noting that \(I(\vartheta^\ast)\) is positive definite, so the inverse
    of the integrated Hessian converges in probability to
    \(I(\vartheta^\ast)^{-1}\). Since the sum of the two scaled gradient terms
    is \(\calO_P(1)\), inversion first bounds the estimation error by
    \(\calO_P(n^{-1/2})\). We can then replace the inverse matrix
    with its limit and apply Slutsky's theorem:
    \[
        \sqrt n\bigl(\widetilde\vartheta_{n,T_n}^{\,\mathsf{SM}}
        -\vartheta^\ast\bigr)
        =-I(\vartheta^\ast)^{-1}\sqrt n\,
        \nabla\widehat{\calR}_n^{\,\mathsf{MLE}}(\vartheta^\ast)+o_P(1)
        \xrightarrow{\mathsf d}
        \mathcal N\bigl(\mathbf0,I(\vartheta^\ast)^{-1}\bigr).
    \]
\end{proof sketch}

\begin{remark}[On Consistency and Asymptotic Normality]\label{rmk:consistency_asymptotic_normality}
    The joint sign flip \((\theta,\nu)\mapsto(-\theta,-\nu)\) exchanges
    the component labels and leaves both the empirical score matching and
    maximum-likelihood objectives invariant. After sign alignment, the theorem
    establishes the same limiting distribution as maximum likelihood, and the
    centered terminal-gradient bound in
    \hyperref[proof:consistency_asymptotic_normality_step_3]{Step~3} of the
    proof of Main Theorem~\ref{thm:consistency_asymptotic_normality}
    also controls the paired discrepancy. With
    \(\widetilde\vartheta_{n,T_n}^{\,\mathsf{SM}}
    :=s_n\widehat\vartheta_{n,T_n}^{\,\mathsf{SM}}\) for \(s_n\) as in
    Main Theorem~\ref{thm:consistency_asymptotic_normality}, and
    \(\widetilde\vartheta_n^{\,\mathsf{MLE}}\) denoting the maximum-likelihood
    estimator aligned to the same ground truth representative, 
    we bound the discrepancy between the two estimators in probability (see
    Proposition~\ref{prop:paired_sm_mle} in Appendix~\ref{sup:setup}) by 
    \begin{equation}\label{eq:paired_sm_mle_main}
        \bigl\|\widetilde\vartheta_{n,T_n}^{\,\mathsf{SM}}
              -\widetilde\vartheta_n^{\,\mathsf{MLE}}\bigr\|
        =\calO_P\!\left(\frac{\bar\alpha_{T_n}}{\sqrt n}\right).
    \end{equation}
    Since \(\bar\alpha_{T_n}\to0\), Eq.~\eqref{eq:paired_sm_mle_main}
    implies \(\sqrt n(\widetilde\vartheta_n^{\,\mathsf{MLE}}
    -\widetilde\vartheta_{n,T_n}^{\,\mathsf{SM}})=\calO_P(\bar\alpha_{T_n})=o_P(1)\),
    and we have both consistency and asymptotic normality with the same limiting distribution
    for the score matching and maximum-likelihood estimators 
    such that \(\sqrt n(\widetilde\vartheta_{n,T_n}^{\,\mathsf{SM}}-\vartheta^\ast) \stackrel{\mathsf d}{\longrightarrow} \mathcal N(\mathbf0,I(\vartheta^\ast)^{-1})\), 
    \(\sqrt{n} (\widetilde\vartheta_{n}^{\,\mathsf{MLE}}-\vartheta^\ast) \stackrel{\mathsf d}{\longrightarrow} \mathcal N(\mathbf0,I(\vartheta^\ast)^{-1})\).

    The fixed points of EM operators are the stationary points of the empirical log-likelihood,
    satisfying
    \(\nabla_\vartheta\widehat{\calR}_n^{\,\mathsf{MLE}}(\vartheta)=\mathbf0\)
    (Proposition~\ref{prop:empirical_em_stationarity} in Appendix~\ref{sup:setup}).
    In particular, every finite global maximum-likelihood estimator is an
    fixed point of EM operators. Therefore, for an EM estimator attaining the global maximum of the
    empirical log-likelihood, the rate \(\bar\alpha_{T_n}/\sqrt n\) in
    Eq.~\eqref{eq:paired_sm_mle_main} also applies to its difference from
    the score matching estimator after alignment to the same ground truth
    \(\vartheta^\ast\).
    We evaluate the statistical behavior of the score matching estimator
    in Appendix~\ref{sup:additional_experiments}
    by comparing score matching against numerical MLE and EM fits across different
    sample sizes, measuring
    the paired score matching--MLE discrepancy
    (Figure~\ref{fig:mle-em-comparison} and Table~\ref{tab:mle-em-comparison}),
    and checking calibration across dimensions and covariate laws
    (Figure~\ref{fig:design-calibration}).
    We also examine the rate \(\bar\alpha_{T_n}/\sqrt n\) in Eq.~\eqref{eq:paired_sm_mle_main}
    across terminal schedules and sample sizes using the root-mean-square
    paired discrepancy in Figure~\ref{fig:terminal-schedules}(a).

    \citet{chewi2025arxiv} related path-integrated score matching to
    likelihood and established asymptotic normality with inverse
    Fisher-information covariance under assumed consistency and local
    likelihood regularity. For MLR, we establish consistency and asymptotic
    normality by proving empirical minimum attainment and verifying local
    likelihood regularity and nonsingularity of the joint Fisher information.
    In contrast, controlling the terminal discrepancy through its objective value, as in
    \citet{chewi2025arxiv}, provides a useful alternative sufficient condition. Specifically,
    \hyperref[proof:consistency_asymptotic_normality_step_1]{Step~1} of the
    proof of Main Theorem~\ref{thm:consistency_asymptotic_normality}
    bounds the empirical likelihood gap by \(\calO_P(\bar\alpha_{T_n})\), so
    \(\bar\alpha_{T_n}=o(n^{-1})\) makes the score matching estimator an
    \(o_P(n^{-1})\)-near maximizer of the empirical likelihood. Schedules of
    the form \(n^{-(1+\delta)}\) with \(\delta>0\), including our experimental
    choice \(n^{-3/2}\), satisfy this stronger condition, under which standard
    M-estimation establishes the same asymptotic normality (see the remark on the
    objective-value transfer argument in Appendix~\ref{sup:setup}). This
    stronger \(\bar\alpha_{T_n}=o(n^{-1})\) requirement stems from bounding the
    terminal discrepancy solely by its value, whereas differentiating the
    bridge and exploiting the zero-mean centering of the terminal parameter
    gradient at the ground truth establishes the paired rate in
    Eq.~\eqref{eq:paired_sm_mle_main} and relaxes the schedule condition to
    any deterministic \(\bar\alpha_{T_n}\to0\).

    If instead \(\bar\alpha_{T_n}\to a\in(0,1)\),
    the aligned score matching estimator remains consistent under the finite 
    second-moment condition, as established in the
    \hyperref[rmk:root_n_terminal_limit]{remark on consistency and \(\sqrt{n}\)-consistency}
    in Appendix~\ref{sup:setup}. Under the fourth-moment condition, this
    estimator is \(\sqrt{n}\)-consistent and asymptotically normal with a
    sandwich covariance \(V_a\) that differs from the inverse Fisher-information
    matrix \(I^{-1}\). We show that the
    population Hessian at the ground truth parameters is positive definite and
    derive this sandwich covariance in the
    \hyperref[rmk:fixed_terminal_covariance]{covariance comparison}
    in Appendix~\ref{sup:setup}.
    To examine the effect of terminal-scale
    selection \(\bar{\alpha}_{T_n}= a\) in finite samples, we vary the value of \(a\) in
    Appendix~\ref{sup:additional_experiments}
    (Figure~\ref{fig:terminal-schedules}(b)) and compare the observed loss of
    statistical efficiency at fixed positive terminal levels against the
    prediction of this sandwich variance \(V_a\).
\end{remark}

\begin{remark}[Terminal Schedules]\label{rmk:terminal_schedule}
    Every deterministic schedule satisfying \(0<\bar\alpha_{T_n}<1\) and
    \(\bar\alpha_{T_n}\to0\) is admissible in
    Main Theorem~\ref{thm:consistency_asymptotic_normality}. For sufficiently
    large \(n\), admissible choices include \(n^{-\gamma}\) for any
    \(\gamma>0\), \(c/n\) for any \(c>0\), and \(1/\ln n\). Even slowly
    vanishing schedules therefore retain both consistency and asymptotic
    normality with inverse Fisher-information covariance under the respective
    moment assumptions.
\end{remark}

\begin{remark}[Stochastic Training]\label{rmk:computation_noise}
    We present stochastic training experiments in Appendix~\ref{sup:additional_experiments} 
    to jointly learn the regression parameters and mixing weights 
    of the parametric score model, evaluated over 50 datasets 
    for each combination of dimension \(d \in \{1, 10\}\) 
    and sample size \(n \in \{1000, 4000, 16000\}\). 
    Forming the path-integrated training objective from the empirical loss in Eq.~\eqref{eq:auxiliary_loss_finite} 
    across sampled diffusion levels \(\bar\alpha_t\) and Gaussian noise \(\xi\), 
    we evaluate the sampled loss using three formulations: 
    the squared denoising integrand in Eq.~\eqref{eq:denoising_integrand}, 
    the conditional Hyv\"arinen integrand in Eq.~\eqref{eq:hyvarinen_integrand}, 
    and a hybrid of the two. 
    After averaging over the Gaussian noise and observed data at each diffusion level, 
    both integrands equal the empirical loss in Eq.~\eqref{eq:auxiliary_loss_finite} 
    up to parameter-independent constants (see Eq.~\eqref{eq:empirical_integrand_equivalence}), 
    so these noise-averaged objectives share the same minimizers and expected parameter gradients. 
    However, their stochastic gradients differ sharply in variance at each fixed diffusion level: 
    the variance of the stochastic gradient grows as \((1-\bar\alpha_t)^{-1}\) at low noise \(\bar\alpha_t \to 1\) for the denoising form, 
    whereas it grows as \(\bar\alpha_t^{-1}\) at high noise \(\bar\alpha_t \to 0\) for the conditional Hyv\"arinen form. 
    Combining the squared denoising integrand at high noise \(\bar\alpha_t \leq 1/2\) with the conditional Hyv\"arinen integrand at low noise \(\bar\alpha_t > 1/2\), 
    the hybrid integrand avoids variance growth at both endpoints (see Figure~\ref{fig:trained-score-model}(a)), 
    retains nearly unchanged optimization accuracy as \(\bar\alpha_{T_n}\) decreases at a fixed iteration budget, 
    and reaches comparable accuracy to either individual integrand in fewer iterations (see Figure~\ref{fig:trained-score-model}(b)). 
    Furthermore, we compare training errors across dimensions \(d\) and sample sizes \(n\) at a fixed iteration budget in Table~\ref{tab:trained-score-model}, 
    supporting the use of the hybrid integrand when extending the diffusion path and choosing the training budget so that optimization error remains small relative to sampling error.
\end{remark}

We examine the corresponding Fisher-metric error and chi-square limits in
Figure~\ref{fig:finite-sample-scaling} for the illustration of established statistical guarantees in Theorem~\ref{thm:consistency_asymptotic_normality}.

\begin{figure}[!tbhp]
    \centering
    \includegraphics[width=\linewidth]{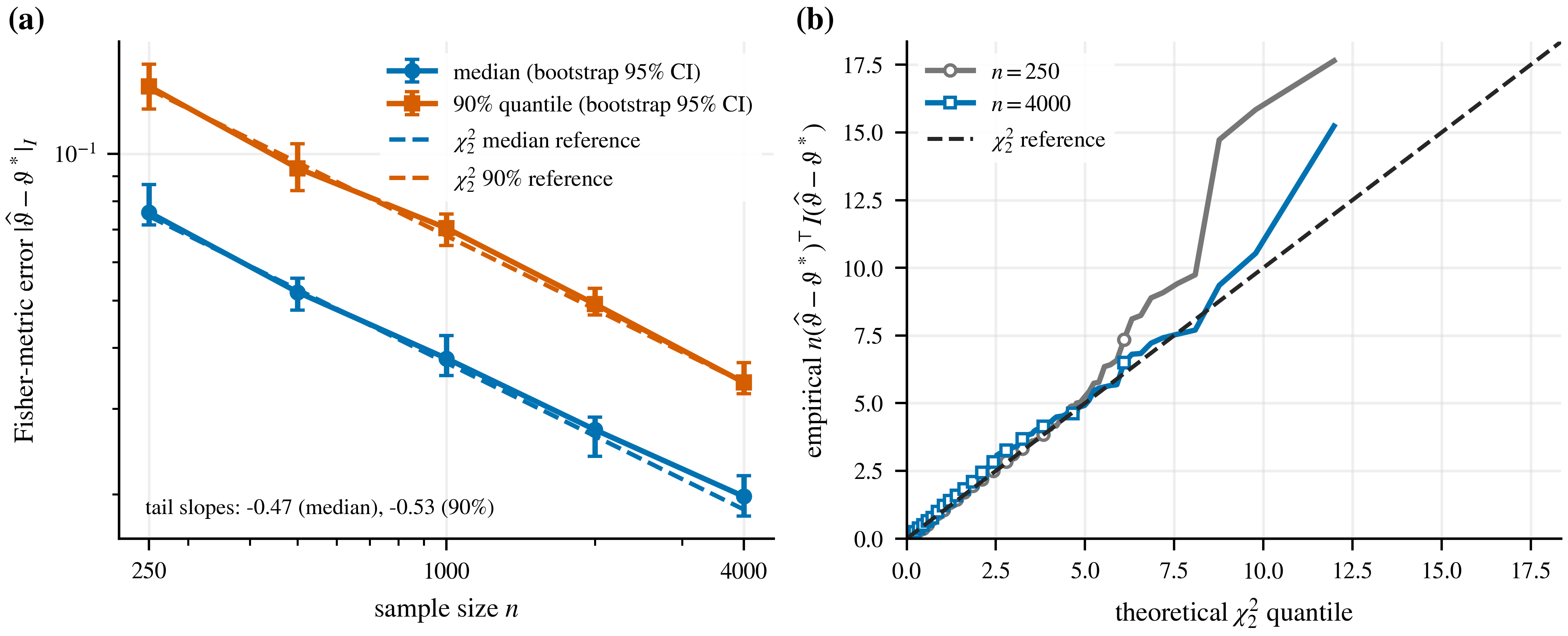}
    \caption{\textbf{Illustration of the finite-sample estimation results in Theorem~\ref{thm:consistency_asymptotic_normality}.}
    The experiment uses a two-parameter Gaussian 2MLR instance, the theorem's
    joint sign alignment, and a terminal sequence whose terminal-scale
    condition is verified in Appendix~\ref{sup:experiment}. Panel (a) reports
    the median and 90\% quantile of the error after joint sign alignment in the
    Fisher metric, together with 95\% bootstrap intervals and the corresponding
    limiting quantiles of $\sqrt{\chi^2_2/n}$. Over the three largest sample
    sizes, the fitted slopes are $-0.471$ and $-0.526$, close to the predicted
    $-1/2$ in the rate of $n^{-1/2}$. Panel (b) compares the squared, $n$-scaled Fisher error with
    $\chi^2_2$ at the smallest and largest sample sizes. Coverage of the joint
    95\% Fisher ellipse ranges from 93.5\% to 96.5\%, equals 95.0\% at the
    largest sample size, and has a corresponding covariance discrepancy of
    0.070. These diagnostics provide finite-sample illustrations under the
    alignment and terminal-scale conditions of Main Theorem~\ref{thm:consistency_asymptotic_normality}. Appendix~\ref{sup:experiment}
    provides additional calibration details.}
    \label{fig:finite-sample-scaling}
\end{figure}

\section{Optimization Analysis of Fixed-Scale and Path-Integrated Score Matching: Cross-Entropy and EM Decompositions}\label{sec:connect}
In this section, we relate the fixed-scale and path-integrated score matching losses to
cross-entropy and the population EM operators to explain its optimization
behavior. We derive equivalent loss decompositions and gradient identities,
and interpret the operator-Jacobian and latent-variance corrections through
the scale derivatives of the EM operators in Eq.~\eqref{eq:sm_em_scale_correction}.
We then derive the exact regression update for preconditioned gradient
descent on the fixed-scale loss in Eq.~\eqref{eq:preconditioned_sm_regression_update}
and the joint update on the path-integrated loss
in Eq.~\eqref{eq:path_integrated_em_update}. The path-integrated update involves
only EM operators at the original and terminal scales.
Finally, we characterize the low- and high-noise limits of both the loss
and its gradients. The proofs of Main
Theorem~\ref{thm:information_theoretical_decomposition},
Propositions~\ref{prop:em_update}--\ref{prop:cross_entropy}
and~\ref{prop:sm_gradients}--\ref{prop:high_noise_effective_coefficient}, and
Corollaries~\ref{cor:sm_gradients_entropy_activation}--\ref{cor:sm_gradients_em_variance} 
in this section are deferred to Appendix~\ref{sup:connect}.

Recall that the cross-entropy is the expected negative log-likelihood
\(\calH(\theta_t,\nu)=\E_{\ast}[F(y_t,\mu_t,\nu)]\), where
\(F(y_t,\mu_t,\nu)=-\ln p(y_t\mid\bx;\theta_t,\pi)\) and
\(\E_{\ast}\) denotes expectation under the ground truth distribution of
\((\bx,y_t)\). To link \(\calH(\theta_t,\nu)\) to population EM, note that
under the candidate model \((\theta_t,\nu)\), the posterior mean of the latent
sign given \((\bx,y_t)\) is \(\tanh(\mu_t y_t+\nu)\), which the population EM
operators \(M\) and \(N\) use to update \(\theta_t\) and the mixing-weight
imbalance \(\tanh\nu\), respectively. Differentiating \(F\) via
Lemma~\ref{lemma:derivatives_negative_log_likelihood} yields the identities for
\(M\) and \(N\) and their derivatives in Proposition~\ref{prop:em_update}, from
which Proposition~\ref{prop:cross_entropy} relates the cross-entropy gradients
to the corresponding EM residuals.

\begin{proposition}[Expectation-Maximization Rules]\label{prop:em_update}
    The population EM operators for the regression parameter \(\theta_t\) and imbalance of mixing weights \(\tanh \nu\) are
    \[
    M(\theta_t, \nu) = \Sigma^{-1}\E_{\ast}[\bx y_t \tanh(\mu_t y_t + \nu)],\qquad
    N(\theta_t, \nu) = \E_{\ast}[\tanh(\mu_t y_t + \nu)].
    \]
    In addition to defining the EM updates, these operators satisfy the scalar
    identity
    \[
    \langle \Sigma M(\theta_t, \nu), \theta_t \rangle
    = \E_{\ast}[\mu_t y_t\tanh(\mu_t y_t + \nu)]
    = \E_{\ast}[y_t^2 - y_t \partial_{y_t} F(y_t, \mu_t, \nu)],
    \]
    while differentiation with respect to \(\nu\) yields
    \(
    \Sigma \nabla_{\nu} M(\theta_t, \nu)
    = \nabla_{\theta_t} N(\theta_t, \nu)
    = \E_{\ast}[\bx y_t \sech^2(\mu_t y_t + \nu)].
    \)
\end{proposition}
\begin{proposition}[Cross-Entropy and EM Update Rules]\label{prop:cross_entropy}
    The cross-entropy \(\calH(\theta_t,\nu)\) satisfies
    \[
    \nabla_{\theta_t} \calH(\theta_t, \nu) = \Sigma \left(\theta_t - M(\theta_t, \nu)\right),\qquad
    \nabla_{\nu} \calH(\theta_t, \nu) = \tanh \nu - N(\theta_t, \nu).
    \]
\end{proposition}

To connect these EM identities to score matching, we first differentiate the
population bridge in Eq.~\eqref{eq:population_sm_mle_bridge} with respect to
its lower endpoint. The resulting identity relates the score matching loss to the rate of change of the
cross-entropy gap as the retained-signal level increases.
\begin{proposition}[Score Matching as a Scale Derivative]\label{prop:sm_cross_entropy_scale_derivative}
    Assume \(\E[\|\bx\|^2]<\infty\). For
    \(0<\bar\alpha_t<1\),
    \[
        \calL_t(\theta,\nu)
        =\partial_{\bar\alpha_t}
        \bigl[\calH(\theta_t,\nu)-\calH(\theta_t^\ast,\nu^\ast)\bigr].
    \]
    Here \(\partial_{\bar\alpha_t}\) denotes differentiation with the original
    candidate parameters, ground truth parameters, and covariate law fixed.
\end{proposition}
\begin{proof}
    First, we express the terminal KL divergence through \(F\). The candidate
    and ground truth models have the same covariate law, which cancels in
    their density ratio, so
    \(\KL\bigl(p(y_t,\bx\mid\theta_t^\ast,\pi^\ast)
    \parallel p(y_t,\bx\mid\theta_t,\pi)\bigr)
    =\E_\ast[\ln\{p(y_t\mid\bx;\theta_t^\ast,\pi^\ast)/
    p(y_t\mid\bx;\theta_t,\pi)\}]\).
    Since \(F\) is the conditional negative log-likelihood and \(\calH\)
    is its ground truth expectation, this equals
    \(\E_\ast[F(y_t,\mu_t,\nu)-F(y_t,\mu_t^\ast,\nu^\ast)]
    =\calH(\theta_t,\nu)-\calH(\theta_t^\ast,\nu^\ast)\).
    Next, we take \(T=t\) in the population bridge,
    Eq.~\eqref{eq:population_sm_mle_bridge}. Its right-hand side is the
    initial KL divergence minus the terminal KL divergence above.
    We differentiate with respect to the lower integration endpoint to
    obtain the stated identity.
    To justify the endpoint differentiation, we couple the response laws
    using the same latent label and Gaussian noise in Eq.~\eqref{eq:yt}.
    The normalized squared score difference in Eq.~\eqref{eq:loss} is then
    continuous in \(\bar\alpha_t\) and bounded by
    \((|\mu_0|+|\mu_0^\ast|)^2/2\). This bound is integrable under
    second moments, so dominated convergence makes \(\calL_t\)
    continuous and the fundamental theorem of calculus applies.
\end{proof}

The next theorem expresses this loss in two equivalent forms. We express it
through the cross-entropy gradient and expected squared activation, then
apply Proposition~\ref{prop:cross_entropy} to express the same loss through
the EM regression operator and latent variance.
\begin{myframe}
\begin{main theorem}[Information Decomposition of the Score Matching Loss]\label{thm:information_theoretical_decomposition}
    The expected squared activation is
    \(\calA(\theta_t, \nu) := \frac{1}{2}\E_{\ast}[\mu_t^2 \tanh^2(\mu_t y_t + \nu)]\)
    and its ground truth counterpart is
    \(\calA(\theta_t^*, \nu^*) := \frac{1}{2}\E_{\ast}[(\mu_t^*)^2 \tanh^2(\mu_t^* y_t + \nu^\ast)]\).
    With these definitions, the score matching loss satisfies
    \[
    \bar{\alpha}_t \calL_t(\theta, \nu) = \langle \nabla_{\theta_t}\calH(\theta_t, \nu), \theta_t \rangle
    - \calA(\theta_t, \nu) + \calA(\theta_t^*, \nu^*).
    \]
    With the latent variance
    \(V(\theta_t, \nu) \equiv \frac{1}{2}\E_{\ast}[\mu_t^2 \sech^2(\mu_t y_t + \nu)]\),
    we can also express the loss as follows.
    \[
    \bar{\alpha}_t \calL_t(\theta, \nu) =
    \|\theta_t\|_{\Sigma}^2/2 -  \langle \Sigma M(\theta_t, \nu), \theta_t \rangle + \|\theta_t^\ast\|_{\Sigma}^2/2 + V(\theta_t, \nu) - V(\theta_t^\ast, \nu^\ast).
    \]
\end{main theorem}
\end{myframe}

\begin{corollary}[Gradients of the Score Matching Loss via Cross-Entropy and Expected Squared Activation]\label{cor:sm_gradients_entropy_activation}
    If \(\E[\|\bx\|^4]<\infty\), the gradients of the score matching loss \(\bar{\alpha}_t \calL_t(\theta, \nu)\) with respect to the regression parameter \(\theta_t\) and the mixing parameter \(\nu\) are as follows. The \(\nu\)-gradient identity requires only \(\E[\|\bx\|^2]<\infty\).
    \[
    \begin{aligned}
        \nabla_{\theta_t} [\bar{\alpha}_t \calL_t(\theta, \nu)] &= \nabla_{\theta_t}\calH(\theta_t, \nu) + \nabla^2_{\theta_t} \calH(\theta_t, \nu) \theta_t - \nabla_{\theta_t} \calA(\theta_t, \nu), \\
        \nabla_\nu [\bar{\alpha}_t \calL_t(\theta, \nu)] &= \langle \nabla^2_{\theta_t \nu} \calH(\theta_t, \nu), \theta_t \rangle - \nabla_\nu \calA(\theta_t, \nu).
    \end{aligned}
    \]
\end{corollary}

\begin{corollary}[Gradients of the Score Matching Loss via EM Operators and Latent Variance]\label{cor:sm_gradients_em_variance}
    If \(\E[\|\bx\|^4]<\infty\), the gradients of the score matching loss \(\bar{\alpha}_t \calL_t(\theta, \nu)\) with respect to \(\theta_t\) and \(\nu\) are as follows. The \(\nu\)-gradient identity requires only \(\E[\|\bx\|^2]<\infty\).
    \[
    \begin{aligned}
        \nabla_{\theta_t} [\bar{\alpha}_t \calL_t(\theta, \nu)] &= \Sigma (\theta_t - M(\theta_t, \nu)) - \big(\nabla_{\theta_t} M(\theta_t, \nu)\big)^\top \Sigma \theta_t + \nabla_{\theta_t} V(\theta_t, \nu), \\
        \nabla_\nu [\bar{\alpha}_t \calL_t(\theta, \nu)] &= -\langle \nabla_{\theta_t} N(\theta_t, \nu), \theta_t \rangle + \nabla_\nu V(\theta_t, \nu).
    \end{aligned}
    \]
\end{corollary}

The normalized EM operators are 
\(
    \bar{M}_t(\theta, \nu) := M(\theta_t, \nu) / \sqrt{\bar{\alpha}_t},\quad
    \bar{N}_t(\theta, \nu) := (N(\theta_t, \nu)-(1-\bar{\alpha}_t)\tanh \nu) / \bar{\alpha}_t.
\)
These definitions imply
\(\theta - \bar{M}_t(\theta, \nu) = (\theta_t - M(\theta_t, \nu))/\sqrt{\bar{\alpha}_t}\),
\(\tanh \nu - \bar{N}_t(\theta, \nu) = (\tanh \nu - N(\theta_t, \nu))/\bar{\alpha}_t\),
and \(\Sigma \nabla_{\nu} \bar{M}_t(\theta, \nu) = \nabla_{\theta} \bar{N}_t(\theta, \nu)\),
which yield the following gradient identities with \(V_t(\theta, \nu) := V(\theta_t, \nu)/\bar{\alpha}_t\).
\begin{proposition}[Gradients of the Score Matching Loss]\label{prop:sm_gradients}
    If \(\E[\|\bx\|^4]<\infty\), the gradients of the score matching loss \(\calL_t(\theta, \nu)\) are as follows. The \(\nu\)-gradient identity requires only \(\E[\|\bx\|^2]<\infty\).
    \[
    \begin{aligned}
        \nabla_{\theta} \calL_t(\theta, \nu) &= \Sigma (\theta - \bar{M}_t(\theta, \nu)) - \left( \nabla_{\theta}\bar{M}_t(\theta, \nu) \right)^\top \Sigma \theta + \nabla_\theta V_t(\theta, \nu),\\
        \nabla_{\nu} \calL_t(\theta, \nu) &= -\langle \nabla_{\theta} \bar{N}_t(\theta, \nu), \theta \rangle + \nabla_{\nu} V_t(\theta, \nu).
    \end{aligned}
    \]
\end{proposition}

\begin{remark}
    As \(\bar{\alpha}_t\to 1\), we have
    \(\theta_t \equiv \theta \sqrt{\bar{\alpha}_t} \to \theta\), so
    \(\bar{M}_t(\theta, \nu) \to M(\theta, \nu)\),
    \(\bar{N}_t(\theta, \nu) \to N(\theta, \nu)\), and
    \(V_t(\theta, \nu) \to V(\theta, \nu)\). Under
    \(\E[\|\bx\|^4]<\infty\), substitution into
    Proposition~\ref{prop:sm_gradients} establishes
    \begin{equation}
    \begin{aligned}
        \lim_{\bar{\alpha}_t\to 1}\nabla_{\theta} \calL_t(\theta, \nu) &= \Sigma (\theta - M(\theta, \nu)) - \left( \nabla_{\theta}M(\theta, \nu) \right)^\top \Sigma \theta + \nabla_\theta V(\theta, \nu),\\
        \lim_{\bar{\alpha}_t\to 1}\nabla_{\nu} \calL_t(\theta, \nu) &= -\langle \nabla_{\theta} N(\theta, \nu), \theta \rangle + \nabla_{\nu} V(\theta, \nu).
    \end{aligned}
    \end{equation}
\end{remark}

\begin{remark}[Optimization Implications of the EM Decomposition]\label{rmk:em_decomposition_optimization}
We next express the preconditioned regression update for the path-integrated
population loss as a relaxed EM step plus a terminal correction, obtained by
writing the fixed-scale score matching gradient as a derivative with respect
to the retained-signal level.

Throughout this argument, assume \(\E[\|\bx\|^4]<\infty\).
We use the scale derivative defined in
Proposition~\ref{prop:sm_cross_entropy_scale_derivative}.
Because the ground truth term is independent of the candidate parameters and
the fourth-moment condition permits interchanging derivatives, we obtain
\(\nabla_\vartheta\calL_t
=\partial_{\bar\alpha_t}\nabla_\vartheta\calH(\theta_t,\nu)\).
By Proposition~\ref{prop:cross_entropy} and the chain rule,
\(\nabla_\theta\calH(\theta_t,\nu)
=\sqrt{\bar\alpha_t}\Sigma(\theta_t-M(\theta_t,\nu))
=\bar\alpha_t\Sigma(\theta-\bar M_t)\) and
\(\nabla_\nu\calH(\theta_t,\nu)=\tanh\nu-N(\theta_t,\nu)\).
Differentiating both identities with respect to the retained-signal level and
applying the product rule yields
\begin{equation}\label{eq:sm_em_scale_correction}
\begin{aligned}
    -\Sigma^{-1}\nabla_\theta\calL_t
    &=\bar M_t-\theta+\bar\alpha_t\,\partial_{\bar\alpha_t}\bar M_t
      =\partial_{\bar\alpha_t}[\bar\alpha_t(\bar M_t-\theta)],\\
    -\nabla_\nu\calL_t&=\partial_{\bar\alpha_t}N(\theta_t,\nu),
\end{aligned}
\end{equation}
where the arguments \((\theta,\nu)\) of \(\bar M_t\), \(\calL_t\), and
\(V_t\) are suppressed here and below.
By Proposition~\ref{prop:sm_gradients}, this scale correction satisfies
\(\bar\alpha_t\,\partial_{\bar\alpha_t}\bar M_t
=\Sigma^{-1}[(\nabla_\theta\bar M_t)^\top\Sigma\theta-\nabla_\theta V_t]\).
Consequently, for preconditioned gradient descent on the fixed-scale score
matching loss with step size \(\eta\), the regression update
\(\theta^+=\theta-\eta\Sigma^{-1}\nabla_\theta\calL_t\) becomes
\begin{equation}\label{eq:preconditioned_sm_regression_update}
    \theta^+=(1-\eta)\theta+\eta\bar M_t
    +\eta\bar\alpha_t\,\partial_{\bar\alpha_t}\bar M_t.
\end{equation}
Here \(\bar M_t\) denotes the EM regression update in the original
coordinates, and the scale correction materially alters this step whenever
\(\bar\alpha_t\|\partial_{\bar\alpha_t}\bar M_t\|_\Sigma\) is comparable to
\(\|\bar M_t-\theta\|_\Sigma\).

Turning to path-integrated score matching with the population loss
\(\int_{\bar\alpha_T}^{\bar\alpha_0}\calL_t(\theta,\nu)\,\d\bar\alpha_t\)
in Eq.~\eqref{eq:population_sm_mle_bridge}, where the candidate parameters
remain fixed across the integral, integrating
Eq.~\eqref{eq:sm_em_scale_correction} from \(\bar\alpha_T\) to
\(\bar\alpha_0=1\) telescopes the regression direction to
\([\bar\alpha_t(\bar M_t-\theta)]_{\bar\alpha_T}^{1}\).
Since \(\bar M_0=M\), preconditioned gradient descent on this
path-integrated score matching loss, with regression preconditioner
\(\Sigma^{-1}\) and step size \(\eta\), yields
\begin{equation}\label{eq:path_integrated_em_update}
\begin{aligned}
    \theta^+&=(1-\eta)\theta+\eta M(\theta,\nu)
        +\eta\bar\alpha_T\bigl(\theta-\bar M_T(\theta,\nu)\bigr),\\
    \nu^+&=\nu+\eta\bigl[N(\theta,\nu)-N(\theta_T,\nu)\bigr].
\end{aligned}
\end{equation}
Path integration therefore leaves only a terminal correction to the relaxed
EM regression step. From the EM expectation formulas,
\(\bar\alpha_T\bar M_T=\sqrt{\bar\alpha_T}M(\theta_T,\nu)\to0\) and
\(N(\theta_T,\nu)\to\tanh\nu\) as \(\bar\alpha_T\to0\), so the regression
update approaches the relaxed EM \((1-\eta)\theta+\eta M(\theta,\nu)=\theta-\eta \Sigma^{-1} \nabla_\theta \mathcal{H}(\theta, \nu)\), recovering the EM
regression step when \(\eta=1\). In the mixing-weight coordinate, the update
approaches the likelihood-gradient step
\(\nu+\eta[N(\theta,\nu)-\tanh\nu]=\nu - \eta \nabla_\nu \mathcal{H}(\theta, \nu)\), whereas EM sets
\(\nu^+=\operatorname{arctanh}N(\theta,\nu)\).
Moreover, evaluated at the ground truth parameters, both the fixed-scale and
terminal regression corrections vanish at every retained-signal level.

\citet{shah2023learning} studied Gaussian mixtures with known equal mixing
weights and identity component covariances, expressing the two-component
score matching update as a relaxed EM step plus a bounded correction to
establish low-noise contraction under separation and initialization
conditions. By contrast, for MLR with general covariates and unknown,
potentially unequal mixing weights, Eq.~\eqref{eq:sm_em_scale_correction}
characterizes the full preconditioned score matching direction through the EM
operators and their scale derivatives. This yields the fixed-scale regression
update in Eq.~\eqref{eq:preconditioned_sm_regression_update}, whereas path
integration eliminates the scale derivatives and produces the joint update in
Eq.~\eqref{eq:path_integrated_em_update} using only EM operators at the
original and terminal scales. To examine these optimization implications,
Appendix~\ref{sup:additional_experiments} compares EM and gradient descent on
the negative log-likelihood against \(\Sigma^{-1}\)-preconditioned gradient
descent on the fixed-scale and path-integrated score matching losses, and
quantifies the fixed-scale correction relative to the EM residual
(Figure~\ref{fig:matched-start-dynamics} and
Table~\ref{tab:matched-start-iterations}).
\end{remark}

\begin{proposition}[Limit at Low Noise and Expectation-Maximization]\label{prop:low_noise_em}
    If \(\E[\|\bx\|^4]<\infty\), the gradients of the score matching loss \(\calL_t(\theta, \nu)\) satisfy the following expansions as \(\bar{\alpha}_t \to 1\).
    \[
    \begin{aligned}
        \nabla_{\theta} \calL_t(\theta, \nu) &= \Sigma (\theta - M(\theta, \nu)) - \left( \nabla_{\theta}M(\theta, \nu) \right)^\top \Sigma \theta + \nabla_\theta V(\theta, \nu) + \calO(1-\bar{\alpha}_t),\\
        \nabla_{\nu} \calL_t(\theta, \nu) &= -\langle \nabla_{\theta} N(\theta, \nu), \theta \rangle + \nabla_{\nu} V(\theta, \nu) + \calO(1-\bar{\alpha}_t).
    \end{aligned}
    \]
\end{proposition}

\begin{remark}[Low-Noise Loss Limit]
For the loss value, fix \(\theta,\theta^\ast\). Under the independence assumptions of
Section~\ref{subsec:2mlr} and
\(\E[\|\bx\|^2]<\infty\), with
\(\mu_0:=\langle\theta,\bx\rangle\) and
\(\mu_0^\ast:=\langle\theta^\ast,\bx\rangle\),
the normalized conditional squared score discrepancy in
Lemma~\ref{lemma:asymptotic_score_loss} is bounded uniformly in
\(\bar\alpha_t\) by
\((|\mu_0|+|\mu_0^\ast|)^2\), which is integrable under the stated
second-moment condition. Dominated convergence over \(\bx\) therefore permits passage to
the low-noise limit
\begin{equation}\label{eq:low_noise_loss_limit}
\lim_{\bar{\alpha}_t \to 1} \calL_t(\theta,\nu)
= \frac{1}{2} \E_{\bx} \E_{y_0 \mid \bx; \theta^\ast, \pi^\ast} \left[ \left( \mu_0 \tanh(\mu_0 y_0 + \nu) - \mu_0^\ast \tanh(\mu_0^\ast y_0 + \nu^\ast) \right)^2 \right].
\end{equation}
Thus the objective limit itself requires only a second moment, whereas
Proposition~\ref{prop:low_noise_em} assumes the stronger fourth-moment condition
to control differentiation under the covariate expectation and obtain the
corresponding gradient expansion with EM operators, latent variance, and an \(\calO(1-\bar\alpha_t)\) remainder.
\end{remark}
\begin{proposition}[Effective Coefficient Limit at High Noise]\label{prop:high_noise_effective_coefficient}
    If \(\E[\|\bx\|^2] < \infty\), the effective coefficient \(c^\ast:=\theta^\ast\tanh\nu^\ast\) determines the limiting objective through
    \(
        \lim_{\bar{\alpha}_t \to 0} \calL_t(\theta, \nu) = \frac{1}{2}\|\theta \tanh \nu-c^\ast\|_{\Sigma}^2.
    \)
    The condition \(\E[\|\bx\|^4] < \infty\) suffices for the \(\calO(\bar\alpha_t)\) gradient expansions as \(\bar{\alpha}_t \to 0\)
    \[
        \nabla_{\theta} \calL_t(\theta, \nu) = \nabla_{\theta} \lim_{\bar{\alpha}_t \to 0} \calL_t(\theta, \nu) + \calO(\bar{\alpha}_t),\qquad \nabla_{\nu} \calL_t(\theta, \nu) = \nabla_{\nu} \lim_{\bar{\alpha}_t \to 0} \calL_t(\theta, \nu) + \calO(\bar{\alpha}_t).
    \]
\end{proposition}

\begin{remark}[High-Noise Limit and the Balanced Slice]
Proposition~\ref{prop:high_noise_effective_coefficient} characterizes the leading high-noise objective, whose gradient dynamics we analyze in Section~\ref{sec:theory}. We check its \(\calO(\bar\alpha_t)\) gradient remainder componentwise in Figure~\ref{fig:endpoint-gradients}(b).
On the balanced slice \(\nu^\ast=\nu=0\), we have \(c^\ast=\mathbf0\), and that leading objective is identically zero, so its gradient in \(\theta\) vanishes throughout the slice.
We report a separate next-order directional diagnostic, illustrated in Figure~\ref{fig:local-spectral}, which compares finite-noise score-matching updates with power steps on the spiked fourth-order covariance in Appendix~\ref{sup:experiment}.
\end{remark}

\begin{figure}[!htbp]
    \centering
    \includegraphics[width=\linewidth]{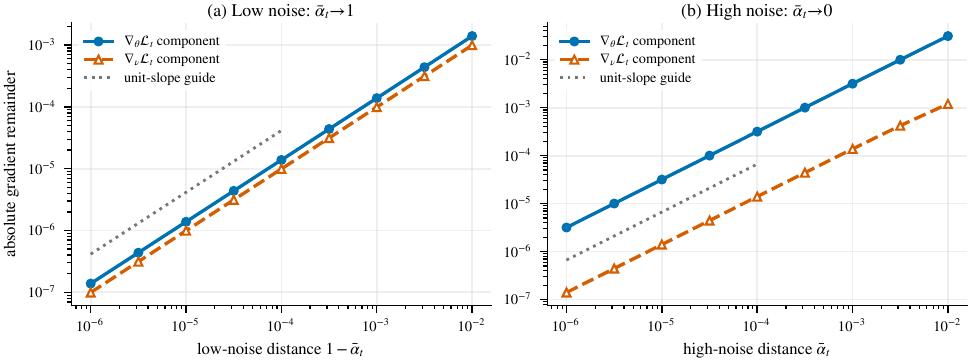}
    \caption{\textbf{First-order scaling of the endpoint gradient remainders in Propositions~\ref{prop:low_noise_em} and~\ref{prop:high_noise_effective_coefficient}.}
    At one fixed model pair, the panels plot componentwise remainders for the
    complete low-noise expression in Proposition~\ref{prop:low_noise_em} and the
    effective high-noise coefficient expression in
    Proposition~\ref{prop:high_noise_effective_coefficient}. The fitted log--log slopes
    track the guides with unit slope, consistent with the respective
    $\calO(1-\bar\alpha_t)$ and $\calO(\bar\alpha_t)$ bounds. The panels report
    one selected instance, whose settings and fitted values appear in
    Appendix~\ref{sup:experiment}.}
    \label{fig:endpoint-gradients}
\end{figure}

\FloatBarrier

\section{High-Noise Gradient Dynamics and Score Matching Blindness}\label{sec:theory}
In this section, we use the endpoint loss and gradient limits in
Section~\ref{sec:connect} to study optimization dynamics and the information
about mixing weights available to fixed-scale score matching.
We first analyze fixed-step gradient descent on the high-noise limiting
population objective in Main Theorem~\ref{thm:global_convergence_high_noise}
and then establish pointwise score matching blindness along fixed high-SNR
ground truth rays in Proposition~\ref{prop:pointwise_sm_blindness}.
These population results complement the statistical guarantees for the
empirical path-integrated estimator in
Main Theorem~\ref{thm:consistency_asymptotic_normality}.
Proofs for the results in this section are deferred to Appendix~\ref{sup:theory}.

\subsection{High-Noise Gradient Dynamics}\label{subsec:loss_landscape}

As shown in Proposition~\ref{prop:high_noise_effective_coefficient}, the high-noise limiting objective is \(\calL(\theta,\nu):=\lim_{\bar\alpha_t\to0}\calL_t(\theta,\nu)\).
We consider the optimization dynamics of gradient descent with a fixed step size \(\eta\):
\(
    \vartheta_{k+1}
    \leftarrow \vartheta_k-\eta\nabla\calL(\vartheta_k),
\) for the iterate \(\vartheta_k=(\theta_k,\nu_k)\) of parameters at time step \(k\).
The following informal Main Theorem~\ref{thm:global_convergence_high_noise} demonstrates the convergence rates of the optimization dynamics. 

\begin{myframe}
\begin{main theorem}[Informal; Global Dynamics as \(\bar{\alpha}_t\to0\)]\label{thm:global_convergence_high_noise}
    On the high-noise limiting loss
    \(\calL(\theta,\nu)=\frac12\|\theta\tanh\nu-c^\ast\|^2\)
    with \(c^\ast:=\theta^\ast\tanh\nu^\ast\) and isotropic covariates
    satisfying \(\Sigma=I_d\), fixed-step gradient descent with an admissible
    initialization has nonincreasing loss and obeys the
    following rates.\\
    (\(c^\ast=\mathbf0\).)
    The loss satisfies \(\calL(\theta_K,\nu_K)=\calO(K^{-2})\).\\
    (\(c^\ast\neq\mathbf0\), generic initializations.)
    Every admissible initialization outside the exceptional saddle basin has
    geometric loss decay after a finite transient and reaches any
    \(\epsilon>0\) in \(\calO(\log(1/\epsilon))\) iterations.\\
    (\(c^\ast\neq\mathbf0\), exceptional initializations.)
    The exceptional saddle basin consists of the admissible initializations
    whose trajectories converge to the saddle manifold
    \(\{(\theta,0):\langle\theta,c^\ast\rangle=0\}\). This basin has
    \((d+1)\)-dimensional Lebesgue measure zero, and every such trajectory has
    limiting loss \(\|c^\ast\|^2/2\).
\end{main theorem}
\end{myframe}

\begin{proof sketch}
    We first establish discrete Lyapunov bounds showing that the loss
    decreases by at least \(\frac{\eta}{2}\|\nabla\calL(\vartheta_k)\|^2\)
    per iteration and that
    \(\calL(\vartheta_k)+\frac12(\|\theta_k\|^2-\sinh^2\nu_k)\)
    is nonincreasing. Together, these bounds keep the regression iterates
    bounded and preserve the admissible step-size condition across
    iterations.

    When \(c^\ast=\mathbf0\), we prove an explicit
    Kurdyka--\L{}ojasiewicz (K\L{}) gradient inequality with exponent
    \(3/4\), bounding \(\|\nabla\calL(\vartheta_k)\|^2\) below by a positive
    multiple of \(\calL(\vartheta_k)^{3/2}\). Combining this lower bound with
    the per-iteration loss decrease yields the \(K^{-2}\) upper bound.

    When \(c^\ast\ne\mathbf0\), the coupled sign recurrences show that each
    admissible trajectory either converges to the saddle manifold or enters
    the region \(\langle\theta_k,c^\ast\rangle\nu_k>0\) after finitely many
    iterations. Once in this region, \(|\tanh\nu_k|\) remains bounded away
    from zero, so the gradient bound
    \(\|\nabla\calL(\vartheta_k)\|^2
    \ge2\tanh^2\nu_k\,\calL(\vartheta_k)\)
    yields geometric loss decay.

    Finally, to characterize the initializations that converge to the saddle
    manifold, we adapt the center-stable-manifold argument of
    \citet{lee2016gradient}. Because the local center-stable manifolds have
    codimension one, analyticity of the update map and local invertibility
    outside a measure-zero critical set ensure that their iterated preimages
    also have measure zero. Thus the admissible saddle basin has zero
    \((d+1)\)-dimensional Lebesgue measure. The precise admissibility
    condition, explicit rate constants, and full proof are deferred to
    Theorem~\ref{thm:global_convergence_high_noise_full} in
    Appendix~\ref{sup:theory}.
\end{proof sketch}

\begin{figure}[!htb]
    \centering
    \includegraphics[width=\linewidth]{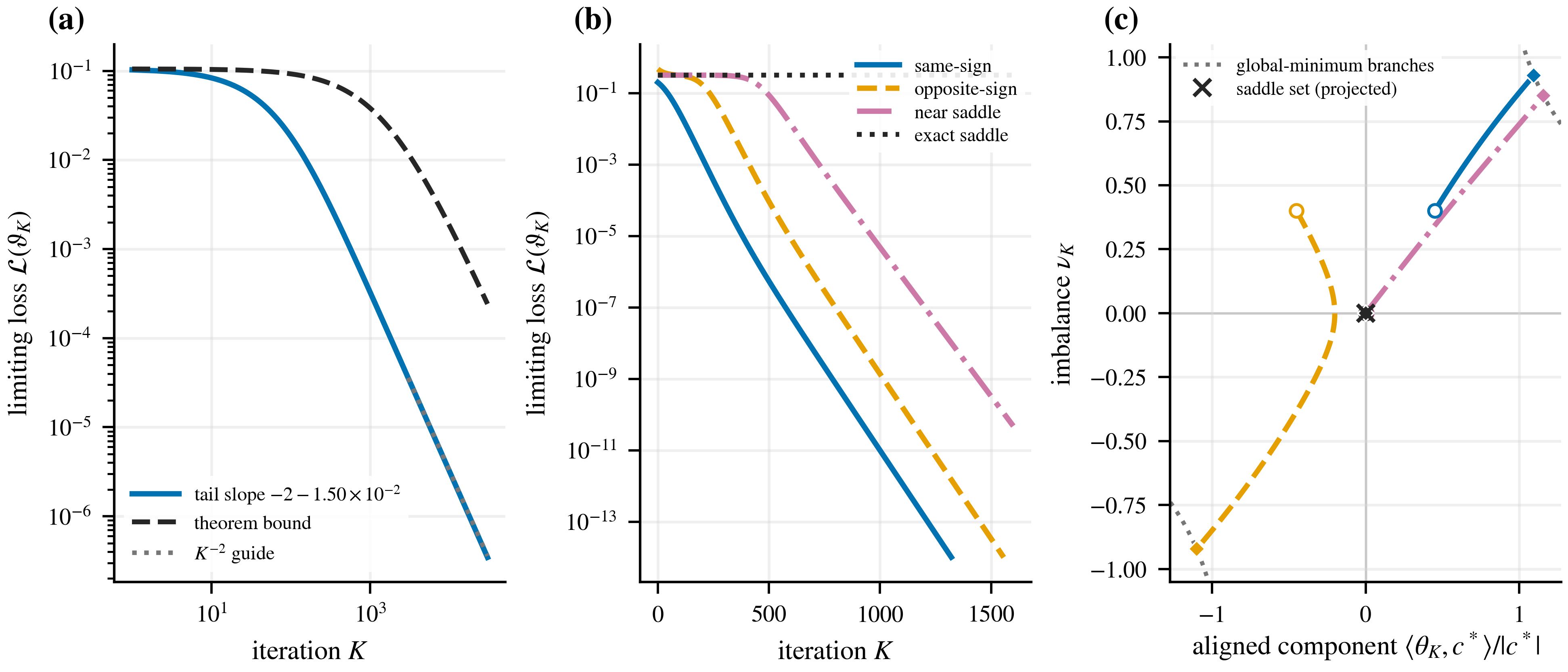}
    \caption{\textbf{Optimization Dynamics of Gradient Descent on the High-Noise Objective in Theorem~\ref{thm:global_convergence_high_noise}.}
    The gradient descent update with a fixed step size is iterated from initializations
    satisfying the admissibility condition in
    Theorem~\ref{thm:global_convergence_high_noise_full}, and every displayed path
    preserves that condition and has nonincreasing loss. (a) For the degenerate
    target, the computed loss remains below the $K^{-2}$ upper bound in
    Theorem~\ref{thm:global_convergence_high_noise}, and its fitted log--log
    tail slope is $-2-1.50\times10^{-2}$, tracking the anchored $K^{-2}$ guide
    over the displayed tail. (b) For $c^\ast\ne\mathbf0$, the three displayed
    generic initializations lie outside the exceptional saddle basin and show
    initialization-dependent transients followed by geometric loss decay,
    consistent with the $\calO(\log(1/\epsilon))$ conclusion. (c) The
    phase portrait in $(\langle\theta_K,c^\ast\rangle/\|c^\ast\|,\nu_K)$
    coordinates marks the initial and final iterates with open circles and
    filled diamonds. The dotted curves are the two branches of the global-minimum
    set $\{(\theta,\nu):\theta\tanh\nu=c^\ast\}$. The saddle set
    $\{(\theta,0):\langle\theta,c^\ast\rangle=0\}$ projects to the single point
    $(0,0)$ in these coordinates and is marked by the black cross. The observed
    degenerate slope diagnoses finite-window agreement with
    the $K^{-2}$ upper bound in Theorem~\ref{thm:global_convergence_high_noise}.
    Appendix~\ref{sup:experiment} provides the configurations for generating the figures.}
    \label{fig:high-noise-dynamics}
\end{figure}

\begin{remark}[Interpreting the Dynamics at High Noise]\label{rmk:global_convergence_high_noise}
    The high-noise limiting objective identifies the ground truth only through
    the effective coefficient \(c^\ast=\theta^\ast\tanh\nu^\ast\), so all
    ground truth pairs with the same product induce the same target. Every
    candidate parameter pair matching this target minimizes the limiting loss,
    and convergence to zero loss recovers this effective coefficient. In the
    degenerate zero-target case, which includes \(\theta^\ast=\mathbf0\) and
    balanced mixtures \(\nu^\ast=0\), the theorem bounds the loss from above
    at a uniform polynomial rate. This upper bound is consistent with
    the K\L{} gradient-inequality exponent \(3/4\) of a singular loss formed
    by a squared product and its associated \(\calO(K^{-2})\) bound on
    function values~\citep{josz2026computing}. For \(c^\ast\neq\mathbf0\), the
    almost-global conclusion holds trajectory by trajectory within the
    admissible class, with both the transient duration and the contraction
    factor depending on the initialization.

    Whereas the \(\calO(\bar\alpha_t)\) gradient expansion in
    Proposition~\ref{prop:high_noise_effective_coefficient} holds pointwise at
    fixed parameters, our long-run convergence guarantees concern the
    high-noise limiting population objective, so extending them to a fixed
    positive retained-signal level or to finite-sample training requires
    control along the full trajectory. At high but finite noise,
    \citet{shah2023learning} used a power-iteration approximation to recover
    the component-mean direction in balanced spherical Gaussian mixtures. In
    our MLR model, the leading high-noise objective retains the effective
    coefficient \(c^\ast\) and becomes flat in \(\theta\) on the balanced
    slice \(\nu=\nu^\ast=0\). For \(c^\ast\ne\mathbf0\), by contrast, the
    coupled product parameterization creates both a manifold of minimizers and
    an exceptional saddle basin, which our proof handles by controlling the
    iterates through a discrete Lyapunov bound and sign recurrences while
    adapting the standard center-stable-manifold argument for saddle
    avoidance~\citep{lee2016gradient}.
\end{remark}
\begin{remark}[Whitening for the Anisotropic Covariates]\label{rmk:whitening_high_noise}
The analysis of optimization behavior in the isotropic-covariate setting can be extended to anisotropic covariates with a
known positive-definite \(\Sigma=\E[\bx\bx^\top]\) via
the whitened quantities \(\widetilde\bx:=\Sigma^{-1/2}\bx\),
\(\widetilde\theta:=\Sigma^{1/2}\theta\), and
\(\widetilde c^\ast:=\Sigma^{1/2}c^\ast\).
Then \(\E[\widetilde\bx\widetilde\bx^\top]=I_d\), and the high-noise
limiting score matching loss \(\calL_\Sigma(\theta,\nu)= \lim_{\bar{\alpha}_t \to 0} \calL_t(\theta,\nu)\) from Proposition~\ref{prop:high_noise_effective_coefficient}
becomes
\begin{equation}\label{eq:whitened_high_noise_loss}
    \calL_\Sigma(\theta,\nu)
    :=\frac12\|\theta\tanh\nu-c^\ast\|_\Sigma^2
    =\frac12\|\widetilde\theta\tanh\nu-\widetilde c^\ast\|^2.
\end{equation}
In the original coordinates of \((\theta, \nu)\), gradient descent on this whitened objective in
\((\widetilde\theta,\nu)\) coincides with the preconditioned gradient descent iteration
\begin{equation}\label{eq:whitened_gradient_iteration}
    \theta_{k+1}=\theta_k-\eta\Sigma^{-1}
        \nabla_\theta\calL_\Sigma(\theta_k,\nu_k),
    \qquad
    \nu_{k+1}=\nu_k-\eta\nabla_\nu\calL_\Sigma(\theta_k,\nu_k).
\end{equation}
Applying the admissibility condition of
Theorem~\ref{thm:global_convergence_high_noise_full} to
\((\widetilde\theta_0,\nu_0)\) and \(\widetilde c^\ast\) therefore yields
the same same convergence gaurantees for Eq.~\eqref{eq:whitened_gradient_iteration},
and because the linear coordinate change is invertible, the exceptional
saddle basin remains a set of Lebesgue measure zero.

Moreover, Eq.~\eqref{eq:whitened_gradient_iteration} is the high-noise
limit of preconditioned gradient descent on the fixed-scale score matching
loss with step size \(\eta\), with the regression update in
Eq.~\eqref{eq:preconditioned_sm_regression_update} as \(\bar\alpha_t\to0\), 
under the fourth-moment condition in Proposition~\ref{prop:high_noise_effective_coefficient}.
To assess this extension under anisotropic covariates, we
evaluate the preconditioned iterations with exact and estimated
second-moment matrices in Appendix~\ref{sup:additional_experiments}
(Figure~\ref{fig:anisotropic-whitening}).
\end{remark}

\subsection{Score Matching Blindness at High SNR}

We next study how the mixing-weight gradient of the fixed-scale score
matching loss behaves at high SNR, as the ground truth regression
parameters grow along a fixed ray with the ground truth mixing weights,
candidate parameters, and diffusion scale held constant.

\begin{proposition}[Pointwise Score Matching Blindness Along a Fixed Ray at High SNR]\label{prop:pointwise_sm_blindness}
    The covariate vector \(\bx\) satisfies \(\E[\|\bx\|^2]<\infty\). The scale of the diffusion noise level
    \(\bar{\alpha}_t\in(0,1]\) and finite model pair
    \((\theta,\nu)\in\R^d\times\R\) are fixed. Let
    \(\theta^\ast\ne\mathbf0\) vary with \(\|\theta^\ast\|\to\infty\) along
    a fixed ray, while the ground truth mixing weights and all other model and
    data-generating quantities remain fixed. When
    \(\theta\ne\mathbf0\), we also impose the following projection condition.
    \[
        \Pr\!\left(\{\langle\theta,\bx\rangle=0\}\cup\{\langle\theta^\ast,\bx\rangle=0\}\right)=0.
    \]
    The expectations defining \(\calL_t\) and \(V_t\) are taken under the
    corresponding varying ground truth.
    At the specified fixed model pair and fixed scale of the diffusion noise level, the following
    pointwise limits hold.
    \[
        \lim_{\|\theta^\ast\|\to\infty}|\nabla_\nu\calL_t(\theta,\nu)|=0,
        \qquad
        \lim_{\|\theta^\ast\|\to\infty}V_t(\theta,\nu)=0.
    \]
\end{proposition}

\begin{remark}[Blindness at High SNR]\label{rmk:pointwise_sm_blindness}
In the fixed-scale high-SNR regime of
Proposition~\ref{prop:pointwise_sm_blindness}, EM and maximum likelihood
retain sensitivity to the mixing weights. Specifically, for a nonzero
candidate \(\theta\) and isotropic Gaussian covariates,
the population EM limit in Lemma~\ref{lemma:high_snr_em_imbalance} is
\(\tanh\nu^+=N(\theta_t,\nu)\to(2/\pi)\arcsin(\rho)\tanh\nu^\ast\) for the next step, where
\(\rho:=\langle\theta,\theta^\ast\rangle/(\|\theta\|\|\theta^\ast\|)\),
matching the infinite-SNR population update in
\citet[Corollary~3.3]{luo2024unveiling}. Although independent of the
incoming value \(\nu\), this limit retains information about the ground truth
mixing-weight imbalance whenever \(\rho\ne0\). By
Proposition~\ref{prop:cross_entropy}, the likelihood gradient with respect
to \(\nu\) therefore converges to
\(\tanh\nu-(2/\pi)\arcsin(\rho)\tanh\nu^\ast\), which remains nonzero
unless \(\tanh\nu\) already matches that limit.

In contrast, the first limit in
Proposition~\ref{prop:pointwise_sm_blindness} establishes score blindness
with respect to the mixing weights, since every finite \(\nu\) parametrizes
a pair of interior candidate mixing weights bijectively through
\(\pi(1)=(1+\tanh\nu)/2, \pi(2)=(1-\tanh\nu)/2\). Even with the ground
truth mixing weights held constant, growing component separation eliminates
the local first-order sensitivity of the fixed-scale score matching loss to
the mixing weights at fixed candidate parameters. Indeed, for any nonzero
candidate \(\theta\) satisfying the projection condition, the posterior
assignment variance \(\sech^2(\mu_t y_t+\nu)\) vanishes almost surely as
posterior responsibilities saturate toward hard component assignments.
Because both terms of the mixing-weight gradient in
Proposition~\ref{prop:sm_gradients} contain this factor, bounding their
integrands and that of \(V_t\) by integrable envelopes independent of the
ground truth signal norm yields both vanishing expectations. Asymptotically
hard component assignments thus annihilate the local mixing-weight gradient
of fixed-scale score matching, whereas the likelihood gradient retains the
imbalance residual above.

This parameter-gradient perspective complements prior analyses of score
insensitivity in covariate-free Gaussian mixtures. Whereas
\citet{wenliang2020blindness} showed that the Fisher divergence between
Gaussian mixtures with common components and distinct mixing weights
vanishes as the components separate,
Proposition~\ref{prop:pointwise_sm_blindness} captures the corresponding
loss of mixing-weight information in MLR through the vanishing gradient in
\(\nu\) and latent variance at fixed candidate parameters and diffusion
scale as the ground truth regression parameters grow along a fixed ray. To
recover mixing weights despite such fixed-scale insensitivity under known
or fixed component distributions,
\citet{schwienhorst2026diffusiondenoising} derived finite-sample error
bounds for known Gaussian component means by integrating denoising over
time with a separation-tuned horizon, while
\citet{dennehy2026mixtureweights} quantified intermediate-scale score
sensitivity to control squared mixing-weight error through the integrated
score loss.
To examine these vanishing limits across covariate laws, we estimate the
decay rates of the mixing-weight gradient and latent variance in
Appendix~\ref{sup:additional_experiments}
(Figure~\ref{fig:design-blindness}).
\end{remark}

\FloatBarrier
\begin{figure}[!ht]
    \centering
    \includegraphics[width=.94\linewidth]{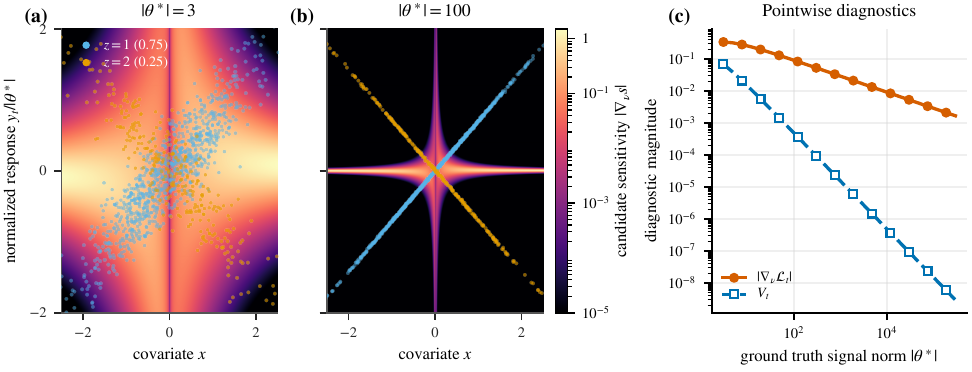}
    \caption{\textbf{Score Matching Blindness at High SNR in Proposition~\ref{prop:pointwise_sm_blindness}.}
    The displayed instance has \(d=1\), \(X\sim\calN(0,1)\),
    \(\bar\alpha_t=\tfrac12\), the fixed candidate
    \((\theta,\nu)=(0.80,-0.35)\), and \(\nu^\ast=0.55\), while
    \(\theta^\ast>0\) grows along the positive ray. This setting satisfies
    Proposition~\ref{prop:pointwise_sm_blindness}. (a)--(b) Following
    \citet{wenliang2020blindness}, samples from the varying ground truth models are overlaid
    on the fixed-candidate sensitivity
    \(|\nabla_\nu s_{\theta_t,\nu}(y_t,x)|=|\mu_t|\sech^2(\mu_ty_t+\nu)\)
    on a log scale. The normalized responses concentrate on
    \(y_t/\|\theta^\ast\|\simeq\pm\sqrt{\bar\alpha_t}\,x\), and the appreciable
    sensitivity concentrates near the ray intersection. (c) Quadrature shows
    that \(|\nabla_\nu\calL_t|\) and the latent variance term
    \(V_t=V(\theta_t,\nu)/\bar\alpha_t\) decrease along the fixed ray, providing
    pointwise evidence at one fixed model and scale. Appendix
    Figure~\ref{fig:high-snr-blindness-controls} provides fixed-scale and
    cross-entropy controls together with an auxiliary moving-candidate
    diagnostic for a separate setting. At high noise,
    the limiting loss in Proposition~\ref{prop:high_noise_effective_coefficient} is
    \(\frac12\|\theta\tanh\nu-c^\ast\|_\Sigma^2\), whose directional signal
    vanishes at balance. At low noise, Proposition~\ref{prop:low_noise_em}
    retains operator-Jacobian and
    posterior-variance corrections beyond the EM residuals.}
    \label{fig:high-snr-blindness}
\end{figure}

\section{Numerical Experiments}\label{sec:experiments}
In this section, we numerically evaluate the statistical behavior of the
path-integrated score matching estimator
(Main Theorem~\ref{thm:consistency_asymptotic_normality}). We then check the
low- and high-noise gradient expansions
(Propositions~\ref{prop:low_noise_em}
and~\ref{prop:high_noise_effective_coefficient}) and examine the high-noise limiting
dynamics (Main Theorem~\ref{thm:global_convergence_high_noise}) and
pointwise score matching blindness at high SNR
(Proposition~\ref{prop:pointwise_sm_blindness}). Experimental settings,
supplementary checks, and an auxiliary finite-noise spectral diagnostic are
deferred to Appendix~\ref{sup:experiment}.

To assess the asymptotic
predictions of Main Theorem~\ref{thm:consistency_asymptotic_normality}, we fit an
MLR instance with Gaussian covariates satisfying the model assumptions under an admissible
vanishing terminal schedule (Figure~\ref{fig:finite-sample-scaling}). Panel~(a) shows error
quantiles in the Fisher metric decreasing approximately as the inverse
square root of the sample size, while Panel~(b) compares the
sample-size-scaled squared Fisher error against the predicted chi-square
distribution.

To check the gradient expansions, we hold the candidate and
ground truth parameters fixed and measure the remainders
(Figure~\ref{fig:endpoint-gradients}). They follow the first-order
scaling predicted by Propositions~\ref{prop:low_noise_em}
and~\ref{prop:high_noise_effective_coefficient}. Panel~(a) evaluates the
full low-noise expression, including the EM-operator derivatives and
latent-variance corrections, while Panel~(b) evaluates the high-noise
effective-coefficient expression.

We study the convergence regimes of
Main Theorem~\ref{thm:global_convergence_high_noise} by running gradient descent on the
high-noise limiting objective from admissible initializations
(Figure~\ref{fig:high-noise-dynamics}). When the
effective coefficient is zero, the loss remains below the inverse-square
upper bound across iterations. When the effective coefficient is nonzero,
the displayed generic trajectories exhibit geometric decay after an initial
transient, whereas a trajectory initialized at the saddle remains
stationary.

To examine the pointwise blindness limit of
Proposition~\ref{prop:pointwise_sm_blindness}, we vary the ground truth signal
strength in one dimension (Figure~\ref{fig:high-snr-blindness}).
As the ground truth regression parameter grows along a fixed ray with the
candidate parameters, diffusion noise level, and ground truth mixing
weights held constant, both the mixing-weight gradient magnitude and the
latent-variance term decay toward zero.

We summarize the extended numerical studies in
Table~\ref{tab:additional-study-guide} and describe them in
Appendix~\ref{sup:additional_experiments}. There, we further assess our
statistical and optimization results by comparing path-integrated score
matching against MLE and EM (Figure~\ref{fig:mle-em-comparison} and
Table~\ref{tab:mle-em-comparison}) and by evaluating stochastic training of
the parametric score matching model (Figure~\ref{fig:trained-score-model}
and Table~\ref{tab:trained-score-model}).

\section{Conclusions}\label{sec:conclusion}
We study variance-preserving diffusion of the response in
symmetric two-component mixed linear regression with unknown mixing weights.
We establish statistical guarantees for the path-integrated estimator and
analyze the optimization of both fixed-scale and path-integrated score
matching losses. For the first question posed in
the introduction, Proposition~\ref{prop:terminal_kl_pinsker_zero_set} shows
that the terminal population KL
divergence at every positive retained-signal level has the two joint-sign
representatives of the ground truth as its global minimizers. Main
Theorem~\ref{thm:consistency_asymptotic_normality} proves that, under finite covariate second moments, the
path-integrated score matching estimator recovers the regression parameters and
mixing weights up to their joint-sign symmetry whenever the deterministic
terminal retained-signal factor \(\bar{\alpha}_{T_n}\to 0\). Under the fourth-moment condition,
it establishes asymptotic normality with inverse-Fisher covariance, matching the
limiting distribution of the maximum-likelihood estimator (MLE).
Under the same assumptions, we also bound the difference
between the exact score matching and maximum-likelihood estimators fitted
to the same sample and aligned to the same ground truth representative
by \(\calO_P(\bar\alpha_{T_n}/\sqrt n)\)
(Eq.~\eqref{eq:paired_sm_mle_main}).

For the second question, we express the fixed-scale score matching loss
through cross-entropy and population Expectation--Maximization (EM) operators,
and decompose its gradient into EM residuals, operator-Jacobian terms, and
latent-variance corrections (Main Theorem~\ref{thm:information_theoretical_decomposition}
and Proposition~\ref{prop:sm_gradients}). These identities yield exact updates
for preconditioned gradient descent. At a fixed scale, the regression update
consists of a relaxed EM step and a correction determined by the scale
derivative of the EM operator (Eq.~\eqref{eq:preconditioned_sm_regression_update}).
Path integration eliminates these scale derivatives and yields a joint update
involving only EM operators at the original and terminal scales
(Eq.~\eqref{eq:path_integrated_em_update}). Its regression component approaches
a relaxed EM step as the terminal retained-signal factor vanishes.

The endpoint and signal-strength results explain how noise changes both the
optimization dynamics and the parameter information retained by score matching.
At low noise, Equation~\eqref{eq:low_noise_loss_limit} identifies the
limiting loss, and Proposition~\ref{prop:low_noise_em} shows that the gradient
retains the EM-related correction structure. At high noise,
Proposition~\ref{prop:high_noise_effective_coefficient} shows that the limiting
loss depends on the regression vector and mixing imbalance only through their
effective coefficient. Under the stated
admissibility conditions, Main
Theorem~\ref{thm:global_convergence_high_noise} proves an
\(\calO(K^{-2})\) loss bound when this effective target is zero and eventual
geometric decay for generic initializations when it is nonzero, with only a
measure-zero exceptional basin attracted to the saddle manifold. Finally,
Proposition~\ref{prop:pointwise_sm_blindness} proves that, along a fixed
high-SNR ground truth ray and at each covered fixed candidate and diffusion
scale, the imbalance gradient and latent-variance term converge pointwise to
zero, eliminating the local first-order signal for changing the candidate
mixing weights.


\FloatBarrier
\bibliographystyle{tmlr}
\bibliography{ref_diffusion_2MLR}

\appendix


\newpage
\part{Appendices}
\parttoc

The appendices are organized as follows.

\begin{itemize}
    \item Appendix~\ref{sup:lemma}: We derive the diffused MLR density and score,
    establish the relative-entropy dissipation and the population and empirical
    likelihood bridges for path-integrated score matching, characterize the
    terminal KL zero set and its Wasserstein relation, and prove the covariance
    formulas for the results in
    Sections~\ref{sec:setup}--\ref{sec:experiments}.
    \item Appendix~\ref{sup:setup}: Using the likelihood bridges from
    Appendix~\ref{sup:lemma}, we prove consistency up to joint sign,
    the aligned Gaussian limit with inverse-Fisher covariance for any
    deterministic vanishing terminal retained-signal schedule, and
    high-probability attainment of the empirical score matching minimum
    to establish the statistical guarantees in Section~\ref{sec:asymptotic}.
    \item Appendix~\ref{sup:connect}: We prove the
    cross-entropy--activation and EM--latent-variance decompositions, derive
    their associated gradient identities, and establish the low- and
    high-noise endpoint results from
    Section~\ref{sec:connect}.
    \item Appendix~\ref{sup:theory}: We prove the convergence rates of the
    high-noise gradient dynamics, establish pointwise score matching blindness
    along fixed high-SNR rays, and derive the hard-assignment EM comparison for
    the results in Section~\ref{sec:theory}.
    \item Appendix~\ref{sup:experiment}: We report the numerical configurations,
    implementation details, calibration calculations, and additional spectral
    and high-SNR blindness diagnostics for the results in
    Section~\ref{sec:experiments}.
    \item Appendix~\ref{sup:additional_experiments}: To assess the finite-sample
    implications of our statistical results, we compare path-integrated score
    matching with MLE and EM, check calibration across dimensions and covariate
    laws, and vary terminal schedules. We also examine stochastic training error
    and gradient variance, the optimization implications of the EM decomposition,
    whitening for high-noise dynamics, and pointwise high-SNR blindness across
    covariate laws.
\end{itemize}

\newpage

\section{Auxiliary Identities and Bounds for Diffused Mixed Linear Regression}\label{sup:lemma}
This appendix develops the model-specific identities and bounds used by the
main proofs, proceeding from the diffused MLR representation and score
formulae to the relative-entropy, terminal-identification, Wasserstein, and
fourth-order covariance results.
\begin{lemma}[Equation~\eqref{eq:yt} derived from the MLR model in Eq.~\eqref{eq:2mlr} and SDE~\eqref{eq:sde}]\label{lemma:mlr_sde}
    Let \(y_0\) be generated by the MLR model in Eq.~\eqref{eq:2mlr} with fixed parameters
    \(\theta,\pi\). Suppose that \(\varepsilon_0\sim\calN(0,1)\) is
    independent of \((\bx,z)\). Let \(y_t\) be generated by
    SDE~\eqref{eq:sde} with noise schedule \(\beta_t\), whose Brownian motion
    is independent of \((\bx,z,\varepsilon_0)\). Then, for every \(t>0\),
    \(p(y_t\mid y_0,\bx;\theta,\pi)=p(y_t\mid y_0)
    =\calN(y_t;y_0\sqrt{\bar\alpha_t},1-\bar\alpha_t)\), and
    Equation~\eqref{eq:yt} holds.
    \[
        y_t = y_0\sqrt{\bar{\alpha}_t} + \xi_t \sqrt{1-\bar{\alpha}_t} = (-1)^{z+1} \langle \theta_t, \bx \rangle + \varepsilon_t = (-1)^{z+1} \mu_t + \varepsilon_t,
    \]
    where \(\bar{\alpha}_t := \exp\left(-\int_0^t \beta_\tau \d \tau \right)\) is the cumulative noise schedule,
    \(\theta_t := \theta \sqrt{\bar{\alpha}_t}\) is the scaled regression
    parameter, and \(\mu_t := \langle \theta_t, \bx \rangle\) is the
    corresponding mean. The Gaussian noises
    \(\xi_t,\varepsilon_t\sim\mathcal{N}(0,1)\) are independent of
    \((\bx,z)\).
\end{lemma}

\begin{lemma}[Negative Log-Likelihood]\label{lemma:nll}
    For the model in Section~\ref{subsec:2mlr}, the negative log-likelihood at
    the scale of the diffusion noise level \(t\) is
    \(F(y_t,\mu_t,\nu):=-\ln p(y_t\mid\bx;\theta_t,\pi)\), where
    \(\theta_t=\theta\sqrt{\bar\alpha_t}\). It has the form
    \[
        F(y_t, \mu_t, \nu) = \frac{1}{2}(y_t^2 + \mu_t^2) - \ln \frac{\cosh(\mu_t y_t+\nu)}{\cosh \nu} + \frac{1}{2} \ln(2\p).
    \]
\end{lemma}

\begin{lemma}[Stein Score Function]\label{lemma:stein_score}
    The Stein score is
    \(s_{\theta_t,\nu}(y_t,\bx):=\nabla_{y_t}\ln
    p(y_t\mid\bx;\theta_t,\pi)\). It has the form
    \[
        s_{\theta_t, \nu} (y_t, \bx) = - y_t + \tanh (\mu_t y_t + \nu) \mu_t.
    \]
\end{lemma}

\begin{proof}[Proof of Lemma~\ref{lemma:mlr_sde}]
    To remove the linear drift, we apply It\^{o}'s formula to \(y_t/\sqrt{\bar{\alpha}_t}\). We use \(\d[1/\sqrt{\bar{\alpha}_t}] = \frac{\beta_t}{2 \sqrt{\bar{\alpha}_t}}\d t\) for \(\bar{\alpha}_t := \exp\left(-\int_0^t \beta_\tau \d \tau \right)\) and substitute \(\d y_t = -\frac{\beta_t}{2}y_t \d t + \sqrt{\beta_t}\d B_t\) from SDE~\eqref{eq:sde}. The drift terms cancel:
    \[
    \d\left[\frac{y_t}{\sqrt{\bar{\alpha}_t}}\right] 
    = \frac{1}{\sqrt{\bar{\alpha}_t}}\d y_t + y_t \d\left[\frac{1}{\sqrt{\bar{\alpha}_t}}\right]
    = \sqrt{\frac{\beta_t}{\bar{\alpha}_t}}\d B_t.
    \]
    We integrate from \(0\) to \(t\) and set \(\theta_t := \theta \sqrt{\bar{\alpha}_t}\).
    Substituting \(y_0 = (-1)^{z+1}\langle \theta, \bx \rangle + \varepsilon_0\) from the MLR model in Eq.~\eqref{eq:2mlr}, we obtain
    \[
    y_t 
    = \sqrt{\bar{\alpha}_t} y_0 + \sqrt{\bar{\alpha}_t}\int_0^t \sqrt{\frac{\beta_\tau}{\bar{\alpha}_\tau}}\d B_\tau
    = (-1)^{z+1} \underbrace{\langle \theta_t, \bx \rangle}_{=:{\color{plum}\;\mu_t}} + \underbrace{\sqrt{\bar{\alpha}_t} \left[\varepsilon_0 + \int_0^t \sqrt{\frac{\beta_\tau}{\bar{\alpha}_\tau}}\d B_\tau\right]}_{=:{\color{royalBlue}\;\varepsilon_t}}
    = (-1)^{z+1} {\color{plum}\mu_t} + {\color{royalBlue}\varepsilon_t}.
    \]
    Since \(\sqrt{\beta_\tau / \bar{\alpha}_\tau}\) is deterministic, the stochastic integral is a zero-mean Gaussian.
    The initial noise satisfies \(\varepsilon_0 \sim \calN(0, 1)\), and the
    forward Brownian motion is independent of \((\bx,z,\varepsilon_0)\).
    Hence their linear combination \(\varepsilon_t\) is also a zero-mean
    Gaussian, with
    \[
    \E[{\color{royalBlue} \varepsilon_t}] = \sqrt{\bar{\alpha}_t} \left( \E[\varepsilon_0] + \E\left[\int_0^t \sqrt{\frac{\beta_\tau}{\bar{\alpha}_\tau}}\d B_\tau\right]\right) = \sqrt{\bar{\alpha}_t} \left(0 + 0\right) = 0.
    \]
    Since \(\varepsilon_0\sim\mathcal{N}(0,1)\) is independent of
    \(\{B_\tau:0\leq\tau\leq t\}\), the variance of \(\varepsilon_t\) is
    the sum of the two variances.
    By It\^{o}'s isometry, \(\Var\big[\int_0^t \sqrt{\frac{\beta_\tau}{\bar{\alpha}_\tau}}\d B_\tau\big] = \int_0^t \frac{\beta_\tau}{\bar{\alpha}_\tau}\d\tau\).
    We evaluate this variance using \(\d \left[ \frac{1}{\bar{\alpha}_\tau}\right] = \frac{\beta_\tau}{\bar{\alpha}_\tau}\d \tau\). Thus
    \(\int_0^t \sqrt{\frac{\beta_\tau}{\bar{\alpha}_\tau}}\d B_\tau \sim \calN(0,1/\bar{\alpha}_t-1)\) and
    \[
    \Var[{\color{royalBlue}\varepsilon_t}] = \bar{\alpha}_t \left(\Var[\varepsilon_0] + \Var\left[\int_0^t \sqrt{\frac{\beta_\tau}{\bar{\alpha}_\tau}}\d B_\tau\right]\right) 
    = \bar{\alpha}_t \left(1 + \int_0^t \frac{\beta_\tau}{\bar{\alpha}_\tau}\d\tau\right) 
    = \bar{\alpha}_t \left(1+ \frac{1}{\bar{\alpha}_\tau}\Big\vert_0^t\right)  
    = 1.
    \]
    Thus, \(\varepsilon_t \sim \mathcal{N}(0, 1)\) and
    \(p(y_t\mid y_0,\bx;\theta,\pi)=p(y_t\mid y_0)
    =\calN(y_t;y_0\sqrt{\bar\alpha_t},1-\bar\alpha_t)\).
    For \(t>0\), it follows that \(\xi_t := (y_t - y_0\sqrt{\bar{\alpha}_t})/\sqrt{1-\bar{\alpha}_t} = \sqrt{\frac{\bar{\alpha}_t}{1-\bar{\alpha}_t}}\int_0^t \sqrt{\frac{\beta_\tau}{\bar{\alpha}_\tau}}\d B_\tau \sim \mathcal{N}(0, 1)\).
    The assumed independence of \(\varepsilon_0\), the Brownian motion, and
    \((\bx,z)\) shows that \(\xi_t\) and \(\varepsilon_t\) are independent
    of \((\bx,z)\).
\end{proof}

\begin{proof}[Proofs of Lemmas~\ref{lemma:nll} and~\ref{lemma:stein_score}]
    We first obtain the conditional density of \(y_t\) by marginalizing the
    joint density over \(z\in\{1,2\}\). Under the model assumptions,
    \(\Pr(z\mid\bx;\theta_t,\pi)=\Pr(z\mid\pi)=\pi(z)\). Hence
    \[
    F(y_t, \mu_t, \nu) := -\ln p(y_t \mid \bx; \theta_t, \pi)
    = -\ln \sum_{z \in \{1,2\}} p(y_t, z\mid \bx;\theta_t, \pi)
    = -\ln \sum_{z \in \{1,2\}} \pi(z) \cdot p(y_t \mid z, \bx; \theta_t,\pi).
    \]
    By Lemma~\ref{lemma:mlr_sde}, \(\varepsilon_t\sim\mathcal{N}(0,1)\),
    independent of \((\bx,z)\). Hence
    \(p(y_t \mid z, \bx; \theta_t, \pi) = p(y_t \mid z, \mu_t)=\calN((-1)^{z+1}\mu_t, 1)\), where \(\mu_t=\langle\theta_t,\bx\rangle\).
    \[
    F(y_t, \mu_t, \nu) = -\ln \sum_{z \in \{1,2\}} \pi(z) \cdot p(y_t \mid z, \mu_t) = -\ln \sum_{z \in \{1,2\}} \pi(z) \cdot \calN((-1)^{z+1}\mu_t, 1).
    \]
    The imbalance parameter \(\nu := (\ln\pi(1)-\ln\pi(2))/2\) is in one-to-one correspondence with the mixing weights. In particular,
    \(\pi(1) = \frac{\mathe^\nu}{2 \cosh \nu}\) and \(\pi(2) = \frac{\mathe^{-\nu}}{2 \cosh \nu}\). Substituting these weights into the mixture density, we obtain
    \begin{eqnarray*}
        F(y_t, \mu_t, \nu)
        &=& -\ln \left( \frac{\mathe^{\nu}}{2 \cosh \nu} \frac{1}{\sqrt{2\p}} \mathe^{-\frac{1}{2}(y_t^2 - 2y_t\mu_t + \mu_t^2)} + \frac{\mathe^{-\nu}}{2 \cosh \nu} \frac{1}{\sqrt{2\p}} \mathe^{-\frac{1}{2}(y_t^2 + 2y_t\mu_t + \mu_t^2)} \right) \\
        &=& \frac{1}{2}(y_t^2 + \mu_t^2) - \ln \left( e^{\nu + y_t\mu_t} + e^{-(\nu + y_t\mu_t)} \right) + \ln(2\cosh \nu) + \frac{1}{2} \ln(2\p) \\
        &=& \frac{1}{2}(y_t^2 + \mu_t^2)
        - \ln \frac{\cosh(\nu + y_t\mu_t)}{\cosh \nu}
        + \frac{1}{2} \ln(2\p).
    \end{eqnarray*}
    For Lemma~\ref{lemma:stein_score}, write \(\nabla\) for \(\nabla_{y_t}\). The Stein score is the negative gradient of \(F\). Since \(\mu_t\), \(\nu\), and \(\frac{1}{2}\ln(2\p)\) are constant in \(y_t\), we apply the chain rule to the log-cosh term:
    \[
    s_{\theta_t, \nu} (y_t, \bx) := \nabla \ln p(y_t \mid \bx; \theta_t, \pi) = -\nabla F(y_t, \mu_t, \nu) = - \left( y_t - \nabla \ln \cosh(\mu_t y_t + \nu) \right) = - y_t + \tanh(\mu_t y_t + \nu) \mu_t.
    \]
\end{proof}

\begin{lemma}[Conditional Relative Entropy Dissipation]\label{lemma:kl_derivative}
    The derivative of the KL divergence between the ground truth and model distributions satisfies
    \[
    \begin{aligned}
        -\frac{\d}{\d t}\KL\left(p(y_t|\bx;\theta_t^\ast,\pi^\ast) \parallel p(y_t|\bx;\theta_t,\pi)\right)
        = \frac{\beta_t}{2}\E_{y_t|\bx;\theta_t^\ast,\pi^\ast}
        \|\nabla \ln p(y_t|\bx;\theta_t,\pi) - \nabla \ln p(y_t|\bx;\theta_t^\ast,\pi^\ast)\|^2.
    \end{aligned}
    \]
\end{lemma}

\begin{lemma}[Log-likelihood Derivative and Score Matching Loss]\label{lemma:log_likelihood_derivative}
    Let \(\Pr_n\) be the empirical measure defined before
    Equation~\eqref{eq:auxiliary_loss_finite}. Then, for every \(t>0\),
    \[
    \frac{\d}{\d t}\E_{\Pr_n}\E_{y_t\mid y_0} \ln p(y_t|\bx;\theta_t,\pi) = \frac{\beta_t}{2}\E_{\Pr_n}\E_{y_t\mid y_0} \left[ \|\nabla \ln p(y_t|\bx;\theta_t,\pi)\|^2 -2 \langle \nabla \ln p(y_t|\bx;\theta_t,\pi), \nabla \ln p(y_t|y_0) \rangle + 1\right].
    \]
\end{lemma}

\begin{proof}[Proof of Lemma~\ref{lemma:kl_derivative}]
    Let \(p_t := p(y_t|\bx;\theta_t,\pi)\) denote the model density, and let
    \(p^*_t := p(y_t|\bx;\theta_t^*,\pi^*)\) denote the ground truth density.
    The negative time derivative of the KL divergence is
    \[
    -\frac{\d}{\d t}\KL(p^*_t \parallel p_t) = -\frac{\d}{\d t} \E_{p^*_t} \ln \frac{p^*_t}{p_t} = \frac{\d}{\d t}\E_{p^*_t}[\ln p_t] - \frac{\d}{\d t}\E_{p^*_t}[\ln p^*_t].
    \]
    The forward Kolmogorov operator for SDE~\eqref{eq:sde} is
    \(\calA^*[f] :=
    -\nabla \cdot\left[f \cdot(-\frac{\beta_t}{2}y_t)\right]
    +\frac{1}{2}\Delta\left[f\cdot (\sqrt{\beta_t})^2\right]
    = \frac{\beta_t}{2} \left[\nabla \cdot (f y_t) + \Delta f\right]\).
    The transition density from Lemma~\ref{lemma:mlr_sde} is
    \(p(y_t\mid y_0)=\calN(y_t;y_0\sqrt{\bar\alpha_t},
    1-\bar\alpha_t)\). Its Fokker--Planck--Kolmogorov equation is
    \[
        \partial_t p(y_t \mid y_0) 
        = \frac{\beta_t}{2} \left[\nabla \cdot (p(y_t\mid y_0) y_t) + \Delta p(y_t\mid y_0)\right]
        \equiv \calA^*[p(y_t \mid y_0)].
    \]
    Integrating this transition density against the fixed initial law, we obtain the Chapman--Kolmogorov relation
    \[
        p_t
        =\E_{y_0\mid\bx;\theta,\pi}[p(y_t\mid y_0)]
        =\sum_{k=1}^2\pi(k)\,
          \calN\!\left(y_t;(-1)^{k+1}\mu_0\sqrt{\bar\alpha_t},1\right).
    \]
    Direct differentiation shows that each of these two Gaussian components
    satisfies the FPK equation above. By linearity of the forward operator,
    the mixture satisfies the same equation:
    \[
        \partial_t p_t=\calA^*[p_t]
        =\frac{\beta_t}{2}\left[\nabla\cdot(p_t y_t)+\Delta p_t\right].
    \]
    Dividing by \(p_t\) and using identities \(\frac{\nabla \cdot (p_t y_t)}{p_t} = 1 + \langle \nabla \ln p_t, y_t \rangle\), \(\frac{\Delta p_t}{p_t} = \Delta \ln p_t + \|\nabla \ln p_t\|^2\), 
    we obtain the derivative of the log-density.
    \[
    \partial_t \ln p_t = \frac{\beta_t}{2} \left[ 1 + \langle \nabla \ln p_t, y_t \rangle + \Delta \ln p_t + \|\nabla \ln p_t\|^2 \right].
    \]
    To account for both time dependence and motion along the diffusion, we apply
    It\^{o}'s formula to \(\ln p_t\) under the forward SDE:
    \[
    \begin{aligned}
    \d \ln p_t &= \partial_t \ln p_t \d t + \langle \nabla \ln p_t, \d y_t \rangle + \frac{1}{2}\Delta \ln p_t (\d y_t)^2 \\
    &= \partial_t \ln p_t \d t + \left\langle \nabla \ln p_t, -\frac{\beta_t}{2}y_t \d t + \sqrt{\beta_t}\d B_t \right\rangle + \frac{\beta_t}{2}\Delta \ln p_t \d t \\
    &= \frac{\beta_t}{2} \left[ \|\nabla \ln p_t\|^2 + 2\Delta \ln p_t + 1 \right]\d t + \sqrt{\beta_t}\langle \nabla \ln p_t, \d B_t \rangle .
    \end{aligned}
    \]
    Under the ground truth path measure \(\mathbb{P}^*\), the state at each \(\tau \in [0,t]\) has marginal density \(p_\tau^\ast\).
    Fubini's theorem moves the expectation inside the time integral, and the Brownian stochastic integral has mean zero.
    \[
    \begin{aligned}
    \E_{p_t^\ast}[\ln p_t] - \E_{p_0^\ast}[\ln p_0]
    &= \E_{\Pr^\ast}\left[ \int_0^t \frac{\beta_\tau}{2}
        \left[ \|\nabla \ln p_\tau\|^2 + 2\Delta \ln p_\tau + 1 \right]\d \tau \right] \\
    &= \int_0^t \frac{\beta_\tau}{2} \E_{p^*_\tau}
        \left[ \|\nabla \ln p_\tau\|^2 + 2\Delta \ln p_\tau + 1 \right] \d \tau.
    \end{aligned}
    \]
    Differentiating this time integral by the fundamental theorem of calculus, we obtain
    \[
    \frac{\d}{\d t}\E_{p_t^\ast}[\ln p_t] = \frac{\beta_t}{2} \E_{p_t^\ast}\left[ \|\nabla \ln p_t\|^2 + 2\Delta \ln p_t + 1 \right]. 
    \]
    We next rewrite the Laplacian term as a product of scores.
    Integration by parts has no boundary contribution:
    \[
    \E_{p^*_t}[\Delta \ln p_t] = \E_{p^*_t}[\nabla \cdot (\nabla \ln p_t)] = -\E_{p^*_t}[\langle \nabla \ln p^*_t, \nabla \ln p_t \rangle].
    \]
    The time derivative can now be expressed entirely through scores:
    \[
    \frac{\d}{\d t}\E_{p_t^\ast}[\ln p_t] = \frac{\beta_t}{2} \E_{p_t^\ast}\left[ \|\nabla \ln p_t\|^2 - 2\langle \nabla \ln p^*_t, \nabla \ln p_t \rangle + 1 \right]. 
    \]
    The same calculation applies to the ground truth distribution:
    \[
    \frac{\d}{\d t}\E_{p^*_t}[\ln p^*_t] = \frac{\beta_t}{2} \E_{p^*_t} \left[ \|\nabla \ln p^*_t\|^2 - 2\|\nabla \ln p^*_t\|^2 + 1 \right] = \frac{\beta_t}{2} \E_{p^*_t} \left[ -\|\nabla \ln p^*_t\|^2 + 1 \right].
    \]
    Subtracting the entropy derivative from the cross-entropy derivative,
    we combine the score terms into a squared difference:
    \[
    -\frac{\d}{\d t}\KL(p^*_t \parallel p_t)
    = \frac{\d}{\d t}\E_{p_t^\ast}[\ln p_t] - \frac{\d}{\d t}\E_{p_t^\ast}[\ln p^*_t]
    = \frac{\beta_t}{2} \E_{p^*_t} \left[ \|\nabla \ln p_t - \nabla \ln p^*_t\|^2 \right].
    \]
    Restoring the full notation \(p_t \equiv p(y_t|\bx;\theta_t,\pi)\) and \(p_t^\ast \equiv p(y_t|\bx;\theta_t^*,\pi^*)\) yields the lemma statement.
\end{proof}

\begin{proof}[Proof of Lemma~\ref{lemma:log_likelihood_derivative}]
    We apply the proof of Lemma~\ref{lemma:kl_derivative} with \(p^\ast_t\)
    replaced by \(p(y_t\mid y_0)\). The resulting identity is
    \[
    \frac{\d}{\d t}\E_{y_t\mid y_0} [\ln p_t] = \frac{\beta_t}{2} \E_{y_t\mid y_0} [\|\nabla \ln p_t\|^2 - 2 \langle \nabla \ln p_t, \nabla \ln p(y_t|y_0) \rangle + 1].
    \]
    Taking the expectation over \(\Pr_n\), interchanging \(\E_{\Pr_n}\) and \(\frac{\d}{\d t}\), and restoring \(p_t \equiv p(y_t|\bx;\theta_t,\pi)\) proves the result.
\end{proof}

\begin{lemma}[Decaying KL Divergence Bound]\label{lemma:decaying_kl_bound}
    Suppose \(0<\bar{\alpha}_t<1\).
    Let the two conditional diffusions start from the initial probability laws
    \(p(y_0|\bx;\theta^*,\pi^*)\) and
    \(p(y_0|\bx;\theta,\pi)\). Their time-\(t\) laws satisfy
    \[
    \KL\left(p(y_t|\bx;\theta_t^*,\pi^*) \parallel p(y_t|\bx;\theta_t,\pi)\right) \le \frac{\bar{\alpha}_t}{2(1 - \bar{\alpha}_t)} W_2^2\left(p(y_0|\bx;\theta^*,\pi^*),\; p(y_0|\bx;\theta,\pi)\right).
    \]
\end{lemma}

\begin{lemma}[Wasserstein Distance for Conditional Distributions]\label{lemma:wasserstein_distance_conditional}
    Fix \(\bx\). Define the two discrete distributions
    \(D:= D(\bx;\theta, \pi) \equiv \pi(1) \delta_{\langle \theta, \bx\rangle} + \pi(2) \delta_{-\langle \theta, \bx\rangle}\) and
    \(D^*:= D(\bx;\theta^*, \pi^*) \equiv \pi^*(1) \delta_{\langle \theta^*, \bx\rangle} + \pi^*(2) \delta_{-\langle \theta^*, \bx\rangle}\).
    Then the squared 2-Wasserstein distance between \(p(y_0 \mid \bx; \theta^*, \pi^*)\) and \(p(y_0 \mid \bx; \theta, \pi)\) is upper bounded by the squared 2-Wasserstein distance between the discrete distributions.
    \[
    W_2^2(p(y_0 \mid \bx; \theta^*, \pi^*), p(y_0 \mid \bx; \theta, \pi)) \leq W_2^2(D, D^*).
    \]
\end{lemma}

\begin{lemma}[Decaying KL Divergence w.r.t. Transition Kernel]\label{lemma:decaying_kl_bound_transition} 
    Suppose \(0<\bar{\alpha}_t<1\).
    Let \(p(y_t \mid y_0)\) be the transition kernel at time \(t\) from a
    fixed initial value \(y_0\). Set \(\mu_0 = \langle \theta, \bx\rangle\)
    and \(\nu = (\ln\pi(1)-\ln\pi(2))/2\). Then
    \[
    \KL\left(p(y_t \mid y_0) \parallel p(y_t\mid \bx; \theta_t, \pi)\right) 
    \leq \frac{\bar{\alpha}_t}{2(1 - \bar{\alpha}_t)} (y_0^2 + \mu_0^2 - 2 y_0 \mu_0 \tanh \nu + 1).
    \]
\end{lemma}

\begin{proof}[Proof of Lemma~\ref{lemma:decaying_kl_bound}]
    For any fixed \(u\in\R\), the transition density from
    Lemma~\ref{lemma:mlr_sde} is
    \[
        q_t(y_t\mid u)
        :=\calN\left(y_t;u\sqrt{\bar\alpha_t},1-\bar\alpha_t\right).
    \]
    Define the initial probability laws
    \(p_0^*:=p(y_0\mid\bx;\theta^*,\pi^*)\) and
    \(p_0:=p(y_0\mid\bx;\theta,\pi)\). We use
    \(\widetilde y_0\) for a value under \(p_0^*\) and \(y_0\) for a value
    under \(p_0\). Since the two transition densities have variance
    \(1-\bar\alpha_t\), their KL divergence depends only on the difference of their means:
    \[
    \KL(q_t(\,\cdot\mid\widetilde y_0)\parallel q_t(\,\cdot\mid y_0))
    =\frac{\bar\alpha_t}{2(1-\bar\alpha_t)}
      \|\widetilde y_0-y_0\|^2.
    \]
    Let \(p_t := p(y_t|\bx;\theta_t,\pi)\) and \(p_t^* := p(y_t|\bx;\theta_t^*,\pi^*)\). Since \(\theta_t = \theta \sqrt{\bar{\alpha}_t}\), we have
    \[
    p_t = \E_{y_0\sim p_0}[q_t(y_t\mid y_0)],\qquad
    p_t^* = \E_{\widetilde y_0\sim p_0^*}
    [q_t(y_t\mid\widetilde y_0)].
    \]
    Let \(\calP(\mathbb{R})\) denote the space of probability measures on \(\mathbb{R}\). For \(p_0^*,p_0\in\calP(\mathbb{R})\), let \(\Gamma(p_0^*,p_0)\) be the set of joint probability measures on \(\mathbb{R}\times\mathbb{R}\) with marginals \(p_0^*\) and \(p_0\).
    Thus \(\gamma \in \Gamma(p_0^*, p_0)\) if and only if every pair of Borel sets \(A,B\subseteq\mathbb{R}\) satisfies
    \(
    \gamma(A \times \mathbb{R}) = p_0^*(A),\, \gamma(\mathbb{R} \times B) = p_0(B)
    \).
    Theorem~4.1 of \citet{villani2009optimal} shows that a lower
    semicontinuous quadratic cost bounded from below attains its infimum.
    Thus an optimal coupling \(\gamma^\star\in\Gamma(p_0^*,p_0)\) satisfies
    \[
    \E_{(\widetilde y_0, y_0) \sim \gamma^\star}[\|\widetilde y_0 - y_0\|^2]
    = \inf_{\gamma \in \Gamma(p_0^*, p_0)} \E_{(\widetilde y_0, y_0) \sim \gamma}[\|\widetilde y_0 - y_0\|^2]
    =: W_2^2(p_0^*, p_0).
    \]
    The optimal coupling \(\gamma^\star\) expresses both marginals \(p_t^*\) and \(p_t\) as expectations on the same probability space.
    \[
    p_t^* = \E_{(\widetilde y_0, y_0) \sim \gamma^\star}
    [q_t(y_t\mid\widetilde y_0)],\qquad
    p_t = \E_{(\widetilde y_0, y_0) \sim \gamma^\star}
    [q_t(y_t\mid y_0)].
    \]
    The integral log-sum inequality is the continuous-density
    form of joint convexity of relative entropy
    \cite[Theorem~2.7.2]{cover2006elements}. It bounds the divergence of these mixtures by the mean
    divergence of their components:
    \[
    \KL(p_t^* \parallel p_t)
    \le \E_{(\widetilde y_0, y_0) \sim \gamma^\star}
    \Big[\KL(q_t(\,\cdot\mid\widetilde y_0)
    \parallel q_t(\,\cdot\mid y_0))\Big].
    \]
    The conditional Gaussian KL formula reduces the right-hand side to the optimal transport cost:
    \[
    \KL(p_t^* \parallel p_t) \le \E_{(\widetilde y_0, y_0) \sim \gamma^\star} \left[ \frac{\bar{\alpha}_t}{2(1 - \bar{\alpha}_t)} \|\widetilde y_0 - y_0\|^2 \right] = \frac{\bar{\alpha}_t}{2(1 - \bar{\alpha}_t)} W_2^2(p_0^*, p_0).
    \]
    Restoring the full notation for the four probability laws yields the lemma
    statement.
\end{proof}

\begin{proof}[Proof of Lemma~\ref{lemma:wasserstein_distance_conditional}]
    Fix \(\bx\) and both parameter pairs. Write
    \(p_0:=p(y_0|\bx;\theta,\pi)\) and
    \(p_0^*:=p(y_0|\bx;\theta^*,\pi^*)\), and let \(D,D^*\) be the
    discrete laws in the lemma statement.
    Set \(\bar{y}_0 = (-1)^{z+1} \langle \theta, \bx\rangle\) in Equation~\eqref{eq:yt}. By Lemma~\ref{lemma:mlr_sde}, \(y_0 = \bar{y}_0 + \varepsilon_0\), with \(\bar{y}_0 \ind \varepsilon_0\) and \(\varepsilon_0\sim \eta := \calN(0, 1)\), and hence
    \[
    \bar{y}_0 \sim D, \quad y_0 = \bar{y}_0 + \varepsilon_0 \sim p_0 = D * \eta.
    \]
    For the ground truth law, let
    \(y_0^* = \bar{y}_0^* + \varepsilon_0^*\), where
    \(\bar{y}_0^*\sim D^*\), \(\varepsilon_0^*\sim\eta\), and the two terms
    are independent. Thus
    \[
    y_0^* = \bar{y}_0^* + \varepsilon_0^* \sim p_0^* = D^* * \eta.
    \]
    Since \(\eta\) is an even regularizing kernel, Lemma~5.2 of
    \citet{santambrogio2015optimal} bounds the distance between the convolutions:
    \[
    W_2^2(p_0^*, p_0) = W_2^2(D^* * \eta, D * \eta) \leq W_2^2(D^*, D) = W_2^2(D, D^*).
    \]
    This establishes the desired upper bound and completes the proof.
\end{proof}

\begin{proof}[Proof of Lemma~\ref{lemma:decaying_kl_bound_transition}]
    Let \(u\) denote a dummy initial value. We write
    \(p_t:=p(y_t\mid\bx;\theta_t,\pi)
    =\E_{u\mid\bx;\theta,\pi}[p(y_t\mid u)]\) and
    \(p(y_t\mid y_0)=\E_{u\sim\delta_{y_0}}[p(y_t\mid u)]\).
    Lemma~\ref{lemma:decaying_kl_bound} applies with the two initial laws
    \(p(u\mid\bx;\theta,\pi)\) and \(\delta_{y_0}\). Hence
    \[
    \KL\left(p(y_t \mid y_0) \parallel p(y_t\mid \bx; \theta_t, \pi)\right)
    \leq \frac{\bar{\alpha}_t}{2(1 - \bar{\alpha}_t)} W_2^2\left(\delta_{y_0},\, p(u \mid \bx; \theta, \pi)\right).
    \]
    The squared 2-Wasserstein distance between \(\delta_{y_0}\) and the
    two-Gaussian mixture is
    \[
    \begin{aligned}
    &W_2^2\left(\delta_{y_0},\, p(u \mid \bx; \theta, \pi)\right)
    = \E_{u\mid\bx;\theta,\pi}[(y_0-u)^2] \\
    =& \pi(1)\bigl[(y_0-\mu_0)^2+1\bigr]
       +\pi(2)\bigl[(y_0+\mu_0)^2+1\bigr]
    = y_0^2 + \mu_0^2 - 2 y_0 \mu_0 \tanh \nu + 1.
    \end{aligned}
    \]
    The first equality holds because a point mass has one coupling with the
    second marginal. The second equality uses
    \(u=(-1)^{z+1}\mu_0+\varepsilon_0\) and
    \(\E[\varepsilon_0^2]=1\).
\end{proof}

\begin{lemma}[Two Losses of Stein Score Function]
    Set \(\theta_t=\theta\sqrt{\bar\alpha_t}\) and
    \(\nu=(\ln\pi(1)-\ln\pi(2))/2\). Define the candidate score by
    \[
        s_{\theta_t,\nu}(y_t,\bx)
        :=\nabla_{y_t}\ln p(y_t\mid\bx;\theta_t,\pi).
    \]
    Set \(\theta_t^\ast=\theta^\ast\sqrt{\bar\alpha_t}\) and
    \(\nu^\ast=(\ln\pi^\ast(1)-\ln\pi^\ast(2))/2\). Define the ground truth score by
    \[
        s_{\theta_t^\ast,\nu^\ast}(y_t,\bx)
        :=\nabla_{y_t}\ln p(y_t\mid\bx;\theta_t^\ast,\pi^\ast).
    \]
    Then
    \[
    \begin{aligned}
    &\E_{y_t \mid \bx; \theta_t^\ast, \pi^\ast}
      \|s_{\theta_t, \nu}(y_t, \bx)
      -s_{\theta_t^\ast, \nu^\ast}(y_t, \bx)\|^2 \\
    &\quad=\E_{y_0\mid\bx;\theta^\ast,\pi^\ast}
      \E_{y_t\mid y_0}\bigl[\|s_{\theta_t, \nu}(y_t, \bx)\|^2\bigr] \\
    &\qquad\quad-2\E_{y_0\mid\bx;\theta^\ast,\pi^\ast}
      \E_{y_t\mid y_0}\bigl[\langle s_{\theta_t, \nu}(y_t, \bx),
      \nabla \ln p(y_t \mid y_0) \rangle\bigr]
      + \mathsf{C}(\mu_t^\ast, \nu^\ast).
    \end{aligned}
    \]
    Here \(\mathsf{C}(\mu_t^\ast, \nu^\ast) := 1-[\mu_t^\ast]^2 \E_{y_t \mid \bx; \theta_t^\ast, \pi^\ast}[\sech^2(\mu_t^\ast y_t + \nu^\ast)] \geq 0\) depends on \(\mu_t^\ast = \langle \theta_t^\ast, \bx \rangle\) and \(\nu^\ast\) alone.
\end{lemma}

\begin{proof}
    To express the marginal score through the transition score, we
    differentiate the Gaussian-mixture marginal and apply Bayes' rule.
    With \(\nabla:=\nabla_{y_t}\), the denoising score identity reads
    \[
        s_{\theta_t^\ast, \nu^\ast}(y_t, \bx) \equiv \nabla \ln p(y_t \mid \bx; \theta_t^\ast, \pi^\ast)
        = \E\!\left[\nabla\ln p(y_t\mid y_0)
          \mid y_t,\bx;\theta^\ast,\pi^\ast\right].
    \]
    Applying the tower property to the cross-term, we obtain
    \[
        \E_{y_t \mid \bx; \theta_t^\ast, \pi^\ast}
        [\langle s_{\theta_t, \nu}, s_{\theta_t^\ast, \nu^\ast}\rangle]
        = \E_{y_0\mid\bx;\theta^\ast,\pi^\ast}
          \E_{y_t\mid y_0}
          [\langle s_{\theta_t, \nu}, \nabla \ln p(y_t \mid y_0)\rangle].
    \]
    After expanding the squared score difference, this cross-term
    identity leaves only the expected squared ground truth score:
    \begin{align*}
        \mathsf{C}
        :=&\ \E_{y_t \mid \bx; \theta_t^\ast, \pi^\ast}
          [\|s_{\theta_t, \nu}-s_{\theta_t^\ast, \nu^\ast}\|^2]
          -\E_{y_0\mid\bx;\theta^\ast,\pi^\ast}\E_{y_t\mid y_0}
          [\|s_{\theta_t, \nu}\|^2]\\
        &+2\E_{y_0\mid\bx;\theta^\ast,\pi^\ast}\E_{y_t\mid y_0}
          [\langle s_{\theta_t, \nu},\nabla\ln p(y_t\mid y_0)\rangle]\\
        =&\ \E_{y_t \mid \bx; \theta_t^\ast, \pi^\ast}
          [\|s_{\theta_t^\ast, \nu^\ast}(y_t, \bx)\|^2],
    \end{align*}
    where suppressed score arguments in the first two lines are
    \((y_t,\bx)\).
    To evaluate this remaining expectation, we integrate by parts.
    The boundary term vanishes, and Lemma~\ref{lemma:stein_score} identifies the score derivative:
    \[
        \E_{y_t \mid \bx; \theta_t^\ast, \pi^\ast}[\|s_{\theta_t^\ast, \nu^\ast}(y_t, \bx)\|^2]
        = -\E_{y_t \mid \bx; \theta_t^\ast, \pi^\ast}[\nabla s_{\theta_t^\ast, \nu^\ast}(y_t, \bx)]
        = 1 - [\mu_t^\ast]^2 \E_{y_t \mid \bx; \theta_t^\ast, \pi^\ast}[\sech^2(\mu_t^\ast y_t + \nu^\ast)].
    \]
    Lemma~\ref{lemma:nll} shows that the conditional density is determined by \(y_t\), \(\mu_t^\ast\), and \(\nu^\ast\). After integrating over \(y_t\), \(\mathsf{C}\) therefore depends on \(\mu_t^\ast\) and \(\nu^\ast\) alone.
\end{proof}

\begin{lemma}[Asymptotic Behavior of the Scaled Score Matching Loss]\label{lemma:asymptotic_score_loss}
    Let \(y_0\) be generated by the MLR model in Eq.~\eqref{eq:2mlr} under the ground truth parameters \(\theta^\ast,\pi^\ast\).
    Let \(y_t\) follow SDE~\eqref{eq:sde} at time \(t\geq0\), with cumulative noise schedule \(\bar{\alpha}_t\in(0,1]\).
    Let \(\mu_0 := \langle \theta, \bx \rangle\) and \(\mu_0^\ast := \langle \theta^\ast, \bx \rangle\) denote the estimated and ground truth means at \(t=0\), 
    such that the means at time \(t\) are \(\mu_t = \mu_0\sqrt{\bar{\alpha}_t}\) and \(\mu_t^\ast = \mu_0^\ast\sqrt{\bar{\alpha}_t}\).
    Let \(\nu := (\ln\pi(1)-\ln\pi(2))/2\) and \(\nu^\ast := (\ln\pi^\ast(1)-\ln\pi^\ast(2))/2\) denote the estimated and ground truth imbalance parameters of the mixing weights.
    Then
    \[
    \begin{aligned} 
        &\lim_{\bar{\alpha}_t \to 0} \frac{1}{\bar{\alpha}_t} \E_{y_t \mid \bx; \theta_t^\ast, \pi^\ast} [\| s_{\theta_t, \nu}(y_t, \bx) - s_{\theta_t^\ast, \nu^\ast}(y_t, \bx) \|^2]
        = \left( \mu_0 \tanh \nu - \mu_0^\ast \tanh \nu^\ast \right)^2, \\
        &\lim_{\bar{\alpha}_t \to 1} \frac{1}{\bar{\alpha}_t} \E_{y_t \mid \bx; \theta_t^\ast, \pi^\ast} [\| s_{\theta_t, \nu}(y_t, \bx) - s_{\theta_t^\ast, \nu^\ast}(y_t, \bx) \|^2]
        = \E_{y_0 \mid \bx; \theta^\ast, \pi^\ast} \left[ \left( \mu_0 \tanh(\mu_0 y_0 + \nu) - \mu_0^\ast \tanh(\mu_0^\ast y_0 + \nu^\ast) \right)^2 \right].
    \end{aligned}
    \]
\end{lemma}

\begin{proof}[Proof of Lemma~\ref{lemma:asymptotic_score_loss}]
    The Stein score formula \(s_{\theta_t, \nu}(y_t, \bx) = -y_t + \mu_t \tanh(\mu_t y_t + \nu)\) established in Lemma~\ref{lemma:stein_score} shows that the linear \(-y_t\) terms cancel in the difference between the estimated and ground truth score functions.
    We place the marginals at all scales of the diffusion noise level on a common probability
    space. Let \(\bar{y}_0^\ast := (-1)^{z+1}\mu_0^\ast\), and let
    \(\varepsilon\sim\calN(0,1)\) be independent of \(z\). At each scale,
    \(\sqrt{\bar\alpha_t}\,\bar y_0^\ast+\varepsilon\) has the same law as
    \(y_t\mid\bx;\theta_t^\ast,\pi^\ast\). Substituting this representation
    and \(\mu_t=\mu_0\sqrt{\bar\alpha_t}\), we express the normalized squared score difference as
    \[
        \frac{1}{\bar{\alpha}_t} \| s_{\theta_t, \nu}(y_t, \bx) - s_{\theta_t^\ast, \nu^\ast}(y_t, \bx) \|^2
        =\underbrace{\left( \mu_0 \tanh\big(\bar{\alpha}_t \mu_0 \bar{y}_0^\ast + \sqrt{\bar{\alpha}_t} \mu_0 \varepsilon + \nu\big) - \mu_0^\ast \tanh\big(\bar{\alpha}_t \mu_0^\ast \bar{y}_0^\ast + \sqrt{\bar{\alpha}_t} \mu_0^\ast \varepsilon + \nu^\ast\big) \right)^2}_{=:\ell_{\mu_0, \mu_0^\ast, \nu, \nu^\ast}(\bar{\alpha}_t, \bar{y}_0^\ast, \varepsilon)}.
    \]   
    Since \(|\tanh(u)|\le1\), the normalized squared score difference
    is bounded for each fixed \(\bx\):
    \[
        |\ell_{\mu_0, \mu_0^\ast, \nu, \nu^\ast}(\bar{\alpha}_t, \bar{y}_0^\ast, \varepsilon)| \le (|\mu_0| + |\mu_0^\ast|)^2 < \infty
        \implies \E_{y_t \mid \bx; \theta_t^\ast, \pi^\ast}\left[\frac{1}{\bar{\alpha}_t} \| s_{\theta_t, \nu}(y_t, \bx) - s_{\theta_t^\ast, \nu^\ast}(y_t, \bx) \|^2\right] < \infty.
    \]
    This integrable bound permits the limit to pass through the expectation by
    the dominated convergence theorem
    \citep[Theorem~1.5.8]{durrett2019probability}.
    The expectation over the common probability space is the iterated
    expectation over \(\bar y_0^\ast\) and \(\varepsilon\). For each
    \(c\in\{0,1\}\), we can therefore write
    \[
    \lim_{\bar{\alpha}_t \to c}\E_{y_t \mid \bx; \theta_t^\ast, \pi^\ast}\left[\frac{1}{\bar{\alpha}_t} \| s_{\theta_t, \nu}(y_t, \bx) - s_{\theta_t^\ast, \nu^\ast}(y_t, \bx) \|^2\right] 
    = \E_{\bar{y}_0^\ast \mid \mu_0^\ast, \nu^\ast}
      \E_{\varepsilon}\left[\lim_{\bar{\alpha}_t \to c}
      \ell_{\mu_0, \mu_0^\ast, \nu, \nu^\ast}
      (\bar{\alpha}_t, \bar{y}_0^\ast, \varepsilon)\right].
    \]
    The pointwise limit for \(c=0\) is
    \[
        \lim_{\bar{\alpha}_t \to 0} \ell_{\mu_0, \mu_0^\ast, \nu, \nu^\ast}(\bar{\alpha}_t, \bar{y}_0^\ast, \varepsilon) = \left( \mu_0 \tanh \nu - \mu_0^\ast \tanh \nu^\ast \right)^2.
    \]
    The pointwise limit for \(c=1\) is
    \[
        \lim_{\bar{\alpha}_t \to 1} \ell_{\mu_0, \mu_0^\ast, \nu, \nu^\ast}(\bar{\alpha}_t, \bar{y}_0^\ast, \varepsilon) = \left( \mu_0 \tanh(\mu_0 (\bar{y}_0^\ast + \varepsilon) + \nu) - \mu_0^\ast \tanh(\mu_0^\ast (\bar{y}_0^\ast + \varepsilon) + \nu^\ast) \right)^2.
    \]
    Since \(\bar{y}_0^\ast+\varepsilon\) has the ground truth law of
    \(y_0\mid\bx;\theta^\ast,\pi^\ast\), the last display proves the second
    equality.
\end{proof}

\begin{lemma}[Wasserstein-2 Distance between Two Discrete Distributions]\label{lemma:wasserstein_distance_discrete}
    Fix \(\bx\). Define the discrete distributions
    \(D \equiv D(\bx;\theta, \pi) := \pi(1) \delta_{\langle \theta, \bx\rangle} + \pi(2) \delta_{-\langle \theta, \bx\rangle}\) and
    \(D^* \equiv D(\bx;\theta^*, \pi^*) := \pi^*(1) \delta_{\langle \theta^*, \bx\rangle} + \pi^*(2) \delta_{-\langle \theta^*, \bx\rangle}\).
    Then for \(\mu_0 \equiv\langle \theta, \bx \rangle, \mu_0^\ast \equiv\langle \theta^*, \bx \rangle\) and
    \(\nu \equiv (\ln\pi(1)-\ln\pi(2))/2, \nu^\ast \equiv (\ln\pi^*(1)-\ln\pi^*(2))/2\), we have
    \[
    W_2^2(D, D^*) = (|\mu_0| - |\mu_0^\ast|)^2 + 2 |\mu_0| |\mu_0^\ast| \left| \sgn(\mu_0) \tanh \nu - \sgn(\mu_0^\ast) \tanh \nu^\ast \right|.
    \]
\end{lemma}

\begin{proof}[Proof of Lemma~\ref{lemma:wasserstein_distance_discrete}]
    We adopt \(\sgn(0)=0\) and set
    \[
        a:=|\mu_0|,\qquad a^\ast:=|\mu_0^\ast|,\qquad
        q:=\sgn(\mu_0)\tanh\nu,\qquad
        q^\ast:=\sgn(\mu_0^\ast)\tanh\nu^\ast.
    \]
    If \(aa^\ast=0\), the claimed identity follows. Otherwise, \(D\) is the law of \(aS\), where \(S\in\{-1,1\}\), \(\Pr(S=1)=(1+q)/2\), and \(\tanh\nu=\pi(1)-\pi(2)\). Likewise, \(D^*\) is the law of \(a^\ast S^\ast\) with \(\Pr(S^\ast=1)=(1+q^\ast)/2\).
    For any coupling of \(S\) and \(S^\ast\),
    \[
        \E[(aS-a^\ast S^\ast)^2]
        =a^2+(a^\ast)^2-2aa^\ast\E[SS^\ast].
    \]
    The maximal coupling has \(\Pr(S\neq S^\ast)=|q-q^\ast|/2\), and hence \(\E[SS^\ast]=1-|q-q^\ast|\). It therefore minimizes the quadratic cost, so
    \[
        W_2^2(D,D^*)=(a-a^\ast)^2+2aa^\ast|q-q^\ast|,
    \]
    which is the stated formula.
\end{proof}

\begin{lemma}[Expectations of Gaussians, Lemma E.2 of~\citet{luo2025structural}]\label{lemma:expectations_gaussian}
    We recall the following identities from Lemma~E.2 of
    \citet{luo2025structural}. Let \(g\) and \(g'\) be standard Gaussian
    variables with \(\E[gg']=\sin\varphi\), where
    \(\varphi\in[-\p/2,\p/2]\). Then
    \[
    \E\left[\sgn\left(g g^{\prime}\right)\right]=\frac{2}{\p} \varphi, \quad \E\left[\left|g g^{\prime}\right|\right]=\frac{2}{\p}[\varphi \sin \varphi+\cos \varphi], \quad \E\left[g^2 \sgn\left(g g^{\prime}\right)\right]=\frac{2}{\p}[\varphi+\sin \varphi \cos \varphi].
    \]
\end{lemma}

\begin{lemma}[Expectation of Wasserstein-2 Distance between Two Discrete Distributions]\label{lemma:expectations_wasserstein_distance_discrete}
    Suppose that \(\bx\) is independent of \((z,\varepsilon_0)\) and
    \(\E[\bx \bx^\top] = I_d\). Then \(D,D^*\) from
    Lemma~\ref{lemma:wasserstein_distance_discrete} satisfy
    \[
    \begin{aligned}
    \E_{\bx} W_2^2(D, D^*) &= \|\theta\|^2 + \|\theta^*\|^2 -2 \langle \theta, \theta^* \rangle \sgn(\nu \nu^*) \tanh(|\nu| \wedge |\nu^*|)
    - 2\E_{\bx} [|\mu_0 \mu_0^*|](1 - \tanh(|\nu| \vee |\nu^*|)) \\
    &\leq \|\theta\|^2 + \|\theta^*\|^2 -2 \langle \theta, \theta^* \rangle \sgn(\nu \nu^*) \tanh(|\nu| \wedge |\nu^*|)
    - 2|\langle \theta, \theta^* \rangle|(1 - \tanh(|\nu| \vee |\nu^*|)).
    \end{aligned}
    \]
    In particular, suppose that \(\bx \sim \calN(\mathbf{0}, I_d)\) and \(\theta,\theta^*\neq\mathbf{0}\). Define \(\rho := \frac{\langle \theta, \theta^* \rangle}{\|\theta\| \|\theta^*\|}\) and \(\varphi := \arcsin |\rho|\), so \(\sin \varphi = |\rho|\). Then
    \[
    \begin{aligned}
    &\E_{\bx \sim \calN(\mathbf{0}, I_d)} W_2^2(D, D^*) \\
    =& \|\theta\|^2 + \|\theta^*\|^2 - 2 \|\theta\| \|\theta^*\|
    \left[\sgn(\rho \nu \nu^\ast) \tanh(|\nu| \wedge |\nu^*|) \sin \varphi + (1 - \tanh(|\nu| \vee |\nu^*|))\frac{2}{\p} [\varphi \sin \varphi+\cos \varphi] \right].
    \end{aligned}
    \]
\end{lemma}

\begin{proof}[Proof of Lemma~\ref{lemma:expectations_wasserstein_distance_discrete}]
    Write \(A:=\mu_0=\langle\theta,\bx\rangle\), \(B:=\mu_0^\ast=\langle\theta^\ast,\bx\rangle\), and
    \[
        m:=\tanh(|\nu|\wedge|\nu^\ast|),
        \qquad
        M:=\tanh(|\nu|\vee|\nu^\ast|).
    \]
    Since \(|\tanh\nu|=\tanh|\nu|\), the identity
    \(|a-b|=\max\{|a|,|b|\}-\sgn(ab)\min\{|a|,|b|\}\) separates the magnitude and sign terms:
    \[
        |AB|\left|\sgn(A)\tanh\nu-\sgn(B)\tanh\nu^\ast\right|
        =|AB|M-AB\,\sgn(\nu\nu^\ast)m.
    \]
    Lemma~\ref{lemma:wasserstein_distance_discrete} therefore yields the pointwise identity
    \[
        W_2^2(D,D^*)
        =A^2+B^2-2AB\,\sgn(\nu\nu^\ast)m-2|AB|(1-M).
    \]
    Taking expectations and using \(\E[\bx\bx^\top]=I_d\) proves the equality.
    Moreover, \(\E|AB|\geq|\E[AB]|=|\langle\theta,\theta^\ast\rangle|\).
    Since \(1-M\geq0\), this proves the stated inequality.

    For the Gaussian specialization, the normalized projections \(G:=A/\|\theta\|\) and \(G^\ast:=B/\|\theta^\ast\|\) are standard Gaussian with correlation \(\rho\). Applying Lemma~\ref{lemma:expectations_gaussian} with the signed angle \(\arcsin\rho\), we obtain the two moments needed above:
    \[
        \E|AB|
        =\frac{2\|\theta\|\|\theta^\ast\|}{\p}
        (\varphi\sin\varphi+\cos\varphi),
        \qquad
        \E[AB]=\sgn(\rho)\|\theta\|\|\theta^\ast\|\sin\varphi.
    \]
    Taking expectations in the pointwise identity and substituting these moments proves the Gaussian formula. If either parameter vector is zero, the general identity reduces to
    \(\E W_2^2(D,D^*)=\|\theta\|^2+\|\theta^\ast\|^2\).
\end{proof}

\begin{proposition}[Proposition~\ref{prop:terminal_kl_pinsker_zero_set}]
    Suppose that \(0<\bar{\alpha}_T\leq 1\), \(\theta^\ast\neq\mathbf{0}\),
    \(\nu^\ast\in\R\), and that \(\bx\) has a density that is
    positive almost everywhere on \(\R^d\) and finite second moments.
    Define the terminal population KL divergence
    \[
        \KL_T(\theta,\nu)
        := \KL\left(
        p(y_T,\bx\mid\theta_T^\ast,\pi^\ast)
        \parallel
        p(y_T,\bx\mid\theta_T,\pi)
        \right).
    \]
    Then
    \[
        \argmin_{(\theta,\nu)\in\R^d\times\R}\KL_T(\theta,\nu)
        =\{(\theta^\ast,\nu^\ast),(-\theta^\ast,-\nu^\ast)\}.
    \]
\end{proposition}

\begin{proof}[Proof of Proposition~\ref{prop:terminal_kl_pinsker_zero_set}]
    Because the marginal distribution of \(\bx\) is fixed across the model parameters, the chain rule for KL divergence yields
    \[
        \KL_T(\theta,\nu)
        = \E_{\bx}\KL\left(
            p(y_T\mid\bx;\theta_T^\ast,\pi^\ast)
            \parallel p(y_T\mid\bx;\theta_T,\pi)
        \right).
    \]
    Applying Pinsker's inequality~\citep{cover2006elements} for each \(\bx\)
    and then taking the expectation over \(\bx\), we obtain a lower bound for the terminal divergence:
    \[
        \KL_T(\theta,\nu)
        \geq 2\E_{\bx}\TV^2\left(
            p(\,\cdot\mid\bx;\theta_T^\ast,\pi^\ast),
            p(\,\cdot\mid\bx;\theta_T,\pi)
        \right).
    \]

    We next identify the equality case. Suppose that \(\KL_T(\theta,\nu)=0\). Since the conditional KL divergence is non-negative, it follows that
    \[
        p(\,\cdot\mid\bx;\theta_T,\pi)
        =p(\,\cdot\mid\bx;\theta_T^\ast,\pi^\ast)
        \quad\text{as distributions for almost every }\bx.
    \]
    By Lemma~\ref{lemma:mlr_sde}, the first two raw moments of the terminal conditional distribution are
    \[
        \E_{y_T\mid\bx;\theta_T,\pi}[y_T]=\mu_T\tanh\nu,
        \qquad
        \E_{y_T\mid\bx;\theta_T,\pi}[y_T^2]=1+\mu_T^2.
    \]
    Equality of the terminal conditional distributions and \(\bar{\alpha}_T>0\) therefore imply
    \[
        \mu_0\tanh\nu=\mu_0^\ast\tanh\nu^\ast,
        \qquad
        \mu_0^2=(\mu_0^\ast)^2
        \quad\text{for almost every }\bx.
    \]
    The second identity is equivalent to
    \[
        \bx^\top\left(\theta\theta^\top-\theta^\ast(\theta^\ast)^\top\right)\bx=0
        \quad\text{almost surely}.
    \]
    Since the density of \(\bx\) is positive almost everywhere on \(\R^d\), the quadratic polynomial on the left-hand side is the zero polynomial. Hence
    \[
        \theta\theta^\top=\theta^\ast(\theta^\ast)^\top.
    \]
    Because \(\theta^\ast\neq\mathbf{0}\), equality of these rank-one matrices implies \(\theta=s\theta^\ast\) for some \(s\in\{+1,-1\}\).
    The first-moment identity now reduces to \(s\tanh\nu=\tanh\nu^\ast\).
    Since the hyperbolic tangent is odd and strictly monotone, \(\nu=s\nu^\ast\).
    The joint sign change \((\theta,\nu)\mapsto(-\theta,-\nu)\) exchanges the two mixture components and leaves the terminal conditional distribution unchanged.
    Thus
    \[
        \KL_T(\theta,\nu)=0
        \quad\Longleftrightarrow\quad
        (\theta,\nu)\in
        \{(\theta^\ast,\nu^\ast),(-\theta^\ast,-\nu^\ast)\}.
    \]
    Lemma~\ref{lemma:wasserstein_distance_discrete} and the same positive-density argument show that
    \(\E_{\bx}W_2^2(D,D^*)=0\) has the same two solutions. Combining these equivalences yields the claimed characterization of the global minimizers.
\end{proof}

\begin{lemma}[Integrated Relative de Bruijn Identity for Conditional MLR]\label{lemma:kl_divergence_SM_loss}
    Let the population score matching loss be \(\calL_t(\theta, \nu) := \frac{1}{2\bar{\alpha}_t} \E_{y_t, \bx; \theta_t^\ast, \pi^\ast} \left[ \| s_{\theta_t, \nu}(y_t, \bx) - s_{\theta_t^\ast, \nu^\ast}(y_t, \bx) \|^2 \right]\)
    as defined in Equation~\eqref{eq:loss}. Suppose that \(\bx\) is
    independent of \((z,\varepsilon_0)\) and that
    \(\bar{\alpha}_t \in (0, 1]\). Then every \(T\geq0\) satisfies
    \[
    \KL(p(y_0, \bx \mid \theta^\ast, \pi^\ast) \parallel p(y_0, \bx \mid \theta, \pi))
    = \KL(p(y_T, \bx \mid \theta_T^\ast, \pi^\ast) \parallel p(y_T, \bx \mid \theta_T, \pi)) + \int_{\bar{\alpha}_T}^{\bar{\alpha}_0} \calL_t(\theta, \nu) \d \bar{\alpha}_t.
    \]
\end{lemma}

\begin{lemma}[Log-likelihood and Finite Sample Score Matching Loss]\label{lemma:log_likelihood_SM_loss_finite}
    Let \(\Pr_n\) be the empirical measure defined before
    Equation~\eqref{eq:auxiliary_loss_finite}, and let \(\calJ_t^n\) be the empirical
    auxiliary loss in Equation~\eqref{eq:auxiliary_loss_finite}. Then every
    \(T>0\) with \(0<\bar\alpha_T<1\) satisfies
    \[
    -\E_{\Pr_n}\ln p(y_0 | \bx; \theta, \pi) = \E_{\Pr_n}\KL \left(p(y_T\mid y_0) \parallel p(y_T | \bx; \theta_T, \pi)\right) + \int_{\bar{\alpha}_T}^{\bar{\alpha}_0} \calJ_t^n(\theta, \nu) \d \bar{\alpha}_t + \frac{1}{2}\left(\ln\frac{1-\bar{\alpha}_T}{\bar{\alpha}_T} + \ln (2\p) + 1\right).
    \]
\end{lemma}

\begin{proof}
    By independence, \(p(\bx \mid z; \theta, \pi) = p(\bx \mid z; \theta_t, \pi) = p(\bx)\). Therefore \(p(y_0,\bx\mid \theta, \pi) = p(\bx) p(y_0\mid \bx; \theta, \pi)\) and \(p(y_t,\bx\mid \theta_t, \pi) = p(\bx) p(y_t\mid \bx; \theta_t, \pi)\).

    Under the ground truth, \(p(y_0, \bx \mid \theta^\ast, \pi^\ast) = p(\bx) p(y_0 \mid \bx; \theta^\ast, \pi^\ast)\) and \(p(y_t, \bx \mid \theta_t^\ast, \pi^\ast) = p(\bx) p(y_t \mid \bx; \theta_t^\ast, \pi^\ast)\). Hence \(\E_{y_0, \bx\mid \theta^\ast, \pi^\ast} [\cdot] = \E_{\bx}\E_{y_0| \bx; \theta^\ast, \pi^\ast} [\cdot]\) and \(\E_{y_t, \bx\mid \theta_t^\ast, \pi^\ast} [\cdot] = \E_{\bx}\E_{y_t| \bx; \theta_t^\ast, \pi^\ast} [\cdot]\).

    We interchange \(\frac{\d}{\d t}\) and \(\E_\bx\) and apply Lemma~\ref{lemma:kl_derivative}.
    The relation \(\frac{\d \bar{\alpha}_t}{\d t} = -\beta_t \bar{\alpha}_t\) then converts the time derivative to a derivative with respect to the retained-signal factor:
    \[
    \begin{aligned}
    &-\frac{\d}{\d t}\KL(p(y_t, \bx \mid \theta_t^\ast, \pi^\ast) \parallel p(y_t, \bx \mid \theta_t, \pi))
    = -\E_\bx \frac{\d}{\d t}\KL(p(y_t\mid \bx; \theta_t^\ast, \pi^\ast) \parallel p(y_t\mid \bx; \theta_t, \pi)) \\
    =& \frac{\beta_t}{2} \E_\bx \E_{y_t| \bx; \theta_t^\ast, \pi^\ast} \left[ \| s_{\theta_t, \nu}(y_t, \bx) - s_{\theta_t^\ast, \nu^\ast}(y_t, \bx) \|^2 \right]
    = \beta_t \bar{\alpha}_t \cdot \calL_t(\theta, \nu)
    = -\calL_t(\theta, \nu)\; \frac{\d \bar{\alpha}_t}{\d t}.
    \end{aligned}
    \]
    We integrate from \(0\) to \(T\), with \(\theta_0=\theta\) and \(\theta_0^\ast=\theta^\ast\), to compare the initial and terminal divergences:
    \[
    \KL(p(y_0, \bx \mid \theta^\ast, \pi^\ast) \parallel p(y_0, \bx \mid \theta, \pi))
    - \KL(p(y_T, \bx \mid \theta_T^\ast, \pi^\ast) \parallel p(y_T, \bx \mid \theta_T, \pi)) = \int_{\bar{\alpha}_T}^{\bar{\alpha}_0} \calL_t(\theta, \nu) \d \bar{\alpha}_t.
    \]
    Rearranging the terms proves the result.
\end{proof}

\begin{proof}[Proof of Lemma~\ref{lemma:log_likelihood_SM_loss_finite}]
    Fix \(\epsilon\in(0,T)\). We apply
    Lemma~\ref{lemma:log_likelihood_derivative} to \(\calJ_t^n\), use
    \(\frac{\d \bar{\alpha}_t}{\d t}=-\beta_t\bar{\alpha}_t\), and
    integrate from \(\epsilon\) to \(T\). Under the common Gaussian coupling,
    \(y_\epsilon\to y_0\) in \(L^2\) and
    \(\theta_\epsilon\to\theta\) as \(\epsilon\downarrow0\). The log-density
    is continuous and has a quadratic envelope for each fixed empirical
    sample, so dominated convergence applies. We therefore let
    \(\epsilon\downarrow0\), interpret the integral at its upper endpoint as
    an improper integral, and obtain
    \[
    -\E_{\Pr_n}\ln p(y_0 \mid \bx; \theta, \pi) + \E_{\Pr_n}\E_{y_T\mid y_0} \ln p(y_T | \bx; \theta_T, \pi) = \int_{\bar{\alpha}_T}^{\bar{\alpha}_0} \calJ_t^n(\theta, \nu) \d \bar{\alpha}_t + \frac{1}{2}\ln\frac{1}{\bar{\alpha}_T}.
    \]
    The Gaussian transition satisfies
    \[
    \begin{aligned}
        p(y_T\mid y_0)
        &=\calN(y_T;y_0\sqrt{\bar\alpha_T},1-\bar\alpha_T),\\
        -\ln p(y_T\mid y_0)
        &=\frac12\ln\!\bigl(2\p(1-\bar\alpha_T)\bigr)
        +\frac{\|y_T-y_0\sqrt{\bar\alpha_T}\|^2}{2(1-\bar\alpha_T)},\\
        \E_{y_T\mid y_0}\!\left[\|y_T-y_0\sqrt{\bar\alpha_T}\|^2\right]
        &=1-\bar\alpha_T.
    \end{aligned}
    \]
    We can now separate the terminal cross-entropy into the KL divergence and the Gaussian transition entropy:
    \[
    \begin{aligned}
        -\E_{y_T\mid y_0} \ln p(y_T | \bx; \theta_T, \pi)
        &= \KL(p(y_T\mid y_0) \parallel p(y_T | \bx; \theta_T, \pi)) - \E_{y_T\mid y_0}[\ln p(y_T\mid y_0)] \\
        &= \KL(p(y_T\mid y_0) \parallel p(y_T | \bx; \theta_T, \pi)) +\frac{1}{2}\left(\ln(1-\bar{\alpha}_T) + \ln(2\p) + 1\right).
    \end{aligned}
    \]
    Taking \(\E_{\Pr_n}\), substituting into the first identity, and rearranging the terms proves the result.
\end{proof}

\begin{lemma}[Derivatives of Negative Log-Likelihood]\label{lemma:derivatives_negative_log_likelihood}
    Let \(F(y_t, \mu_t, \nu):=-\ln p(y_t \mid \bx; \theta_t, \pi)\) be the negative log-likelihood from Lemma~\ref{lemma:nll}. Then
    \[
    y_t \partial_{y_t} F(y_t, \mu_t, \nu) - y_t^2 = \mu_t \partial_{\mu_t} F(y_t, \mu_t, \nu) - \mu_t^2 = (\nabla_\nu F(y_t, \mu_t, \nu) - \tanh \nu)\mu_t y_t = - \tanh(\mu_t y_t + \nu) \mu_t y_t.
    \]
    Moreover, \(\partial^2_{y_t y_t} F(y_t, \mu_t, \nu) =  1 - \mu_t^2\sech^2(\mu_t y_t + \nu)\), and
    \[
        \|\partial_{y_t} F(y_t, \mu_t, \nu) \|^2  - \partial^2_{y_t y_t} F(y_t, \mu_t, \nu) = 2 y_t \partial_{y_t} F(y_t, \mu_t, \nu) + \left( \mu_t^2 - y_t^2 - 1\right).
    \]
\end{lemma}

\begin{proof}
    The derivatives of \(F\) from Lemma~\ref{lemma:nll} with respect to \(y_t\), \(\mu_t\), and \(\nu\) are
    \[
    \partial_{y_t} F(y_t, \mu_t, \nu) = y_t - \tanh(\mu_t y_t + \nu) \mu_t,\;
    \partial_{\mu_t} F(y_t, \mu_t, \nu) = \mu_t - \tanh(\mu_t y_t + \nu)y_t,
    \]
    \[
    \nabla_{\nu} F(y_t, \mu_t, \nu) = \tanh \nu - \tanh(\mu_t y_t + \nu).
    \]
    We complete the proof of the first identity by rearranging terms in the above expressions and multiplying by \(y_t, \mu_t\) and \(\mu_t y_t\) to obtain \(-\tanh(\mu_t y_t + \nu) \mu_t y_t\).
    For the second identity, use \(\tanh^2(\cdot) = 1 - \sech^2(\cdot)\), \(\d \tanh(\cdot)/\d (\cdot) = \sech^2(\cdot)\), and \(\partial_{y_t} F(y_t, \mu_t, \nu) = y_t - \tanh(\mu_t y_t + \nu) \mu_t\). Then
    \[
    \begin{aligned}
    \partial^2_{y_t y_t} F(y_t, \mu_t, \nu) &= 1 - \mu_t^2\sech^2(\mu_t y_t + \nu) 
    = 1- \mu_t^2 + \mu_t^2 \tanh^2(\mu_t y_t + \nu) 
    = 1- \mu_t^2 + \|y_t - \partial_{y_t} F(y_t, \mu_t, \nu)\|^2\\
    &= \|\partial_{y_t} F(y_t, \mu_t, \nu) \|^2 - 2 y_t \partial_{y_t} F(y_t, \mu_t, \nu) + \left(y_t^2  + 1 - \mu_t^2 \right).
    \end{aligned}
    \]
    Therefore, the last identity follows by rearranging terms in the above expressions.
\end{proof}


\begin{lemma}[Spiked Covariance under General Gaussian Covariance]
    \label{lemma:generalized_spiked_covariance}
    Let \(\bx\sim\calN(0,\Sigma)\), where \(\Sigma\succ0\), and define
    \(\calC^\ast:=\E[\bx\bx^\top(\bx^\top\theta^\ast)^2]\).  With
    \(c:=(\theta^\ast)^\top\Sigma\theta^\ast\) and
    \(a:=\Sigma^{1/2}\theta^\ast\),
    \[
        \calC^\ast=c\Sigma+2\Sigma\theta^\ast
        (\Sigma\theta^\ast)^\top,
        \qquad
        \Sigma^{-1/2}\calC^\ast\Sigma^{-1/2}
        =cI_d+2aa^\top.
    \]
    If \(\theta^\ast\ne\mathbf0\), the whitened matrix has principal
    eigenspace \(\operatorname{span}\{a\}\), with eigenvalue \(3c\), and
    eigenvalue \(c\) on \(a^\perp\), with multiplicity \(d-1\). The
    equivalent generalized eigenproblem \(\calC^\ast v=\lambda\Sigma v\) has leading
    eigenspace \(\operatorname{span}\{\theta^\ast\}\), lower eigenspace
    \(\{v:v^\top\Sigma\theta^\ast=0\}\), and generalized spectral gap \(2c\).
    When \(\Sigma=I_d\), these are the ordinary eigenspaces of \(\calC^\ast\).
\end{lemma}
    
\begin{proof}
    Isserlis' identity reads
    \(\E[X_iX_jX_kX_l]=\Sigma_{ij}\Sigma_{kl}
    +\Sigma_{ik}\Sigma_{jl}+\Sigma_{il}\Sigma_{jk}\)
    \citep{isserlis1918formula}.  Contracting this identity with
    \(\theta_k^\ast\theta_l^\ast\) yields
    \(\calC^\ast=c\Sigma+2\Sigma\theta^\ast
    (\Sigma\theta^\ast)^\top\), and whitening yields the displayed rank-one
    form.  Since \(\|a\|^2=c\), that matrix maps \(a\) to \(3ca\) and every
    \(z\perp a\) to \(cz\). Under \(w=\Sigma^{1/2}v\),
    \(\calC^\ast v=\lambda\Sigma v\) is equivalent to
    \(\Sigma^{-1/2}\calC^\ast\Sigma^{-1/2}w=\lambda w\), so the
    generalized eigenspaces and gap follow from the whitened calculation.
\end{proof}

\begin{lemma}[Spiked Covariance for Elliptical Distributions]
    \label{lemma:elliptical_spiked_covariance}
    Let \(\bx\in\R^d\) follow a zero-mean elliptical distribution with
    covariance \(\Sigma\succ0\) and finite fourth moments.
    Define the excess kurtosis parameter
    \[
    \kappa \;:=\; \frac{\E[\|\Sigma^{-1/2}\bx\|^4]}{d(d+2)} - 1 \;>\; -1,
    \]
    where \(\Sigma^{-1/2}\) is the symmetric inverse square root of \(\Sigma\).
    Then the fourth-order moments satisfy
    \[
    \E[X_i X_j X_k X_l]
    = (1+\kappa)\bigl( \Sigma_{ij}\Sigma_{kl} + \Sigma_{ik}\Sigma_{jl} + \Sigma_{il}\Sigma_{jk} \bigr).
    \]
    The fourth cumulant tensor
    \[
    \Delta_{ijkl} := \E[X_i X_j X_k X_l] - \bigl( \Sigma_{ij}\Sigma_{kl} + \Sigma_{ik}\Sigma_{jl} + \Sigma_{il}\Sigma_{jk} \bigr)
    \]
    satisfies \(\Delta_{ijkl} = \kappa\,(\Sigma_{ij}\Sigma_{kl} + \Sigma_{ik}\Sigma_{jl} + \Sigma_{il}\Sigma_{jk})\).
    
    For \(\calC^\ast:=\E[\bx\bx^\top(\bx^\top\theta^\ast)^2]\), set
    \(c:=(\theta^\ast)^\top\Sigma\theta^\ast\) and
    \(a:=\Sigma^{1/2}\theta^\ast\).  Then
    \[
        \calC^\ast=(1+\kappa)
        \bigl[c\Sigma+2\Sigma\theta^\ast(\Sigma\theta^\ast)^\top\bigr],
        \qquad
        \Sigma^{-1/2}\calC^\ast\Sigma^{-1/2}
        =(1+\kappa)(cI_d+2aa^\top).
    \]
    If \(\theta^\ast\ne\mathbf0\), the whitened eigenvalues are
    \(3(1+\kappa)c\) on \(\operatorname{span}\{a\}\) and
    \((1+\kappa)c\) on \(a^\perp\). The equivalent generalized problem
    \(\calC^\ast v=\lambda\Sigma v\) has leading eigenspace
    \(\operatorname{span}\{\theta^\ast\}\), lower eigenspace
    \(\{v:v^\top\Sigma\theta^\ast=0\}\), and generalized spectral gap
    \(2(1+\kappa)c\).  For \(\Sigma=I_d\), these are ordinary eigenspaces.
\end{lemma}
    
\begin{proof}
    Write \(\bx=R\Sigma^{1/2}\bu\), where \(R\ge0\) is independent of
    the uniform direction \(\bu\in\mathbb S^{d-1}\)
    \cite[Theorem~2.5(iii), p.~30]{fang1990symmetric}. Rotational invariance
    implies that the fourth-moment tensor of \(\bu\) has the form
    \(c_u(\delta_{ij}\delta_{kl}+\delta_{ik}\delta_{jl}
    +\delta_{il}\delta_{jk})\) for a scalar \(c_u\). Taking the double trace
    and using \(\|\bu\|=1\), we find \(c_u=1/[d(d+2)]\). Therefore
    \[
        \E[u_i u_j u_k u_l]
        =\frac{\delta_{ij}\delta_{kl}+\delta_{ik}\delta_{jl}
        +\delta_{il}\delta_{jk}}{d(d+2)}.
    \]
    Together with \(R=\|\Sigma^{-1/2}\bx\|\), this identity proves the
    stated fourth-moment formula.
    Contracting it with \(\theta_k^\ast\theta_l^\ast\) establishes the formula for
    \(\calC^\ast\).  The remaining claims follow from the whitened calculation
    in Lemma~\ref{lemma:generalized_spiked_covariance}, multiplied by
    \(1+\kappa\).
\end{proof}

\section{Proofs of Consistency, Asymptotic Normality for Path-Integrated Score Matching}\label{sup:setup}
This appendix uses the likelihood bridges from Appendix~\ref{sup:lemma} to
restate and prove Main Theorem~\ref{thm:consistency_asymptotic_normality}
in Section~\ref{sec:asymptotic},
including consistency modulo joint sign, the aligned Gaussian limit with
inverse-Fisher covariance for any deterministic vanishing terminal
retained-signal schedule, and the attainment argument required by the
empirical estimator.
\begin{myframe}
\begin{main theorem}[Main Theorem~\ref{thm:consistency_asymptotic_normality}]\label{thm:consistency_asymptotic_normality_appendix}
    The dimension \(d\) is fixed, and the ground truth parameters
    \((\theta^\ast,\nu^\ast)\in\R^d\times\R\) satisfy
    \(\theta^\ast\ne\mathbf0\). The deterministic terminal-horizon sequence
    \(\{T_n\}_{n\geq1}\) satisfies
    \(0<\bar\alpha_{T_n}<1\) and \(\bar\alpha_{T_n}\to0\),
    and the covariate vector \(\bx\) has a density that is positive almost everywhere
    and satisfies \(\E\|\bx\|^2<\infty\).
    The estimator \(\widehat{\vartheta}_{n,T_n}^{\,\mathsf{SM}} \equiv (\widehat{\theta}_{n,T_n}^{\,\mathsf{SM}},\widehat{\nu}_{n,T_n}^{\,\mathsf{SM}})\) is \textsf{consistent} up
    to a joint sign change. We choose
    \(
        s_n \in \argmin_{s\in \{ +1, -1\}} \left\| s\,\left(\widehat{\theta}_{n,T_n}^{\,\mathsf{SM}},\widehat{\nu}_{n,T_n}^{\,\mathsf{SM}}\right) - (\theta^\ast,\nu^\ast) \right\|
    \).
    Then, as \(n \to \infty\),
    \[
        s_n \left(\widehat{\theta}_{n,T_n}^{\,\mathsf{SM}},\widehat{\nu}_{n,T_n}^{\,\mathsf{SM}}\right) \xrightarrow{\Pr} (\theta^\ast,\nu^\ast).
    \]
    If \(\E[\|\bx\|^4] < \infty\), then the aligned score matching estimator is \textsf{asymptotically normal}. As \(n\to\infty\),
    \[
        \sqrt{n} \left[ s_n \left(\widehat{\theta}_{n,T_n}^{\,\mathsf{SM}},\widehat{\nu}_{n,T_n}^{\,\mathsf{SM}}\right) - (\theta^\ast,\nu^\ast) \right]
        \xrightarrow{\mathsf{d}} \mathcal{N}(\mathbf{0}, I(\theta^\ast,\nu^\ast)^{-1}),
    \]
    where \(I\) denotes the Fisher information matrix.
\end{main theorem}
\end{myframe}

\begin{proof}[Proof of Main Theorem~\ref{thm:consistency_asymptotic_normality_appendix}]
\phantomsection\label{proof:consistency_asymptotic_normality}
    We compare objective values through the likelihood bridge, as in
    \citet{chewi2025arxiv}. We use it to establish consistency in the
    conditional MLR model, including joint sign ambiguity and unknown mixing
    weights. For asymptotic normality, we bound the centered terminal
    parameter score after averaging over the diffusion noise. This last step
    only requires a vanishing terminal retained-signal factor.

    \textbf{Proof-specific notation.} In addition to the notation fixed in
    Section~\ref{subsec:notations}, define
    \begin{itemize}
        \item \(\calE_n(\vartheta)
        :=\widehat{\calR}_n^{\,\mathsf{MLE}}(\vartheta)
        -\widehat{\calR}_n^{\,\mathsf{MLE}}(\vartheta^\ast)\) and
        \(\calE(\vartheta):=\E[\calE_n(\vartheta)]\), where the expectation
        is over the i.i.d. sample. These quantities are the empirical
        contrast and the population excess risk.
        \item \(c_n:=\bar\alpha_{T_n}/[2(1-\bar\alpha_{T_n})]\) and
        \(\Sigma_n:=\E_{\Pr_n}[\bx\bx^\top]\).
        \item \(\widetilde{\vartheta}_{n,T_n}^{\,\mathsf{SM}}
        :=s_n\widehat{\vartheta}_{n,T_n}^{\,\mathsf{SM}}\). This estimator is
        sign-aligned in both coordinates, with \(s_n\) as defined in the theorem.
        \item \(m_{\vartheta}(y_0,\bx):=\ln p(y_0\mid\bx;\vartheta)\).

    \end{itemize}

    We set \(T=T_n\) in
    Lemma~\ref{lemma:log_likelihood_SM_loss_finite} and define
    \(\mathsf{c}(\bar{\alpha}_{T_n})
    :=\frac{1}{2}\left(\ln\frac{1-\bar{\alpha}_{T_n}}
    {\bar{\alpha}_{T_n}}+\ln(2\p)+1\right)\). The empirical
    bridge then reads
    \[
        \widehat{\calR}_{n, T_n}^{\,\mathsf{SM}}(\vartheta) = \widehat{\calR}_n^{\,\mathsf{MLE}}(\vartheta) - \Delta_n(\vartheta) - \mathsf{c}(\bar{\alpha}_{T_n}).
    \]
    Whenever the empirical minimum is attained, optimality implies
    \(\widehat{\calR}_{n,T_n}^{\,\mathsf{SM}}
    (\widehat{\vartheta}_{n,T_n}^{\,\mathsf{SM}})
    \leq\widehat{\calR}_{n,T_n}^{\,\mathsf{SM}}(\vartheta^\ast)\).
    Substitution isolates the excess risk
    \(\mathcal{E}_n(\widehat{\vartheta}_{n,T_n}^{\,\mathsf{SM}})
    \equiv \widehat{\calR}_n^{\,\mathsf{MLE}}
    (\widehat{\vartheta}_{n,T_n}^{\,\mathsf{SM}})
    -\widehat{\calR}_n^{\,\mathsf{MLE}}(\vartheta^\ast)\).
    Since \(\Delta_n(\vartheta^\ast)\ge 0\), we obtain
    \[
        \mathcal{E}_n(\widehat{\vartheta}_{n,T_n}^{\,\mathsf{SM}})
        \leq \Delta_n(\widehat{\vartheta}_{n,T_n}^{\,\mathsf{SM}})
        -\Delta_n(\vartheta^\ast)
        \leq \Delta_n(\widehat{\vartheta}_{n,T_n}^{\,\mathsf{SM}}).
    \]

    \vspace{0.5em}
    \phantomsection\label{proof:consistency_asymptotic_normality_step_1}
    \noindent\textbf{Step 1. Boundedness of \(\widehat{\theta}_{n,T_n}^{\,\mathsf{SM}}\) and control of the excess risk.} \\
    Lemma~\ref{lemma:decaying_kl_bound_transition} bounds the decaying KL
    divergence \(\Delta_n\). We first work on samples with a nonempty
    score matching argmin. Set
    \(c_n \equiv \frac{\bar{\alpha}_{T_n}}{2(1-\bar{\alpha}_{T_n})}\).
    To express the discrepancy bound in terms of the empirical regression
    norm, we apply the Cauchy--Schwarz inequality and use
    \(|\tanh(\cdot)|\leq 1\):
    \[
    \begin{aligned}
        \mathcal{E}_n(\widehat{\vartheta}_{n,T_n}^{\,\mathsf{SM}})
        &\leq \Delta_n(\widehat{\vartheta}_{n,T_n}^{\,\mathsf{SM}})\\
        &\leq c_n \E_{\Pr_n} \left[y_0^2 + (\widehat{\theta}_{n,T_n}^{\,\mathsf{SM}})^\top \bx \bx^\top \widehat{\theta}_{n,T_n}^{\,\mathsf{SM}}
        - 2 y_0 ((\widehat{\theta}_{n,T_n}^{\,\mathsf{SM}})^\top \bx) \tanh \widehat{\nu}_{n,T_n}^{\,\mathsf{SM}} + 1\right]\\
        &\leq c_n \left( \|\widehat{\theta}_{n,T_n}^{\,\mathsf{SM}}\|_{\Sigma_n}^2 + 2\|\widehat{\theta}_{n,T_n}^{\,\mathsf{SM}}\|_{\Sigma_n} \sqrt{\mathbb{E}_{\Pr_n}[y_0^2]} + \mathbb{E}_{\Pr_n}[y_0^2] + 1\right).
    \end{aligned}
    \]

    For a lower bound on the same excess risk, we start from the negative
    log-likelihood in Lemma~\ref{lemma:nll}. We apply
    \(|x| \geq \ln\frac{\cosh(x+x')}{\cosh x'} \geq -|x|\) and the
    Cauchy--Schwarz inequality for \(\Sigma_n\) to obtain the lower bound
    \[
    \begin{aligned}
        \mathcal{E}_n(\widehat{\vartheta}_{n,T_n}^{\,\mathsf{SM}})
        &= \E_{\Pr_n}\left[
        \begin{aligned}
            &\frac{1}{2}(\widehat{\theta}_{n,T_n}^{\,\mathsf{SM}})^\top
              \bx \bx^\top \widehat{\theta}_{n,T_n}^{\,\mathsf{SM}}
              -\frac{1}{2} (\theta^\ast)^\top \bx \bx^\top \theta^\ast \\
            &\quad - \ln \frac{\cosh((\widehat{\theta}_{n,T_n}^{\,\mathsf{SM}})^\top
              \bx y_0+\widehat{\nu}_{n,T_n}^{\,\mathsf{SM}})}{\cosh\widehat{\nu}_{n,T_n}^{\,\mathsf{SM}}}
              + \ln \frac{\cosh((\theta^\ast)^\top \bx y_0+\nu^\ast)}{\cosh \nu^\ast}
        \end{aligned}
        \right]\\
        &\geq \frac{1}{2} \|\widehat{\theta}_{n,T_n}^{\,\mathsf{SM}}\|_{\Sigma_n}^2
        - \|\widehat{\theta}_{n,T_n}^{\,\mathsf{SM}}\|_{\Sigma_n}
          \sqrt{\mathbb{E}_{\Pr_n}[y_0^2]} \\
        &\quad - \|\theta^\ast\|_{\Sigma_n} \sqrt{\mathbb{E}_{\Pr_n}[y_0^2]}
        - \frac{1}{2}\|\theta^\ast\|_{\Sigma_n}^2.
    \end{aligned}
    \]

    Comparing the quadratic lower bound with the \(c_n\)-scaled upper bound,
    we find that the empirical norm of the score matching estimator
    satisfies
    \[
        \underbrace{\left( \frac{1}{2} - c_n \right)}_{:=A_n} \|\widehat{\theta}_{n,T_n}^{\,\mathsf{SM}}\|_{\Sigma_n}^2 - \underbrace{(1 + 2c_n) \sqrt{\mathbb{E}_{\Pr_n}[y_0^2]}}_{:= B_n} \|\widehat{\theta}_{n,T_n}^{\,\mathsf{SM}}\|_{\Sigma_n}
        \leq \underbrace{c_n(\E_{\Pr_n}[y_0^2] + 1 ) + \|\theta^\ast\|_{\Sigma_n} \sqrt{\mathbb{E}_{\Pr_n}[y_0^2]} + \frac{1}{2}\|\theta^\ast\|_{\Sigma_n}^2}_{C_n}.
    \]

    Under \(\bar\alpha_{T_n}\to0\), we have
    \(c_n=\bar\alpha_{T_n}/[2(1-\bar\alpha_{T_n})]=o(1)\).
    Since \(A_n=\frac{1}{2}-c_n\to\frac{1}{2}\), we have \(A_n>0\) for all large \(n\).
    Since \(\bx\) has finite population second moments \(\E[\bx\bx^\top]\),
    we can apply Markov's inequality to bound the empirical moments:
    \(\|\theta^\ast\|_{\Sigma_n}^2 = \calO_P(1)\) and
    \(\mathbb{E}_{\Pr_n}[y_0^2]=\calO_P(1)\).
    First, consider the empirical weighted norm \(\|\theta^\ast\|_{\Sigma_n}^2 = (\theta^\ast)^\top \Sigma_n \theta^\ast\) with \(\Sigma_n \equiv \E_{\Pr_n}[\bx \bx^\top]\).
    Its expectation under the data-generating distribution is finite.
    By Markov's inequality, for every \(M > 0\),
    \[
        \lim_{M \to \infty} \limsup_{n \to \infty} \mathbb{P} \left( \|\theta^\ast\|_{\Sigma_n}^2 > M \right)
        \leq \lim_{M \to \infty} \frac{\mathbb{E}\left[(\theta^\ast)^\top \E_{\Pr_n}[\bx\bx^\top] \theta^\ast\right]}{M}
        = \lim_{M \to \infty} \frac{(\theta^\ast)^\top \mathbb{E}[\bx\bx^\top] \theta^\ast}{M} = 0.
    \]
    Second, consider the empirical second moment \(\mathbb{E}_{\Pr_n}[y_0^2]\).
    Under the ground truth distribution, the response satisfies \(y_0 = (-1)^{z+1} \langle \theta^\ast, \bx \rangle + \varepsilon_0\), where \((\bx,z)\ind\varepsilon_0\) and \(\varepsilon_0\sim\mathcal{N}(0,1)\).
    Its second moment satisfies
    \(
        \mathbb{E} \left[ \mathbb{E}_{\Pr_n}[y_0^2] \right] = \mathbb{E}[y_0^2] = (\theta^\ast)^\top \mathbb{E}[\bx\bx^\top] \theta^\ast + 1
    \).
    Since \(\E[\bx \bx^\top]\) is finite, another application of Markov's
    inequality shows that
    \[
        \lim_{M \to \infty} \limsup_{n \to \infty} \mathbb{P} \left( \mathbb{E}_{\Pr_n}[y_0^2] > M \right)
        \leq \lim_{M \to \infty} \frac{\E\left[\mathbb{E}_{\Pr_n}[y_0^2]\right]}{M}
        = \lim_{M \to \infty} \frac{(\theta^\ast)^\top \mathbb{E}[\bx\bx^\top] \theta^\ast + 1}{M} = 0.
    \]
    The two inequalities above guarantee \(\|\theta^\ast\|_{\Sigma_n}^2 = \calO_P(1),\, \mathbb{E}_{\Pr_n}[y_0^2]=\calO_P(1)\).
    It follows that
    \[
        B_n := (1 + 2c_n)\sqrt{\mathbb{E}_{\Pr_n}[y_0^2]} = \calO(1) \sqrt{\calO_P(1)} = \calO_P(1),
    \]
    \[
        C_n:=c_n(\E_{\Pr_n}[y_0^2]+1)+\|\theta^\ast\|_{\Sigma_n}
        \sqrt{\E_{\Pr_n}[y_0^2]}+\frac12\|\theta^\ast\|_{\Sigma_n}^2
        =o(1)\calO_P(1)+\calO_P(1)=\calO_P(1).
    \]
    This leaves a quadratic inequality of the form \(A_n \|\widehat{\theta}_{n,T_n}^{\,\mathsf{SM}}\|_{\Sigma_n}^2 - B_n \|\widehat{\theta}_{n,T_n}^{\,\mathsf{SM}}\|_{\Sigma_n} \le C_n\), where \(A_n \to \frac{1}{2}\) and \(B_n, C_n = \calO_P(1)\). Hence \(A_n \geq \frac{1}{4}\) for all large \(n\).
    Solving the quadratic inequality shows that the empirical norm is bounded
    in probability.
    \[
        \|\widehat{\theta}_{n,T_n}^{\,\mathsf{SM}}\|_{\Sigma_n} \leq \frac{B_n + \sqrt{B_n^2 + 4 A_n C_n}}{2 A_n} \leq \frac{B_n}{A_n} + \sqrt{\frac{C_n}{A_n}} \leq 4\calO_P(1) + \sqrt{4\calO_P(1)} = \calO_P(1).
    \]
    On samples with an empty argmin, the convention in
    Section~\ref{subsec:notations} sets
    \(\widehat\vartheta_{n,T_n}^{\,\mathsf{SM}}=(\mathbf0,0)\). Hence the
    empirical norm bound holds unconditionally. The fallback value also
    satisfies the transition-kernel KL bound. In particular,
    \(\Delta_n((\mathbf0,0))\leq c_n(\E_{\Pr_n}[y_0^2]+1)=\calO_P(\bar\alpha_{T_n})=o_P(1)\).
    We can now substitute the empirical norm bound into the KL discrepancy bound.
    Since \(c_n=\calO(\bar\alpha_{T_n})=o(1)\),
    \[
        \Delta_n(\widehat{\vartheta}_{n,T_n}^{\,\mathsf{SM}})\leq c_n\calO_P(1)=\calO_P(\bar\alpha_{T_n})=o_P(1).
    \]
    On samples with a nonempty argmin, we also have
    \(\mathcal{E}_n(\widehat{\vartheta}_{n,T_n}^{\,\mathsf{SM}})
    \leq\Delta_n(\widehat{\vartheta}_{n,T_n}^{\,\mathsf{SM}})\).
    Moreover, since \(\bx\) has a density that is positive almost everywhere on \(\mathbb{R}^d\), \(\Sigma := \E[\bx \bx^\top]\) is positive definite,
    and its smallest eigenvalue satisfies \(\lambda_{\min}(\Sigma) > 0\).
    To transfer the empirical norm bound to the Euclidean norm, we control
    the smallest eigenvalue using Weyl's Perturbation Theorem
    (Theorem~4.3.1 on page~239 of~\citet{horn2012matrix}) and the Weak Law of
    Large Numbers (Theorem~2.2.14 of~\citet{durrett2019probability}):
    \(
        |\lambda_{\min}(\Sigma_n) - \lambda_{\min}(\Sigma)| \leq \|\Sigma_n - \Sigma\|_2  = o_P(1)
    \).
    Therefore, \(|\lambda_{\min}(\Sigma_n) - \lambda_{\min}(\Sigma)| / \lambda_{\min}(\Sigma) = o_P(1)\). For all large \(n\),
    \(\lambda_{\min}(\Sigma_n) \geq \lambda_{\min}(\Sigma) - |\lambda_{\min}(\Sigma_n) - \lambda_{\min}(\Sigma)| = (1-o_P(1)) \lambda_{\min}(\Sigma) > 0\) with high probability and
    \[
    \|\widehat{\theta}_{n,T_n}^{\,\mathsf{SM}}\|
    \leq \sqrt{\frac{(\widehat{\theta}_{n,T_n}^{\,\mathsf{SM}})^\top\Sigma_n (\widehat{\theta}_{n,T_n}^{\,\mathsf{SM}})}{\lambda_{\min}(\Sigma_n)}} \leq \frac{\|\widehat{\theta}_{n,T_n}^{\,\mathsf{SM}}\|_{\Sigma_n}}{\sqrt{(1-o_P(1))\lambda_{\min}(\Sigma)}} = \mathcal{O}_P(1).
    \]

    \vspace{0.5em}
    \phantomsection\label{proof:consistency_asymptotic_normality_step_2}
    \noindent\textbf{Step 2. Convergence to the ground truth and boundedness of \(\widehat{\vartheta}_{n,T_n}^{\,\mathsf{SM}}\).} \\
    Taking the expectation of the empirical average under the i.i.d.
    ground truth sample expresses the population excess risk as a conditional
    KL divergence:
    \[
        \mathcal{E}(\vartheta)
        = \E_{\bx}\!\left[
          \KL\!\left(p(\,\cdot\mid\bx;\vartheta^\ast)
          \parallel p(\,\cdot\mid\bx;\vartheta)\right)\right]
        \geq 0.
    \]
    The equality condition for KL divergence shows that
    \[
        \mathcal{E}(\vartheta) = 0 \Leftrightarrow
        p(\,\cdot\mid\bx;\vartheta^\ast)
        =p(\,\cdot\mid\bx;\vartheta)
        \quad\text{for almost every }\bx.
    \]
    Their conditional moment generating function is
    \[
        \E_{y_0\mid\bx;\vartheta}[\exp(t y_0)]
        =\left[\cosh(t\theta^\top\bx)+\tanh\nu\cdot\sinh(t\theta^\top\bx)\right]\exp(t^2/2).
    \]
    In particular, equality of the conditional distributions implies equality of their first two raw moments.
    \[
        (\theta^\top\bx)\tanh\nu=((\theta^\ast)^\top\bx)\tanh\nu^\ast,
        \qquad
        (\theta^\top\bx)^2=((\theta^\ast)^\top\bx)^2
        \quad\text{almost surely}.
    \]
    The second identity is equivalent to
    \[
        \bx^\top\left(\theta\theta^\top-\theta^\ast(\theta^\ast)^\top\right)\bx=0
        \quad\text{almost surely}.
    \]
    Since \(\bx\) has a density that is positive almost everywhere on \(\R^d\), this quadratic polynomial is zero on \(\R^d\), and hence
    \(\theta\theta^\top=\theta^\ast(\theta^\ast)^\top\).
    Because \(\theta^\ast\neq\mathbf{0}\), it follows that \(\theta=s\theta^\ast\) for some \(s\in\{+1,-1\}\).
    The first-moment identity then requires \(s\tanh\nu=\tanh\nu^\ast\), so the oddness and strict monotonicity of \(\tanh\) imply \(\nu=s\nu^\ast\).
    The reverse implication follows because the joint sign flip exchanges the two mixture components and leaves the conditional distribution unchanged.
    This agrees with Proposition~\ref{prop:terminal_kl_pinsker_zero_set} at
    \(\bar{\alpha}_0=1\) and establishes identifiability for every
    \(\vartheta\in\R^d\times\R\).
    \[
        \mathcal{E}(\vartheta)=0
        \quad\Longleftrightarrow\quad
        \vartheta\equiv(\theta,\nu)=s(\theta^\ast,\nu^\ast)=s\vartheta^\ast,
        \qquad s\in\{+1,-1\}.
    \]
    From Step~1, we have
    \(\|\widehat{\theta}_{n,T_n}^{\,\mathsf{SM}}\|=\mathcal{O}_P(1)\).
    Thus, for every \(\delta>0\), we can choose a finite radius
    \(R_\delta>\|\theta^\ast\|\) and an index \(n_\delta\) such that
    \[
        \inf_{n\geq n_\delta}
        \Pr\!\left(\widehat{\theta}_{n,T_n}^{\,\mathsf{SM}}
        \in K_{R_\delta}\right)\geq 1-\delta,
        \qquad
        K_{R_\delta}:=\{\theta\in\R^d\mid\|\theta\|\leq R_\delta\}.
    \]
    We work on the compact parameter set
    \(K_\delta:=K_{R_\delta}\times\overline{\R}\), where
    \(\overline{\R}:=\R\cup\{-\infty,+\infty\}\), equipped with
    the topology induced by the map \(\nu\mapsto\tanh\nu\) onto
    \([-1,1]\), with \(\tanh(\pm\infty)=\pm1\).
    Near either finite ground truth representative, this topology agrees
    with the usual Euclidean topology.
    At its two boundary points, the conditional model has the limits
    \(p(\,\cdot\mid\bx;\theta,\pm\infty)
    =\mathcal{N}(\pm\langle\theta,\bx\rangle,1)\).
    To exclude these boundary limits, we apply Fatou's lemma
    (Theorem~1.5.5 of~\citet{durrett2019probability}):
    \[
        \liminf_{|\nu| \to \infty} \mathcal{E}(\theta, \nu)
        \geq \E_{\bx}\!\left[\min_{s\in\{+1,-1\}}
        \KL\!\left(p(\,\cdot\mid\bx;\vartheta^\ast)
        \parallel\mathcal{N}\!\left(s\langle\theta,\bx\rangle,1\right)\right)\right]>0.
    \]
    Since \(|\nu^\ast|<\infty\) and \(\theta^\ast\ne\mathbf0\), the ground truth conditional distribution has two distinct Gaussian components with positive mixing weights for almost every \(\bx\).
    The identifiability of Gaussian mixtures implies that the KL divergence to
    either single Gaussian is positive for almost every \(\bx\). Define
    \[
        g(\theta):=\E_{\bx}\!\left[\min_{s\in\{+1,-1\}}
        \KL\!\left(p(\,\cdot\mid\bx;\vartheta^\ast)
        \parallel\mathcal N\!\left(s\langle\theta,\bx\rangle,1\right)\right)\right].
    \]
    The integrand is continuous in \(\theta\). Convexity of KL divergence in
    its first argument bounds it, for \(\theta\in K_{R_\delta}\), by
    \(\frac12(\|\theta^\ast\|+R_\delta)^2\|\bx\|^2\). The finite second
    moment of \(\bx\) and dominated convergence therefore make \(g\)
    continuous on \(K_{R_\delta}\). We also have \(g(\theta)>0\) at every
    \(\theta\in K_{R_\delta}\). Continuity on this compact set ensures a
    strictly positive minimum,
    \(\kappa_\delta:=\min_{\theta\in K_{R_\delta}}g(\theta)>0\).
    \[
        \inf_{\theta \in K_{R_\delta}} \liminf_{|\nu| \to \infty}
        \mathcal{E}(\theta,\nu) \geq \kappa_\delta>0,
        \qquad
        \min_{\vartheta\in K_\delta}\calE(\vartheta)
        =\calE(\vartheta^\ast)=\calE(-\vartheta^\ast)=0.
    \]
    Hence
    \(\argmin_{\vartheta\in K_\delta}\calE(\vartheta)
    =\{\vartheta^\ast,-\vartheta^\ast\}\).
    The ambient parameter space \(\R^d\times\overline{\R}\) contains
    \(K_\delta\). The global maximizers of the expected conditional
    log-likelihood on the ambient space are
    \(\mathcal{S}^\ast:=\{\vartheta^\ast,-\vartheta^\ast\}\).
    \[
        \mathcal{S}^\ast
        =\left\{\vartheta'\in\R^d\times\overline{\R}\ \middle|\
        \E_{y_0,\bx\mid\vartheta^\ast}[m_{\vartheta'}]
        =\sup_{\vartheta\in\R^d\times\overline{\R}}
        \E_{y_0,\bx\mid\vartheta^\ast}[m_\vartheta]\right\}
        =\left\{\vartheta'\in\R^d\times\overline{\R}\ \middle|\
        \calE(\vartheta')=0\right\}.
    \]
    Let \(U\subset K_\delta\) be a neighborhood and set
    \(R_U:=\sup_{(\theta,\nu)\in U}\|\theta\|<\infty\).
    Lemma~\ref{lemma:nll} and the inequality
    \(\bigl|\ln[\cosh(x+x')/\cosh x']\bigr|\leq |x|\), followed by passage to the
    boundary limits when needed, imply
    \[
    \begin{aligned}
    \E\!\left[\sup_{\vartheta\in U}\{-m_\vartheta(y_0,\bx)\}\right]
    &\leq \frac12\E[y_0^2]
        +\frac12R_U^2\E[\|\bx\|^2]
        +R_U\E[|y_0|\|\bx\|]
        +\frac12\ln(2\p)\\
    &\leq \frac12\E[y_0^2]
        +\frac12R_U^2\E[\|\bx\|^2]
        +R_U\sqrt{\E[y_0^2]\E[\|\bx\|^2]}
        +\frac12\ln(2\p)<\infty.
    \end{aligned}
    \]
    This envelope verifies the local domination condition in Wald's
    consistency theorem for M-estimators with nonunique maxima
    (Theorem~5.14 on page~48 of~\citet{van2000asymptotic}). The conditional
    log-likelihood is continuous in each finite parameter and has continuous
    limits at \(\nu=\pm\infty\), so it is upper semicontinuous on
    \(K_\delta\).

    To apply the consistency theorem, we must also check likelihood
    near-optimality. We verify this condition by combining score matching optimality with the
    empirical bridge. Let \(\mathcal A_n\) denote the event that the
    empirical score matching argmin is nonempty, and define
    \(D_n:=\sup_{\vartheta}\E_{\Pr_n}m_\vartheta
    -\E_{\Pr_n}m_{\widehat{\vartheta}_{n,T_n}^{\,\mathsf{SM}}}\) on the entire
    sample space under the estimator convention in
    Section~\ref{subsec:notations}. Equivalently,
    \(D_n=\widehat{\calR}_n^{\,\mathsf{MLE}}
    (\widehat{\vartheta}_{n,T_n}^{\,\mathsf{SM}})
    -\inf_{\vartheta}\widehat{\calR}_n^{\,\mathsf{MLE}}(\vartheta)\).
    On \(\mathcal A_n\), score matching optimality means that
    \(\widehat{\calR}_{n,T_n}^{\,\mathsf{SM}}(\widehat{\vartheta}_{n,T_n}^{\,\mathsf{SM}})-\widehat{\calR}_{n,T_n}^{\,\mathsf{SM}}(\vartheta)\leq0\)
    for every candidate parameter \(\vartheta\).
    We then apply the empirical bridge, whose parameter-independent constant
    cancels, and use \(\Delta_n(\vartheta)\geq0\) to obtain
    \[
        \widehat{\calR}_n^{\,\mathsf{MLE}}(\widehat{\vartheta}_{n,T_n}^{\,\mathsf{SM}})
        -\widehat{\calR}_n^{\,\mathsf{MLE}}(\vartheta)
        \quad=\widehat{\calR}_{n,T_n}^{\,\mathsf{SM}}(\widehat{\vartheta}_{n,T_n}^{\,\mathsf{SM}})
        -\widehat{\calR}_{n,T_n}^{\,\mathsf{SM}}(\vartheta)
        +\Delta_n(\widehat{\vartheta}_{n,T_n}^{\,\mathsf{SM}})-\Delta_n(\vartheta)
        \quad\leq\Delta_n(\widehat{\vartheta}_{n,T_n}^{\,\mathsf{SM}}).
    \]
    Taking the supremum over \(\vartheta\), we obtain
    \(0\leq D_n\leq\Delta_n(\widehat{\vartheta}_{n,T_n}^{\,\mathsf{SM}})\) on \(\mathcal A_n\).
    Step~1 established
    \(\Delta_n(\widehat{\vartheta}_{n,T_n}^{\,\mathsf{SM}})=\calO_P(\bar\alpha_{T_n})=o_P(1)\)
    on the entire sample space, while
    \(\Pr(\mathcal A_n^c)\to0\) follows from
    Lemma~\ref{lemma:sm_minimum_attainment}. Therefore, for every
    \(\epsilon>0\),
    \(\Pr(D_n>\epsilon)\leq\Pr(\mathcal A_n^c)
    +\Pr\{\Delta_n(\widehat{\vartheta}_{n,T_n}^{\,\mathsf{SM}})>\epsilon\}\to0\).
    Thus \(D_n=o_P(1)\), which is the near-optimality condition required
    for consistency.
    We apply Wald's theorem to the compact set
    \(K_\delta \equiv K_{R_\delta} \times \overline{\mathbb{R}}\). The
    equivalence \(\widehat{\vartheta}_{n,T_n}^{\,\mathsf{SM}} \in K_\delta
    \Leftrightarrow \widehat{\theta}_{n,T_n}^{\,\mathsf{SM}} \in
    K_{R_\delta}\) lets us write its conclusion as
    \[
        \lim_{n\to\infty} \Pr\!\left(
        \left\{\min_{s \in \{+1, -1\}} \|s\, \widehat{\vartheta}_{n,T_n}^{\,\mathsf{SM}} - \vartheta^\ast\| \geq \epsilon\right\}
        \cap\left\{\widehat{\theta}_{n,T_n}^{\,\mathsf{SM}} \in K_{R_\delta}\right\}\right) = 0
        \quad\text{for every }\epsilon > 0.
    \]
    The probability that the estimator leaves \(K_\delta\) is at most
    \(\delta\) for all sufficiently large \(n\), so
    \[
        \limsup_{n\to\infty} \Pr\left(\min_{s \in \{+1, -1\}} \|s\, \widehat{\vartheta}_{n,T_n}^{\,\mathsf{SM}} - \vartheta^\ast\| \geq \epsilon\right)
        \leq 1 - \liminf_{n\to\infty}\Pr\left(\widehat{\theta}_{n,T_n}^{\,\mathsf{SM}} \in K_{R_\delta}\right) \leq \delta
        \quad\text{for every }\epsilon > 0.
    \]
    Since \(\delta>0\) is arbitrary,
    \[\min_{s \in \{+1, -1\}} \|s\, \widehat{\vartheta}_{n,T_n}^{\,\mathsf{SM}} - \vartheta^\ast\| \xrightarrow{\Pr} 0.\]
    This also implies \(\min_{s \in \{+1, -1\}} \|s\, \widehat{\theta}_{n,T_n}^{\,\mathsf{SM}} - \theta^\ast\| = o_P(1)\).
    The full parameter vector and its mixing coordinate are therefore
    bounded in probability:
    \[
    \|\widehat{\vartheta}_{n,T_n}^{\,\mathsf{SM}}\| \leq \|\vartheta^\ast\| + \min_{s \in \{+1, -1\}} \|s\, \widehat{\vartheta}_{n,T_n}^{\,\mathsf{SM}} - \vartheta^\ast\| = \calO(1) + o_P(1) = \calO_P(1),
    \]
    \[
    |\widehat{\nu}_{n,T_n}^{\,\mathsf{SM}}| =  \sqrt{\|\widehat{\vartheta}_{n,T_n}^{\,\mathsf{SM}}\|^2 - \|\widehat{\theta}_{n,T_n}^{\,\mathsf{SM}}\|^2}
    \leq \|\widehat{\vartheta}_{n,T_n}^{\,\mathsf{SM}}\| = \calO_P(1).
    \]

    \vspace{0.5em}
    \phantomsection\label{proof:consistency_asymptotic_normality_step_3}
    \noindent\textbf{Step 3. Asymptotic normality.}\\
    Step~1 establishes the regression bound and vanishing terminal
    discrepancy used in Step~2 to prove consistency of the sign-aligned
    full parameter vector. By consistency, the segment joining the
    sign-aligned estimator to \(\vartheta^\ast\) lies in any fixed
    neighborhood of \(\vartheta^\ast\) with probability tending to one.
    Under the fourth-moment assumption,
    we separately bound the centered terminal gradient and the terminal
    Hessian. The condition \(\bar\alpha_{T_n}=o(1)\) makes these corrections
    negligible, so the likelihood score determines the limiting distribution.
    Take \(U\) to be a fixed compact convex neighborhood of
    \(\vartheta^\ast\), with \(R_U=\sup_{(\theta,\nu)\in U}\|\theta\|\)
    as in Step~2.

    \textbf{The terminal parameter score.}
    We use \(y_t,\theta_t,\mu_t\), and \(\xi_t\) from
    Section~\ref{subsec:2mlr}. All derivatives below are with respect to
    the original parameters \(\vartheta=(\theta,\nu)\), including the
    factor \(\sqrt{\bar\alpha_t}\) when differentiating
    \(\theta_t=\sqrt{\bar\alpha_t}\theta\).
    The expectation \(\E_{y_t\mid y_0}\) integrates over the
    parameter-independent diffusion noise. To separate the terminal correction
    from the likelihood score, we differentiate the bridge:
    \[
    \begin{aligned}
        \nabla\Delta_n(\vartheta)
        &=-\E_{\Pr_n}\E_{y_{T_n}\mid y_0}
        [\nabla_\vartheta\ln p(y_{T_n}\mid\bx;\theta_{T_n},\nu)],\\
        \nabla\widehat{\calR}_{n,T_n}^{\,\mathsf{SM}}(\vartheta)
        &=-\E_{\Pr_n}[\nabla m_\vartheta]-\nabla\Delta_n(\vartheta).
    \end{aligned}
    \]
    The Gaussian noise moments and bounded derivatives of \(\tanh\)
    justify these differentiations under the noise expectation for each observation.

    We first average the score over the diffusion noise to expose the
    cancellation of its leading centered-noise term. Applying Gaussian integration by parts in \(\xi_t\),
    with \(\E[\xi_t f(\xi_t)]=\E[f'(\xi_t)]\), we obtain
    \[
    \begin{aligned}
        \E_{y_t\mid y_0}[\nabla_\theta\ln p(y_t\mid\bx;\theta_t,\nu)]
        &=\bar\alpha_t\bx\Bigl\{
        y_0\E_{y_t\mid y_0}[\tanh(\mu_t y_t+\nu)]\\
        &\qquad+\mu_0\bigl[(1-\bar\alpha_t)
        \E_{y_t\mid y_0}[\sech^2(\mu_t y_t+\nu)]-1\bigr]\Bigr\},\\
        \E_{y_t\mid y_0}[\nabla_\nu\ln p(y_t\mid\bx;\theta_t,\nu)]
        &=\E_{y_t\mid y_0}[\tanh(\mu_t y_t+\nu)]-\tanh\nu.
    \end{aligned}
    \]
    The bounds \(|\tanh|\leq1\) and \(0\leq\sech^2\leq1\) imply
    \[
        \bigl\|\E_{y_t\mid y_0}[\nabla_\theta\ln p(y_t\mid\bx;\theta_t,\nu)]\bigr\|
        \leq\bar\alpha_t\|\bx\|(|y_0|+|\mu_0|).
    \]
    For the other coordinate, we separate the deterministic shift from the
    centered Gaussian shift using
    \(\mu_t y_t=\bar\alpha_t\mu_0y_0+
    \sqrt{\bar\alpha_t(1-\bar\alpha_t)}\mu_0\xi_t\).
    The Lipschitz bound for \(\tanh\), followed by Taylor's formula with
    \(\sup|\tanh''|\leq2\), bounds their contribution to the averaged score:
    \[
        \bigl|\E_{y_t\mid y_0}[\nabla_\nu\ln p(y_t\mid\bx;\theta_t,\nu)]\bigr|
        \leq\bar\alpha_t|\mu_0y_0|
        +\bar\alpha_t(1-\bar\alpha_t)\mu_0^2.
    \]
    Under the ground truth law, the conditional density of \(y_t\) given
    \(\bx\) is \(p(\,\cdot\mid\bx;\theta_t^\ast,\nu^\ast)\).
    Its parameter score integrates to zero. The tower property therefore
    preserves this centering after averaging over the diffusion noise:
    \[
        \E\!\left[
        \left.\E_{y_t\mid y_0}[\nabla_\vartheta\ln p(y_t\mid\bx;\theta_t,\nu)]
        \right|_{\vartheta=\vartheta^\ast}\right]=\mathbf0.
    \]
    To make the squared bound explicit, we substitute
    \(|\mu_0^\ast|\leq\|\theta^\ast\|\|\bx\|\) in both coordinate bounds.
    The conditional second moment
    \(\E[y_0^2\mid\bx]=1+(\mu_0^\ast)^2\), as in Step~1,
    then implies
    \[
    \begin{aligned}
        &\E\left\|
        \left.\E_{y_t\mid y_0}[\nabla_\vartheta\ln p(y_t\mid\bx;\theta_t,\nu)]
        \right|_{\vartheta=\vartheta^\ast}\right\|^2\\
        &\quad\leq\bar\alpha_t^2(1+\|\theta^\ast\|^2)
        \E\!\left[\|\bx\|^2(|y_0|+\|\theta^\ast\|\|\bx\|)^2\right]\\
        &\quad\leq2\bar\alpha_t^2(1+\|\theta^\ast\|^2)
        \left[\E\|\bx\|^2+2\|\theta^\ast\|^2\E\|\bx\|^4\right].
    \end{aligned}
    \]
    The factor multiplying \(\bar\alpha_t^2\) is finite under the
    fourth-moment assumption and depends only on the ground truth regression
    vector and the covariate moments. To pass from one observation to the
    empirical average, expand the definition of \(\E_{\Pr_n}\):
    \[
    \begin{aligned}
        \nabla\Delta_n(\vartheta^\ast)
        &=-\frac1n\sum_{i=1}^n
        \left.\E_{y_{T_n}\mid y_0^{(i)}}
        [\nabla_\vartheta\ln p(y_{T_n}\mid\bx^{(i)};\theta_{T_n},\nu)]
        \right|_{\vartheta=\vartheta^\ast}.
    \end{aligned}
    \]
    For each fixed \(n\), these summands are independent and identically
    distributed functions of the original observations, and each has mean
    zero by the score identity above. Thus, when expanding the squared norm
    of the sum, all cross inner products have expectation zero. Only the
    \(n\) diagonal terms remain:
    \[
    \begin{aligned}
        \E\|\nabla\Delta_n(\vartheta^\ast)\|^2
        &=\frac1{n^2}\sum_{i=1}^n\E\left\|
        \left.\E_{y_{T_n}\mid y_0^{(i)}}
        [\nabla_\vartheta\ln p(y_{T_n}\mid\bx^{(i)};\theta_{T_n},\nu)]
        \right|_{\vartheta=\vartheta^\ast}\right\|^2\\
        &=\frac1n\E\left\|
        \left.\E_{y_{T_n}\mid y_0}
        [\nabla_\vartheta\ln p(y_{T_n}\mid\bx;\theta_{T_n},\nu)]
        \right|_{\vartheta=\vartheta^\ast}\right\|^2
        =\calO\!\left(\frac{\bar\alpha_{T_n}^2}{n}\right).
    \end{aligned}
    \]
    The factor \(\sqrt n\) compensates for this \(1/n\) reduction in the
    second moment:
    \[
        \E\|\sqrt n\,\nabla\Delta_n(\vartheta^\ast)\|^2
        =n\E\|\nabla\Delta_n(\vartheta^\ast)\|^2
        =\calO(\bar\alpha_{T_n}^2).
    \]
    Thus \(\E[\nabla\Delta_n(\vartheta^\ast)]=\mathbf0\), and we apply Markov's
    inequality to conclude that
    \(\sqrt n\,\nabla\Delta_n(\vartheta^\ast)=\calO_P(\bar\alpha_{T_n})\).
    The factor \(\bar\alpha_{T_n}^2\) comes from canceling the centered noise
    terms before taking the squared norm.

    \textbf{The terminal Hessian.}
    Differentiating the regression score once more, we obtain
    \[
        \E_{y_t\mid y_0}[\nabla_{\theta\theta}^2\ln p(y_t\mid\bx;\theta_t,\nu)]
        =\bar\alpha_t\bx\bx^\top
        \bigl\{\E_{y_t\mid y_0}[y_t^2\sech^2(\mu_t y_t+\nu)]-1\bigr\}.
    \]
    For the mixed derivative, apply Gaussian integration by parts as above.
    For the \(\nu\nu\) derivative, apply the deterministic-shift and
    centered-Gaussian Taylor bounds to \(\sech^2\), whose first and second
    derivatives have absolute value at most \(2\).
    Since \(\E_{y_t\mid y_0}[y_t^2]=\bar\alpha_t y_0^2+1-\bar\alpha_t\),
    the three Hessian blocks satisfy
    \[
    \begin{aligned}
        \bigl\|\E_{y_t\mid y_0}[\nabla_{\theta\theta}^2\ln p(y_t\mid\bx;\theta_t,\nu)]\bigr\|_2
        &\leq\bar\alpha_t\|\bx\|^2(\bar\alpha_t y_0^2+2),\\
        \bigl\|\E_{y_t\mid y_0}[\nabla_{\theta\nu}^2\ln p(y_t\mid\bx;\theta_t,\nu)]\bigr\|
        &\leq\bar\alpha_t\|\bx\|(|y_0|+2|\mu_0|),\\
        \bigl|\E_{y_t\mid y_0}[\nabla_{\nu\nu}^2\ln p(y_t\mid\bx;\theta_t,\nu)]\bigr|
        &\leq2\bar\alpha_t|\mu_0y_0|
        +\bar\alpha_t(1-\bar\alpha_t)\mu_0^2.
    \end{aligned}
    \]
    The operator norm of a symmetric block matrix is at most the sum of
    the norms of its two diagonal blocks and its off-diagonal block.
    Using \(|\mu_0|\leq R_U\|\bx\|\) on \(U\), we therefore obtain
    \[
        \sup_{\vartheta\in U}\|\nabla^2\Delta_n(\vartheta)\|_2
        \leq\bar\alpha_{T_n}\E_{\Pr_n}\!\left[ \|\bx\|^2(y_0^2+2+2R_U+R_U^2)
        +(1+2R_U)\|\bx\||y_0|\right] =\calO_P(\bar\alpha_{T_n})=o_P(1).
    \]
    The bracketed envelope is integrable because
    \(\E[\|\bx\|^2y_0^2]\leq\E\|\bx\|^2+
    \|\theta^\ast\|^2\E\|\bx\|^4<\infty\).

    \textbf{Local likelihood curvature.}
    The conditional log-likelihood
    \(m_\vartheta(y_0,\bx)=\ln p(y_0\mid\bx;\vartheta)\) is constructed
    from \(\ln\cosh(\cdot)\) and linear maps
    (Lemma~\ref{lemma:nll}).
    Hence the map \(\vartheta\mapsto m_\vartheta(y_0,\bx)\) is smooth in
    \(\vartheta\equiv(\theta,\nu)\) for all \(y_0,\bx\).
    In particular, the derivative \(\nabla m_{\vartheta^\ast}\) with respect to \(\vartheta\) at \(\vartheta=\vartheta^\ast\) exists.
    The Hessian blocks of the conditional log-likelihood are
    \[
    \begin{aligned}
        \nabla_{\theta\theta}^2m_\vartheta
        &=\bx\bx^\top\bigl[y_0^2\sech^2(y_0\mu_0+\nu)-1\bigr],\\
        \nabla_{\theta\nu}^2m_\vartheta
        &=\bx y_0\sech^2(y_0\mu_0+\nu),\\
        \nabla_{\nu\nu}^2m_\vartheta
        &=\sech^2(y_0\mu_0+\nu)-\sech^2\nu.
    \end{aligned}
    \]
    We bound their norms uniformly over \(\vartheta\) by a linear
    combination of \(\|\bx\|^2(y_0^2+1)\),
    \(\|\bx\||y_0|\), and \(1\). Under the fourth-moment condition and the
    representation \(y_0=(-1)^{z+1}(\theta^\ast)^\top\bx+\varepsilon_0\),
    this envelope is integrable. Because the Hessian entries are continuous
    in \(\vartheta\), the Dominated Convergence Theorem
    \cite[Theorem~1.5.8]{durrett2019probability} justifies differentiating
    under the expectation twice and shows that the population criterion has a
    second-order Taylor expansion at \(\vartheta^\ast\), governed by
    \[
        \nabla^2\E_{y_0,\bx\mid\vartheta^\ast}[m_\vartheta]
        \big|_{\vartheta=\vartheta^\ast}
        =\E_{y_0,\bx\mid\vartheta^\ast}[\nabla^2m_{\vartheta^\ast}].
    \]

    For each fixed \(\bx\), the Gaussian mixture density and its first two
    parameter derivatives admit integrable bounds on bounded parameter
    neighborhoods. We may therefore differentiate
    \(\int p(y_0\mid\bx;\vartheta)\,dy_0=1\) twice.
    Substituting
    \(\nabla_\vartheta^2p=p(\nabla^2m_\vartheta+
    \nabla m_\vartheta\nabla m_\vartheta^\top)\), evaluating at
    \(\vartheta^\ast\), and taking expectation over \(\bx\), we obtain the
    information matrix equality
    \[
        I(\vartheta^\ast) :=\E_{y_0,\bx\mid\vartheta^\ast} [\nabla m_{\vartheta^\ast}\nabla
        m_{\vartheta^\ast}^\top] =-\E_{y_0,\bx\mid\vartheta^\ast} [\nabla^2m_{\vartheta^\ast}].
    \]
    Applying the weak law of large numbers
    \citep[Theorem~2.2.14]{durrett2019probability} entrywise to the
    integrable Hessian establishes its convergence at the ground truth:
    \[
        \nabla^2\widehat{\calR}_n^{\,\mathsf{MLE}}(\vartheta^\ast)
        =-\E_{\Pr_n}[\nabla^2m_{\vartheta^\ast}] \xrightarrow{P}
        -\E_{y_0,\bx\mid\vartheta^\ast}[\nabla^2m_{\vartheta^\ast}] =I(\vartheta^\ast).
    \]

    We verify nonsingularity through the score coordinates.
    For any vector \(\vartheta=(\theta,\nu)\in\R^d\times\R\), if
    \(\vartheta^\top I(\vartheta^\ast)\vartheta=0\), then the corresponding linear
    combination of the score coordinates vanishes almost surely.
    \[
        (\mu_0y_0+\nu)\tanh(\mu_0^\ast y_0+\nu^\ast)
        -\mu_0\mu_0^\ast-\nu\tanh\nu^\ast=0.
    \]
    For almost every \(\bx\), the conditional density
    \(y_0\mapsto p(y_0\mid\bx;\vartheta^\ast)\) is positive on \(\R\).
    The displayed expression is continuous in
    \(y_0\), so it vanishes for every \(y_0\). For every such \(\bx\) with
    \(\mu_0^\ast\ne0\), division by \(y_0\) followed by \(y_0\to+\infty\) forces
    \(\mu_0=0\), after which the nonconstancy of
    \(y_0\mapsto\tanh(\mu_0^\ast y_0+\nu^\ast)\) forces \(\nu=0\). Because
    \(\theta^\ast\ne\mathbf0\) and the density of \(\bx\) is
    positive almost everywhere, \(\mu_0^\ast\ne0\) almost surely, and
    \(\mu_0=\theta^\top\bx=0\) almost surely forces \(\theta=0\). Thus the Fisher information
    matrix \(I(\vartheta^\ast)\) is positive definite and nonsingular.

    We now control the likelihood Hessian along the random Taylor segment.
    On compact \(U\), continuity and the integrable Hessian envelope allow
    us to apply the uniform law of large numbers entrywise.
    With the population excess risk \(\calE\) defined at the start of the
    proof, \(\nabla^2\calE(\vartheta)
    =-\E_{y_0,\bx\mid\vartheta^\ast}[\nabla^2m_\vartheta]\), so
    \[
        \sup_{\vartheta\in U}
        \left\|\nabla^2\widehat{\calR}_n^{\,\mathsf{MLE}}(\vartheta)
        -\nabla^2\calE(\vartheta)\right\|_2=o_P(1).
    \]
    Let \(\lambda\in[0,1]\) parametrize this segment. By consistency from
    Step~2, the entire segment shrinks to the ground truth:
    \[
        \sup_{0\leq\lambda\leq1} \left\|(1-\lambda)\vartheta^\ast+
        \lambda\widetilde\vartheta_{n,T_n}^{\,\mathsf{SM}}-\vartheta^\ast\right\|
        =\left\|\widetilde\vartheta_{n,T_n}^{\,\mathsf{SM}} -\vartheta^\ast\right\|=o_P(1).
    \]
    Hence the entire segment lies in \(U\) with probability tending to one.
    By dominated convergence, \(\nabla^2\calE\) is continuous.
    Since the segment shrinks to \(\vartheta^\ast\) and
    \(\nabla^2\calE(\vartheta^\ast)=I(\vartheta^\ast)\), we obtain
    \[
        \sup_{0\leq\lambda\leq1}
        \left\|\nabla^2\calE
        \bigl((1-\lambda)\vartheta^\ast+\lambda\widetilde\vartheta_{n,T_n}^{\,\mathsf{SM}}\bigr)
        -I(\vartheta^\ast)\right\|_2=o_P(1).
    \]
    On the event that the segment lies in \(U\), adding and subtracting
    the population Hessian separates the sampling error from this variation:
    \[
    \begin{aligned}
        &\sup_{0\leq\lambda\leq1} \left\|\nabla^2\widehat{\calR}_n^{\,\mathsf{MLE}}
        \bigl((1-\lambda)\vartheta^\ast+\lambda\widetilde\vartheta_{n,T_n}^{\,\mathsf{SM}}\bigr)
        -I(\vartheta^\ast)\right\|_2\\
        &\leq\sup_{\vartheta\in U} \left\|\nabla^2\widehat{\calR}_n^{\,\mathsf{MLE}}(\vartheta)
        -\nabla^2\calE(\vartheta)\right\|_2 +\sup_{0\leq\lambda\leq1} \left\|\nabla^2\calE
        \bigl((1-\lambda)\vartheta^\ast+\lambda\widetilde\vartheta_{n,T_n}^{\,\mathsf{SM}}\bigr)
        -I(\vartheta^\ast)\right\|_2 =o_P(1)+o_P(1)=o_P(1).
    \end{aligned}
    \]
    Thus the empirical negative-log-likelihood Hessian converges uniformly
    to \(I(\vartheta^\ast)\) along this segment. Integrating over
    \(\lambda\in[0,1]\) preserves this limit.

    \textbf{The estimating equation.}
    Both terms of the empirical bridge are invariant under joint sign
    change. On the attainment event, the aligned estimator is therefore a
    finite unconstrained minimizer and satisfies
    \[
        \nabla\widehat{\calR}_{n,T_n}^{\,\mathsf{SM}}
        \bigl(\widetilde\vartheta_{n,T_n}^{\,\mathsf{SM}}\bigr)=\mathbf0.
    \]
    With probability tending to one, its segment from \(\vartheta^\ast\)
    lies inside \(U\). The likelihood-curvature convergence and terminal
    Hessian bound above show that the integral score matching Hessian along
    this segment converges to \(I(\vartheta^\ast)\), which is positive definite.
    We can therefore invert this integral Hessian to solve the Taylor
    expansion of the estimating equation for the parameter error:
    \[
        \sqrt n\bigl(\widetilde\vartheta_{n,T_n}^{\,\mathsf{SM}}-\vartheta^\ast\bigr)
        =\left[\int_0^1\nabla^2\widehat{\calR}_{n,T_n}^{\,\mathsf{SM}}
        \bigl((1-\lambda)\vartheta^\ast+\lambda\widetilde\vartheta_{n,T_n}^{\,\mathsf{SM}}\bigr)\,\mathrm{d}\lambda\right]^{-1}
        \cdot\left[\sqrt n\,\E_{\Pr_n}[\nabla m_{\vartheta^\ast}] +\sqrt
        n\,\nabla\Delta_n(\vartheta^\ast)\right].
    \]
    The observations are i.i.d., and their likelihood scores have mean zero
    and finite covariance \(I(\vartheta^\ast)\). The multivariate central
    limit theorem \citep[Theorem~3.10.7]{durrett2019probability} therefore applies:
    \[
        \sqrt n\,\E_{\Pr_n}[\nabla m_{\vartheta^\ast}]
        =-\sqrt n\,\nabla\widehat{\calR}_n^{\,\mathsf{MLE}}(\vartheta^\ast)
        \xrightarrow{\mathsf d}\mathcal N(\mathbf0,I(\vartheta^\ast)).
    \]
    Since \(\sqrt n\,\nabla\Delta_n(\vartheta^\ast)
    =\calO_P(\bar\alpha_{T_n})=o_P(1)\), the terminal correction vanishes.
    We conclude by Slutsky's theorem that
    \[
        \sqrt n\left(\widetilde\vartheta_{n,T_n}^{\,\mathsf{SM}} -\vartheta^\ast\right)
        =I(\vartheta^\ast)^{-1}\sqrt n\,\E_{\Pr_n}[\nabla m_{\vartheta^\ast}] +o_P(1)
        \xrightarrow{\mathsf d}\mathcal N(\mathbf0,I(\vartheta^\ast)^{-1}).
    \]
    The nonattainment event has probability tending to zero, so this
    conclusion also holds under the estimator's fallback convention.
\end{proof}

\begin{remark}[The objective-value transfer argument]
\label{rmk:objective_value_transfer}
Under the fourth-moment assumptions of
Theorem~\ref{thm:consistency_asymptotic_normality}, the likelihood
near-optimality argument also proves asymptotic normality under the stronger condition
\(\bar{\alpha}_{T_n}=o(n^{-1})\). \hyperref[proof:consistency_asymptotic_normality_step_1]{Steps~1} and~\hyperref[proof:consistency_asymptotic_normality_step_2]{2} of the proof of Main Theorem~\ref{thm:consistency_asymptotic_normality_appendix} bound the likelihood gap as follows:
\[
    D_n=\widehat{\calR}_n^{\,\mathsf{MLE}}
    (\widetilde\vartheta_{n,T_n}^{\,\mathsf{SM}})
    -\inf_\vartheta\widehat{\calR}_n^{\,\mathsf{MLE}}(\vartheta) 
    \leq \Delta_n(\widetilde{\vartheta}_{n,T_n}^{\,\mathsf{SM}}) = \mathcal{O}_P(\bar{\alpha}_{T_n})
    =o_P(n^{-1}).
\]
Joint-sign alignment leaves the likelihood unchanged. On the attainment
event \(\mathcal A_n\) defined in \hyperref[proof:consistency_asymptotic_normality_step_2]{Step~2} of the proof of Main Theorem~\ref{thm:consistency_asymptotic_normality_appendix},
\(0\leq D_n\leq\Delta_n(\widehat\vartheta_{n,T_n}^{\,\mathsf{SM}})\).
Thus, for every \(\epsilon>0\),
\[
    \Pr(nD_n>\epsilon)
    \leq\Pr(\mathcal A_n^c)
       +\Pr\!\left(n\Delta_n(\widehat\vartheta_{n,T_n}^{\,\mathsf{SM}})>\epsilon\right)
    \longrightarrow0,
\]
which justifies the displayed rate under the fallback convention as well.

Together with consistency, this \(D_n = o_P(n^{-1})\) is the likelihood near-optimality condition
in Theorem~5.23 of \emph{Asymptotic Statistics} by
\citet[p.~53]{van2000asymptotic}. \hyperref[proof:consistency_asymptotic_normality_step_3]{Step~3} of the proof of Main Theorem~\ref{thm:consistency_asymptotic_normality_appendix} verifies the
smoothness, quadratic expansion, and nonsingular Fisher information required
by Theorem~5.23. Its remaining local Lipschitz condition follows directly
from the likelihood score. For a bounded convex neighborhood \(U\) of
\(\vartheta^\ast\), write
\(M:=\sup_{(\theta,\nu)\in U}\|\theta\|\) and
\(\dot m(y_0,\bx):=\sup_{\vartheta\in U}\|\nabla m_\vartheta(y_0,\bx)\|\).
By the mean value theorem,
\[
    |m_{\vartheta_1}-m_{\vartheta_2}|
    \leq\dot m\|\vartheta_1-\vartheta_2\|,
    \qquad \vartheta_1,\vartheta_2\in U.
\]
The score formulas and
\(|y_0|\leq\|\theta^\ast\|\|\bx\|+|\varepsilon_0|\) imply
\[
    \dot m(y_0,\bx) \leq\|\bx\|(|y_0|+M\|\bx\|)+2 \leq\left(M+\|\theta^\ast\|+\tfrac12\right)\|\bx\|^2
    +\tfrac12\varepsilon_0^2+2.
\]
The last inequality follows from \(uv\leq(u^2+v^2)/2\). Since the Gaussian
observation noise has a finite fourth moment, \(\E\dot m^2<\infty\)
under \(\E\|\bx\|^4<\infty\).
We apply the cited theorem with population log-likelihood Hessian
\(-I(\vartheta^\ast)\). The asymptotic covariance is then
\(I(\vartheta^\ast)^{-1}I(\vartheta^\ast)I(\vartheta^\ast)^{-1}
=I(\vartheta^\ast)^{-1}\).

This argument transfers efficiency by bounding objective values and is the
route closest to the likelihood-bridge analysis in the proof of
Theorem~4.1 of \citet{chewi2025arxiv}.
The stronger terminal condition makes the value bound
\(\Delta_n(\widehat\vartheta_{n,T_n}^{\,\mathsf{SM}})=\calO_P(\bar{\alpha}_{T_n})\)
sufficient for an \(o_P(n^{-1})\) likelihood gap. The centered-gradient
argument in \hyperref[proof:consistency_asymptotic_normality_step_3]{Step~3} of the proof of Main Theorem~\ref{thm:consistency_asymptotic_normality_appendix} instead establishes
\(\sqrt n\,\nabla\Delta_n(\vartheta^\ast)=\calO_P(\bar{\alpha}_{T_n})\), and therefore
proves the theorem for every deterministic \(\bar{\alpha}_{T_n}\to0\).
\end{remark}

\begin{lemma}[High-probability attainment of the empirical minimum]
\label{lemma:sm_minimum_attainment}
    The dimension \(d\) is fixed, and the observations are i.i.d. from the
    MLR model with \(\theta^\ast\ne\mathbf0\) and
    \(\nu^\ast\in\R\). The covariate vector \(\bx\) has a density that is positive
    almost everywhere on \(\R^d\) and satisfies \(\E\|\bx\|^2<\infty\).
    The terminal horizon satisfies
    \(0<\bar\alpha_{T_n}<1\) and \(\bar\alpha_{T_n}\to0\). Then the probability
    that the empirical score matching objective attains its minimum over
    \(\R^{d+1}\) tends to one. A sample-measurable exact minimizer can be
    selected whenever this minimum exists. The same attainment and
    measurability conclusions hold for the empirical negative log-likelihood.
\end{lemma}

\begin{proof}
    We use \(\Sigma_n\) from the proof of
    Main Theorem~\ref{thm:consistency_asymptotic_normality_appendix} and
    the response notation of Section~\ref{subsec:2mlr}. In particular,
    \[
        \mu_0=\langle\theta,\bx\rangle,\qquad
        \mu_{T_n}y_{T_n}
        =\bar\alpha_{T_n}\mu_0y_0+
          \sqrt{\bar\alpha_{T_n}(1-\bar\alpha_{T_n})}\,\mu_0\xi_{T_n}.
    \]
    For the asymptotic attainment claim, it suffices to consider \(n\geq d\),
    when the covariate-density assumption makes \(\Sigma_n\) positive
    definite almost surely. We work on this event until the measurability
    argument below.
    Using the empirical bridge in
    Lemma~\ref{lemma:log_likelihood_SM_loss_finite} and the mixture formula
    in Lemma~\ref{lemma:nll}, we express the criterion after subtracting
    its value at \((\mathbf0,0)\) as
    \[
    \begin{aligned}
        \mathsf G_n(\theta,\nu) &:=\widehat{\calR}_{n,T_n}^{\,\mathsf{SM}}(\theta,\nu)
        -\widehat{\calR}_{n,T_n}^{\,\mathsf{SM}}(\mathbf0,0)\\
        &=\frac{1-\bar\alpha_{T_n}}2\theta^\top\Sigma_n\theta -\E_{\Pr_n}\ln\cosh(\mu_0y_0+\nu)
        +\E_{\Pr_n}\E_{y_{T_n}\mid y_0} \ln\cosh(\mu_{T_n}y_{T_n}+\nu).
    \end{aligned}
    \]

    To include the boundary limits \(\nu=\pm\infty\), we work on the
    extended set \(\overline{\R}:=\R\cup\{-\infty,+\infty\}\) from
    \hyperref[proof:consistency_asymptotic_normality_step_2]{Step~2}
    of the proof of Main Theorem~\ref{thm:consistency_asymptotic_normality_appendix},
    with \(\tanh(\pm\infty)=\pm1\). For \(\nu\in\overline{\R}\), define
    \(\psi_\nu(z):=\ln(\cosh z+\tanh\nu\,\sinh z)\).
    The logarithm's argument is a convex combination of \(e^z\) and
    \(e^{-z}\), so it is strictly positive, including at the endpoints.
    In particular, \(\psi_{+\infty}(z)=z\) and
    \(\psi_{-\infty}(z)=-z\).
    For finite \(\nu\),
    \(\ln\cosh(\nu+z)=\ln\cosh\nu+\psi_\nu(z)\).
    Substituting this identity into \(\mathsf G_n\), the two
    \(\ln\cosh\nu\) terms cancel. We therefore define its extension to
    \(\nu\in\overline{\R}\) by
    \[
        \overline{\mathsf G}_n(\theta,\nu) :=\frac{1-\bar\alpha_{T_n}}2\theta^\top\Sigma_n\theta
        -\E_{\Pr_n}\psi_\nu(\mu_0y_0) +\E_{\Pr_n}\E_{y_{T_n}\mid y_0} \psi_\nu(\mu_{T_n}y_{T_n}).
    \]
    Thus \(\overline{\mathsf G}_n\) agrees with \(\mathsf G_n\) for
    finite \(\nu\). To verify continuity at the boundary, we must also
    pass limits through the Gaussian expectation. The explicit formula for
    \(\psi_\nu(z)\) is continuous in \((z,\tanh\nu)\).
    For finite \(\nu\), its derivative with respect to \(z\) is
    \(\psi_\nu'(z)=\tanh(z+\nu)\), while the
    endpoint derivatives are \(\pm1\). Thus \(\psi_\nu\) is \(1\)-Lipschitz
    and, since \(\psi_\nu(0)=0\), satisfies \(|\psi_\nu(z)|\leq|z|\).
    Substituting the Gaussian representation above into this bound, we obtain
    \[
        |\psi_\nu(\mu_{T_n}y_{T_n})|
        \leq|\mu_0|\left(\bar\alpha_{T_n}|y_0|
        +\sqrt{\bar\alpha_{T_n}(1-\bar\alpha_{T_n})}\,|\xi_{T_n}|\right).
    \]
    For each observed pair, the right-hand side has an integrable bound
    uniform over bounded \(\theta\)-sets because \(\E|\xi_{T_n}|<\infty\).
    Dominated convergence therefore justifies passing parameter limits
    through the noise expectation. Hence
    \(\overline{\mathsf G}_n\) is finite and jointly continuous on
    \(\R^d\times\overline{\R}\).

    To prove coercivity, we bound the difference between the two
    \(\psi_\nu\) terms before averaging over the observations.
    Applying the \(1\)-Lipschitz bound to the Gaussian representation,
    we obtain
    \[
        \left|\E_{y_{T_n}\mid y_0}\psi_\nu(\mu_{T_n}y_{T_n})
        -\psi_\nu(\mu_0y_0)\right|
        \leq|\mu_0|\left((1-\bar\alpha_{T_n})|y_0|
        +\sqrt{\bar\alpha_{T_n}(1-\bar\alpha_{T_n})}\,\E|\xi_{T_n}|\right).
    \]
    We bound the two empirical terms separately. Since
    \(\E_{\Pr_n}[\mu_0^2]=\|\theta\|_{\Sigma_n}^2\),
    we apply the Cauchy--Schwarz inequality to each term:
    \[
        \E_{\Pr_n}|\mu_0y_0|
        \leq\|\theta\|_{\Sigma_n}\sqrt{\E_{\Pr_n}[y_0^2]},
        \qquad
        \bigl(\E_{\Pr_n}|\mu_0|\bigr)^2
        \leq\E_{\Pr_n}[\mu_0^2]=\|\theta\|_{\Sigma_n}^2.
    \]
    We therefore define
    \[
        L_n:=(1-\bar\alpha_{T_n})\sqrt{\E_{\Pr_n}[y_0^2]}
        +\sqrt{\bar\alpha_{T_n}(1-\bar\alpha_{T_n})}\,
          \E|\xi_{T_n}|.
    \]
    The empirical average of the preceding Lipschitz bound is then at
    most \(L_n\|\theta\|_{\Sigma_n}\).
    Subtracting this bound from the quadratic term, we obtain
    \[
        \overline{\mathsf G}_n(\theta,\nu)\geq\frac{1-\bar{\alpha}_{T_n}}{2}\|\theta\|_{\Sigma_n}^2-L_n\|\theta\|_{\Sigma_n}.
    \]
    Since \(\Sigma_n\) is positive definite, this bound makes
    \(\overline{\mathsf G}_n\) coercive in \(\theta\), uniformly over
    \(\nu\in\overline{\R}\). By continuity and compactness of
    \(\overline{\R}\), its argmin on \(\R^d\times\overline{\R}\)
    is nonempty and compact.

    We next compare the boundary values with the ground truth value.
    Using \(\psi_{\pm\infty}(z)=\pm z\) and \(\E[\xi_{T_n}]=0\),
    we evaluate the criterion at \(\nu=\pm\infty\).
    We complete the square in \(\theta\) to compute the common minimum
    at the two boundary points:
    \[
    \begin{aligned}
        \overline{\mathsf G}_n(\theta,\pm\infty)
        &=(1-\bar\alpha_{T_n})\left(
           \frac12\theta^\top\Sigma_n\theta
           \mp\theta^\top\E_{\Pr_n}[\bx y_0]\right),\\
        b_n:=\min_\theta\overline{\mathsf G}_n(\theta,\pm\infty)
        &=-\frac{1-\bar\alpha_{T_n}}2
           \E_{\Pr_n}[\bx y_0]^\top\Sigma_n^{-1}\E_{\Pr_n}[\bx y_0].
    \end{aligned}
    \]
    To evaluate these limits, we begin with the model's conditional-mean identity
    \(\E[y_0\mid\bx]=\mu_0^\ast\tanh\nu^\ast\).
    Multiplying by \(\bx\) and taking expectation, we find
    \(\E[\bx y_0]=\Sigma\theta^\ast\tanh\nu^\ast\).
    Finite second moments ensure that both \(\bx\bx^\top\) and
    \(\bx y_0\) are integrable, so the weak law applies to their empirical
    averages. Continuity of matrix inversion at \(\Sigma\succ0\) and
    \(\bar\alpha_{T_n}\to0\) then allow substitution into the expression
    for \(b_n\). Thus
    \[
        \Sigma_n\xrightarrow{\Pr}\Sigma,\qquad \E_{\Pr_n}[\bx y_0]\xrightarrow{\Pr}
        \Sigma\theta^\ast\tanh\nu^\ast, \qquad b_n\xrightarrow{\Pr}b_\infty
        :=-\tfrac12\tanh^2(\nu^\ast)(\theta^\ast)^\top\Sigma\theta^\ast.
    \]

    We now evaluate the criterion at the ground truth.
    The empirical moments \(\E_{\Pr_n}|\mu_0^\ast y_0|\) and
    \(\E_{\Pr_n}|\mu_0^\ast|\) are \(\calO_P(1)\) by the weak law.
    Using the \(1\)-Lipschitz property of \(\ln\cosh\) in the
    Gaussian representation, we bound the terminal term as follows,
    with the final convergence following from \(\bar\alpha_{T_n}\to0\):
    \[
        \left|\E_{\Pr_n}\E_{y_{T_n}\mid y_0} \ln\cosh(\mu_{T_n}^\ast y_{T_n}+\nu^\ast)
        -\ln\cosh\nu^\ast\right| \leq\bar\alpha_{T_n}\E_{\Pr_n}|\mu_0^\ast y_0|
        +\sqrt{\bar\alpha_{T_n}(1-\bar\alpha_{T_n})}\, \E|\xi_{T_n}|\E_{\Pr_n}|\mu_0^\ast| =o_P(1).
    \]
    Thus the terminal term equals \(\ln\cosh\nu^\ast+o_P(1)\).
    Also, \(\bar\alpha_{T_n}\E_{\Pr_n}[(\mu_0^\ast)^2]=o_P(1)\), so
    the coefficient of the quadratic term can be replaced by \(1/2\).
    Substituting these two approximations into \(\mathsf G_n\), we obtain
    the empirical average below. The bound
    \(|\psi_{\nu^\ast}(\mu_0^\ast y_0)|\leq|\mu_0^\ast y_0|\)
    ensures integrability, allowing us to apply the weak law:
    \[
        \mathsf G_n(\theta^\ast,\nu^\ast)
        =\E_{\Pr_n}\!\left[\frac12(\mu_0^\ast)^2-\psi_{\nu^\ast}(\mu_0^\ast y_0)\right]+o_P(1)
        \xrightarrow{\Pr}g_\ast:=\E\!\left[\frac12(\mu_0^\ast)^2-\psi_{\nu^\ast}(\mu_0^\ast y_0)\right].
    \]
    To determine the sign of \(b_\infty-g_\ast\), we compare the ground
    truth conditional density with the unit-variance Gaussian having the
    same conditional mean. Substituting the log-density formula from
    Lemma~\ref{lemma:nll} into their log-density ratio, we find
    \[
        \ln\frac{p(y_0\mid\bx;\theta^\ast,\nu^\ast)}
        {\calN(y_0;\tanh\nu^\ast\mu_0^\ast,1)}
        =\psi_{\nu^\ast}(\mu_0^\ast y_0)-\frac12(\mu_0^\ast)^2
        -\tanh\nu^\ast\mu_0^\ast y_0
        +\frac12\tanh^2(\nu^\ast)(\mu_0^\ast)^2.
    \]
    Taking expectation under the ground truth and using
    \(\E[y_0\mid\bx]=\tanh\nu^\ast\mu_0^\ast\), we obtain
    \(-g_\ast\) from the first two terms and \(b_\infty\) from the last two.
    Therefore,
    \[
        b_\infty-g_\ast=\E_\bx\KL\!\left(p(\,\cdot\mid\bx;\theta^\ast,\nu^\ast)\parallel\calN(\tanh\nu^\ast\mu_0^\ast,1)\right)>0.
    \]
    The inequality is strict because the ground truth conditional variance is
    \(1+\sech^2(\nu^\ast)(\mu_0^\ast)^2>1\) for almost every \(\bx\), whereas
    the boundary Gaussian has variance one. The two conditional laws are
    therefore distinct. Since both empirical criterion values converge
    to their respective limits, this positive gap persists with probability
    tending to one:
    \[
        b_n-\mathsf G_n(\theta^\ast,\nu^\ast)
        \xrightarrow{\Pr}b_\infty-g_\ast>0,
        \qquad \Pr\{\mathsf G_n(\theta^\ast,\nu^\ast)<b_n\}\longrightarrow1.
    \]
    On this event, the attained minimum over
    \(\R^d\times\overline{\R}\) lies strictly below every boundary value:
    \[
        \min\overline{\mathsf G}_n
        \leq\mathsf G_n(\theta^\ast,\nu^\ast)<b_n
        \leq\overline{\mathsf G}_n(\theta,\pm\infty)
        \qquad(\theta\in\R^d).
    \]
    Every minimizer of \(\overline{\mathsf G}_n\) must therefore have
    finite \(\nu\).
    At finite parameters, \(\overline{\mathsf G}_n=\mathsf G_n\), and
    centering the original objective does not change its minimizers. Hence
    \[
        \Pr\!\left\{\argmin_{\vartheta\in\R^{d+1}}\widehat{\calR}_{n,T_n}^{\,\mathsf{SM}}(\vartheta)\ne\varnothing\right\}\longrightarrow1.
    \]

    For measurability, \(\mathsf G_n\) is measurable in the sample and
    continuous in the parameters, so the measurable minimum theorem applies
    on each cube \([-k,k]^{d+1}\), \(k\geq1\).
    The unrestricted infimum is the infimum of the resulting measurable
    minimum values. It is attained exactly when one of these values equals
    it, a measurable event. We select a measurable minimizer from the first
    such cube and set the estimator to \((\mathbf0,0)\) otherwise.

    Setting \(\bar\alpha_{T_n}=0\) in the expression for \(\mathsf G_n\)
    recovers the centered empirical negative log-likelihood.
    The same coercivity, boundary-gap, and measurable-selection arguments
    establish the corresponding assertions for the MLE.
\end{proof}

\begin{remark}[Consistency and \(\sqrt{n}\)-consistency for terminal limits below one]
\label{rmk:root_n_terminal_limit}
Assume the model and covariate conditions of
Main Theorem~\ref{thm:consistency_asymptotic_normality}, including
\(\E\|\bx\|^2<\infty\), and let
\(0<\bar\alpha_{T_n}<1\) with
\(\bar\alpha_{T_n}\to a\in[0,1)\) deterministically.
The empirical score matching minimum is attained with probability tending
to one. Every measurable exact global minimizer, with the fallback
convention in Section~\ref{subsec:notations} and joint-sign alignment,
satisfies
\[
    \widetilde\vartheta_{n,T_n}^{\,\mathsf{SM}}
          -\vartheta^\ast=o_P(1).
\]
If additionally \(\E\|\bx\|^4<\infty\), then
\[
    \bigl\|\widetilde\vartheta_{n,T_n}^{\,\mathsf{SM}}
          -\vartheta^\ast\bigr\|=\calO_P(n^{-1/2}).
\]
The case \(a=0\) follows from
Main Theorem~\ref{thm:consistency_asymptotic_normality_appendix}.
For \(a\in(0,1)\), we first prove consistency under the second-moment
condition, using the notation of the preceding proof.

For \(n\geq d\), the coercivity and continuity argument in the
preceding proof ensures that \(\overline{\mathsf G}_n\) attains its
minimum over \(\R^d\times\overline{\R}\) almost surely.
To bound the regression parameters of all its minimizers, we combine
the same coercivity inequality with
\(\overline{\mathsf G}_n(\mathbf0,\nu)=0\).
For every \((\theta,\nu)\in\argmin\overline{\mathsf G}_n\),
\[
\begin{aligned}
    \frac{1-\bar\alpha_{T_n}}{2}\|\theta\|_{\Sigma_n}^2
       -L_n\|\theta\|_{\Sigma_n}
    &\leq\overline{\mathsf G}_n(\theta,\nu)\leq0,\\
    L_n
    &=(1-\bar\alpha_{T_n})\sqrt{\E_{\Pr_n}[y_0^2]}
       +\sqrt{\bar\alpha_{T_n}(1-\bar\alpha_{T_n})}\,
          \E|\xi_{T_n}|=\calO_P(1),\\
    \|\theta\|
    &\leq\frac{2L_n}
       {(1-\bar\alpha_{T_n})\sqrt{\lambda_{\min}(\Sigma_n)}}
       =\calO_P(1).
\end{aligned}
\]
Here \(\Sigma_n\xrightarrow{P}\Sigma\succ0\) and
\(1-\bar\alpha_{T_n}\to1-a>0\).

To compare the empirical and population criteria, we first write
\(\overline{\mathsf G}_n\) as an average of contributions from individual
observations. Using \(\E_{\Pr_n}[\mu_0^2]=\theta^\top\Sigma_n\theta\)
and the Gaussian representation of \(\mu_{T_n}y_{T_n}\), we have
\[
    \overline{\mathsf G}_n(\theta,\nu)
    =\E_{\Pr_n}\!\left[
       \frac{1-\bar\alpha_{T_n}}2\mu_0^2-\psi_\nu(\mu_0y_0)
       +\E_\xi\psi_\nu\!\left(\bar\alpha_{T_n}\mu_0y_0
          +\sqrt{\bar\alpha_{T_n}(1-\bar\alpha_{T_n})}\,\mu_0\xi\right)
       \right],
\]
where \(\xi\sim\mathcal N(0,1)\) is independent of \((y_0,\bx)\).
Here \(\E_\xi\) averages the diffusion noise for each observed pair,
and \(\E_{\Pr_n}\) then averages over the \(n\) observed pairs.
Taking expectation over the i.i.d. sample replaces \(\E_{\Pr_n}\)
by the ground truth expectation \(\E\), while keeping the deterministic
level \(\bar\alpha_{T_n}\) unchanged. Thus
\(\E\overline{\mathsf G}_n\) is the population criterion at that level.
Replacing the level by its limit \(a\), and using
\(\E[\mu_0^2]=\theta^\top\Sigma\theta\), we define
\[
    \overline{\mathsf G}_a(\theta,\nu) := {}\frac{1-a}{2}\theta^\top\Sigma\theta
    -\E[\psi_\nu(\mu_0y_0)] +\E\E_\xi\!\left[ \psi_\nu\!\left(a\mu_0y_0+
    \sqrt{a(1-a)}\,\mu_0\xi\right)\right].
\]

Because the terminal level changes with \(n\), we need an integrable
bound independent of \(\bar\alpha_{T_n}\).
The bound \(|\psi_\nu(z)|\leq|z|\) controls both terms involving
\(\psi_\nu\), including the Gaussian expectation.
For \(\|\theta\|\leq R\), we then combine the bounds
\(|\mu_0|\leq R\|\bx\|\), \(1+\bar\alpha_{T_n}\leq2\), and
\(\sqrt{\bar\alpha_{T_n}(1-\bar\alpha_{T_n})}\leq1/2\) to obtain
\[
\begin{aligned}
    &\left|\frac{1-\bar\alpha_{T_n}}{2}\mu_0^2-\psi_\nu(\mu_0y_0)
       +\E_\xi\psi_\nu\!\left(\bar\alpha_{T_n}\mu_0y_0
          +\sqrt{\bar\alpha_{T_n}(1-\bar\alpha_{T_n})}\,\mu_0\xi\right)\right|\\
    &\qquad\leq\frac{1-\bar\alpha_{T_n}}{2}\mu_0^2
       +(1+\bar\alpha_{T_n})|\mu_0y_0|
       +\sqrt{\bar\alpha_{T_n}(1-\bar\alpha_{T_n})}\,|\mu_0|\E|\xi|\\
    &\qquad\leq\frac{R^2}{2}\|\bx\|^2+2R|y_0|\|\bx\|
       +\frac R2\E|\xi|\|\bx\|.
\end{aligned}
\]
The right-hand side is a common integrable envelope because
\(\E\|\bx\|^2<\infty\) and
\(\E[y_0^2]=1+(\theta^\ast)^\top\Sigma\theta^\ast<\infty\).
The mixed term is integrable as well, since
\(\E[|y_0|\|\bx\|]\leq
(\E[y_0^2]\E\|\bx\|^2)^{1/2}<\infty\).

To apply a uniform law of large numbers despite the changing terminal
level, we consider the family obtained by allowing the retained-signal
argument in this observation contribution to range over \([0,1]\).
For fixed \(R\), its parameters \(\theta\), \(\tanh\nu\), and
retained-signal level range over the compact set
\(\{\theta\in\R^d:\|\theta\|\leq R\}\times[-1,1]\times[0,1]\).
For each \((y_0,\bx)\), continuity of \(\psi_\nu\) and dominated
convergence over \(\xi\) ensure that the observation contribution is
jointly continuous on this set, including \(\nu=\pm\infty\).
The same integrable envelope bounds every member of the family. The uniform
law of large numbers for a continuous family on a compact index set
therefore applies. In particular, evaluating at each deterministic
\(\bar\alpha_{T_n}\) and using the empirical-average identity above,
we obtain
\[
    \sup_{\|\theta\|\leq R,\,\nu\in\overline{\R}}
       \bigl|\overline{\mathsf G}_n(\theta,\nu)
             -\E\overline{\mathsf G}_n(\theta,\nu)\bigr|=o_P(1).
\]

It remains to compare the population criteria at
\(\bar\alpha_{T_n}\) and \(a\). Dominated convergence with the same
envelope ensures joint continuity of the population criterion in
\(\theta\), \(\tanh\nu\), and the retained-signal level. Compactness makes this continuity
uniform, so changing the level from \(\bar\alpha_{T_n}\) to \(a\)
changes the population criterion uniformly by \(o(1)\).
This is a deterministic error, whereas the preceding \(o_P(1)\)
term measures the sampling fluctuations. Combining them, we obtain
\[
\begin{aligned}
    \sup_{\|\theta\|\leq R,\,\nu\in\overline{\R}}
    \bigl|\overline{\mathsf G}_n(\theta,\nu)
          -\overline{\mathsf G}_a(\theta,\nu)\bigr|
    &\leq\sup_{\|\theta\|\leq R,\,\nu\in\overline{\R}}
       \bigl|\overline{\mathsf G}_n(\theta,\nu)
             -\E\overline{\mathsf G}_n(\theta,\nu)\bigr|\\
    &\quad+\sup_{\|\theta\|\leq R,\,\nu\in\overline{\R}}
       \bigl|\E\overline{\mathsf G}_n(\theta,\nu)
             -\overline{\mathsf G}_a(\theta,\nu)\bigr|
     =o_P(1)+o(1)=o_P(1).
\end{aligned}
\]

To determine the population minimizers, we subtract the criterion's value
at the ground truth. For finite \(\nu\), we use the mixture log-density
formula in Lemma~\ref{lemma:nll} to express this difference as the
initial KL divergence minus the terminal KL divergence at level \(a\).
We then apply the population bridge in
Eq.~\eqref{eq:population_sm_mle_bridge} to represent it as an integral
of the nonnegative score matching loss:
\[
    \overline{\mathsf G}_a(\theta,\nu) -\overline{\mathsf G}_a(\theta^\ast,\nu^\ast)
    =\int_a^1\calL_t(\theta,\nu)\,\mathrm d\bar\alpha_t\geq0.
\]
To include \(\nu=\pm\infty\), we need a bound on the integrand that
is independent of \(\nu\). The common \(-y_t\) term in
Eq.~\eqref{eq:score} cancels from the score difference.
Using \(|\tanh|\leq1\) and
\(\mu_t=\sqrt{\bar\alpha_t}\,\theta^\top\bx\), we obtain
\[
\begin{aligned}
    \left|s_{\theta_t,\nu}(y_t,\bx)
          -s_{\theta_t^\ast,\nu^\ast}(y_t,\bx)\right|
    &=\left|\mu_t\tanh(\mu_ty_t+\nu)
          -\mu_t^\ast\tanh(\mu_t^\ast y_t+\nu^\ast)\right|\\
    &\leq|\mu_t|+|\mu_t^\ast|
     \leq\sqrt{\bar\alpha_t}\,
          (\|\theta\|+\|\theta^\ast\|)\|\bx\|.
\end{aligned}
\]
We square this bound and apply the loss definition in
Eq.~\eqref{eq:loss}. The factor \(\bar\alpha_t\) then cancels against
the normalization \(1/(2\bar\alpha_t)\), so
\[
    \calL_t(\theta,\nu)
    \leq\frac{1}{2\bar\alpha_t}\E\bigl[(|\mu_t|+|\mu_t^\ast|)^2\bigr]
    \leq\tfrac12(\|\theta\|+\|\theta^\ast\|)^2\E\|\bx\|^2.
\]
This finite bound is uniform in \(\nu\) and the retained-signal level.
Together with continuity of \(\overline{\mathsf G}_a\), it allows us to
extend the identity to \(\nu=\pm\infty\) by dominated convergence.

We can now determine when the integral vanishes. Since \(\calL_t\)
is an expected squared score difference, a zero integral forces
\(\calL_t=0\) at almost every retained-signal level in \((a,1)\).
Choose one such level. The positive ground truth response density and
continuity of the scores extend their equality to every response,
for almost every \(\bx\). The conditional
log-densities therefore differ by a constant in the response, and
normalization forces this constant to be zero.
The conditional densities thus have the same first two raw moments.
These moments are \(\sqrt{\bar\alpha_t}\tanh\nu\,\theta^\top\bx\)
and \(1+\bar\alpha_t(\theta^\top\bx)^2\) for the candidate model.
Equating them with the ground truth moments and canceling the positive
scale factors, we obtain, for almost every \(\bx\),
\[
\begin{aligned}
    (\theta^\top\bx)^2&=((\theta^\ast)^\top\bx)^2,
    &\tanh\nu\,\theta^\top\bx
       &=\tanh\nu^\ast\,(\theta^\ast)^\top\bx.
\end{aligned}
\]
By the argument in
\hyperref[proof:consistency_asymptotic_normality_step_2]{Step~2}
of the proof of Main Theorem~\ref{thm:consistency_asymptotic_normality_appendix},
the first equality and the positive covariate density force
\(\theta=s\theta^\ast\) for some \(s\in\{+1,-1\}\).
Since \(\theta^\ast\ne\mathbf0\), the second equality then forces
\(\tanh\nu=s\tanh\nu^\ast\), hence \(\nu=s\nu^\ast\).
Conversely, both joint sign choices preserve the conditional model and
attain zero loss. Therefore,
\[
    \argmin\overline{\mathsf G}_a
    =\bigl\{s(\theta^\ast,\nu^\ast):s\in\{+1,-1\}\bigr\}
    \subset\R^{d+1}.
\]

We now turn uniform convergence of the criteria into parameter
consistency. To exclude empirical minimizers far from both ground truth
points, we first establish a positive population gap there.
We measure distance in \((\theta,\tanh\nu)\), where even the boundary
values \(\nu=\pm\infty\) have finite coordinates.
For fixed \(R>\|\theta^\ast\|\) and \(\epsilon>0\), the set with
\(\|\theta\|\leq R\) at distance at least \(\epsilon\) from both
population minimizers is compact. On this set the population gap is
continuous and strictly positive. If the set is nonempty, the extreme
value theorem ensures that this gap attains a positive minimum:
\[
    \inf_{\substack{\|\theta\|\leq R,\,\nu\in\overline{\R}\\
       \min_{s\in\{+1,-1\}}
       \|(\theta,\tanh\nu)-s(\theta^\ast,\tanh\nu^\ast)\|\geq\epsilon}}
    \bigl[\overline{\mathsf G}_a(\theta,\nu)
          -\overline{\mathsf G}_a(\theta^\ast,\nu^\ast)\bigr]>0.
\]
Every minimizer of \(\overline{\mathsf G}_n\) with
\(\|\theta\|\leq R\) satisfies
\(\overline{\mathsf G}_n(\theta,\nu)
\leq\overline{\mathsf G}_n(\theta^\ast,\nu^\ast)\).
Subtracting this nonpositive empirical difference from the population
gap leaves two approximation errors. We bound each by the uniform
convergence estimate:
\[
\begin{aligned}
    0\leq\overline{\mathsf G}_a(\theta,\nu)
       -\overline{\mathsf G}_a(\theta^\ast,\nu^\ast)
    &\leq[\overline{\mathsf G}_a-\overline{\mathsf G}_n](\theta,\nu)
       +[\overline{\mathsf G}_n-\overline{\mathsf G}_a](\theta^\ast,\nu^\ast)\\
    &\leq2\sup_{\|\theta'\|\leq R,\,\nu'\in\overline{\R}}
       \bigl|\overline{\mathsf G}_n(\theta',\nu')
             -\overline{\mathsf G}_a(\theta',\nu')\bigr|=o_P(1).
\end{aligned}
\]
For fixed \(R\) and \(\epsilon\), this uniform \(o_P(1)\) upper
bound excludes every empirical minimizer in the separated set with
probability tending to one. The earlier common \(\calO_P(1)\) bound
on their regression parameters makes the probability of any minimizer
outside the ball arbitrarily small, for all sufficiently large \(n\),
by increasing \(R\).
Combining these two bounds, we conclude
\[
    \sup_{(\theta,\nu)\in\argmin\overline{\mathsf G}_n}
    \min_{s\in\{+1,-1\}}
    \|(\theta,\tanh\nu)-s(\theta^\ast,\tanh\nu^\ast)\|=o_P(1).
\]
At \(\nu=\pm\infty\), the mixing coordinate is \(\tanh\nu=\pm1\),
at distance at least \(1-|\tanh\nu^\ast|>0\) from either ground truth
mixing coordinate. The convergence above therefore excludes all boundary
minimizers with probability tending to one.
At finite parameters, \(\overline{\mathsf G}_n\) differs from the
original empirical objective only by a parameter-independent constant.
Its finite minimizers therefore also minimize that objective, so
\[
    \Pr\!\left\{\argmin_{\vartheta\in\R^{d+1}} \widehat{\calR}_{n,T_n}^{\,\mathsf{SM}}(\vartheta)
    \ne\varnothing\right\}\longrightarrow1.
\]
To return to the original parameters, we use continuity
of \(\operatorname{arctanh}\) at both \(\pm\tanh\nu^\ast\) and the
identity \(\operatorname{arctanh}(s\tanh\nu)=s\nu\) for finite
\(\nu\) and \(s\in\{+1,-1\}\). The convergence up to a joint sign
therefore also holds in the original coordinates \((\theta,\nu)\).
For the measurable estimator and fallback convention already specified,
we align the sign in \((\theta,\nu)\) and conclude
\[
    \bigl\|\widetilde\vartheta_{n,T_n}^{\,\mathsf{SM}}-\vartheta^\ast\bigr\|
    =\min_{s\in\{+1,-1\}}
    \left\|s\bigl(\widehat\theta_{n,T_n}^{\,\mathsf{SM}},
                   \widehat\nu_{n,T_n}^{\,\mathsf{SM}}\bigr)
              -(\theta^\ast,\nu^\ast)\right\|=o_P(1).
\]

For \(\sqrt{n}\)-consistency, assume additionally \(\E\|\bx\|^4<\infty\).
The fourth-moment bounds in
\hyperref[proof:consistency_asymptotic_normality_step_3]{Step~3} of the
proof of Main Theorem~\ref{thm:consistency_asymptotic_normality_appendix}
control the gradient:
\[
    \sqrt n\,\nabla\widehat{\calR}_{n,T_n}^{\,\mathsf{SM}}(\vartheta^\ast) =-\sqrt
    n\,\E_{\Pr_n}[\nabla m_{\vartheta^\ast}] -\sqrt n\,\nabla\Delta_n(\vartheta^\ast)
    =\calO_P(1)+\calO_P(\bar\alpha_{T_n})=\calO_P(1).
\]
Since the sign-aligned estimator is consistent, the entire segment joining
\(\widetilde\vartheta_{n,T_n}^{\,\mathsf{SM}}\) to \(\vartheta^\ast\)
shrinks to \(\vartheta^\ast\) in probability. We can therefore apply the
uniform Hessian convergence proved in the
\hyperref[rmk:fixed_terminal_covariance]{covariance comparison}
along this segment. The averaged Hessian converges to a positive definite
matrix, so it is invertible with probability tending to one and its
inverse is \(\calO_P(1)\).
On the attainment event, the finite global minimizer satisfies
the first-order condition
\(\nabla\widehat{\calR}_{n,T_n}^{\,\mathsf{SM}}
(\widetilde\vartheta_{n,T_n}^{\,\mathsf{SM}})=\mathbf0\).
Expanding this gradient about \(\vartheta^\ast\), we obtain
\[
    \sqrt n\bigl(\widetilde\vartheta_{n,T_n}^{\,\mathsf{SM}} -\vartheta^\ast\bigr)
    =-\left[\int_0^1\nabla^2\widehat{\calR}_{n,T_n}^{\,\mathsf{SM}} \bigl((1-\lambda)\vartheta^\ast+
    \lambda\widetilde\vartheta_{n,T_n}^{\,\mathsf{SM}}\bigr) \,\mathrm d\lambda\right]^{-1}
    \cdot\sqrt n\,\nabla\widehat{\calR}_{n,T_n}^{\,\mathsf{SM}} (\vartheta^\ast)=\calO_P(1).
\]
\end{remark}

\begin{proposition}[Paired difference between score matching and maximum likelihood]
\label{prop:paired_sm_mle}
Assume the conditions of Theorem~\ref{thm:consistency_asymptotic_normality},
including fixed dimension, \(\E\|\bx\|^4<\infty\), and a deterministic
sequence \(\bar{\alpha}_{T_n}\to0\). Consider exact global score matching
and maximum-likelihood estimators based on the same i.i.d.\ observations.
Choose each minimizer as a measurable function of the sample whenever it
exists, and take \((\mathbf0,0)\) otherwise; Lemma~\ref{lemma:sm_minimum_attainment}
justifies this selection. Align both estimates to the same
ground truth representative by setting
\[
    \widetilde\vartheta_n^{\,\mathsf{MLE}}
    \in\argmin_{\vartheta\in\{
    \widehat\vartheta_n^{\,\mathsf{MLE}},-\widehat\vartheta_n^{\,\mathsf{MLE}}\}}
    \|\vartheta-\vartheta^\ast\|,
\]
and using \(\widetilde\vartheta_{n,T_n}^{\,\mathsf{SM}}\) defined in the
\hyperref[proof:consistency_asymptotic_normality]{proof} of Main Theorem~\ref{thm:consistency_asymptotic_normality_appendix}. Resolve either sign-alignment tie by retaining the
original sign. Then
\[
    \left\|\widetilde\vartheta_{n,T_n}^{\,\mathsf{SM}}
         -\widetilde\vartheta_n^{\,\mathsf{MLE}}\right\|
    =\calO_P\!\left(\frac{\bar{\alpha}_{T_n}}{\sqrt n}\right).
\]
\end{proposition}

\begin{proof}
\textit{MLE expansion.}
We first establish the MLE's consistency and expansion from the
likelihood arguments in the proof of Main
Theorem~\ref{thm:consistency_asymptotic_normality_appendix}.
By Lemma~\ref{lemma:sm_minimum_attainment}, a finite measurable MLE
exists with probability tending to one. On this event, optimality
implies
\(\calE_n(\widehat\vartheta_n^{\,\mathsf{MLE}})
=\widehat{\calR}_n^{\,\mathsf{MLE}}(\widehat\vartheta_n^{\,\mathsf{MLE}})
-\widehat{\calR}_n^{\,\mathsf{MLE}}(\vartheta^\ast)\leq0\).

To bound the regression component, we start from the negative-log-likelihood
formula in Lemma~\ref{lemma:nll}. The parameter-independent terms
cancel in \(\calE_n\). Since \((\ln\cosh)'=\tanh\) and
\(|\tanh|\leq1\), each logarithmic term satisfies
\(\left|\ln\frac{\cosh(\mu_0y_0+\nu)}{\cosh\nu}\right|
\leq|\mu_0y_0|\).
The empirical Cauchy--Schwarz inequality and
\(\E_{\Pr_n}[\mu_0^2]=\|\theta\|_{\Sigma_n}^2\) imply
\(\E_{\Pr_n}|\mu_0y_0|\leq\|\theta\|_{\Sigma_n}
\sqrt{\E_{\Pr_n}[y_0^2]}\).
Applying these bounds to the candidate and ground truth terms in the
likelihood contrast, we recover the lower bound derived in
\hyperref[proof:consistency_asymptotic_normality_step_1]{Step~1}
of the proof of Main Theorem~\ref{thm:consistency_asymptotic_normality_appendix}:
\[
    \calE_n(\vartheta)
    \geq\frac12\|\theta\|_{\Sigma_n}^2
       -\frac12\|\theta^\ast\|_{\Sigma_n}^2
       -\bigl(\|\theta\|_{\Sigma_n}+\|\theta^\ast\|_{\Sigma_n}\bigr)
          \sqrt{\E_{\Pr_n}[y_0^2]}.
\]
At \(\vartheta=\widehat\vartheta_n^{\,\mathsf{MLE}}\), optimality
makes \(\calE_n\leq0\). Moving the negative terms to the right-hand
side, we obtain
\[
\begin{aligned}
    \frac12\|\widehat\theta_n^{\,\mathsf{MLE}}\|_{\Sigma_n}^2
    &\leq\|\widehat\theta_n^{\,\mathsf{MLE}}\|_{\Sigma_n}
       \sqrt{\E_{\Pr_n}[y_0^2]}
       +\|\theta^\ast\|_{\Sigma_n}\sqrt{\E_{\Pr_n}[y_0^2]}
       +\frac12\|\theta^\ast\|_{\Sigma_n}^2.
\end{aligned}
\]
Solving this quadratic inequality using the moment bounds and
\(\Sigma_n\xrightarrow{P}\Sigma\succ0\) from the same step, we obtain
\(\|\widehat\theta_n^{\,\mathsf{MLE}}\|=\calO_P(1)\).

To obtain consistency of the full vector, we apply the argument in
\hyperref[proof:consistency_asymptotic_normality_step_2]{Step~2}
of the proof of Main Theorem~\ref{thm:consistency_asymptotic_normality_appendix}.
On each bounded regression-parameter set, the positive population
likelihood gap at \(\nu=\pm\infty\) excludes the single-component
boundary limits, and joint-sign identifiability leaves only
\(\pm\vartheta^\ast\) as population minimizers. The MLE attains the
global empirical negative-log-likelihood minimum, so its gap to that minimum is
zero. The same Wald consistency argument therefore proves
\(\widetilde\vartheta_n^{\,\mathsf{MLE}}-\vartheta^\ast=o_P(1)\).

Sign alignment preserves the likelihood objective. Hence the finite
aligned MLE is stationary. Expanding its zero-gradient condition around the
ground truth, we obtain
\[
\begin{aligned}
    \mathbf0&=\nabla\widehat{\calR}_n^{\,\mathsf{MLE}}
       (\widetilde\vartheta_n^{\,\mathsf{MLE}})\\
     &=\nabla\widehat{\calR}_n^{\,\mathsf{MLE}}(\vartheta^\ast)
       +\left[\int_0^1\nabla^2\widehat{\calR}_n^{\,\mathsf{MLE}}
       \bigl((1-\lambda)\vartheta^\ast
       +\lambda\widetilde\vartheta_n^{\,\mathsf{MLE}}\bigr)
       \,\mathrm{d}\lambda\right]
       (\widetilde\vartheta_n^{\,\mathsf{MLE}}-\vartheta^\ast).
\end{aligned}
\]
Consistency makes this entire segment shrink to \(\vartheta^\ast\).
The uniform likelihood-Hessian convergence, continuity, and information
identity established in
\hyperref[proof:consistency_asymptotic_normality_step_3]{Step~3}
of the proof of Main Theorem~\ref{thm:consistency_asymptotic_normality_appendix}
therefore apply along this segment:
\[
    \int_0^1\nabla^2\widehat{\calR}_n^{\,\mathsf{MLE}}
    \bigl((1-\lambda)\vartheta^\ast
    +\lambda\widetilde\vartheta_n^{\,\mathsf{MLE}}\bigr)
    \,\mathrm{d}\lambda
    =I(\vartheta^\ast)+o_P(1).
\]
Since \(I(\vartheta^\ast)\succ0\), this matrix is invertible with
probability tending to one, with inverse
\(I(\vartheta^\ast)^{-1}+o_P(1)\).
Moreover,
\(-\sqrt n\,\nabla\widehat{\calR}_n^{\,\mathsf{MLE}}(\vartheta^\ast)
=\sqrt n\,\E_{\Pr_n}[\nabla m_{\vartheta^\ast}]=\calO_P(1)\)
by the score central limit theorem in that step. Multiplying the Taylor
identity by \(\sqrt n\) and inverting its integral matrix, we obtain
\[
    \sqrt n(\widetilde\vartheta_n^{\,\mathsf{MLE}}-\vartheta^\ast)
    =I(\vartheta^\ast)^{-1}\sqrt n\,\E_{\Pr_n}[\nabla m_{\vartheta^\ast}]
      +o_P(1).
\]
In particular,
\(\widetilde\vartheta_n^{\,\mathsf{MLE}}-\vartheta^\ast
=\calO_P(n^{-1/2})\).
This bound lets us control the terminal parameter score at the MLE.

\textit{Paired comparison.}
Recall the terminal discrepancy \(\Delta_n\) and the fixed bounded
convex neighborhood \(U\) of \(\vartheta^\ast\) as in \hyperref[proof:consistency_asymptotic_normality_step_3]{Step~3} of the proof of Main Theorem~\ref{thm:consistency_asymptotic_normality_appendix}.
With probability tending to one, both minima are attained and
\(\widetilde\vartheta_n^{\,\mathsf{MLE}},
\widetilde\vartheta_{n,T_n}^{\,\mathsf{SM}}\in U\).
Work on this event for the following Taylor calculations.
We first transfer the terminal-gradient bound from the ground truth to the
MLE. The terminal-gradient and Hessian bounds in \hyperref[proof:consistency_asymptotic_normality_step_3]{Step~3} of the proof of Main Theorem~\ref{thm:consistency_asymptotic_normality_appendix} imply
\[
\begin{aligned}
    \|\nabla\Delta_n(\widetilde\vartheta_n^{\,\mathsf{MLE}})\|
    &\leq\|\nabla\Delta_n(\vartheta^\ast)\|
       +\sup_{\vartheta\in U}\|\nabla^2\Delta_n(\vartheta)\|_2
        \|\widetilde\vartheta_n^{\,\mathsf{MLE}}-\vartheta^\ast\|\\
    &=\calO_P(\bar\alpha_{T_n}/\sqrt n)
       +\calO_P(\bar\alpha_{T_n})\calO_P(n^{-1/2})
     =\calO_P(\bar\alpha_{T_n}/\sqrt n).
\end{aligned}
\]
The likelihood gradient vanishes at
\(\widetilde\vartheta_n^{\,\mathsf{MLE}}\). Thus, by the empirical bridge,
its score matching gradient consists only of the terminal correction:
\[
    \nabla\widehat{\calR}_{n,T_n}^{\,\mathsf{SM}}
    (\widetilde\vartheta_n^{\,\mathsf{MLE}})
    =-\nabla\Delta_n(\widetilde\vartheta_n^{\,\mathsf{MLE}}).
\]
The score matching gradient vanishes at
\(\widetilde\vartheta_{n,T_n}^{\,\mathsf{SM}}\).
We now expand between the two estimates to express their difference in
terms of this correction. With interpolation parameter
\(\lambda\in[0,1]\), the integral Taylor formula reads
\[
\begin{aligned}
    &\left[\int_0^1\nabla^2\widehat{\calR}_{n,T_n}^{\,\mathsf{SM}}
       \bigl((1-\lambda)\widetilde\vartheta_n^{\,\mathsf{MLE}}
       +\lambda\widetilde\vartheta_{n,T_n}^{\,\mathsf{SM}}\bigr)\,\mathrm{d}\lambda\right]\\
    &\qquad\cdot\bigl(\widetilde\vartheta_{n,T_n}^{\,\mathsf{SM}}
       -\widetilde\vartheta_n^{\,\mathsf{MLE}}\bigr)
       =\nabla\Delta_n(\widetilde\vartheta_n^{\,\mathsf{MLE}}).
\end{aligned}
\]
Both endpoints converge to \(\vartheta^\ast\), so their entire segment
converges to that point. The shrinking-neighborhood Hessian argument and
terminal-Hessian bound in \hyperref[proof:consistency_asymptotic_normality_step_3]{Step~3} of the proof of Main Theorem~\ref{thm:consistency_asymptotic_normality_appendix} imply that the integral matrix above
converges in probability to \(I(\vartheta^\ast)\).
Since this limit is positive definite, the inverse has operator norm
\(\calO_P(1)\) on an event with probability tending to one.
Inverting the displayed equation proves the stated paired rate.
The nonattainment events have vanishing probability, so this bound
also holds under the fallback convention.
\end{proof}

The paired result concerns global minimizers with the diffusion-noise
expectation evaluated exactly. The next proposition derives the empirical
EM updates and their relation to likelihood stationarity. The following
remark proves global MLE uniqueness up to sign with probability tending
to one and distinguishes this property from convergence of a numerical
EM run.

\begin{proposition}[Empirical EM fixed points and likelihood stationarity]
\label{prop:empirical_em_stationarity}
For the observed sample, suppose
\(\Sigma_n=\E_{\Pr_n}[\bx\bx^\top]\succ0\).
For a finite current candidate \(\vartheta=(\theta,\nu)\), with
\(\mu_0=\langle\theta,\bx\rangle\),
empirical EM updates the parameters to \(\vartheta'=(\theta',\nu')\), where
\[
\begin{aligned}
    \theta'&=\Sigma_n^{-1}\E_{\Pr_n}
       [\bx y_0\tanh(\mu_0y_0+\nu)],\\
    \nu'&=\operatorname{arctanh}\!\left(
       \E_{\Pr_n}[\tanh(\mu_0y_0+\nu)]\right).
\end{aligned}
\]
The average posterior sign updates the mixing imbalance \(\tanh\nu\);
applying \(\operatorname{arctanh}\) converts it to the parameter \(\nu\).
The empirical negative log-likelihood satisfies
\[
\begin{aligned}
    \nabla_\theta\widehat{\calR}_n^{\,\mathsf{MLE}}(\theta,\nu)
    &=\Sigma_n\theta-\E_{\Pr_n}[\bx y_0\tanh(\mu_0y_0+\nu)]
      =\Sigma_n(\theta-\theta'),\\
    \nabla_\nu\widehat{\calR}_n^{\,\mathsf{MLE}}(\theta,\nu)
    &=\tanh\nu-\E_{\Pr_n}[\tanh(\mu_0y_0+\nu)]
      =\tanh\nu-\tanh\nu'.
\end{aligned}
\]
Consequently,
\[
    \nabla_\vartheta\widehat{\calR}_n^{\,\mathsf{MLE}}(\vartheta)=\mathbf0
    \quad\Longleftrightarrow\quad
    \vartheta'=\vartheta.
\]
\end{proposition}

\begin{proof}
We first compute the posterior mean of the latent sign under the current
candidate \(\vartheta\) using Bayes' rule:
\[
    \E_{z\mid y_0,\bx;\vartheta}[(-1)^{z+1}]
    =\tanh(\mu_0y_0+\nu).
\]
The EM step maximizes the conditional expected complete-data log-likelihood
over the new parameters \(\vartheta'=(\theta',\nu')\).
After averaging over the sample, this objective is
\[
\begin{aligned}
    &\E_{\Pr_n}\E_{z\mid y_0,\bx;\vartheta}
       [\ln p(y_0,z\mid\bx;\vartheta')]\\
    &\quad=-\tfrac12(\theta')^\top\Sigma_n\theta'
      +(\theta')^\top\E_{\Pr_n}[\bx y_0\tanh(\mu_0y_0+\nu)]\\
    &\qquad\quad+\nu'\E_{\Pr_n}[\tanh(\mu_0y_0+\nu)]
      -\ln(2\cosh\nu')-\tfrac12\E_{\Pr_n}[y_0^2]
      -\tfrac12\ln(2\p).
\end{aligned}
\]
The last two terms are independent of \(\vartheta'\).
Since \(\Sigma_n\succ0\), the regression part has a unique maximizer,
obtained by setting its gradient to zero. Solving this equation proves the
displayed update for \(\theta'\).
Every observed value of \(\tanh(\mu_0y_0+\nu)\) lies strictly between
\(-1\) and \(1\), so its empirical average also lies in \((-1,1)\).
The mixing-weight part is strictly concave in \(\nu'\), and its derivative
vanishes exactly when
\[
    \tanh\nu'=\E_{\Pr_n}[\tanh(\mu_0y_0+\nu)].
\]
Applying \(\operatorname{arctanh}\) proves the stated finite update for \(\nu'\).

To obtain the likelihood identities, apply
Lemma~\ref{lemma:derivatives_negative_log_likelihood} to each observation
at \(\bar\alpha_0=1\).
Differentiating the finite sample average, we obtain
\[
\begin{aligned}
    \nabla_\theta\widehat{\calR}_n^{\,\mathsf{MLE}}(\theta,\nu)
    &=\E_{\Pr_n}[\bx(\mu_0-y_0\tanh(\mu_0y_0+\nu))]
      =\Sigma_n(\theta-\theta'),\\
    \nabla_\nu\widehat{\calR}_n^{\,\mathsf{MLE}}(\theta,\nu)
    &=\E_{\Pr_n}[\tanh\nu-\tanh(\mu_0y_0+\nu)]
      =\tanh\nu-\tanh\nu'.
\end{aligned}
\]
These are the sample counterparts of Proposition~\ref{prop:cross_entropy}.
Invertibility of \(\Sigma_n\) and injectivity of \(\tanh\) make the zero
gradient condition equivalent to \(\theta'=\theta\) and \(\nu'=\nu\),
as claimed.
Under the covariate-density assumption of
Theorem~\ref{thm:consistency_asymptotic_normality}, the sample vectors span
\(\R^d\) almost surely for \(n\ge d\), so \(\Sigma_n\succ0\) almost surely.
\end{proof}

\begin{remark}[Global MLE uniqueness and numerical EM output]
\label{rmk:consistent_em_fixed_point}
Assume the conditions of Proposition~\ref{prop:paired_sm_mle}, including
fixed dimension and \(\E\|\bx\|^4<\infty\).
\par\nopagebreak[4]
\textit{Global MLE uniqueness.}
With probability tending to one, the empirical negative log-likelihood
has exactly two global minimizers, exchanged by the joint sign flip
\((\theta,\nu)\mapsto(-\theta,-\nu)\). They represent the same mixture
distribution. With the measurable aligned MLE and its nonattainment
completion defined in Proposition~\ref{prop:paired_sm_mle},
\[
    \Pr\!\left\{
        \argmin_{\vartheta\in\R^{d+1}}
        \widehat{\calR}_n^{\,\mathsf{MLE}}(\vartheta)
        =\bigl\{\widetilde\vartheta_n^{\,\mathsf{MLE}},
               -\widetilde\vartheta_n^{\,\mathsf{MLE}}\bigr\}
    \right\}\longrightarrow1.
\]
Each finite global minimizer is an EM fixed point by
Proposition~\ref{prop:empirical_em_stationarity}. These fixed points are
likelihood targets for the numerical EM procedure specified below.

\textit{Existence and consistency.}
Lemma~\ref{lemma:sm_minimum_attainment} proves finite MLE attainment with
probability tending to one and constructs a sample-measurable selection on the
attainment event. Proposition~\ref{prop:paired_sm_mle} proves that its
aligned version is consistent and asymptotically normal. It remains to
show that all global minimizers belong to this one joint-sign pair.

\textit{Local uniqueness.}
The integrable Hessian envelope and continuity established in \hyperref[proof:consistency_asymptotic_normality_step_3]{Step~3} of the proof of Main Theorem~\ref{thm:consistency_asymptotic_normality_appendix}
allow us to choose a fixed \(r\in(0,\|\vartheta^\ast\|/2)\) and
\(U:=\overline B_r(\vartheta^\ast)\) such that
\[
    \E\!\left[\sup_{\vartheta\in U}
      \|\nabla^2m_\vartheta-\nabla^2m_{\vartheta^\ast}\|_2\right]
    <\tfrac14\lambda_{\min}(I(\vartheta^\ast)).
\]
Since \(\nabla^2\widehat{\calR}_n^{\,\mathsf{MLE}}=-\E_{\Pr_n}[\nabla^2m_\vartheta]\),
\[
\begin{aligned}
    \sup_{\vartheta\in U}
    \|\nabla^2\widehat{\calR}_n^{\,\mathsf{MLE}}(\vartheta)
         -I(\vartheta^\ast)\|_2
    &\leq
    \|\nabla^2\widehat{\calR}_n^{\,\mathsf{MLE}}(\vartheta^\ast)
         -I(\vartheta^\ast)\|_2\\
    &\quad+\E_{\Pr_n}\!\left[\sup_{\vartheta\in U}
       \|\nabla^2m_\vartheta-\nabla^2m_{\vartheta^\ast}\|_2\right].
\end{aligned}
\]
The first term tends to zero in probability, and the law of large numbers
makes the second converge to the expectation bounded above.
Consequently, with probability tending to one,
\[
    \nabla^2\widehat{\calR}_n^{\,\mathsf{MLE}}(\vartheta)
    \succeq\tfrac12\lambda_{\min}(I(\vartheta^\ast))I_{d+1}
    \qquad\text{for every }\vartheta\in U.
\]
The empirical negative log-likelihood is then strictly convex on \(U\).
By attainment and consistency, the aligned MLE lies in the interior of
\(U\) and is stationary with probability tending to one.
It is therefore the only stationary point in \(U\), and
Proposition~\ref{prop:empirical_em_stationarity} makes it the unique EM
fixed point there. By sign symmetry, uniqueness also holds in
\(-U\). The two balls are disjoint by the choice of \(r\).

\textit{Global uniqueness up to sign.}
It remains to show that every global minimizer lies in \(U\cup(-U)\),
where local uniqueness has already been established. We use the excess risk
\(\calE_n(\vartheta)=\widehat{\calR}_n^{\,\mathsf{MLE}}(\vartheta)
-\widehat{\calR}_n^{\,\mathsf{MLE}}(\vartheta^\ast)\) and its
expectation \(\calE\).
At every global minimizer, optimality relative to \((\mathbf0,0)\)
and the quadratic likelihood bound in
\hyperref[proof:consistency_asymptotic_normality_step_1]{Step~1} of the
proof of Main Theorem~\ref{thm:consistency_asymptotic_normality_appendix}
ensure that
\[
    0\geq\calE_n(\theta,\nu)-\calE_n(\mathbf0,0)
    \geq\tfrac12\|\theta\|_{\Sigma_n}^2
       -\sqrt{\E_{\Pr_n}[y_0^2]}\,\|\theta\|_{\Sigma_n}.
\]
For \(n\ge d\), we solve this quadratic inequality to obtain the common bound
\(\|\theta\|\leq2\sqrt{\E_{\Pr_n}[y_0^2]/\lambda_{\min}(\Sigma_n)}\).
The weak law yields \(\Sigma_n\xrightarrow{\Pr}\Sigma\succ0\) and
\(\E_{\Pr_n}[y_0^2]\xrightarrow{\Pr}\E y_0^2<\infty\), so we can
choose a fixed radius \(R>\|\theta^\ast\|\) whose ball contains the
regression parameters of all global minimizers with probability tending to one.

By \hyperref[proof:consistency_asymptotic_normality_step_2]{Step~2} of that
proof, \(\calE\) is continuous on the compact set
\(\{\|\theta\|\leq R\}\times\overline{\R}\) and vanishes only at
\(\pm\vartheta^\ast\). Here \(\overline{\R}=\R\cup\{-\infty,+\infty\}\)
has the topology induced by \(\tanh\nu\).
The closed subset outside the interiors of \(U\) and \(-U\) contains
neither population minimizer. By the extreme value theorem, \(\calE\)
has a positive minimum there. In particular,
\[
    \inf_{\substack{\|\theta\|\leq R,\,\nu\in\overline{\R}\\
       (\theta,\nu)\notin U\cup(-U)}}
       \calE(\theta,\nu)>0.
\]
On the full compact set, the log-density formula in Lemma~\ref{lemma:nll} makes the
per-observation contrast \(m_{\vartheta^\ast}-m_\vartheta\) continuous,
with integrable envelope \(R^2\|\bx\|^2+2R|y_0|\|\bx\|\).
The uniform law of large numbers therefore applies to \(\calE_n\).
Moreover, optimality relative to \(\vartheta^\ast\) means that
\(\calE_n(\theta,\nu)\leq0\) at every global minimizer. Thus, for all
such minimizers with \(\|\theta\|\leq R\),
\[
    0\leq\calE(\theta,\nu)
    \leq\sup_{\|\theta'\|\leq R,\,\nu'\in\overline{\R}}
       |\calE_n(\theta',\nu')-\calE(\theta',\nu')|=o_P(1).
\]
This common upper bound excludes the positive-gap region with
probability tending to one. Every global minimizer therefore lies in
\(U\cup(-U)\).

\textit{Numerical EM output.}
The uniqueness statement concerns the global minimizing pair. The origin
is an additional fixed point, so there are at least three distinct fixed points
when the two nonzero global minimizers exist. Any exact fixed-point selection
that is consistent after alignment enters \(U\) with probability tending
to one and therefore coincides with
\(\widetilde\vartheta_n^{\,\mathsf{MLE}}\) on that event.

The numerical estimator \(\widehat\vartheta_n^{\,\mathsf{EM}}\) is the
finite iterate returned by the EM procedure in
Appendix~\ref{sup:additional_experiments}.
Figure~\ref{fig:mle-em-comparison} and Table~\ref{tab:mle-em-comparison}
compare its statistical accuracy with the numerical MLE and score matching
fits on the same datasets.
\end{remark}

\begin{remark}[Covariance comparison at a positive terminal limit]
\label{rmk:fixed_terminal_covariance}
Under the fourth-moment assumptions of
Theorem~\ref{thm:consistency_asymptotic_normality}, consider a deterministic
schedule satisfying \(0<\bar\alpha_{T_n}<1\) and
\(\bar\alpha_{T_n}\to a\in(0,1)\).
Let \(\widetilde\vartheta_{n,T_n}^{\,\mathsf{SM}}\)
denote the score matching estimator at terminal
level \(\bar\alpha_{T_n}\), aligned to the ground truth \(\vartheta^\ast\).
The \hyperref[rmk:root_n_terminal_limit]{remark on consistency and \(\sqrt{n}\)-consistency}
following Lemma~\ref{lemma:sm_minimum_attainment} establishes its
consistency.
We compare it with \(\widetilde\vartheta_n^{\,\mathsf{MLE}}\), the MLE
defined in Proposition~\ref{prop:paired_sm_mle}, computed from the same
observations and aligned to the same ground truth \(\vartheta^\ast\).
The original log-density is
\(m_\vartheta(y_0,\bx):=\ln p(y_0\mid\bx;\theta,\nu)\), as in
the proof of Main Theorem~\ref{thm:consistency_asymptotic_normality_appendix}.
Thus \(\nabla m_{\vartheta^\ast}\) is the parameter gradient of the
original log-density, evaluated at the ground truth.
Its second-moment matrix under the ground truth law is the Fisher
information \(I:=I(\vartheta^\ast)
=\E[\nabla m_{\vartheta^\ast}(\nabla m_{\vartheta^\ast})^\top]\).

The diffused response has the Gaussian representation
\(y_{T_n}=\sqrt{\bar\alpha_{T_n}}\,y_0+
\sqrt{1-\bar\alpha_{T_n}}\,\xi\), where \(\xi\sim\calN(0,1)\) is
independent of \((y_0,\bx)\).
The terminal log-density is the same function \(m\), evaluated at
the diffused response \(y_{T_n}\) and the scaled parameter
\((\sqrt{\bar\alpha_{T_n}}\,\theta,\nu)\).
We differentiate with respect to the original joint parameter
\(\vartheta=(\theta,\nu)\), with \(y_{T_n}\) and \(\bx\) fixed.
Averaging this derivative over the added noise, we define
\(g_{\bar\alpha_{T_n}}(y_0,\bx)
:=\left.\E_\xi\!\left[
\nabla_\vartheta m_{(\sqrt{\bar\alpha_{T_n}}\,\theta,\nu)}
(y_{T_n},\bx)\right]\right|_{\vartheta=\vartheta^\ast}\).
The regression part of this derivative includes the chain-rule factor
\(\sqrt{\bar\alpha_{T_n}}\).
After the noise average, \(g_{\bar\alpha_{T_n}}\) depends only on
the original observation \((y_0,\bx)\).
We define \(g_a(y_0,\bx)\) by the same expression with
\(\bar\alpha_{T_n}\) replaced by its limit \(a\).
We omit the arguments \((y_0,\bx)\) below.
As in Step~3 of that proof,
\(\nabla\Delta_n(\vartheta^\ast)=-\E_{\Pr_n}[g_{\bar\alpha_{T_n}}]\).
Differentiating the empirical bridge in
Eq.~\eqref{eq:empirical_sm_mle_bridge} therefore expresses the score
matching gradient as the difference of the two empirical score averages:
\[
    \nabla\widehat{\calR}_{n,T_n}^{\,\mathsf{SM}}(\vartheta^\ast)
    =-\E_{\Pr_n}[\nabla m_{\vartheta^\ast}]
     -\nabla\Delta_n(\vartheta^\ast)
    =-\E_{\Pr_n}[\nabla m_{\vartheta^\ast}-g_{\bar\alpha_{T_n}}].
\]
To compute the covariance between the two estimators, we relate
\(g_a\) to the original likelihood score.
Because the Gaussian transition kernel is parameter independent,
differentiating the diffusion convolution shows that the terminal
parameter score equals
\(\E[\nabla m_{\vartheta^\ast}\mid\sqrt a\,y_0+\sqrt{1-a}\,\xi,\bx]\).
Since this score has mean zero, its Fisher information matrix with
respect to the original parameters is
\(I_a:=\operatorname{Cov}\!\bigl(
\E[\nabla m_{\vartheta^\ast}\mid\sqrt a\,y_0+\sqrt{1-a}\,\xi,\bx]
\bigr)\).
Since \(g_a\) averages this score conditional on \((y_0,\bx)\),
conditioning first on the original observation and then on the diffused
response and \(\bx\) establishes
\(\E[g_a(\nabla m_{\vartheta^\ast})^\top]=I_a\).

Differentiating the empirical bridge in Eq.~\eqref{eq:empirical_sm_mle_bridge}
twice and applying the information matrix equality, as in
\hyperref[proof:consistency_asymptotic_normality_step_3]{Step~3} of the
proof of Main Theorem~\ref{thm:consistency_asymptotic_normality_appendix},
we obtain \(I-I_a\) as the limiting population Hessian.
To justify its inverse in the Taylor expansion below, we show that
\(I-I_a\succ0\).
By the law of total covariance,
\[
    I-I_a=\E[\operatorname{Cov}(\nabla m_{\vartheta^\ast}\mid
       \sqrt a\,y_0+\sqrt{1-a}\,\xi,\bx)]\succeq0.
\]
For \(v\in\R^{d+1}\) with \(v^\top(I-I_a)v=0\), the
conditional variance of \(v^\top\nabla m_{\vartheta^\ast}\)
vanishes almost surely.
The positive mixture and Gaussian transition densities ensure that
the conditional density of \(y_0\) given the diffused response and
\(\bx\) is positive on \(\R\).
Thus \(v^\top\nabla m_{\vartheta^\ast}\) must be constant in
\(y_0\) almost everywhere, for almost every \(\bx\).
The parameter-score identity
\(\E[\nabla m_{\vartheta^\ast}\mid\bx]=\mathbf0\) makes this
constant zero, so \(v^\top Iv=0\).
Since \(I\succ0\) by Step~3 of the same proof, we conclude that
\(v=\mathbf0\), proving \(I-I_a\succ0\).

The empirical gradient depends on \(g_{\bar\alpha_{T_n}}\), so we first
justify replacing it by \(g_a\) on the \(\sqrt n\) scale.
The noise-averaged score formulas in
\hyperref[proof:consistency_asymptotic_normality_step_3]{Step~3} of the
proof of Main Theorem~\ref{thm:consistency_asymptotic_normality_appendix}
are continuous in the retained-signal level and have a common
square-integrable envelope under the fourth-moment assumption.
Dominated convergence therefore establishes
\(\E\|g_{\bar\alpha_{T_n}}-g_a\|^2\to0\).
Both noise averages have mean zero by the parameter-score identity.
Taking expectation over the i.i.d.\ original sample, independence yields
\[
    \E\left\|\sqrt n\,\E_{\Pr_n}
    [g_{\bar\alpha_{T_n}}-g_a]\right\|^2
    =\E\|g_{\bar\alpha_{T_n}}-g_a\|^2\longrightarrow0.
\]
Hence \(\sqrt n\,\E_{\Pr_n}[g_{\bar\alpha_{T_n}}-g_a]=o_P(1)\)
by Markov's inequality, with no additional rate condition on
\(\bar\alpha_{T_n}\to a\).

On compact neighborhoods of \(\vartheta^\ast\) and \(a\in(0,1)\),
the common integrable Hessian envelope from
\hyperref[proof:consistency_asymptotic_normality_step_3]{Step~3}
of the same proof and joint continuity justify a uniform law of large
numbers and continuity of the population Hessian. Consistency makes the
Taylor segment approach \(\vartheta^\ast\), so together with
\(\bar\alpha_{T_n}\to a\) we obtain
\(\bigl\|\int_0^1\nabla^2\widehat{\calR}_{n,T_n}^{\,\mathsf{SM}}
\bigl((1-\lambda)\vartheta^\ast+
\lambda\widetilde\vartheta_{n,T_n}^{\,\mathsf{SM}}\bigr)
\,\mathrm d\lambda-(I-I_a)\bigr\|_2=o_P(1)\).

On the attainment event, the aligned estimator satisfies
\(\nabla\widehat{\calR}_{n,T_n}^{\,\mathsf{SM}}
(\widetilde\vartheta_{n,T_n}^{\,\mathsf{SM}})=\mathbf0\).
We apply Taylor's integral formula at \(\vartheta^\ast\) and substitute
the gradient identity above:
\[
    \left[\int_0^1\nabla^2\widehat{\calR}_{n,T_n}^{\,\mathsf{SM}}
    \bigl((1-\lambda)\vartheta^\ast+
    \lambda\widetilde\vartheta_{n,T_n}^{\,\mathsf{SM}}\bigr)
    \,\mathrm d\lambda\right]\sqrt n\bigl(
    \widetilde\vartheta_{n,T_n}^{\,\mathsf{SM}}-\vartheta^\ast\bigr)
    =\sqrt n\,\E_{\Pr_n}[\nabla m_{\vartheta^\ast}-g_{\bar\alpha_{T_n}}].
\]
We now replace the averaged Hessian by \(I-I_a+o_P(1)\), where the
remainder is a matrix controlled by the preceding operator-norm bound.
On the right, \(\sqrt n\,\E_{\Pr_n}[g_{\bar\alpha_{T_n}}-g_a]=o_P(1)\)
permits replacing \(g_{\bar\alpha_{T_n}}\) by \(g_a\), with a vector remainder:
\[
    \bigl[I-I_a+o_P(1)\bigr]\sqrt n\bigl(
    \widetilde\vartheta_{n,T_n}^{\,\mathsf{SM}}-\vartheta^\ast\bigr)
    =\sqrt n\,\E_{\Pr_n}[\nabla m_{\vartheta^\ast}-g_a]+o_P(1).
\]
The centered sum on the right is \(\calO_P(1)\), since its second
moment equals \(\E\|\nabla m_{\vartheta^\ast}-g_a\|^2<\infty\),
independently of \(n\).
Because \(I-I_a\succ0\), the inverse matrix on the left converges
in probability to \((I-I_a)^{-1}\).
Multiplying by this inverse establishes the score matching expansion,
with an \(o_P(1)\) remainder.
For the MLE, the expansion in the proof of
Proposition~\ref{prop:paired_sm_mle} uses the likelihood Hessian limit
\(I\) and the original score \(\nabla m_{\vartheta^\ast}\).
Combining these expansions, we obtain
\[
    \sqrt n
    \begin{pmatrix}
        \widetilde\vartheta_{n,T_n}^{\,\mathsf{SM}}-\vartheta^\ast\\
        \widetilde\vartheta_n^{\,\mathsf{MLE}}-\vartheta^\ast
    \end{pmatrix}
    =\sqrt n\,\E_{\Pr_n}\!\left[
    \begin{pmatrix}
        (I-I_a)^{-1}(\nabla m_{\vartheta^\ast}-g_a)\\
        I^{-1}\nabla m_{\vartheta^\ast}
    \end{pmatrix}\right]+o_P(1).
\]
Each vector inside this empirical average is a function of one original
observation \((y_0,\bx)\). These vectors are therefore i.i.d., have
mean zero by the parameter-score identity, and have finite second
moments under the fourth-moment assumption.

To apply the multivariate central limit theorem, we compute the covariance
of this vector. For the score matching block, we expand its covariance
and substitute the cross-moment identity
\(\E[g_a(\nabla m_{\vartheta^\ast})^\top]=I_a\):
\[
    V_a:=(I-I_a)^{-1}\bigl(I-2I_a+\E[g_ag_a^\top]\bigr)(I-I_a)^{-1}.
\]
The MLE block has covariance \(I^{-1}\).
Since both blocks depend on the same observation, we also compute their
cross-covariance. Using the same cross-moment identity, we obtain
\((I-I_a)^{-1}\E[(\nabla m_{\vartheta^\ast}-g_a)
(\nabla m_{\vartheta^\ast})^\top]I^{-1}
=(I-I_a)^{-1}(I-I_a)I^{-1}=I^{-1}\).

The multivariate central limit theorem and Slutsky's theorem now
establish the joint limit
\[
    \sqrt n
    \begin{pmatrix}
        \widetilde\vartheta_{n,T_n}^{\,\mathsf{SM}}-\vartheta^\ast\\
        \widetilde\vartheta_n^{\,\mathsf{MLE}}-\vartheta^\ast
    \end{pmatrix}
    \xrightarrow{\mathsf d}
    \calN\!\left(\mathbf0,
    \begin{pmatrix}
        V_a&I^{-1}\\
        I^{-1}&I^{-1}
    \end{pmatrix}\right).
\]
We subtract the second block from the first to compare the two estimators.
The covariance of the limiting difference is
\(V_a+I^{-1}-I^{-1}-I^{-1}=V_a-I^{-1}\). Hence
\[
    \sqrt n\bigl(\widetilde\vartheta_{n,T_n}^{\,\mathsf{SM}}
               -\widetilde\vartheta_n^{\,\mathsf{MLE}}\bigr)
    \xrightarrow{\mathsf d}\calN(0,V_a-I^{-1}).
\]
To compute the variance increase without subtracting two nearly equal
covariance matrices, we evaluate the equivalent expression
\[
\begin{aligned}
    V_a-I^{-1}
    &=(I-I_a)^{-1}\operatorname{Cov}(g_a-I_aI^{-1}\nabla m_{\vartheta^\ast})(I-I_a)^{-1}\\
    &=(I-I_a)^{-1}\bigl(\E[g_ag_a^\top]-I_aI^{-1}I_a\bigr)(I-I_a)^{-1}.
\end{aligned}
\]
This expression is positive semidefinite. We compute the
benchmarks in Appendix~\ref{sup:additional_experiments} from it.
Thus \(V_a-I^{-1}\) measures both the asymptotic variance increase over
MLE and the covariance of the paired difference on the sampling scale.
In particular, the relative asymptotic variance increase in coordinate
\(j\) is
\[
    \frac{[V_a]_{jj}}{[I^{-1}]_{jj}}-1
    =\frac{[V_a-I^{-1}]_{jj}}{[I^{-1}]_{jj}},
    \qquad j=1,\ldots,d+1.
\]
We evaluate them by quadrature
to plot the solid curves in
Figure~\ref{fig:terminal-schedules}(b), with \(j=d+1\) for \(\nu\)
and averaging the ratios over \(j=1,\ldots,d\) for \(\theta\).
\end{remark}

\begin{remark}[Fixed-Sample Nonattainment]
\label{rmk:fixed_sample_nonattainment}
Finite samples can have an empty finite-parameter argmin. For example, take \(n=d=1\), \(0<\bar\alpha_{T_1}<1\), \(x\ne0\), and
\(y_0\ne0\), so that \(\mu_0=\theta x\).
For the centered criterion \(\mathsf G_1\) from the proof of
Lemma~\ref{lemma:sm_minimum_attainment}, we apply Jensen's inequality
and the \(1\)-Lipschitz property of \(\ln\cosh\) to obtain
\[
    \mathsf G_1(\theta,\nu)\geq\frac{1-\bar\alpha_{T_1}}{2}(|\mu_0|-|y_0|)^2-\frac{1-\bar\alpha_{T_1}}{2}y_0^2.
\]
The inequality is strict for every finite \(\nu\) when \(\mu_0\ne0\), while the
criterion is zero at \(\mu_0=0\). The negative lower bound is approached
with \(\mu_0=y_0\) as \(\nu\to+\infty\), or with
\(\mu_0=-y_0\) as \(\nu\to-\infty\).
Thus the finite-parameter argmin is empty for this sample, although
Lemma~\ref{lemma:sm_minimum_attainment} shows that this obstruction has
vanishing probability in the stated asymptotic regime.
\end{remark}

\section{Proofs of Cross-Entropy and EM Decompositions for Score Matching}\label{sup:connect}
This appendix restates and proves the fixed-scale decompositions,
gradient identities, and noise-limit results from
Section~\ref{sec:connect}: Propositions~\ref{prop:em_update},
\ref{prop:cross_entropy}, \ref{prop:sm_gradients},
\ref{prop:low_noise_em}, and~\ref{prop:high_noise_effective_coefficient}, Main
Theorem~\ref{thm:information_theoretical_decomposition}, and
Corollaries~\ref{cor:sm_gradients_entropy_activation}
and~\ref{cor:sm_gradients_em_variance}.

\begin{proposition}[Proposition~\ref{prop:em_update}]
    The EM update rules for the regression parameter \(\theta_t\) and the mixing-weight imbalance \(\tanh\nu\) are
    \[
    M(\theta_t, \nu) = \Sigma^{-1}\E_{\ast}[\bx y_t \tanh(\mu_t y_t + \nu)],\;\;
    N(\theta_t, \nu) = \E_{\ast}[\tanh(\mu_t y_t + \nu)].
    \]
    They also satisfy
    \[
    \langle \Sigma M(\theta_t, \nu), \theta_t \rangle = \E_{\ast}[\mu_t y_t\tanh(\mu_t y_t + \nu)] = \E_{\ast}[y_t^2 - y_t \partial_{y_t} F(y_t, \mu_t, \nu)] .
    \]
    Their derivatives satisfy
    \[\Sigma \nabla_{\nu} M(\theta_t, \nu) = \nabla_{\theta_t} N(\theta_t, \nu) = \E_{\ast}[\bx y_t \sech^2(\mu_t y_t + \nu)].\]
\end{proposition}

\begin{proof}[Proof for Proposition~\ref{prop:em_update}]
    The calculation adapts Appendix~B of \citet{luo2024unveiling} and
    Proposition~3 of \citet{luo2025structural} to the diffused parameter
    \(\theta_t\). Those results used the clean-scale parameter. Here we replace
    it with \(\theta_t\) from Equation~\eqref{eq:yt}.

    The marginal law of \(\bx\) is fixed across \((\theta_t,\nu)\), so
    we use the conditional negative log-likelihood \(F(y_t,\mu_t,\nu)\).
    On bounded sets of candidate parameters, its first parameter derivatives are
    bounded by \(C(1+\|\bx\|^2+\|\bx\||y_t|)\), which is integrable
    under the second-moment assumption. We may therefore differentiate
    under the expectation. Since \(\mu_t=\langle\theta_t,\bx\rangle\), we obtain
    \[
    M(\theta_t, \nu) = \theta_t - \Sigma^{-1} \nabla_{\theta_t} \E_{\ast}[F(y_t,\mu_t,\nu)]
    = \theta_t - \Sigma^{-1} \E_{\bx} \E_{y_t\mid \bx; \theta_t^\ast, \pi^\ast} [\bx \partial_{\mu_t} F(y_t, \mu_t, \nu)],
    \]
    \[
    N(\theta_t, \nu) = \tanh \nu - \nabla_\nu \E_{\ast}[F(y_t,\mu_t,\nu)]
    = \tanh \nu - \E_{\bx} \E_{y_t\mid \bx; \theta_t^\ast, \pi^\ast} [\nabla_{\nu} F(y_t, \mu_t, \nu)].
    \]
    By Lemma~\ref{lemma:derivatives_negative_log_likelihood},
    \(\partial_{\mu_t}F=\mu_t-y_t\tanh(\mu_ty_t+\nu)\) and
    \(\nabla_\nu F=\tanh\nu-\tanh(\mu_ty_t+\nu)\). Substitution yields
    \[
    M(\theta_t, \nu) = \Sigma^{-1} \E_{\ast} [\bx y_t \tanh(\mu_t y_t + \nu)],\;\;
    N(\theta_t, \nu) = \E_{\ast} [\tanh(\mu_t y_t + \nu)].
    \]
    The same derivative formulas imply
    \(\mu_t\partial_{\mu_t}F-\mu_t^2
    =y_t\partial_{y_t}F-y_t^2
    =-\mu_ty_t\tanh(\mu_ty_t+\nu)\). Therefore
    \[
    \langle \Sigma M(\theta_t, \nu), \theta_t \rangle = \E_{\ast}[\mu_t y_t\tanh(\mu_t y_t + \nu)] = \E_{\ast}[y_t^2 - y_t \partial_{y_t} F(y_t, \mu_t, \nu)] .
    \]
    The identity \(\tanh'(u)=\sech^2(u)\) and the relation
    \(\nabla_{\theta_t}\mu_t=\bx\) identify the derivatives of the two operator integrands.
    The dominated convergence theorem justifies moving the derivatives through
    \(\E_\ast\)~\citep[Theorem~1.5.8]{durrett2019probability}.
    \[
        \Sigma \nabla_{\nu} M(\theta_t, \nu) 
        = \E_{\ast}[\bx y_t \nabla_{\nu} \tanh(\mu_t y_t + \nu)]
        = \E_{\ast}[\bx y_t \sech^2(\mu_t y_t + \nu)]
        = \E_{\ast}[\nabla_{\theta_t} \tanh(\mu_t y_t + \nu)]
        = \nabla_{\theta_t}N(\theta_t, \nu).
    \]
\end{proof}

\begin{proposition}[Proposition~\ref{prop:cross_entropy}]
    The cross-entropy \(\calH(\theta_t,\nu)\) defined in
    Section~\ref{subsec:em_update} satisfies
    \[
    \nabla_{\theta_t} \calH(\theta_t, \nu) = \Sigma \left(\theta_t - M(\theta_t, \nu)\right),\;\;
    \nabla_{\nu} \calH(\theta_t, \nu) = \tanh \nu - N(\theta_t, \nu).
    \]
\end{proposition}

\begin{proof}[Proof for Proposition~\ref{prop:cross_entropy}]
    Recall the operator expressions established in the preceding proof:
    \[
    M(\theta_t,\nu)=\theta_t-
    \Sigma^{-1}\E_\ast[\bx\,\partial_{\mu_t}F(y_t,\mu_t,\nu)],
    \qquad
    N(\theta_t,\nu)=\tanh\nu-
    \E_\ast[\nabla_\nu F(y_t,\mu_t,\nu)].
    \]
    The same integrable bounds permit differentiation through the expectation
    \citep[Theorem~1.5.8]{durrett2019probability}. Since
    \(\nabla_{\theta_t}\mu_t=\bx\), we isolate the expected derivatives in these expressions:
    \[
    \Sigma \left(\theta_t - M(\theta_t, \nu) \right)
    =\E_{\ast}[\bx \partial_{\mu_t} F(y_t, \mu_t, \nu)]
    = \E_{\ast}[\nabla_{\theta_t} F(y_t, \mu_t, \nu)]
    = \nabla_{\theta_t} \E_{\ast}[F(y_t, \mu_t, \nu)]
    = \nabla_{\theta_t} \calH(\theta_t, \nu),
    \]
    \[
    \tanh \nu - N(\theta_t, \nu) = \E_{\ast}[\nabla_{\nu} F(y_t, \mu_t, \nu)] = \nabla_{\nu} \E_{\ast}[F(y_t, \mu_t, \nu)] = \nabla_{\nu} \calH(\theta_t, \nu).
    \]
\end{proof}

\begin{myframe}
\begin{main theorem}[Main Theorem~\ref{thm:information_theoretical_decomposition}]
    We define the expected squared activation
    \(\calA(\theta_t, \nu) := \frac{1}{2}\E_{\ast}[\mu_t^2 \tanh^2(\mu_t y_t + \nu)]\)
    and its ground truth counterpart
    \(\calA(\theta_t^*, \nu^*) := \frac{1}{2}\E_{\ast}[(\mu_t^*)^2 \tanh^2(\mu_t^* y_t + \nu^\ast)]\).
    The score matching loss satisfies
    \[
    \bar{\alpha}_t \calL_t(\theta, \nu) = \langle \nabla_{\theta_t}\calH(\theta_t, \nu), \theta_t \rangle
    - \calA(\theta_t, \nu) + \calA(\theta_t^*, \nu^*).
    \]
    We also define the latent variance
    \(V(\theta_t, \nu) \equiv \frac{1}{2}\E_{\ast}[\mu_t^2
    \sech^2(\mu_t y_t + \nu)]\). The equivalent representation is
    \[
    \bar{\alpha}_t \calL_t(\theta, \nu) =
    \|\theta_t\|_{\Sigma}^2/2 -  \langle \Sigma M(\theta_t, \nu), \theta_t \rangle + \|\theta_t^\ast\|_{\Sigma}^2/2 + V(\theta_t, \nu) - V(\theta_t^\ast, \nu^\ast).
    \]
    Here \(V(\theta_t^\ast, \nu^\ast) \equiv \frac{1}{2}\E_{\ast}[(\mu_t^\ast)^2 \sech^2(\mu_t^\ast y_t + \nu^\ast)]\) is constant in the candidate parameters.
\end{main theorem}
\end{myframe}

\begin{proof}[Proof for Main Theorem~\ref{thm:information_theoretical_decomposition}]
    The candidate score is
    \(s_{\theta_t,\nu}(y_t,\bx)=-y_t+\mu_t\tanh(\mu_ty_t+\nu)\).
    The loss compares this score with its ground truth counterpart:
    \[
    \calL_t(\theta, \nu) 
    = \frac{1}{2\bar{\alpha}_t} \E_{y_t, \bx \mid \theta_t^\ast, \pi^\ast}
    \|s_{\theta_t, \nu}(y_t, \bx) - s_{\theta_t^\ast, \nu^\ast}(y_t, \bx)\|^2.
    \]
    By Lemmas~\ref{lemma:nll} and~\ref{lemma:stein_score},
    \(s_{\theta_t,\nu}=-\partial_{y_t}F(y_t,\mu_t,\nu)\) and
    \(p(y_t\mid\bx;\theta_t^\ast,\pi^\ast)
    =\exp[-F(y_t,\mu_t^\ast,\nu^\ast)]\). Hence
    \[
    \calL_t(\theta, \nu) 
    = \frac{1}{2\bar{\alpha}_t} \E_{\bx} \int_{\R} \|-\partial_{y_t} F(y_t, \mu_t, \nu) + \partial_{y_t} F(y_t, \mu_t^\ast, \nu^\ast)\|^2 \exp(-F(y_t, \mu_t^\ast, \nu^\ast)) \d y_t.
    \]
    To rewrite the score products in terms of response derivatives,
    we integrate by parts in \(y_t\). The boundary terms vanish.
    We apply this step to the mixed term and to the squared ground truth score:
    \[
    2 \bar{\alpha}_t \calL_t(\theta, \nu) 
    = \E_{\ast}\left[\|\partial_{y_t} F(y_t, \mu_t, \nu) \|^2  - 2 \partial^2_{y_t y_t} F(y_t, \mu_t, \nu) + \partial^2_{y_t y_t} F(y_t, \mu_t^*, \nu^*)\right].
    \]
    By Lemma~\ref{lemma:derivatives_negative_log_likelihood},
    \(\|\partial_{y_t}F\|^2-\partial^2_{y_ty_t}F
    =2[y_t\partial_{y_t}F-y_t^2]+(\mu_t^2+y_t^2-1)\) and
    \(\partial^2_{y_ty_t}F=1-\mu_t^2\sech^2(\mu_ty_t+\nu)\).
    From Equation~\eqref{eq:yt}, we also have
    \(\E[y_t^2\mid\bx]=(\mu_t^\ast)^2+1\). These identities yield
    \[
        2 \bar{\alpha}_t \calL_t(\theta, \nu) 
        = 2\E_{\ast}\left[y_t \partial_{y_t} F(y_t, \mu_t, \nu) - y_t^2 \right] + \E_{\bx}[\mu_t^2 + (\mu_t^*)^2] + \E_{\ast}[\mu_t^2 \sech^2(\mu_t y_t +\nu)] - \E_{\ast}[(\mu_t^*)^2 \sech^2(\mu_t^* y_t +\nu^\ast)].
    \]
    Proposition~\ref{prop:em_update} identifies the first expectation as
    \(-\langle\Sigma M(\theta_t,\nu),\theta_t\rangle\). The covariance
    identity is
    \(\E[\mu_t^2+(\mu_t^\ast)^2]
    =\|\theta_t\|_\Sigma^2+\|\theta_t^\ast\|_\Sigma^2\). Thus
    \[
        2 \bar{\alpha}_t \calL_t(\theta, \nu) 
        = \|\theta_t\|_{\Sigma}^2 - 2\langle \Sigma M(\theta_t, \nu), \theta_t \rangle + \|\theta_t^*\|_{\Sigma}^2
        + \E_{\ast}[\mu_t^2 \sech^2(\mu_t y_t +\nu)] - \E_{\ast}[(\mu_t^*)^2 \sech^2(\mu_t^* y_t +\nu^\ast)].
    \]
    This is the latent-variance representation. For the activation form, we combine
    Proposition~\ref{prop:cross_entropy}, the identity
    \(\sech^2(u)=1-\tanh^2(u)\), and the two covariance identities to obtain
    \[
        2 \bar{\alpha}_t \calL_t(\theta, \nu) = 2 \langle \nabla_{\theta_t} \calH(\theta_t, \nu), \theta_t \rangle
        - \E_{\ast}[\mu_t^2 \tanh^2(\mu_t y_t +\nu)] + \E_{\ast}[(\mu_t^*)^2 \tanh^2(\mu_t^* y_t +\nu^\ast)].
    \]
    Division by two and the definitions of \(\calA(\theta_t,\nu)\) and
    \(\calA(\theta_t^\ast,\nu^\ast)\) complete the proof.
\end{proof}

\begin{corollary}[Corollary~\ref{cor:sm_gradients_entropy_activation}]
    If \(\E[\|\bx\|^4]<\infty\), the gradients of
    \(\bar{\alpha}_t\calL_t(\theta,\nu)\) with respect to \(\theta_t\) and
    \(\nu\) are as follows. The \(\nu\)-gradient identity requires only
    \(\E[\|\bx\|^2]<\infty\).
    \[
    \begin{aligned}
        \nabla_{\theta_t} [\bar{\alpha}_t \calL_t(\theta, \nu)] &= \nabla_{\theta_t}\calH(\theta_t, \nu) + \nabla^2_{\theta_t} \calH(\theta_t, \nu) \theta_t - \nabla_{\theta_t} \calA(\theta_t, \nu), \\
        \nabla_\nu [\bar{\alpha}_t \calL_t(\theta, \nu)] &= \langle \nabla^2_{\theta_t \nu} \calH(\theta_t, \nu), \theta_t \rangle - \nabla_\nu \calA(\theta_t, \nu).
    \end{aligned}
    \]
\end{corollary}

\begin{proof}[Proof for Corollary~\ref{cor:sm_gradients_entropy_activation}]
    The regression Hessian integrand of \(\calH\) is bounded in norm by
    \(\|\bx\|^2(1+y_t^2)\). Since
    \(\E[y_t^2\mid\bx]=1+(\mu_t^\ast)^2\), the fourth-moment
    assumption makes this bound integrable. The mixed
    \((\theta_t,\nu)\)-derivative integrand is bounded by
    \(\|\bx\||y_t|\), which requires only second moments.
    The first parameter derivatives of the activation integrand have an
    envelope \(C\|\bx\|^2\) on bounded sets of candidate parameters, using
    the boundedness of \(\tanh(\mu_t y_t+\nu)\) and
    \(\mu_t y_t\sech^2(\mu_t y_t+\nu)\).
    These envelopes justify differentiation under the expectations.
    Differentiate the first identity in Main
    Theorem~\ref{thm:information_theoretical_decomposition} with respect to
    \(\theta_t\) and \(\nu\). For the inner product, the product rule reads
    \(\nabla_{\theta_t}\langle \nabla_{\theta_t} \calH(\theta_t, \nu),
    \theta_t \rangle = \nabla_{\theta_t} \calH(\theta_t, \nu) +
    \nabla^2_{\theta_t} \calH(\theta_t, \nu) \theta_t\) and
    \(\nabla_\nu \langle \nabla_{\theta_t} \calH(\theta_t, \nu), \theta_t
    \rangle = \langle \nabla^2_{\theta_t \nu} \calH(\theta_t, \nu),
    \theta_t \rangle\). These formulas prove the two identities.
\end{proof}

\begin{corollary}[Corollary~\ref{cor:sm_gradients_em_variance}]
    If \(\E[\|\bx\|^4]<\infty\), the gradients of
    \(\bar{\alpha}_t\calL_t(\theta,\nu)\) with respect to \(\theta_t\) and
    \(\nu\) are as follows. The \(\nu\)-gradient identity requires only
    \(\E[\|\bx\|^2]<\infty\).
    \[
    \begin{aligned}
        \nabla_{\theta_t} [\bar{\alpha}_t \calL_t(\theta, \nu)] &= \Sigma (\theta_t - M(\theta_t, \nu)) - \big(\nabla_{\theta_t} M(\theta_t, \nu)\big)^\top \Sigma \theta_t + \nabla_{\theta_t} V(\theta_t, \nu), \\
        \nabla_\nu [\bar{\alpha}_t \calL_t(\theta, \nu)] &= -\langle \nabla_{\theta_t} N(\theta_t, \nu), \theta_t \rangle + \nabla_\nu V(\theta_t, \nu).
    \end{aligned}
    \]
\end{corollary}

\begin{proof}[Proof for Corollary~\ref{cor:sm_gradients_em_variance}]
    The EM Jacobian integrand is bounded by
    \(\|\Sigma^{-1}\|\,\|\bx\|^2y_t^2\), so its expectation is finite
    under fourth moments. The integrands for \(\nabla_\nu M\) and
    \(\nabla_{\theta_t}N\) are bounded by a constant times
    \(\|\bx\||y_t|\), and the first derivatives of the latent-variance
    integrand are bounded by \(C\|\bx\|^2\) on bounded
    sets of candidate parameters. Thus second moments suffice for the
    \(\nu\)-gradient identity.
    Differentiate the second identity in Main
    Theorem~\ref{thm:information_theoretical_decomposition} with respect to
    \(\theta_t\) and \(\nu\). The quadratic and inner-product terms have derivatives
    \(\nabla_{\theta_t}\|\theta_t\|_{\Sigma}^2 = 2\Sigma \theta_t\) and
    \(\nabla_{\theta_t} \langle \Sigma M(\theta_t, \nu), \theta_t \rangle
    = \Sigma M(\theta_t, \nu) + \big(\nabla_{\theta_t} M(\theta_t,
    \nu)\big)^\top \Sigma \theta_t\). By Proposition~\ref{prop:em_update}, the mixing derivative satisfies
    \(\nabla_\nu \langle \Sigma M(\theta_t, \nu), \theta_t \rangle
    = \langle \Sigma \nabla_\nu M(\theta_t, \nu), \theta_t \rangle
    = \langle \nabla_{\theta_t} N(\theta_t, \nu), \theta_t \rangle\).
    These formulas prove the two identities.
\end{proof}

We define the normalized EM operators at time step \(t\) by
\[
    \bar{M}_t(\theta, \nu) := M(\theta_t, \nu) / \sqrt{\bar{\alpha}_t},\quad
    \bar{N}_t(\theta, \nu) := (N(\theta_t, \nu)-(1-\bar{\alpha}_t)\tanh \nu) / \bar{\alpha}_t.
\]
They satisfy
\(\theta - \bar{M}_t(\theta, \nu) = (\theta_t - M(\theta_t, \nu))/\sqrt{\bar{\alpha}_t}\),
\(\tanh \nu - \bar{N}_t(\theta, \nu) = (\tanh \nu - N(\theta_t, \nu))/\bar{\alpha}_t\),
and \(\Sigma \nabla_{\nu} \bar{M}_t(\theta, \nu) = \nabla_{\theta} \bar{N}_t(\theta, \nu)\).
We also set \(V_t(\theta,\nu):=V(\theta_t,\nu)/\bar\alpha_t\).

\begin{proposition}[Proposition~\ref{prop:sm_gradients}]
    If \(\E[\|\bx\|^4]<\infty\), the gradients of the score matching
    loss in Equation~\eqref{eq:loss} are as follows. The \(\nu\)-gradient
    identity requires only \(\E[\|\bx\|^2]<\infty\).
    \[
    \begin{aligned}
        \nabla_{\theta} \calL_t(\theta, \nu) &= \Sigma (\theta - \bar{M}_t(\theta, \nu)) - \left( \nabla_{\theta}\bar{M}_t(\theta, \nu) \right)^\top \Sigma \theta + \nabla_\theta V_t(\theta, \nu),\\
        \nabla_{\nu} \calL_t(\theta, \nu) &= -\langle \nabla_{\theta} \bar{N}_t(\theta, \nu), \theta \rangle + \nabla_{\nu} V_t(\theta, \nu).
    \end{aligned}
    \]
\end{proposition}

\begin{remark}
    Use one \(\xi\sim\calN(0,1)\), independent of \((y_0,\bx)\),
    to couple the marginal diffused responses by
    \(y_t=\sqrt{\bar\alpha_t}y_0+\sqrt{1-\bar\alpha_t}\xi\).
    Then \(y_t\to y_0\) almost surely as \(\bar\alpha_t\to1\).
    The functions \(\tanh\) and \(\sech^2\) are bounded. The operator
    integrands and their first parameter derivatives therefore have a common
    integrable envelope when \(\E[\|\bx\|^4]<\infty\). The Gaussian noise has
    moments of every order. These bounds allow passage to the limit by dominated convergence:
    \(\bar{M}_t(\theta, \nu) \to M(\theta, \nu)\),
    \(\bar{N}_t(\theta, \nu) \to N(\theta, \nu)\), and
    \(V_t(\theta, \nu) \to V(\theta, \nu)\), together with the parameter
    derivative limits used in Proposition~\ref{prop:sm_gradients}.
    Substituting these limits into that proposition, we obtain
    \[
    \begin{aligned}
        \lim_{\bar{\alpha}_t\to 1}\nabla_{\theta} \calL_t(\theta, \nu) &= \Sigma (\theta - M(\theta, \nu)) - \left( \nabla_{\theta}M(\theta, \nu) \right)^\top \Sigma \theta + \nabla_\theta V(\theta, \nu),\\
        \lim_{\bar{\alpha}_t\to 1}\nabla_{\nu} \calL_t(\theta, \nu) &= -\langle \nabla_{\theta} N(\theta, \nu), \theta \rangle + \nabla_{\nu} V(\theta, \nu).
    \end{aligned}
    \]
    Proposition~\ref{prop:low_noise_em} strengthens these limits by proving an
    \(\calO(1-\bar\alpha_t)\) remainder.
\end{remark}

\begin{proof}[Proof for Proposition~\ref{prop:sm_gradients}]
    Recall the normalized EM operators
    \(\bar{M}_t(\theta, \nu) := M(\theta_t, \nu) / \sqrt{\bar{\alpha}_t}\),
    \(\bar{N}_t(\theta, \nu) := (N(\theta_t, \nu)-(1-\bar{\alpha}_t)\tanh \nu) / \bar{\alpha}_t\),
    and \(V_t(\theta, \nu) := V(\theta_t, \nu)/\bar{\alpha}_t\).
    These definitions guarantee
    \(\theta - \bar{M}_t(\theta, \nu) = (\theta_t - M(\theta_t, \nu))/\sqrt{\bar{\alpha}_t}\),
    \(\tanh \nu - \bar{N}_t(\theta, \nu) = (\tanh \nu - N(\theta_t, \nu))/\bar{\alpha}_t\),
    and \(\Sigma \nabla_{\nu} \bar{M}_t(\theta, \nu) = \nabla_{\theta} \bar{N}_t(\theta, \nu)\).
    Dividing the identities of Corollary~\ref{cor:sm_gradients_em_variance} by \(\bar{\alpha}_t\) and applying the chain rule with \(\theta_t = \sqrt{\bar{\alpha}_t}\,\theta\) yields the stated gradients of \(\calL_t(\theta, \nu)\).
\end{proof}

\begin{proposition}[Proposition~\ref{prop:low_noise_em}]
    If \(\E[\|\bx\|^4]<\infty\), then the gradients of the score
    matching loss satisfy the following expansions as
    \(\bar{\alpha}_t\to1\).
    \[
    \begin{aligned}
        \nabla_{\theta} \calL_t(\theta, \nu) &= \Sigma (\theta - M(\theta, \nu)) - \left( \nabla_{\theta}M(\theta, \nu) \right)^\top \Sigma \theta + \nabla_\theta V(\theta, \nu) + \calO(1-\bar{\alpha}_t),\\
        \nabla_{\nu} \calL_t(\theta, \nu) &= -\langle \nabla_{\theta} N(\theta, \nu), \theta \rangle + \nabla_{\nu} V(\theta, \nu) + \calO(1-\bar{\alpha}_t).
    \end{aligned}
    \]
\end{proposition}

\begin{proof}[Proof of Proposition~\ref{prop:low_noise_em}]
    We write \(\mu_0=\langle\theta,\bx\rangle\),
    \(\mu_0^\ast=\langle\theta^\ast,\bx\rangle\), and
    \(a=\bar\alpha_t\). Since the
    loss depends only on the marginal law at each scale,
    Equation~\eqref{eq:yt} permits the common coupling
    \[
        y_t=(-1)^{z+1}\sqrt a\,\mu_0^\ast+\varepsilon,\qquad
        \mu_t=\sqrt a\,\mu_0,\qquad
        \mu_t^\ast=\sqrt a\,\mu_0^\ast,
    \]
    where \(\varepsilon\sim\calN(0,1)\) is independent of \((\bx,z)\).
    In Equations~\eqref{eq:score} and~\eqref{eq:loss}, the linear terms
    in the two scores cancel. We therefore work with the normalized score difference
    \[
        G_t:=\mu_0\tanh(\mu_t y_t+\nu)
        -\mu_0^\ast\tanh(\mu_t^\ast y_t+\nu^\ast),\qquad
        \calL_t(\theta,\nu)=\frac12\E_{\bx,z,\varepsilon}[G_t^2].
    \]
    The bounds below justify the identities
    \[
        \nabla_\theta\calL_t=\E[G_t\nabla_\theta G_t],\qquad
        \nabla_\nu\calL_t=\E[G_t\nabla_\nu G_t].
    \]

    We now control the dependence on \(a\) near one. For
    \(a\in[1/2,1]\), the derivatives of the two activation arguments are
    \[
        \partial_a(\mu_t y_t)
        =\frac{\mu_t y_t}{2a}
        +\frac{(-1)^{z+1}}{2}\mu_0\mu_0^\ast,\qquad
        \partial_a(\mu_t^\ast y_t)
        =\frac{\mu_t^\ast y_t}{2a}
        +\frac{(-1)^{z+1}}{2}(\mu_0^\ast)^2.
    \]
    For each fixed \(c\in\R\), the function \(\tanh(u+c)\) and the
    quantities \(u^k\tanh^{(j)}(u+c)\), with
    \(j\in\{1,2\}\) and \(k\in\{0,1,2\}\), are bounded over \(u\in\R\).
    Applying the product and chain rules with these bounds, we can choose a
    constant \(C\), independent of \(a\), such that
    \[
        |G_t|+\|\nabla_\theta G_t\|+|\nabla_\nu G_t|
        \le C\|\bx\|,\qquad
        |\partial_aG_t|+\|\partial_a\nabla_\theta G_t\|
        +|\partial_a\nabla_\nu G_t|
        \le C(\|\bx\|+\|\bx\|^3).
    \]
    The \(a\)-derivatives of the two loss-gradient integrands
    \(G_t\nabla_\theta G_t\) and \(G_t\nabla_\nu G_t\) are thus bounded by
    \(C(\|\bx\|^2+\|\bx\|^4)\). The fourth-moment assumption makes this
    envelope integrable. Differentiation under the expectation is valid
    \citep[Theorem~2.27(b)]{folland1999real}, and
    \[
        \sup_{a\in[1/2,1]}
        \left\{\|\partial_a\nabla_\theta\calL_t(\theta,\nu)\|
        +|\partial_a\nabla_\nu\calL_t(\theta,\nu)|\right\}<\infty.
    \]
    Integration from \(a\) to one shows that both loss gradients equal their
    endpoint values plus \(\calO(1-a)\). At \(a=1\), the normalized operators
    reduce to \(M\), \(N\), and \(V\).
    Proposition~\ref{prop:sm_gradients} identifies the endpoint gradients
    with the two leading terms in the statement. Substituting
    \(a=\bar\alpha_t\) completes the proof.

\end{proof}

\begin{proposition}[Proposition~\ref{prop:high_noise_effective_coefficient}]
    If \(\E[\|\bx\|^2] < \infty\), then \(\lim_{\bar{\alpha}_t \to 0} \calL_t(\theta, \nu) = \frac{1}{2}\|\theta \tanh \nu - \theta^\ast \tanh \nu^\ast\|_{\Sigma}^2\). The condition \(\E[\|\bx\|^4] < \infty\) suffices for the \(\calO(\bar\alpha_t)\) gradient expansions
    \[
        \nabla_{\theta} \calL_t(\theta, \nu) = \nabla_{\theta} \lim_{\bar{\alpha}_t \to 0} \calL_t(\theta, \nu) + \calO(\bar{\alpha}_t),\qquad \nabla_{\nu} \calL_t(\theta, \nu) = \nabla_{\nu} \lim_{\bar{\alpha}_t \to 0} \calL_t(\theta, \nu) + \calO(\bar{\alpha}_t)
    \]
    as \(\bar{\alpha}_t \to 0\).
\end{proposition}

\begin{proof}[Proof of Proposition~\ref{prop:high_noise_effective_coefficient}]
    With \(\mu_0=\langle\theta,\bx\rangle\) and \(\mu_0^\ast=\langle\theta^\ast,\bx\rangle\), set
    \[
        f(w):=\frac12\left[\mu_0\tanh(\nu+\mu_0w)-\mu_0^\ast\tanh(\nu^\ast+\mu_0^\ast w)\right]^2.
    \]
    Substituting the ground truth model in Equation~\eqref{eq:yt} into the normalized loss in Equation~\eqref{eq:loss} yields
    \[
        \calL_t(\theta,\nu)=\E_{\bx,z,\varepsilon_t}\!\left[f\!\left(\sqrt{\bar\alpha_t}\,\varepsilon_t+\bar\alpha_t(-1)^{z+1}\mu_0^\ast\right)\right].
    \]
    Since \(|f(w)|\le\frac12(|\mu_0|+|\mu_0^\ast|)^2\), dominated convergence under \(\E[\|\bx\|^2]<\infty\) yields
    \[
        \lim_{\bar\alpha_t\to0}\calL_t(\theta,\nu)=\E_{\bx}[f(0)]=\frac12\|\theta\tanh\nu-\theta^\ast\tanh\nu^\ast\|_\Sigma^2.
    \]

    To control the gradient remainder, we differentiate the Gaussian expectation
    in its mean and variance. Conditional on \((\bx,z)\), Gaussian integration
    by parts shows that, for \(a>0\),
    \[
        \frac{\d}{\d a}\E_{\varepsilon_t}\!\left[f\!\left(\sqrt a\,\varepsilon_t+a(-1)^{z+1}\mu_0^\ast\right)\right]=\E_{\varepsilon_t}\!\left[(-1)^{z+1}\mu_0^\ast f'\!\left(\sqrt a\,\varepsilon_t+a(-1)^{z+1}\mu_0^\ast\right)+\frac12f''\!\left(\sqrt a\,\varepsilon_t+a(-1)^{z+1}\mu_0^\ast\right)\right].
    \]
    For \(j\in\{1,2,3\}\), \(k\in\{0,1,2\}\), and \(\eta\) in any
    bounded set, the functions \(\tanh\) and
    \(v^k\tanh^{(j)}(\eta+v)\) have a common finite bound. Applying the product
    and chain rules with this bound, we can choose a constant \(C<\infty\), common to a bounded
    neighborhood of \((\theta,\nu)\), such that the derivatives in \(\mu_0\)
    and \(\nu\) of the integrand on the right satisfy
    \[
        \|\bx\|\left|\partial_{\mu_0}\left\{(-1)^{z+1}\mu_0^\ast f'(w)+\frac12f''(w)\right\}\right|+\left|\nabla_\nu\left\{(-1)^{z+1}\mu_0^\ast f'(w)+\frac12f''(w)\right\}\right|\le C\bigl(1+\|\bx\|^4\bigr)
    \]
    for every \(w\in\R\). The fourth-moment condition makes this envelope integrable. It ensures continuity of the parameter gradients at \(a=0\) and permits differentiation under the expectations. We then integrate the preceding Gaussian identity from zero to \(\bar\alpha_t\). Since \(\nabla_\theta=\bx\partial_{\mu_0}\), we obtain
    \[
        \nabla_\theta\calL_t(\theta,\nu)=\nabla_\theta\lim_{\bar\alpha_t\to0}\calL_t(\theta,\nu)+\calO(\bar\alpha_t),\qquad \nabla_\nu\calL_t(\theta,\nu)=\nabla_\nu\lim_{\bar\alpha_t\to0}\calL_t(\theta,\nu)+\calO(\bar\alpha_t),
    \]
    which proves the claim.
\end{proof}

\section{Proofs for High-Noise Gradient Dynamics and Score Matching Blindness}\label{sup:theory}
In this appendix we present the full quantitative statement
(Theorem~\ref{thm:global_convergence_high_noise_full}) and proof of
Main Theorem~\ref{thm:global_convergence_high_noise}, prove the complementary
high-SNR blindness result in Proposition~\ref{prop:pointwise_sm_blindness}, and
establish the companion hard-assignment limit that supports the comparison
with population EM.

\begin{myframe}
\begin{main theorem}[Formal statement of Main Theorem~\ref{thm:global_convergence_high_noise}]\label{thm:global_convergence_high_noise_full}
    Assume \(\Sigma=\E[\bx\bx^\top]=I_d\), set
    \(c^*:=\theta^*\tanh\nu^*\), and define the high-noise loss
    \[
        \calL(\theta,\nu)
        :=\lim_{\bar{\alpha}_t\to0}\calL_t(\theta,\nu)
        =\frac12\|\theta\tanh\nu-c^*\|^2.
    \]
    For \(\vartheta=(\theta,\nu)\), set
    \(J(\vartheta):=\frac12(\|\theta\|^2-\sinh^2\nu)\) and
    \(x_+:=\max\{x,0\}\).  Given \(\eta>0\), let
    \(\Phi(\vartheta):=\vartheta-\eta\nabla\calL(\vartheta)\) and
    \[
        \calD_\eta:=\left\{\vartheta\in\R^{d+1}:32\eta\max\!\left\{\calL(\vartheta),[J(\vartheta)+\calL(\vartheta)]_+,\|c^*\|^2,1\right\}\le1\right\}.
    \]
    The set \(\calD_\eta\) is forward invariant under \(\Phi\), and
    the loss is nonincreasing along its trajectories. If
    \(\vartheta_0\in\calD_\eta\) and
    \(\vartheta_{k+1}=\Phi(\vartheta_k)\), then, at every index
    \(k\in\mathbb Z_{\ge0}\), \(\vartheta_k\in\calD_\eta\) and
    \(\calL(\vartheta_{k+1})\le\calL(\vartheta_k)\).
    Write \(\calL_0:=\calL(\vartheta_0)\) and
    \(B_0:=4\max\{\sqrt{\calL_0},\sqrt{[J(\vartheta_0)+\calL_0]_+},\|c^*\|,1\}\).
    \begin{enumerate}[label=(\alph*)]
        \item \textbf{Degenerate target.} If \(c^*=\mathbf0\) and
        \(\calL_0>0\), then
        \[
            \calL(\vartheta_K)
            \le \left(\calL_0^{-1/2}+\frac{\eta K}{\sqrt6\,B_0}\right)^{-2},
            \qquad K\in\mathbb Z_{\ge0}.
        \]
        Thus \(\calL(\vartheta_K)=\calO(K^{-2})\).  If \(\calL_0=0\),
        all subsequent losses vanish.
        \item \textbf{Null saddle basin.} Suppose
        \(c^*\neq\mathbf0\), and define
        \[
            \calS:=\{(\theta,0):\langle\theta,c^*\rangle=0\},\qquad \calW^s(\calS):=\left\{\vartheta\in\R^{d+1}:\exists\vartheta_s\in\calS,\ \lim_{k\to\infty}\Phi^k(\vartheta)=\vartheta_s\right\}.
        \]
        The exceptional set
        \(\calD_\eta\cap\calW^s(\calS)\) has zero
        \((d+1)\)-dimensional Lebesgue measure.  Initializations in this set
        satisfy \(\calL(\vartheta_k)\to\|c^*\|^2/2\).
        \item \textbf{Generic geometric convergence.} Suppose
        \(c^*\neq\mathbf0\) and
        \(\vartheta_0\in\calD_\eta\setminus\calW^s(\calS)\).
        Then there exist \(\tau\in\mathbb Z_{\ge0}\) and
        \(\underline u>0\), both depending on \(\vartheta_0\), such that
        \[
            \calL(\vartheta_{\tau+j})
            \le e^{-\eta\underline u^2j}\calL(\vartheta_\tau),
            \qquad j\in\mathbb Z_{\ge0}.
        \]
        With
        \(K_\epsilon:=\inf\{k\in\mathbb Z_{\ge0}:\calL(\vartheta_k)\le\epsilon\}\),
        the required iteration count satisfies
        \(K_\epsilon=\calO_{\vartheta_0,\eta,c^*}(\log(1/\epsilon))\) as
        \(\epsilon\downarrow0\).
    \end{enumerate}
\end{main theorem}
\end{myframe}
\begin{proof}[Proof of Theorem~\ref{thm:global_convergence_high_noise}]
    Fix \(\vartheta_0=(\theta_0,\nu_0)\in\calD_\eta\), and write
    \(J_0:=J(\vartheta_0)\), \(\calL_0:=\calL(\vartheta_0)\),
    \(B:=B_0\), and \(C:=\|c^*\|\). By admissibility,
    \[
        2\eta B^2=32\eta\max\{\calL_0,[J_0+\calL_0]_+,C^2,1\}\le1,\qquad \eta\le\frac1{2B^2}\le\frac1{32}.
    \]

    Define the confinement set and its ambient convex cylinder by
    \[
        \Omega:=\{\vartheta:\calL(\vartheta)\le\calL_0,\
        J(\vartheta)+\calL(\vartheta)\le J_0+\calL_0\},\qquad
        \overline\Omega:=\{(\theta,\nu):\|\theta\|\le B,\ \nu\in\mathbb R\}.
    \]
    We combine the loss and energy bounds to prove
    \(\Omega\subset\overline\Omega\). This confinement bounds
    \(\theta\), while \(\nu\) ranges over \(\mathbb R\) in the ambient cylinder.

    By the definition of \(B\),
    \[
        B\ge4,\qquad C\le\frac B4,\qquad C^2\le\frac{B^2}{16},\qquad \calL_0\le\frac{B^2}{16},\qquad [J_0+\calL_0]_+\le\frac{B^2}{16}.
    \]
    For \((\theta,\nu)\in\Omega\), the product satisfies
    \[
        \|\theta\|\tanh|\nu|\le\|\theta\tanh\nu-c^*\|+C\le\sqrt{2\calL_0}+C\le\sqrt{2C^2+4\calL_0}.
    \]
    If \(\nu=0\), then
    \(\|\theta\|^2+C^2=2[J(\theta,0)+\calL(\theta,0)]\), so membership in
    \(\Omega\) implies
    \[
        \|\theta\|^2+C^2\le2(J_0+\calL_0)\quad\Longrightarrow\quad J_0+\calL_0\ge0,\qquad \|\theta\|^2\le2(J_0+\calL_0)\le\frac18B^2.
    \]
    If \(\nu\neq0\), then \(\tanh^{-2}\nu=1+\sinh^{-2}\nu\), so the loss
    bound above controls the regression norm when the imbalance is large.
    The energy bound defining \(\Omega\) controls it when the imbalance
    is small:
    \[
        \|\theta\|^2\le(2C^2+4\calL_0)(1+\sinh^{-2}\nu),\qquad \|\theta\|^2=\sinh^2\nu+2J(\vartheta)\le\sinh^2\nu+2(J_0+\calL_0).
    \]
    Moreover,
    \[
        2C^2+4\calL_0\le\frac38B^2,\qquad 2(J_0+\calL_0)\le2[J_0+\calL_0]_+\le\frac18B^2.
    \]
    Here we used \(J_0+\calL_0\le[J_0+\calL_0]_+\).
    Taking the minimum of the two preceding bounds and successively applying
    \(\min\{a+b,c+d\}\le\max\{a,c\}+\min\{b,d\}\) and
    \(\min\{a/x,x\}\le\sqrt a\) for \(a\ge0\) and \(x>0\) therefore yields
    \[
        \|\theta\|^2\le\max\{2C^2+4\calL_0,2(J_0+\calL_0)\}+\min\!\left\{\frac{2C^2+4\calL_0}{\sinh^2\nu},\sinh^2\nu\right\}\le\frac38B^2+\sqrt{\frac38}\,B<\frac34B^2.
    \]
    The last strict inequality follows from \(B\ge4\).  Thus
    \(\sup_{\Omega}\|\theta\|\le\sqrt3B/2<B\), proving
    \(\Omega\subset\overline\Omega\).
    Put \(u:=\tanh\nu\in(-1,1)\), so \(\nabla_\nu u=1-u^2\).  The first
    derivatives are
    \[
        \nabla_\theta\calL=u(u\theta-c^*),\qquad \nabla_\nu\calL=(1-u^2)\langle\theta,u\theta-c^*\rangle=(1-u^2)(u\|\theta\|^2-\langle\theta,c^*\rangle).
    \]
    The second derivatives are
    \[
        \nabla^2_{\theta\theta}\calL=u^2I_d,\qquad \nabla^2_{\theta\nu}\calL=(1-u^2)\{(u\theta-c^*)+u\theta\}=(1-u^2)(2u\theta-c^*),
    \]
    and
    \[
        \nabla^2_{\nu\nu}\calL=-2u(1-u^2)(u\|\theta\|^2-\langle\theta,c^*\rangle)+(1-u^2)^2\|\theta\|^2=(1-u^2)\{2u\langle\theta,c^*\rangle+(1-3u^2)\|\theta\|^2\}.
    \]
    Thus
    \[
        \nabla^2\calL=\begin{pmatrix}u^2I_d&(1-u^2)(2u\theta-c^*)\\(1-u^2)(2u\theta-c^*)^\top&(1-u^2)\{2u\langle\theta,c^*\rangle+(1-3u^2)\|\theta\|^2\}\end{pmatrix}.
    \]
    By symmetry it suffices to consider \(u\in[0,1]\).
    Since \(\frac{\mathrm d}{\mathrm du}[2u(1-u^2)]=2(1-3u^2)\), the first
    expression is maximized at \(u=1/\sqrt3\).  In addition,
    \((1-u^2)(1-3u^2)=3(u^2-2/3)^2-1/3\in[-1/3,1]\).  Hence
    \[
        \max_{|u|\le1}|2u(1-u^2)|=\frac4{3\sqrt3},\qquad \max_{|u|\le1}|(1-u^2)(1-3u^2)|=1.
    \]
    On \(\overline\Omega\), \(\|\theta\|\le B\), \(C\le B/4\), and
    \(B\ge4\), and therefore
    \[
        \|\nabla^2_{\theta\theta}\calL\|_2=u^2\le1\le\frac{B^2}{16},
    \]
    \[
        \|\nabla^2_{\theta\nu}\calL\|_2\le\frac4{3\sqrt3}\|\theta\|+C\le\frac4{3\sqrt3}B+\frac B4\le\left(\frac1{3\sqrt3}+\frac1{16}\right)B^2,
    \]
    and
    \[
        |\nabla^2_{\nu\nu}\calL|\le\frac4{3\sqrt3}\|\theta\|C+\|\theta\|^2\le\frac4{3\sqrt3}B\frac B4+B^2=\left(1+\frac1{3\sqrt3}\right)B^2.
    \]
    The last step in the mixed-block bound uses \(B\le B^2/4\), which also
    implies \(B/4\le B^2/16\). To control the Taylor remainder, we bound
    the Hessian norm by the sum of its four block norms:
    \(\|\nabla^2\calL\|_2\le\|\nabla^2_{\theta\theta}\calL\|_2+
    2\|\nabla^2_{\theta\nu}\calL\|_2+|\nabla^2_{\nu\nu}\calL|\).
    Since \(1/\sqrt3<3/4\), the resulting coefficient satisfies
    \(19/16+1/\sqrt3<31/16<2\). Consequently,
    \[
        \sup_{\vartheta\in\overline\Omega}\|\nabla^2\calL(\vartheta)\|_2\le\left\{\frac1{16}+2\left(\frac1{3\sqrt3}+\frac1{16}\right)+\left(1+\frac1{3\sqrt3}\right)\right\}B^2=\left(\frac{19}{16}+\frac1{\sqrt3}\right)B^2<2B^2.
    \]

    \textbf{Step 1. Descent, confinement, and admissibility.}
    Since \(\vartheta_0\in\Omega\), the induction starts at \(k=0\).
    Suppose as the induction hypothesis that \(\vartheta_k\in\Omega\), and set
    \(u_k:=\tanh\nu_k\) and \(q_k:=u_k\theta_k-c^*\).  The first component
    of the gradient is \(\nabla_\theta\calL(\vartheta_k)=u_kq_k\).  Since
    \(\eta\le1/32\), the residual bound is
    \[
        \|q_k\|\le|u_k|\|\theta_k\|+C\le\frac{\sqrt3}{2}B+\frac B4=\left(\frac{\sqrt3}{2}+\frac14\right)B<\frac54B,
    \]
    and the resulting one-step movement satisfies
    \[
        \|\theta_{k+1}\|=\|\theta_k-\eta u_kq_k\|\le\|\theta_k\|+\eta|u_k|\|q_k\|\le\left(\frac{\sqrt3}{2}+\frac1{32}\frac54\right)B=\left(\frac{\sqrt3}{2}+\frac5{128}\right)B<B.
    \]
    Here the last strict inequality follows from
    \(\sqrt3/2<7/8\) and \(7/8+5/128=117/128<1\).  Since
    \(\overline\Omega\) allows every \(\nu\in\mathbb R\), this bound is
    sufficient to place \(\vartheta_{k+1}\) in \(\overline\Omega\).
    Thus \(\vartheta_k,\vartheta_{k+1}\in\overline\Omega\), and convexity of
    \(\overline\Omega\) places the segment joining them in \(\overline\Omega\).
    The Hessian bound therefore applies throughout the Taylor expansion
    between consecutive iterates. We obtain
    \[
        \calL(\vartheta_{k+1})-\calL(\vartheta_k)\le\langle\nabla\calL(\vartheta_k),-\eta\nabla\calL(\vartheta_k)\rangle+B^2\|\eta\nabla\calL(\vartheta_k)\|^2=-\eta(1-\eta B^2)\|\nabla\calL(\vartheta_k)\|^2\le-\frac{\eta}{2}\|\nabla\calL(\vartheta_k)\|^2.
    \]
    Here the Taylor remainder coefficient is
    \(B^2=\frac12(2B^2)\), and the final inequality uses
    \(\eta B^2\le1/2\).

    The component updates are
    \[
        \theta_{k+1}=\theta_k-\eta u_kq_k,\qquad \nu_{k+1}=\nu_k-\eta(1-u_k^2)\langle\theta_k,q_k\rangle.
    \]
    Expanding the squared regression norm and substituting the imbalance
    increment, we obtain
    \[
        \|\theta_{k+1}\|^2-\|\theta_k\|^2=-2\eta u_k\langle\theta_k,q_k\rangle+\eta^2u_k^2\|q_k\|^2=\eta^2u_k^2\|q_k\|^2+\frac{2u_k}{1-u_k^2}(\nu_{k+1}-\nu_k).
    \]
    To control the change in the energy, we apply the supporting-line
    inequality for the convex function \(\nu\mapsto\sinh^2\nu\):
    \[
        \sinh^2\nu_{k+1}-\sinh^2\nu_k\ge\sinh(2\nu_k)(\nu_{k+1}-\nu_k)=\frac{2u_k}{1-u_k^2}(\nu_{k+1}-\nu_k).
    \]
    Combining these two displays with the descent inequality above
    yields
    \[
        J(\vartheta_{k+1})-J(\vartheta_k)\le\frac{\eta^2}{2}u_k^2\|q_k\|^2=\frac{\eta^2}{2}\|\nabla_\theta\calL(\vartheta_k)\|^2\le\frac{\eta^2}{2}\|\nabla\calL(\vartheta_k)\|^2\le\eta\bigl[\calL(\vartheta_k)-\calL(\vartheta_{k+1})\bigr].
    \]
    Because \(\eta<1\), the loss decrease absorbs the possible increase in
    the energy. Combining the two inequalities, we find
    \[
        [J(\vartheta_{k+1})+\calL(\vartheta_{k+1})]-[J(\vartheta_k)+\calL(\vartheta_k)]\le-(1-\eta)[\calL(\vartheta_k)-\calL(\vartheta_{k+1})]\le0.
    \]
    Hence \(\vartheta_{k+1}\in\Omega\), and induction proves confinement.
    The two trajectory-dependent quantities in the admissibility condition
    are therefore nonincreasing:
    \[
        \max\{\calL(\vartheta_{k+1}),[J(\vartheta_{k+1})+\calL(\vartheta_{k+1})]_+,C^2,1\}\le\max\{\calL_0,[J_0+\calL_0]_+,C^2,1\}\le\frac1{32\eta}.
    \]
    Thus admissibility is preserved.  In summary,
    \[
        \vartheta_k\in\Omega\subset\overline\Omega,\qquad \|\theta_k\|\le\frac{\sqrt3}{2}B,\qquad \vartheta_k\in\calD_\eta,\qquad k\in\mathbb Z_{\ge0}.
    \]

    \textbf{Step 2. Degenerate target.}
    Suppose \(c^*=\mathbf0\).  If \(\calL_0=0\),
    the descent inequality from Step~1 shows that all subsequent losses
    vanish.  We may
    therefore assume \(\calL_0>0\) and work until a possible finite iterate
    with zero loss, after which the same conclusion applies.

    At \(c^*=\mathbf0\), the gradient and its squared norm simplify to
    \[
        \nabla_\theta\calL=u^2\theta,\qquad \nabla_\nu\calL=u(1-u^2)\|\theta\|^2,\qquad \|\nabla\calL\|^2=u^4\|\theta\|^2+u^2(1-u^2)^2\|\theta\|^4=2\calL\{u^2+\|\theta\|^2(1-u^2)^2\}.
    \]
    For \(\vartheta\in\Omega\), if \(u^2>1/2\), then
    \[
        u^2+\|\theta\|^2(1-u^2)^2\ge u^2>\frac12\ge\frac{\|\theta\||u|}{\sqrt3B}.
    \]
    If \(u^2\le1/2\), the inequality \((|u|-\|\theta\|/2)^2\ge0\) instead yields
    \[
        u^2+\|\theta\|^2(1-u^2)^2\ge u^2+\frac14\|\theta\|^2\ge\|\theta\||u|\ge\frac{\|\theta\||u|}{\sqrt3B}.
    \]
    Here we used the confinement bound
    \(\|\theta\|\le\sqrt3B/2\) and \(B\ge4\).  We obtain
    \[
        \|\nabla\calL(\vartheta)\|^2=2\calL(\vartheta)\{u^2+\|\theta\|^2(1-u^2)^2\}\ge\frac{2\calL(\vartheta)\sqrt{2\calL(\vartheta)}}{\sqrt3B}=2\sqrt{\frac23}\,\frac{\calL(\vartheta)^{3/2}}{B}.
    \]
    As long as \(\calL(\vartheta_{k+1})>0\), we apply the supporting-line
    inequality for \(x\mapsto x^{-1/2}\). This converts the loss decrease
    into an increment bounded below by a constant, using the gradient bound
    just established:
    \[
        \calL(\vartheta_{k+1})^{-1/2}-\calL(\vartheta_k)^{-1/2}\ge\frac{\calL(\vartheta_k)-\calL(\vartheta_{k+1})}{2\calL(\vartheta_k)^{3/2}}\ge\frac{\eta\|\nabla\calL(\vartheta_k)\|^2}{4\calL(\vartheta_k)^{3/2}}\ge\frac{\eta}{\sqrt6\,B}.
    \]
    If \(\calL(\vartheta_K)>0\), telescoping and inversion prove
    \[
        \calL(\vartheta_K)\le\left(\calL_0^{-1/2}+\frac{\eta K}{\sqrt6\,B_0}\right)^{-2}.
    \]
    If \(\calL(\vartheta_K)=0\), the same bound is immediate, and
    the descent inequality makes all later losses vanish.  Thus the displayed
    sublinear bound holds for every \(K\in\mathbb Z_{\ge0}\).

    For Steps 3 through 6, assume \(c^*\ne\mathbf0\), so \(C>0\).

    \textbf{Step 3. Saddle points and the boundary of the sign regions.}
    If \(\vartheta=(\theta,0)\in\calS\), then \(u=0\),
    \(\nabla_\theta\calL(\vartheta)=u(u\theta-c^*)=0\), and
    \(\nabla_\nu\calL(\vartheta)=(1-u^2)\langle\theta,u\theta-c^*\rangle
    =-\langle\theta,c^*\rangle=0\).  Hence \(\Phi(\vartheta)=\vartheta\),
    so an initialization in \(\calS\) remains fixed.

    Fix any boundary index
    \(r\in\mathbb Z_{\ge0}\) with \(\langle\theta_r,c^*\rangle\nu_r=0\) and
    \(\vartheta_r\notin\calS\).  The boundary
    point enters the matching-sign region after one update.  If
    \(\langle\theta_r,c^*\rangle=0\) and \(\nu_r\ne0\), then
    \(u_r\nu_r>0\), \(0<u_r(1-u_r^2)/\nu_r\le1\), and
    \(\eta\|\theta_r\|^2\le3\eta B^2/4\le3/8\).  Therefore,
    \[
        \langle\theta_{r+1},c^*\rangle\nu_r=\eta C^2u_r\nu_r>0,\qquad \frac{\nu_{r+1}}{\nu_r}=1-\eta\|\theta_r\|^2\frac{u_r(1-u_r^2)}{\nu_r}\ge1-\eta\|\theta_r\|^2\ge\frac58>0,\qquad \langle\theta_{r+1},c^*\rangle\nu_{r+1}>0.
    \]
    If
    \(\langle\theta_r,c^*\rangle\ne0\) and \(\nu_r=0\), then
    \(u_r=0\) and \(q_r=-c^*\), so the same conclusion follows from
    the update
    \[
        \theta_{r+1}=\theta_r-\eta u_rq_r=\theta_r,\qquad \nu_{r+1}=\nu_r-\eta(1-u_r^2)\langle\theta_r,q_r\rangle=\eta\langle\theta_r,c^*\rangle,\qquad \langle\theta_{r+1},c^*\rangle\nu_{r+1}=\eta\langle\theta_r,c^*\rangle^2>0.
    \]
    Thus, in either boundary case,
    \[
        \langle\theta_{r+1},c^*\rangle\nu_{r+1}>0.
    \]

    \textbf{Step 4. Matching signs.}
    Because \(\calL(\theta,\nu)=\calL(-\theta,-\nu)\) and
    \(\Phi(-\theta,-\nu)=-\Phi(\theta,\nu)\), it suffices to
    analyze a time \(r\) at which
    \(\langle\theta_r,c^*\rangle>0\) and \(\nu_r>0\).
    For this and the next step, we introduce the normalized projection
    \(w_k:=\langle\theta_k,c^*\rangle/C^2\). In these coordinates, the
    component updates from Step~1 become
    \[
        w_{k+1}=(1-\eta u_k^2)w_k+\eta u_k,\qquad \nu_{k+1}-\nu_k=\eta(1-u_k^2)(C^2w_k-u_k\|\theta_k\|^2).
    \]

    Define \(w_{\min,r}:=\min\{w_r,1\}>0\) and
    \(u_{\min,r}:=\min\{u_r,C^2w_{\min,r}/(2B^2)\}>0\).
    The joint induction starts at \(k=r\).  Suppose that
    \(w_k\ge w_{\min,r}\) and \(u_k\ge u_{\min,r}\) for some \(k\ge r\).
    Since \(1-\eta u_k^2>0\), we can bound the regression recurrence below
    without reversing the inequality:
    \[
        w_{k+1}=(1-\eta u_k^2)w_k+\eta u_k\ge\bigl[(1-\eta u_k^2)+\eta u_k\bigr]\min\{w_k,1\}=\bigl[1+\eta u_k(1-u_k)\bigr]\min\{w_k,1\}\ge w_{\min,r}.
    \]
    Moreover,
    \(0<C^2w_{\min,r}/B^2\le C^2/B^2\le1/16\).  If
    \(u_k\le C^2w_{\min,r}/B^2\), the second matching-sign recurrence and
    the confinement bound \(\|\theta_k\|^2\le3B^2/4\) show that the
    imbalance cannot decrease:
    \[
        C^2w_k-u_k\|\theta_k\|^2\ge C^2w_{\min,r}-\frac34C^2w_{\min,r}=\frac14C^2w_{\min,r}>0\quad\Longrightarrow\quad \nu_{k+1}\ge\nu_k\quad\Longrightarrow\quad u_{k+1}\ge u_k\ge u_{\min,r}.
    \]
    Otherwise, \(u_k>C^2w_{\min,r}/B^2\).  Dropping the nonnegative
    term involving \(C^2w_k\) in that recurrence and using
    \(2\eta B^2\le1\), we obtain
    \[
        \nu_{k+1}\ge\operatorname{arctanh}(u_k)-\eta(1-u_k^2)u_k\|\theta_k\|^2\ge\operatorname{arctanh}(u_k)-\eta B^2u_k\ge\operatorname{arctanh}(u_k)-\frac{u_k}{2}.
    \]
    The map \(x\mapsto\operatorname{arctanh}x-x/2\) has derivative
    \((1+x^2)/[2(1-x^2)]>0\) on \((0,1)\). By the subtraction formula,
    \(\tanh(a-b)\ge\tanh a-\tanh b\) for \(a\ge b\ge0\).  We also have
    \(\tanh x\le x\) for \(x\ge0\).  Hence
    \[
        u_{k+1}\ge\tanh\!\left(\operatorname{arctanh}(u_k)-\frac{u_k}{2}\right)\ge\frac{C^2w_{\min,r}}{B^2}-\tanh\!\left(\frac{C^2w_{\min,r}}{2B^2}\right)\ge\frac{C^2w_{\min,r}}{2B^2}\ge u_{\min,r}.
    \]
    This completes the joint inductive step and yields
    \[
        w_k\ge w_{\min,r},\qquad
        u_k\ge u_{\min,r},\qquad k\in\mathbb Z_{\ge r}.
    \]
    In particular, the matching-sign region is forward invariant.
    For \(k\ge r\), the gradient satisfies
    \[
        \|\nabla\calL(\vartheta_k)\|^2\ge\|\nabla_\theta\calL(\vartheta_k)\|^2=u_k^2\|q_k\|^2=2u_k^2\calL(\vartheta_k)\ge2u_{\min,r}^2\calL(\vartheta_k).
    \]
    Combining this inequality with the descent inequality from Step~1 and using
    \(1-x\le e^{-x}\) yields
    \[
        \calL(\vartheta_{r+j})\le(1-\eta u_{\min,r}^2)^j\calL(\vartheta_r)\le e^{-\eta u_{\min,r}^2j}\calL(\vartheta_r),\qquad j\in\mathbb Z_{\ge0}.
    \]

    \textbf{Step 5. Mixed signs, finite exit or convergence to \(\calS\).}
    Suppose \(\langle\theta_0,c^*\rangle\nu_0<0\), and define the exit time
    \(T:=\inf\{k\in\mathbb Z_{\ge0}:
    \langle\theta_k,c^*\rangle\nu_k\ge0\}\), with
    \(\inf\varnothing:=\infty\).  Thus
    \(T\in\mathbb Z_{\ge0}\cup\{\infty\}\).
    The estimates below apply at indices \(k<T\).

    Put \(\theta_k^\perp:=\theta_k-w_kc^*\) and
    \(G_k:=\calL(\vartheta_k)-C^2/2\).  Then
    \(\langle\theta_k^\perp,c^*\rangle=0\), \(w_ku_k<0\), and
    \[
        G_k
        =\frac12\|\theta_k^\perp\|^2u_k^2
         +\frac{C^2}{2}(w_ku_k-2)w_ku_k>0.
    \]
    The update has the form
    \[
        \theta_{k+1}^\perp=(1-\eta u_k^2)\theta_k^\perp,\qquad w_{k+1}=(1-\eta u_k^2)w_k+\eta u_k,\qquad \nu_{k+1}-\nu_k=\eta C^2(1-u_k^2)\left[w_k-u_k\left(w_k^2+\frac{\|\theta_k^\perp\|^2}{C^2}\right)\right].
    \]
    Because the signs are opposite before exit, the first two recurrences
    imply
    \(\|\theta_k^\perp\|\le\|\theta_0^\perp\|\) and
    \(|w_k|\le\max\{|w_0|,\eta\}\).  These bounds hold at
    \(0\le k\le T\) when \(T<\infty\), and at every finite index when
    \(T=\infty\).
    The final mixed-sign recurrence has sign opposite
    to \(\nu_k\), while its magnitude is at most
    \(\eta[B^2+C^2\max\{|w_0|,\eta\}]\).
    Hence the same index ranges satisfy
    \(|\nu_k|\le\max\{|\nu_0|,\eta[B^2+C^2\max\{|w_0|,\eta\}]\}\).

    Set
    \(c_\nu:=\sech^2(\max\{|\nu_0|,\eta[B^2+C^2\max\{|w_0|,\eta\}]\})>0\)
    and \(c_{\mathrm{mix}}:=\min\{2C^2/B^2,c_\nu C\}>0\).  Orthogonality of
    \(\theta_k^\perp\) and \(c^*\) implies
    \(\|\nabla_\theta\calL(\vartheta_k)\|^2
      =C^2u_k^2+2u_k^2G_k\) and
    \(|\nabla_\nu\calL(\vartheta_k)|\ge c_\nu C^2|w_k|\).
    The confinement bound from Step~1 also implies
    \(G_k\le3B^2u_k^2/8+C^2|w_ku_k|\).
    Since
    \(\frac14C^2u_k^2+c_\nu^2C^4w_k^2
      \ge c_\nu C^3|w_ku_k|\),
    \[
        \|\nabla\calL(\vartheta_k)\|^2\ge\frac34C^2u_k^2+c_\nu C^3|w_ku_k|\ge c_{\mathrm{mix}}\left(\frac38B^2u_k^2+C^2|w_ku_k|\right)\ge c_{\mathrm{mix}}G_k.
    \]
    Here \(c_{\mathrm{mix}}\le2C^2/B^2\le1/8\), so
    \(0<\eta c_{\mathrm{mix}}/2<1\).
    Combining the descent inequality from Step~1 with the preceding mixed-sign
    gradient lower bound yields
    \[
        G_{k+1}
        \le\left(1-\frac{\eta c_{\mathrm{mix}}}{2}\right)G_k
        \le e^{-\eta c_{\mathrm{mix}}/2}G_k,
        \qquad k<T.
    \]
    
    Suppose now that \(T=\infty\). By the preceding contraction estimate,
    \(G_k\to0\), while
    \(\|\nabla\calL(\vartheta_k)\|^2\ge C^2u_k^2\) and
    \(\calL(\vartheta_k)>C^2/2\) throughout the mixed-sign region. Summing
    the descent inequality from Step~1 therefore bounds the total squared
    imbalance:
    \[
        \sum_{k=0}^\infty u_k^2
        \le\frac{2G_0}{\eta C^2}<\infty.
    \]
    This summability bound implies \(u_k\to0\), and hence
    \(\nu_k\to0\).  Since
    \(0\le\eta u_k^2\le1/32\),
    \(-\log(1-\eta u_k^2)\le(32/31)\eta u_k^2\).  Thus the positive
    product \(\prod_{k=0}^\infty(1-\eta u_k^2)\) converges to a positive
    limit, and
    the first mixed-sign recurrence shows that \(\theta_k^\perp\) converges
    to some \(\theta_s\perp c^*\).

    Because \(\nu_k\to0\), its increments vanish.  In the mixed-sign
    region, the final mixed-sign recurrence satisfies
    \(\eta C^2(1-u_k^2)|w_k|\le|\nu_{k+1}-\nu_k|\), so \(w_k\to0\).
    We conclude that
    \[
        \vartheta_k\longrightarrow(\theta_s,0)\in\calS.
    \]
    Thus \(T=\infty\) implies
    \(\vartheta_0\in\calW^s(\calS)\).

    Now suppose \(\vartheta_0\notin\calW^s(\calS)\).  If the
    initialization has matching signs, take \(\tau=0\).  A boundary
    initialization outside \(\calS\) is handled by the boundary-entry
    calculation in Step~3.  For a mixed-sign initialization, the implication
    \(T=\infty\Rightarrow\vartheta_0\in\calW^s(\calS)\) proved above rules out
    \(T=\infty\).  After the finite exit, a landing in \(\calS\) would make
    \(\vartheta_0\) belong to
    \(\calW^s(\calS)\).  Thus the exit either has matching signs or
    lands on the boundary outside \(\calS\), where
    the boundary-entry calculation applies.  Therefore
    \(\tau:=\inf\{k\in\mathbb Z_{\ge0}:
      \langle\theta_k,c^*\rangle\nu_k>0\}<\infty\).

    Set
    \[
        \underline u:=\min\!\left\{|\tanh\nu_\tau|,\frac{\min\{|\langle\theta_\tau,c^*\rangle|,C^2\}}{2B^2}\right\}>0.
    \]
    We apply the matching-sign rate from Step~4 at time \(\tau\):
    \[
        \calL(\vartheta_{\tau+j})
        \le e^{-\eta\underline u^2j}\calL(\vartheta_\tau),
        \qquad j\in\mathbb Z_{\ge0}.
    \]
    If \(\calL(\vartheta_\tau)=0\), then \(K_\epsilon\le\tau\).
    Otherwise,
    \[
        K_\epsilon
        \le\tau+
        \left\lceil
          \frac{[\log(\calL(\vartheta_\tau)/\epsilon)]_+}
               {\eta\underline u^2}
        \right\rceil.
    \]
    This proves the logarithmic complexity for each initial point, with
    \(\tau\) and \(\underline u\) depending on that point.

    If \(\vartheta_0\in\calW^s(\calS)\), continuity implies
    \(\calL(\vartheta_k)\to C^2/2\).

    \textbf{Step 6. Null saddle basin.}
    We adapt the standard center-stable-manifold argument for avoiding saddle
    points under gradient descent~\citep{lee2016gradient}. The argument below
    also treats points at which the update map is singular.
    It remains to prove
    \(\mathrm{Vol}_{d+1}(\calD_\eta\cap\calW^s(\calS))=0\).
    The claim is immediate if \(\calD_\eta\) is empty.  Otherwise, its
    definition implies
    \[
        \eta\le\frac{1}{32\max\{C^2,1\}},\qquad \eta C\le\frac1{32}.
    \]

    \emph{Reachable saddle limits.}
    Let \(\vartheta_0\in
    \calD_\eta\cap\calW^s(\calS)\) and
    \(\Phi^k(\vartheta_0)\to\vartheta_s=(\theta_s,0)\in\calS\).
    The confinement argument above applies to this initialization. By
    the definition of \(\calD_\eta\), it bounds the saddle limit by
    \(\|\theta_s\|^2\le3/(8\eta)\).  Thus all such limits lie in the
    compact set
    \(\calK_\eta:=\{(\theta,0)\in\calS:\|\theta\|^2\le3/(8\eta)\}\).

    \emph{Local center-stable disks.}
    At \(\vartheta_s=(\theta_s,0)\in\calS\), the Hessian kernel is
    \(\{(v,0):v\in\mathbb R^d,\ \langle v,c^*\rangle=0\}\).  Direct
    diagonalization on its two-dimensional complement identifies the two
    remaining eigenvalues:
    \[
        \lambda_\pm(\vartheta_s)
        =\frac12\left(\|\theta_s\|^2\pm\sqrt{\|\theta_s\|^4+4C^2}\right),\qquad \lambda_-<0<\lambda_+.
    \]
    The derivative \(D\Phi(\vartheta_s)\) has eigenvalue \(1\) with
    multiplicity \(d-1\), together with
    \(1-\eta\lambda_+(\vartheta_s)\) and
    \(1-\eta\lambda_-(\vartheta_s)\).
    The inequality \(\sqrt{a^2+4b^2}\le a+2b\),
    the global step-size bounds above, and
    \(\eta\|\theta_s\|^2\le3/8\) imply
    \(0<\eta\lambda_+(\vartheta_s)\le13/32<1\).  Hence the eigenvalues of
    \(D\Phi(\vartheta_s)\) form center, stable, and unstable groups
    whose dimensions are \(d-1\), \(1\), and \(1\).
    In particular, \(D\Phi(\vartheta_s)\) is nonsingular.  The Inverse
    Function Theorem makes \(\Phi\) a \(C^1\) local diffeomorphism near
    each point of \(\calK_\eta\).

    The local discrete-time Center-Stable Manifold Theorem
    \cite[Theorem~III.7, pp.~65\textendash 66]{shub1987global} therefore yields an
    open neighborhood \(U_{\vartheta_s}\) and a \(d\)-dimensional embedded
    \(C^1\) disk \(W^{cs}_{\mathrm{loc}}(\vartheta_s)\) satisfying
    \[
        \left\{\vartheta\in U_{\vartheta_s}:
          \forall n\in\mathbb Z_{\ge0},\quad
          \Phi^n(\vartheta)\in U_{\vartheta_s}\right\}
        \subseteq W^{cs}_{\mathrm{loc}}(\vartheta_s).
    \]
    Each such disk is Lebesgue-null in \(\mathbb R^{d+1}\). To cover all
    possible saddle limits, we use compactness to select points
    \(\vartheta_{s,1},\ldots,\vartheta_{s,N}\in\calK_\eta\)
    such that \(\calK_\eta\subseteq\bigcup_{i=1}^N U_{\vartheta_{s,i}}\).  Write
    \(U_i:=U_{\vartheta_{s,i}}\) and
    \(W_i^{cs}:=W^{cs}_{\mathrm{loc}}(\vartheta_{s,i})\).

    If an admissible trajectory converges to a point of \(\calK_\eta\), it
    eventually remains in some \(U_i\).  The local trapping property then
    places a finite
    iterate in \(W_i^{cs}\).  Thus
    \[
        \calD_\eta\cap\calW^s(\calS)
        \subseteq\bigcup_{i=1}^N\bigcup_{\ell=0}^\infty
          \Phi^{-\ell}(W_i^{cs}).
    \]

    \emph{Null preimages.}
    We next show that inverse images of null sets remain null. The determinant
    of \(D\Phi\) is real analytic. Its value at the origin is
    \(\det D\Phi(0,0)=1-\eta^2C^2>0\), so the determinant is not identically zero. Its critical set
    \(\operatorname{Crit}(\Phi):=
      \{\vartheta:\det D\Phi(\vartheta)=0\}\) is Lebesgue-null
    \cite[Section~4.1, p.~83]{krantz2002primer}.

    Let \(Z\subseteq\mathbb R^{d+1}\) be Lebesgue-null. At every point outside
    \(\operatorname{Crit}(\Phi)\), the Inverse Function Theorem ensures a
    neighborhood \(U\) on which \(\Phi\) has a \(C^1\) inverse. We choose an
    open ball \(B\) whose compact closure lies in \(\Phi(U)\), and set
    \(O=(\Phi|_U)^{-1}(B)\). The inverse derivative is bounded on
    \(\overline B\). Since \(B\) is convex, the inverse is Lipschitz on \(B\).
    By second countability, we can select a collection \(\{O_j\}_{j\ge1}\)
    of these neighborhoods that covers the regular points. Lipschitz maps preserve null sets
    \cite[Sections~2.2 and~2.4]{evans2015measure}, so
    \[
        \Phi^{-1}(Z)\setminus\operatorname{Crit}(\Phi)
        \subseteq\bigcup_{j=1}^\infty
          (\Phi|_{O_j})^{-1}\bigl(Z\cap\Phi(O_j)\bigr),
        \qquad
        \mathrm{Vol}_{d+1}
          \bigl(\Phi^{-1}(Z)
            \setminus\operatorname{Crit}(\Phi)\bigr)=0.
    \]
    The part inside \(\operatorname{Crit}(\Phi)\) is also null. Therefore
    \[
        \mathrm{Vol}_{d+1}(Z)=0\quad\Longrightarrow\quad \mathrm{Vol}_{d+1}\bigl(\Phi^{-1}(Z)\bigr)=0.
    \]

    Repeating the argument shows that each
    \(\Phi^{-\ell}(W_i^{cs})\) is null. The countable union on the
    right of the stable-basin cover above is measurable and null.
    Its subset \(\calD_\eta\cap\calW^s(\calS)\) is measurable and null
    by completeness of Lebesgue measure.
\end{proof}

\begin{proposition}[Proposition~\ref{prop:pointwise_sm_blindness}]
    Assume \(\E[\|\bx\|^2]<\infty\).  Fix
    \(\bar{\alpha}_t\in(0,1]\), a finite model pair
    \((\theta,\nu)\in\R^d\times\R\). Let
    \(\theta^\ast\ne\mathbf0\) vary with \(\|\theta^\ast\|\to\infty\) along
    a fixed ray, while the ground truth mixing weights and all other model and
    data-generating quantities remain fixed. If
    \(\theta\ne\mathbf0\), assume
    \[
        \Pr\!\left(\{\langle\theta,\bx\rangle=0\}\cup\{\langle\theta^\ast,\bx\rangle=0\}\right)=0.
    \]
    The expectations defining \(\calL_t\) and \(V_t\) are taken under the
    corresponding varying ground truth.
    Then, pointwise at the fixed model pair and fixed scale of the diffusion noise level,
    \[
        \lim_{\|\theta^\ast\|\to\infty}|\nabla_\nu\calL_t(\theta,\nu)|=0,
        \qquad
        \lim_{\|\theta^\ast\|\to\infty}V_t(\theta,\nu)=0.
    \]
\end{proposition}

\begin{proof}[Proof of Proposition~\ref{prop:pointwise_sm_blindness}]
    Recall the established scaled projections and forward observation
    \[
        \mu_t:=\sqrt{\bar{\alpha}_t}\langle\theta,\bx\rangle,\qquad \mu_t^\ast:=\sqrt{\bar{\alpha}_t}\langle\theta^\ast,\bx\rangle,\qquad y_t=(-1)^{z+1}\mu_t^\ast+\varepsilon_t,
    \]
    where \(\varepsilon_t\sim\calN(0,1)\) is independent of \((\bx,z)\).
    The \(\nu\)-gradient identity in
    Proposition~\ref{prop:sm_gradients}, which holds under second moments,
    together with the definition of \(V_t\), implies
    \[
        V_t(\theta,\nu)=\frac{1}{2\bar{\alpha}_t}
        \E[\mu_t^2\sech^2(\mu_ty_t+\nu)],
        \qquad
        \nabla_\nu\calL_t(\theta,\nu)
        =-\frac{1}{\bar{\alpha}_t}\E[(\mu_ty_t
        +\mu_t^2\tanh(\mu_ty_t+\nu))\sech^2(\mu_ty_t+\nu)].
    \]

    If \(\theta=\mathbf 0\), then \(\mu_t=0\) almost surely, so both
    quantities vanish for every signal norm. Hence assume
    \(\theta\neq\mathbf 0\).

    For every \(x\in\R\),
    \[
        |x|\sech^2x\le\frac2e,\qquad \sech^2x\le1,\qquad |\tanh x|\sech^2x\le1.
    \]
    To remove the growing ground truth signal from the bound, we substitute
    \(\mu_ty_t=(\mu_ty_t+\nu)-\nu\) and apply these inequalities:
    \[
        |(\mu_ty_t+\mu_t^2\tanh(\mu_ty_t+\nu))\sech^2(\mu_ty_t+\nu)|\le\frac2e+|\nu|+\mu_t^2.
    \]
    The variance integrand is bounded in absolute value by \(\mu_t^2\).
    Moreover,
    \[
        \E[\mu_t^2]=\bar{\alpha}_t\E[\langle\theta,\bx\rangle^2]\le\bar{\alpha}_t\|\theta\|^2\E[\|\bx\|^2]<\infty.
    \]
    Thus both absolute integrands have integrable majorants independent of
    \(\|\theta^\ast\|\).  By the projection assumption, almost surely
    \[
        \mu_ty_t+\nu=(-1)^{z+1}\bar{\alpha}_t\langle\theta,\bx\rangle\langle\theta^\ast,\bx\rangle+\sqrt{\bar{\alpha}_t}\langle\theta,\bx\rangle\varepsilon_t+\nu,
    \]
    whose absolute value tends to infinity. Each integrand in the
    identities therefore tends to zero almost surely. The dominated
    convergence theorem, followed by division by the fixed positive
    \(\bar{\alpha}_t\), proves both limits.
\end{proof}

\begin{lemma}[Hard-Assignment Limit of the EM Imbalance Operator]\label{lemma:high_snr_em_imbalance}
    Under the assumptions of
    Proposition~\ref{prop:pointwise_sm_blindness}, suppose
    \(\theta\ne\mathbf0\), and define
    \(q(\theta,\theta^\ast):=\E_{\bx}[\sgn(
    \langle\theta,\bx\rangle\langle\theta^\ast,\bx\rangle)]\), which is
    constant along the fixed ray. Let \(N_{\theta^\ast}\) and
    \(\calH_{\theta^\ast}\) denote the EM imbalance operator and
    cross-entropy when their ground truth expectations are taken under
    \((\theta^\ast,\nu^\ast)\). Then, as \(\|\theta^\ast\|\to\infty\) along
    the fixed ray,
    \[
        N_{\theta^\ast}(\theta_t,\nu)\longrightarrow
        q(\theta,\theta^\ast)\tanh\nu^\ast,
        \qquad
        \nabla_\nu N_{\theta^\ast}(\theta_t,\nu)\longrightarrow0,
    \]
    and
    \[
        \nabla_\nu\calH_{\theta^\ast}(\theta_t,\nu)\longrightarrow
        \tanh\nu-q(\theta,\theta^\ast)\tanh\nu^\ast,
        \qquad
        \nabla_\nu^2\calH_{\theta^\ast}(\theta_t,\nu)\longrightarrow\sech^2\nu>0.
    \]
    If \(\bx\sim\calN(\mathbf0,I_d)\) and
    \(\rho:=\langle\theta,\theta^\ast\rangle/(\|\theta\|\|\theta^\ast\|)\), then
    \(q(\theta,\theta^\ast)=(2/\pi)\arcsin\rho\).
\end{lemma}

\begin{proof}
    Under the ground truth \((\theta^\ast,\nu^\ast)\), the candidate
    activation is \(\mu_ty_t=(-1)^{z+1}\bar\alpha_t
    \langle\theta,\bx\rangle\langle\theta^\ast,\bx\rangle+
    \sqrt{\bar\alpha_t}\langle\theta,\bx\rangle\varepsilon_t\). The projection
    condition implies \(\tanh(\mu_ty_t+\nu)\to
    (-1)^{z+1}\sgn(\langle\theta,\bx\rangle
    \langle\theta^\ast,\bx\rangle)\) and
    \(\sech^2(\mu_ty_t+\nu)\to0\) almost surely. Bounded convergence,
    independence of \(z\) and \(\bx\), and
    \(\E[(-1)^{z+1}]=\tanh\nu^\ast\) then imply
    \[
        N_{\theta^\ast}(\theta_t,\nu)\longrightarrow q(\theta,\theta^\ast)\tanh\nu^\ast,\qquad \nabla_\nu N_{\theta^\ast}(\theta_t,\nu)=\E_\ast[\sech^2(\mu_ty_t+\nu)]\longrightarrow0.
    \]
    Substituting these limits into Proposition~\ref{prop:cross_entropy} and
    its derivative proves the two cross-entropy limits. For isotropic Gaussian covariates,
    Lemma~\ref{lemma:expectations_gaussian} applied to the standardized
    projections identifies \(q(\theta,\theta^\ast)=(2/\pi)\arcsin\rho\).
\end{proof}

\section{Experimental Details and Additional Numerical Results}\label{sup:experiment}
\renewcommand{\floatpagefraction}{0.85}
This appendix develops the covariance interpretation for the auxiliary local
spectral diagnostic and describes the numerical settings for the figures placed
alongside the results in Sections~\ref{sec:asymptotic}, \ref{sec:connect},
and~\ref{sec:theory}, together with their synthesis in
Section~\ref{sec:experiments}. Unless stated otherwise,
$\bx\sim\calN(0,I_d)$.
All computations use double precision, unit response-noise variance, and the
convention for the mixing weights in Section~\ref{subsec:2mlr}.

\suppressfloats[t]

\par\smallskip
\noindent\textbf{Experimental notation.}
Table~\ref{tab:experiment-notation} cross-references quantities introduced
in the theorem statements or their preceding formal appendix. Numerical
diagnostics such as fitted slopes, quantiles, and coverage are described in words.

\begin{table}[!htbp]
    \centering
    \caption{\textbf{Experimental notation.} Main-paper quantities reused in the numerical experiments.}
    \label{tab:experiment-notation}
    \small
    \setlength{\tabcolsep}{5pt}
    \renewcommand{\arraystretch}{1.13}
    \begin{tabular}{@{}>{\raggedright\arraybackslash}p{0.31\linewidth}>{\raggedright\arraybackslash}p{0.63\linewidth}@{}}
        \toprule
        \textbf{Quantity} & \textbf{Use in the experiments} \\
        \midrule
        $\vartheta=(\theta,\nu)$ and $\vartheta^\ast=(\theta^\ast,\nu^\ast)$
        & Regression parameters and mixing-weight imbalance, and their ground truth values. \\
        $\bar\alpha_t$ and $\bar\alpha_{T_n}$
        & Retained-signal factor and its terminal value. Values near one mean
          low diffusion noise; values near zero mean high diffusion noise. \\
        $(\widehat\theta_{n,T_n}^{\,\mathsf{SM}},
          \widehat\nu_{n,T_n}^{\,\mathsf{SM}})$ and $s_n$
        & Integrated score matching estimator and the joint sign alignment in
          Theorem~\ref{thm:consistency_asymptotic_normality}. \\
        $I(\theta^\ast,\nu^\ast)$
        & Fisher-information benchmark for the standardized aligned estimator. \\
        $M$, $N$, $V$, and $V_t$
        & Terms in the complete low-noise reference and the blindness study. \\
        $\Sigma$
        & Covariate covariance in the endpoint and limiting-loss calculations. \\
        $\calC^\ast$
        & Fourth-order signal covariance in the auxiliary spectral
          diagnostic. \\
        $c^\ast=\theta^\ast\tanh\nu^\ast$, $\calL(\vartheta)$,
          $J(\vartheta)$, $\Phi$, and $\calD_\eta$
        & Effective target, high-noise limiting loss, balance quantity,
          gradient-descent update map, and admissible set, respectively. \\
        $\eta$, $\calL_0$, and $B_0$
        & Step size, initial loss, and explicit-envelope constant in the
          dynamics study. \\
        $\calS$, $\calW^s(\calS)$, and $K_\epsilon$
        & Exceptional saddle set, its stable set, and the first iteration
          $K$ for which $\calL(\vartheta_K)\le\epsilon$. \\
        \bottomrule
    \end{tabular}
\end{table}

\par\smallskip
\noindent\textbf{Numerical settings.}
Tables~\ref{tab:experiment-statistical-settings}
and~\ref{tab:experiment-dynamics-settings} collect the settings in
the established notation.  A separate entry
lists the spectral study as an auxiliary local-mechanism diagnostic.

\begin{table}[!htbp]
    \centering
    \caption{\textbf{Finite-sample, endpoint, and blindness settings.}}
    \label{tab:experiment-statistical-settings}
    \small
    \setlength{\tabcolsep}{5pt}
    \renewcommand{\arraystretch}{1.14}
    \begin{tabular}{@{}>{\raggedright\arraybackslash}p{0.22\linewidth}>{\raggedright\arraybackslash}p{0.72\linewidth}@{}}
        \toprule
        \textbf{Study} & \textbf{Model and numerical setting} \\
        \midrule
        Finite sample
        & $d=1$, $(\theta^\ast,\nu^\ast)=(1.10,0.55)$,
          $n=250\cdot2^j$ for $j=0,\ldots,4$, with 200 independent repetitions at
          each $n$, $\bar\alpha_0=1$, and
          $\bar\alpha_{T_n}=n^{-3/2}$.  We evaluate the diffusion integral by 24-point
          Gauss--Legendre quadrature and the expectation over $\xi$ by 11-point
          Gauss--Hermite quadrature. Each fit has four starts. We compute the Fisher
          information with an 80-point Gauss--Hermite rule in each Gaussian
          variable and an exact sum over the two latent classes. \\
        \addlinespace[2pt]
        Endpoints
        & $d=1$, $(\theta^\ast,\nu^\ast)=(1.10,0.55)$, and
          $(\theta,\nu)=(0.45,-0.35)$.  We evaluate both the Gaussian covariate and
          response-noise expectations with 140-point Gauss--Hermite quadrature.
          For $j=0,\ldots,8$, we set
          $\bar\alpha_t=1-10^{-6+j/2}$ on the low-noise grid and
          $\bar\alpha_t=10^{-6+j/2}$ on the high-noise grid. \\
        \addlinespace[2pt]
        Blindness
        & $d=1$, $\bar\alpha_t=\tfrac12$, $(\theta,\nu)=(0.80,-0.35)$,
          $\nu^\ast=0.55$, and
          $\theta^\ast=3\times10^{j/5}>0$ for $j=0,\ldots,25$.
          We evaluate expectations by deterministic quadrature as detailed below. \\
        \bottomrule
    \end{tabular}
\end{table}

\begin{table}[!htbp]
    \centering
    \caption{\textbf{Limiting-dynamics and spectral-diagnostic settings.}}
    \label{tab:experiment-dynamics-settings}
    \small
    \setlength{\tabcolsep}{5pt}
    \renewcommand{\arraystretch}{1.14}
    \begin{tabular}{@{}>{\raggedright\arraybackslash}p{0.22\linewidth}>{\raggedright\arraybackslash}p{0.72\linewidth}@{}}
        \toprule
        \textbf{Study} & \textbf{Model and numerical setting} \\
        \midrule
        Degenerate dynamics
        & $d=2$, $c^\ast=\mathbf0$, $\eta=1/50=2\times10^{-2}$,
          $\theta_0=(\sinh(0.70),0)$, and $\nu_0=0.70$, so that
          $J(\vartheta_0)=0$.  The map $\Phi$ is iterated for 30,000
          steps and the theorem's explicit envelope is evaluated at each step. \\
        \addlinespace[2pt]
        Nondegenerate dynamics
        & $d=2$, $c^\ast=(0.80,0)$, and $\eta=1/50=2\times10^{-2}$.  Every displayed start is
          checked to lie in $\calD_\eta$ before the map $\Phi$ is
          iterated for 1,600 steps.  The four starts are
          $((0.45,0.15),0.40)$, $((-0.45,0.25),0.40)$,
          $((5\times10^{-3},0.65),0)$, and the exact saddle
          $((0,0.65),0)$. \\
        \addlinespace[2pt]
        Spectral diagnostic
        & $d=20$, $\Sigma=I_d$,
          $\theta^\ast=(1,0,\ldots,0)$, $\nu^\ast=\nu=0$,
          $\bar\alpha_t=5\times10^{-3}$, raw $\theta$-update step size
          $\eta_\theta=20$
          (whose product with $\bar\alpha_t$ is $10^{-1}$), and
          $\|\theta_0\|=2\times10^{-2}$.  The initial absolute cosine with
          $\theta^\ast$ is $10^{-1}$. For the Gaussian and
          covariance-normalized spherical Student $t$ designs with ten degrees
          of freedom, the experiment uses $1.6\times10^5$ common draws and 14
          normalized updates. \\
        \bottomrule
    \end{tabular}
\end{table}

\par\smallskip
\noindent\textbf{Finite-sample estimation and calibration.}
The implementation uses the conditional Hyv\"arinen integrand
\(\ell_t^{\mathsf H}\) defined in Eq.~\eqref{eq:hyvarinen_integrand} of
Appendix~\ref{sup:additional_experiments}.
Its relation to the empirical denoising loss in Eq.~\eqref{eq:auxiliary_loss_finite}
and the squared denoising integrand in Eq.~\eqref{eq:denoising_integrand} appears
in Eq.~\eqref{eq:empirical_integrand_equivalence}.
We average \(\ell_t^{\mathsf H}\) over the observed pairs and integrate over
$\bar\alpha_t\in[\bar\alpha_{T_n},\bar\alpha_0]$ and over
$\xi\sim\calN(0,1)$.  Since $\bar\alpha_{T_n}=n^{-3/2}\to0$, the selected
terminal retained-signal schedule satisfies the condition in
Theorem~\ref{thm:consistency_asymptotic_normality}.

For each fit, we run L-BFGS-B with $\theta\in[-3,3]$, $\nu\in[-2.5,2.5]$, at most
140 iterations, function tolerance $10^{-11}$, and projected-gradient
tolerance $2\times10^{-7}$.  Two starts are a data-moment estimate and its
joint negative.  The first coordinate of the data-moment estimate is
$\sqrt{\max\{n^{-1}\sum_i (y_0^{(i)})^2-1,\,8\times10^{-2}\}}$, and its second coordinate
is obtained by applying $\operatorname{arctanh}$ to the ratio of
$(n^{-1}\sum_i\bx^{(i)}y_0^{(i)})/
\max\{n^{-1}\sum_i\|\bx^{(i)}\|^2,10^{-10}\}$ to this first coordinate, clipped
to remain $8\times10^{-2}$ away from $\{\pm1\}$ by restricting it to
$[-1+8\times10^{-2},\,1-8\times10^{-2}]$.  The other two starts are $(0.55,0.15)$ and
$(-0.55,-0.15)$.  Among the four starts, we retain the fit with the smallest
objective and then apply the joint sign $s_n$ to both estimated coordinates.
The aligned error is this estimate minus the ground truth pair. In
Figure~\ref{fig:finite-sample-scaling}(a), we plot its norm induced by
$I(\theta^\ast,\nu^\ast)$, called the Fisher norm.
By Theorem~\ref{thm:consistency_asymptotic_normality}, $n$
times its square converges to the two-degree-of-freedom chi-square law $\chi^2_2$.
Hence the median and 90\% references are the square roots of the corresponding
$\chi^2_2$ quantiles divided by $n$. The empirical quantile bars are 95\%
percentile-bootstrap intervals from 2,000 resamples with seed 17031.
Over the three largest sample sizes, the fitted log--log slopes of the
median and 90\% quantile are $-0.471$ and $-0.526$, respectively, close to
the predicted $-1/2$.
Panel (b)
compares all ordered values of $n$ times the squared Fisher error with
$\chi^2_2$ quantiles at $n=250$ and $n=4000$.

For each sample size $n$, Table~\ref{tab:finite-normal-calibration} reports the
fraction of aligned estimates for which $n$ times the squared Fisher error is
at most the 95\% $\chi^2_2$ quantile. Its covariance discrepancy is the
Frobenius distance between the empirical covariance of the $\sqrt n$-scaled
aligned errors and $I(\theta^\ast,\nu^\ast)^{-1}$, divided by the Frobenius norm
of $I(\theta^\ast,\nu^\ast)^{-1}$. All 1,000 optimization runs terminated successfully.

\begin{table}[!htbp]
    \centering
    \caption{\textbf{Finite-sample calibration.} Joint 95\% Fisher-ellipse coverage and relative covariance discrepancy.}
    \label{tab:finite-normal-calibration}
    \small
    \setlength{\tabcolsep}{8pt}
    \renewcommand{\arraystretch}{1.12}
    \begin{tabular}{@{}ccc@{}}
        \toprule
        Sample size $n$ & Joint 95\% Fisher-ellipse coverage
            & Relative covariance discrepancy \\
        \midrule
        250  & 93.5\% & 0.220 \\
        500  & 96.5\% & 0.093 \\
        1000 & 94.0\% & 0.127 \\
        2000 & 94.5\% & 0.089 \\
        4000 & 95.0\% & 0.070 \\
        \bottomrule
    \end{tabular}
\end{table}

\FloatBarrier
\par\smallskip
\noindent\textbf{Endpoint-gradient checks.}
We sum exactly over the latent class and integrate over both the Gaussian covariate
and response noise with 140-point Gauss--Hermite quadrature at
every scale of the diffusion noise level.  At low noise, the $\theta$ component of the complete
reference in Proposition~\ref{prop:low_noise_em} is the sum of
\(\Sigma(\theta-M(\theta,\nu))\),
\(-\bigl(\nabla_\theta M(\theta,\nu)\bigr)^\top\Sigma\theta\), and
\(\nabla_\theta V(\theta,\nu)\).  The $\nu$ component is
$-\langle\nabla_\theta N(\theta,\nu),\theta\rangle+\nabla_\nu V(\theta,\nu)$.
We compute the plotted remainder from the direct $\bar\alpha_t=1$ gradient to stabilize
the subtraction at the smallest scales. As a separate check, that endpoint differs
from the complete operator expression in Proposition~\ref{prop:low_noise_em} by
$2.24964\times10^{-9}$ in Euclidean norm. Its $\theta$- and $\nu$-component
residuals are $7.606\times10^{-10}$ and $2.117\times10^{-9}$.  This separate
identity check confirms that the endpoint reference combines the EM
residual with the operator-Jacobian and posterior-variance corrections.
At high noise, the reference is the gradient of the effective-product limiting
loss in Proposition~\ref{prop:high_noise_effective_coefficient}.

Table~\ref{tab:endpoint-component-slopes} records the componentwise log--log
slope deviations from the first-order benchmark.  Each fit uses
the five smallest of the nine endpoint distances.

\begin{table}[!htbp]
    \centering
    \caption{\textbf{Endpoint-remainder slopes.} We report each fitted slope as one plus its deviation from the unit-slope benchmark.}
    \label{tab:endpoint-component-slopes}
    \small
    \setlength{\tabcolsep}{8pt}
    \renewcommand{\arraystretch}{1.12}
    \begin{tabular}{@{}lcc@{}}
        \toprule
        Gradient quantity & Low-noise slope (panel a) & High-noise slope (panel b) \\
        \midrule
        $\theta$ component & $1+2.3191\times10^{-5}$ & $1-3.8124\times10^{-5}$ \\
        $\nu$ component    & $1+1.0247\times10^{-5}$ & $1-3.00596\times10^{-4}$ \\
        \bottomrule
    \end{tabular}
\end{table}

\FloatBarrier
\par\smallskip
\begin{remark}[Spiked fourth-order covariance and the local spectral diagnostic]
Proposition~\ref{prop:high_noise_effective_coefficient} identifies
\(\calL(\theta,\nu)=\frac12\|\theta\tanh\nu-c^\ast\|_\Sigma^2\), where
\(c^\ast:=\theta^\ast\tanh\nu^\ast\), as the leading high-noise objective.
For Gaussian covariates and for elliptical covariates with zero mean,
covariance \(\Sigma\succ0\), and finite fourth moments, define the fourth-order
signal covariance using the covariate marginal by
\[
    \calC^\ast:=\E_{\bx}\!\left[\bx\bx^\top
    (\bx^\top\theta^\ast)^2\right].
\]
To evaluate this fourth-order moment, we apply
Lemmas~\ref{lemma:generalized_spiked_covariance}
and~\ref{lemma:elliptical_spiked_covariance} and obtain
\[
    \calC^\ast=
    \begin{cases}
        \sigma_\mu^2\Sigma
        +2\Sigma\theta^\ast(\Sigma\theta^\ast)^\top,
        & \bx\sim\calN(0,\Sigma),\\[0.25em]
        (1+\kappa)\!\left[
        \sigma_\mu^2\Sigma
        +2\Sigma\theta^\ast(\Sigma\theta^\ast)^\top\right],
        & \bx\text{ elliptical},
    \end{cases}
    \qquad
    \sigma_\mu^2:=\|\theta^\ast\|_\Sigma^2,
\]
where
\(\kappa:=\E_{\bx}[\|\Sigma^{-1/2}\bx\|^4]/(d(d+2))-1>-1\).
To identify the leading eigenspace, we whiten this covariance. If
\(\theta^\ast\ne\mathbf0\) and
\(v^\ast:=\Sigma^{1/2}\theta^\ast\), the resulting matrix is
\(\sigma_\mu^2I_d+2v^\ast(v^\ast)^\top\) in the Gaussian case and
\((1+\kappa)[\sigma_\mu^2I_d+2v^\ast(v^\ast)^\top]\) in the elliptical
case, so \(\operatorname{span}\{v^\ast\}\) is the unique leading eigenspace
after whitening.  This also makes
\(\operatorname{span}\{\theta^\ast\}\) the leading generalized eigenspace
of \(\calC^\ast v=\lambda\Sigma v\).  The corresponding generalized spectral
gap is \(2\sigma_\mu^2\) in the Gaussian case and
\(2(1+\kappa)\sigma_\mu^2\) in the elliptical case.  When
\(\Sigma=I_d\), the Gaussian covariance reduces to
\(\|\theta^\ast\|^2I_d+2\theta^\ast(\theta^\ast)^\top\), whereas the
elliptical covariance is \((1+\kappa)\) times this matrix.  These covariance
identities define the auxiliary directional comparator used below.
\end{remark}

On the balanced slice \(\nu^\ast=\nu=0\), the leading high-noise objective is
flat in \(\theta\), so directional information arises at higher order.
Figure~\ref{fig:endpoint-gradients}(b) tests the \(\calO(\bar\alpha_t)\)
remainder in Proposition~\ref{prop:high_noise_effective_coefficient}.
We examine the next-order directional information in
Figure~\ref{fig:local-spectral} at small positive \(\bar\alpha_t\). Using \(n\) covariate
draws \(\{\bx^{(i)}\}_{i=1}^n\) with \(\Sigma=I_d\), we form the sample spiked
covariance
\[
    \widehat{\calC}^\ast
    :=\frac1n\sum_{i=1}^n \bx^{(i)}(\bx^{(i)})^\top
    \bigl((\bx^{(i)})^\top\theta^\ast\bigr)^2
\]
and the normalized power step
\[
    \theta\;\longmapsto\;
    \frac{(I_d+2\gamma\widehat{\calC}^\ast)\theta}
   {\|(I_d+2\gamma\widehat{\calC}^\ast)\theta\|},
    \qquad\gamma:=\eta_\theta\bar\alpha_t.
\]
In parallel, with \(\nu\) clamped at zero, we apply the finite-noise score-matching
update \(\theta\leftarrow\theta-\eta_\theta\nabla_\theta\calL_t(\theta,0)\) and
renormalize to a fixed radius. The two trajectories use respective
normalization radii \(1\) and \(2\times10^{-2}\); the absolute
cosine \(|\cos\angle(\theta_k,\theta^\ast)|\) plotted in
Figure~\ref{fig:local-spectral} is invariant under these rescalings. The label \(t_{10}\) denotes spherical
Student-\(t\) covariates with ten degrees of freedom (covariance-normalized),
contrasted with isotropic Gaussian covariates. The experiment compares
the directions of the normalized trajectories. This next-order diagnostic
complements the proposition's leading-order remainder test.
When \(\Sigma=I_d\), we interpret the power step in terms of ordinary eigenvectors.
Anisotropic covariance would require the corresponding
generalized-eigenvector map.

\begin{figure}[!htbp]
    \centering
    \includegraphics[width=0.74\linewidth]{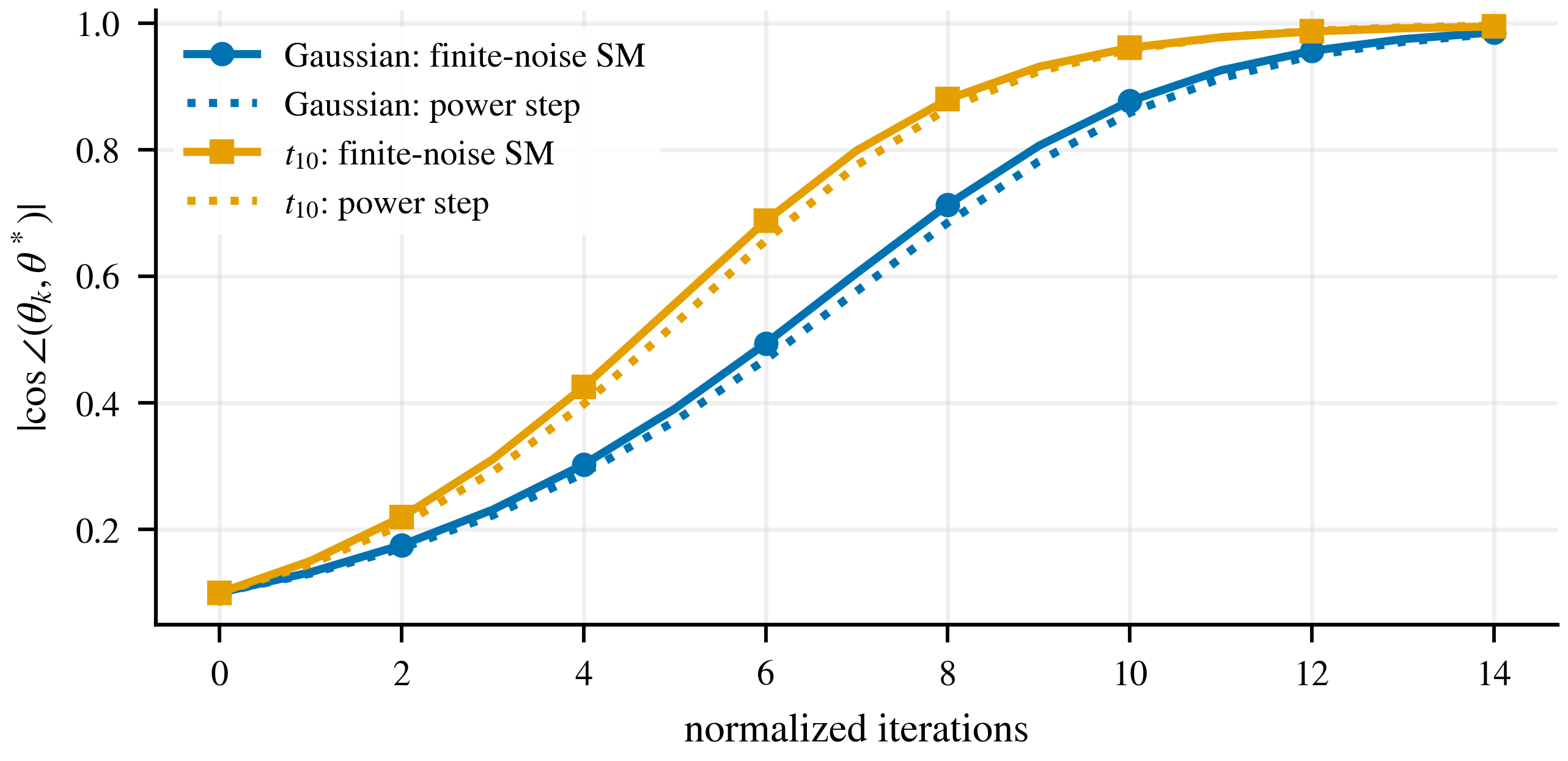}
    \caption{\textbf{Auxiliary next-order directional comparison on the
    balanced slice \(\nu^\ast=\nu=0\).}
    Absolute cosine \(|\cos\angle(\theta_k,\theta^\ast)|\) for normalized
    finite-noise score-matching updates
    \(\theta\leftarrow\theta-\eta_\theta\nabla_\theta\calL_t(\theta,0)\)
    (solid) and normalized power steps
    \(\theta\mapsto(I_d+2\gamma\widehat{\calC}^\ast)\theta/\|(I_d+2\gamma\widehat{\calC}^\ast)\theta\|\)
    with \(\gamma=\eta_\theta\bar\alpha_t\) (dotted).
    Curves labeled \(t_{10}\) use spherical Student-\(t\) covariates with ten
    degrees of freedom; the remaining curves use isotropic Gaussians.}
    \label{fig:local-spectral}
\end{figure}

\par\smallskip
\noindent\textbf{Limiting-dynamics diagnostics.}
The implementation iterates
$\vartheta_{k+1}=\Phi(\vartheta_k)$.  For the degenerate target,
\(\calL_0=\tfrac12\sinh^2(0.70)\tanh^2(0.70)\approx0.1050944\),
\(J(\vartheta_0)=0\), and \(B_0=4\).
Since $\eta=1/50$, the left-hand side in the definition of $\calD_\eta$ is
$32\eta=16/25=0.64<1$. The reference curve in
Figure~\ref{fig:high-noise-dynamics}(a) is the theorem's bound
$\bigl(\calL_0^{-1/2}+\eta K/(\sqrt6\,B_0)\bigr)^{-2}$.
We fit a linear regression of $\log\calL(\vartheta_K)$ on $\log K$ for every
integer $K$ from $3\times10^3$ through $3\times10^4$.  The fitted slope is
$-2-1.5035\times10^{-2}$. It is $1.5035\times10^{-2}$ steeper than the
$K^{-2}$ reference. The coefficient of determination is
$1-4.6533\times10^{-7}$. This fit is a finite-window diagnostic. The pointwise
comparison with the explicit bound checks the theorem's $\calO(K^{-2})$ upper
bound.

For $c^\ast\ne\mathbf0$, the three displayed generic starts in
Table~\ref{tab:dynamics-initializations} are selected to show different
entrance times while remaining in $\calD_\eta$. The exact saddle
start satisfies $\nu_0=0$ and $\langle\theta_0,c^\ast\rangle=0$ and verifies
the predicted plateau $\|c^\ast\|^2/2=0.32$. Across the one degenerate and four
nondegenerate paths,
the maximum admissibility
left-hand side is $0.64$ and the largest one-step loss increase is zero (at the
stationary saddle), consistent along these trajectories with the theorem's
forward-invariance and monotonicity conclusions.

\begin{table}[!htbp]
    \centering
    \caption{\textbf{Nondegenerate-dynamics initializations.}}
    \label{tab:dynamics-initializations}
    \small
    \setlength{\tabcolsep}{5pt}
    \renewcommand{\arraystretch}{1.12}
    \begin{tabular}{@{}>{\raggedright\arraybackslash}p{0.18\linewidth}>{\raggedright\arraybackslash}p{0.23\linewidth}>{\centering\arraybackslash}p{0.14\linewidth}>{\raggedright\arraybackslash}p{0.37\linewidth}@{}}
        \toprule
        \textbf{Start} & $\theta_0$ & $\nu_0$ & \textbf{Role} \\
        \midrule
        Same-sign & $(0.45,0.15)$ & $0.40$
          & Attraction to the matching-sign branch. \\
        Opposite-sign & $(-0.45,0.25)$ & $0.40$
          & Crosses toward the target branch. \\
        Near saddle & $(5\times10^{-3},0.65)$ & $0$
          & Perturbation from the saddle manifold. \\
        Exact saddle & $(0,0.65)$ & $0$
          & Stationary at loss $\|c^\ast\|^2/2$. \\
        \bottomrule
    \end{tabular}
\end{table}

For each displayed generic start, the analysis fits a linear regression of
$\log\calL(\vartheta_K)$ on $K$ while the loss lies between
$10^{-3}$ and $10^{-12}$. Here $K_{10^{-2}}$, for example, is the first
iteration $K$ with $\calL(\vartheta_K)\le10^{-2}$.
Table~\ref{tab:dynamics-hitting-times} reports the four hitting times, the
fitted slope, and one minus its coefficient of determination.

\begin{table}[!htbp]
    \centering
    \caption{\textbf{Hitting times and geometric tail fits.}}
    \label{tab:dynamics-hitting-times}
    \small
    \setlength{\tabcolsep}{7pt}
    \renewcommand{\arraystretch}{1.12}
    \begin{tabular}{@{}lrrrrrr@{}}
        \toprule
        Start & $K_{10^{-2}}$ & $K_{10^{-4}}$ & $K_{10^{-6}}$
              & $K_{10^{-8}}$ & Tail slope & Tail lack of fit \\
        \midrule
        Same-sign     & 136 & 296 & 475  & 680  & $-0.022705$ & $1.71\times10^{-3}$ \\
        Opposite-sign & 328 & 498 & 694  & 907  & $-0.021873$ & $6.67\times10^{-4}$ \\
        Near saddle   & 614 & 844 & 1082 & 1322 & $-0.019248$ & $5.68\times10^{-6}$ \\
        \bottomrule
    \end{tabular}
\end{table}

We omit values below $10^{-14}$ from the semilogarithmic panel as a roundoff
display cutoff.

\par\smallskip
\noindent\textbf{Pointwise high-SNR blindness.}
In the experiment in Figure~\ref{fig:high-snr-blindness}, we take the scalar covariate
$x\sim\calN(0,1)$ and a fixed positive ground truth direction.
For this numerical derivation, we work in the boundary-layer coordinate
$u=\sqrt{\|\theta^\ast\|}\,x$.  The diffused ground truth signal norm is
$\sqrt{\bar\alpha_t}\|\theta^\ast\|$.  Because $\bar\alpha_t$ is fixed in
each sweep covered by Proposition~\ref{prop:pointwise_sm_blindness}, it diverges as
$\|\theta^\ast\|\to\infty$.
With the established latent class $z\in\{1,2\}$ and
$\varepsilon_t\sim\calN(0,1)$, the candidate activation is the sum of
$(-1)^{z+1}\bar\alpha_t\theta u^2$,
$\sqrt{\bar\alpha_t}\theta u\varepsilon_t/\sqrt{\|\theta^\ast\|}$, and $\nu$,
so its nonsaturated region has fixed width in $u$.  We evaluate
$\nabla_\nu\calL_t$ and $V_t$ by quadrature, multiplying the former by
$\|\theta^\ast\|^{1/2}$ and the latter by
$\|\theta^\ast\|^{3/2}$ before integrating over
$(-\infty,0]\cup[0,\infty)$, and undo the
scaling afterward. This prevents the adaptive quadrature routine from having
to resolve a shrinking neighborhood of $x=0$ and
keeps the reported components comparable in magnitude.
These factors are numerical preconditioners tailored to this
one-dimensional Gaussian design and are undone after integration.  The same
adaptive call also evaluates two diagnostic components.  The first is a
second copy of the preconditioned score matching gradient obtained by direct
differentiation.  The second is the scaled boundary correction used for the
cross-entropy gradient.

We sum exactly over the two latent classes. We evaluate the response-noise
expectation with paired 80- and 160-node Gauss--Hermite rules. We compare their
results to assess response-noise discretization separately from the adaptive
integration in \(u\).
For the latter, we apply
vector quadrature with absolute tolerance $10^{-11}$, relative tolerance
$10^{-9}$, and subdivision limit 500.  Its returned error estimate covers the
adaptive \(u\)-integration; the paired-order comparison assesses
Gauss--Hermite discretization.  A production
run is rejected if any reported scaled component has paired-order
difference exceeding $10^{-8}+10^{-6}$ times the larger absolute result from
the 80- and 160-node rules.  At the
    publication settings, the direct-score formula, derived by direct
    differentiation, differs from the integration-by-parts expression by at most
$2.85\times10^{-10}$ in absolute value over the displayed sweep.  The largest
80-versus-160-node difference in the score gradient obtained by direct
differentiation is also $2.85\times10^{-10}$.  The largest adaptive error
estimate for the scaled five-component vector is $1.60\times10^{-10}$.  We evaluate
the plotted score gradient with the direct formula because it preserves exact-model cancellation, and
retain the integration-by-parts formula as an independent identity check.
We take absolute values after integrating the signed quantities.

We generate panels (a)--(b) of Figure~\ref{fig:high-snr-blindness} from 1,200
common-random-number draws at $\|\theta^\ast\|=3$ and $\|\theta^\ast\|=100$.  Their background is the fixed
candidate's response-score sensitivity
$|\nabla_\nu s_{\theta_t,\nu}(y_t,x)|$ on a common logarithmic scale.  Panel
(c) contains the imbalance gradient and posterior-variance term stated in
Proposition~\ref{prop:pointwise_sm_blindness}. At the largest displayed
$\|\theta^\ast\|$, the computed values of $|\nabla_\nu\calL_t|$ and $V_t$ are
$1.6688\times10^{-3}$ and $2.9399\times10^{-9}$.

\begin{figure}[!htbp]
    \centering
    \includegraphics[width=\linewidth]{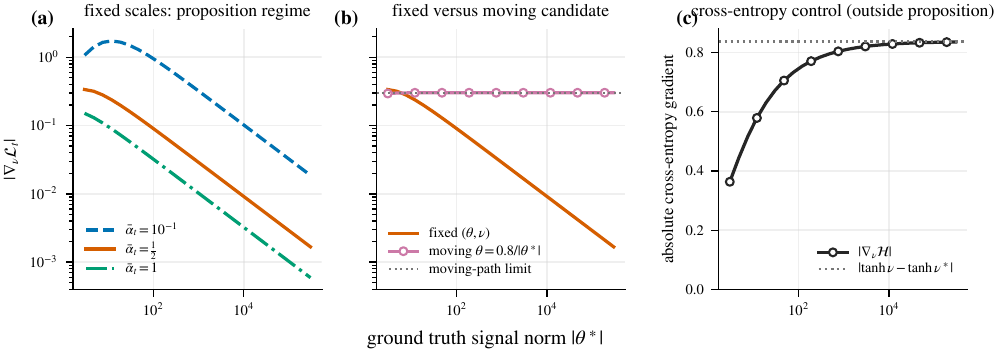}
    \caption{\textbf{Scope and objective controls for pointwise high-SNR blindness.}
    \textbf{(a)} Fixed-\(\bar\alpha_t\) curves provide separate checks covered by the proposition as the
    ground truth signal \(\|\theta^\ast\|\) grows, for
    \(\bar\alpha_t\in\{10^{-1},\tfrac12,1\}\).  \textbf{(b)} The moving-candidate curve uses
    \(\theta=0.8/\|\theta^\ast\|\), whose activation has a finite limit.
    The moving-candidate score matching gradient
    $|\nabla_\nu\calL_t(0.8/\|\theta^\ast\|,\nu)|$ approaches $0.302302$.  That curve is
    outside the fixed-candidate hypothesis of
    Proposition~\ref{prop:pointwise_sm_blindness}
    and illustrates its pointwise-in-candidate scope.  \textbf{(c)} The separate population
    cross-entropy control is evaluated along the same fixed-candidate ray as
    Figure~\ref{fig:high-snr-blindness}. It approaches
    $|\tanh\nu-\tanh\nu^\ast|\approx0.836896$, showing that this selected
    cross-entropy diagnostic retains imbalance signal.}
    \label{fig:high-snr-blindness-controls}
\end{figure}

The cross-entropy control appears in Appendix
Figure~\ref{fig:high-snr-blindness-controls}(c).  It evaluates
$\nabla_\nu\calH(\theta_t,\nu)=\tanh\nu-
\E_\ast[\tanh(\mu_t y_t+\nu)]$ along the same fixed-candidate ray. For the
displayed aligned one-dimensional specialization, the candidate activation is
the sum of $(-1)^{z+1}\bar\alpha_t\theta\|\theta^\ast\|x^2$,
$\sqrt{\bar\alpha_t}\theta x\varepsilon_t$, and $\nu$.  Its hyperbolic tangent
therefore converges almost surely to $(-1)^{z+1}$. To pass to the gradient
expectation, we apply bounded convergence and obtain
$|\nabla_\nu\calH(\theta_t,\nu)|\to
|\tanh(-0.35)-\tanh(0.55)|\approx0.836896$.  Its nonzero limit therefore
agrees with Lemma~\ref{lemma:high_snr_em_imbalance} and provides a contrast.
The score matching diagnostics lose their local imbalance signal, whereas this
cross-entropy gradient remains informative. This is a fixed-candidate
comparison of population gradients evaluated at the specified parameters.

Figure~\ref{fig:high-snr-blindness-controls}(a)--(b) provides two scope checks.
In the fixed-scale comparisons, each value of \(\bar\alpha_t\) is held constant while
\(\|\theta^\ast\|\) varies, so each curve is a separate specialization covered by
Proposition~\ref{prop:pointwise_sm_blindness}. The moving-candidate comparison instead
sets \(\theta=0.8/\|\theta^\ast\|\) while retaining the main experiment's scale of the diffusion noise level,
ground truth ray, and imbalance parameters.  Because the candidate changes
with \(\|\theta^\ast\|\), this curve illustrates the proposition's
pointwise-in-candidate scope.

To evaluate the limit along the moving candidate used in the control, we
differentiate the loss and pass the limit under the expectation by dominated
convergence. We obtain
$\nabla_\nu\calL_t(0.8/\|\theta^\ast\|,\nu)\to
-0.8\,\E_\ast[(-1)^{z+1}x^2\sech^2(0.8\bar\alpha_t(-1)^{z+1}x^2+\nu)]$.
Numerical integration of this expression yields the signed limiting value
$-0.3023024610$ to ten decimal places, with magnitude
$0.3023024610$.  The plotted absolute
gradient at $\|\theta^\ast\|=3\times10^5$ differs from that magnitude by
$3.2\times10^{-12}$. This is the discrepancy at the largest displayed signal
norm. The
largest 80-versus-160-node difference over the control sweep is
$8.72\times10^{-8}$.

\FloatBarrier
\newpage
\section{Additional Comparisons, Schedules, Training, Dynamics, and Robustness Studies}\label{sup:additional_experiments}
\renewcommand{\floatpagefraction}{0.85}
\setlength{\textfloatsep}{8pt plus 2pt minus 2pt}
\setlength{\intextsep}{6pt plus 2pt minus 2pt}
\setlength{\floatsep}{6pt plus 2pt minus 2pt}
\raggedbottom

This appendix presents numerical studies that complement
Section~\ref{sec:experiments} and Appendix~\ref{sup:experiment}.

\noindent\textbf{Purpose and organization of the numerical studies.}
To assess the finite-sample implications of
Main Theorem~\ref{thm:consistency_asymptotic_normality} and
Proposition~\ref{prop:paired_sm_mle}, we compare score matching with the MLE
and EM and measure the paired score matching--MLE discrepancy
(Figure~\ref{fig:mle-em-comparison} and Table~\ref{tab:mle-em-comparison}).
We examine the theorem's calibration predictions across dimensions and
covariate laws (Figure~\ref{fig:design-calibration}), and investigate its
terminal-schedule condition by varying vanishing schedules and measuring
the efficiency loss at fixed positive terminal levels
(Figure~\ref{fig:terminal-schedules}).

To examine stochastic training of the model using the finite-sample loss
in Eq.~\eqref{eq:auxiliary_loss_finite} integrated over diffusion levels, we
compare training error and gradient variance for the denoising, conditional
Hyv\"arinen, and hybrid integrands (Figure~\ref{fig:trained-score-model} and
Table~\ref{tab:trained-score-model}).
Equation~\eqref{eq:empirical_integrand_equivalence} relates the noise-averaged
denoising and Hyv\"arinen forms to this loss up to parameter-independent constants.
To examine the optimization implications of
Corollary~\ref{cor:sm_gradients_em_variance}, we compare EM with gradient
descent and measure the operator-Jacobian and latent-variance corrections
to the EM residual (Figure~\ref{fig:matched-start-dynamics} and
Table~\ref{tab:matched-start-iterations}).

To assess the whitening extension of the high-noise convergence result, we
evaluate the preconditioned iteration in Eq.~\eqref{eq:whitened_gradient_iteration}
with exact and estimated covariances (Figure~\ref{fig:anisotropic-whitening}).
Finally, to investigate the vanishing limits in
Proposition~\ref{prop:pointwise_sm_blindness}, we estimate the decay rates of
the imbalance gradient and latent-variance term across covariate laws
(Figure~\ref{fig:design-blindness}).
Table~\ref{tab:additional-study-guide} summarizes the purpose and location
of each study.

The statistical comparisons estimate both the regression parameters and
mixing weights. Each repetition draws an independent sample of size \(n\),
fits the competing estimators to that common dataset, and aligns the
estimates to the ground truth by the joint sign \(s_n\).
Tables~\ref{tab:additional-notation} and~\ref{tab:additional-settings}
summarize the notation and settings, supplementing
Tables~\ref{tab:experiment-notation}--\ref{tab:experiment-dynamics-settings}
in Appendix~\ref{sup:experiment}.

\begin{table}[!htb]
    \centering
    \caption{\textbf{Guide to the numerical studies in Appendix~\ref{sup:additional_experiments}.}
    The studies are listed in the order in which they appear, and the locations
    include the accompanying discussion of each study.}
    \label{tab:additional-study-guide}
    \small
    \setlength{\tabcolsep}{4pt}
    \renewcommand{\arraystretch}{1.08}
    \begin{tabular}{@{}>{\raggedright\arraybackslash}p{0.20\linewidth}>{\raggedright\arraybackslash}p{0.17\linewidth}>{\raggedright\arraybackslash}p{0.57\linewidth}@{}}
        \toprule
        Study & Location & Purpose \\
        \midrule
        Score matching versus MLE and EM
        & Figure~\ref{fig:mle-em-comparison}, Table~\ref{tab:mle-em-comparison}
        & We compare the Fisher-normalized error and coverage
          with the inverse-Fisher limit of
          Main Theorem~\ref{thm:consistency_asymptotic_normality}, assess the
          paired rate in Proposition~\ref{prop:paired_sm_mle}, compare numerical
          EM with the MLE using the fixed-point characterization in
          Proposition~\ref{prop:empirical_em_stationarity}, and measure fitting
          cost. \\
        \addlinespace[3pt]
        Statistical calibration across dimensions and covariate laws
        & Figure~\ref{fig:design-calibration}
        & We test the calibration predicted by
          Main Theorem~\ref{thm:consistency_asymptotic_normality} in fixed
          dimensions \(d\in\{2,10,50\}\) under six covariate laws, including
          Gamma and uniform designs beyond the positive-density assumption of
          Main Theorem~\ref{thm:consistency_asymptotic_normality} and
          Proposition~\ref{prop:terminal_kl_pinsker_zero_set}. \\
        \addlinespace[3pt]
        Terminal schedules and statistical efficiency
        & Figure~\ref{fig:terminal-schedules}
        & We examine the condition \(\bar\alpha_{T_n}\to0\) of
          Main Theorem~\ref{thm:consistency_asymptotic_normality} and the
          paired rate in Proposition~\ref{prop:paired_sm_mle} under slower
          schedules, and compare the variance increase at fixed positive terminal
          levels \(\bar\alpha_{T_n}=a\) with the covariance \(V_a\) of the
          \hyperref[rmk:fixed_terminal_covariance]{covariance
          comparison} in Appendix~\ref{sup:setup}. \\
        \addlinespace[3pt]
        Stochastic training accuracy and gradient variance
        & Figure~\ref{fig:trained-score-model}, Table~\ref{tab:trained-score-model}
        & We train the model by integrating the finite-sample
          loss in Eq.~\eqref{eq:auxiliary_loss_finite} over diffusion levels, using the
          denoising and conditional Hyv\"arinen forms related to this loss
          by Eq.~\eqref{eq:empirical_integrand_equivalence}, and their hybrid.
          We measure optimization error relative to the quadrature reference and gradient variance at
          both endpoints of the diffusion path, as discussed in the
          \hyperref[rmk:computation_noise]{remark on computational cost} in
          Section~\ref{sec:asymptotic}. \\
        \addlinespace[3pt]
        EM and gradient descent
        & Figure~\ref{fig:matched-start-dynamics}, Table~\ref{tab:matched-start-iterations}
        & We compare EM and likelihood gradient descent with preconditioned
          gradient descent on the fixed-scale and path-integrated score
          matching losses, using \(\Sigma=I_d\).
          We examine the joint Hessian conditioning of the gradient objectives
          and quantify the operator-Jacobian and latent-variance corrections
          to the EM residual in Corollary~\ref{cor:sm_gradients_em_variance}
          and Eq.~\eqref{eq:sm_em_scale_correction}. \\
        \addlinespace[3pt]
        Whitening for high-noise dynamics
        & Figure~\ref{fig:anisotropic-whitening}
        & We assess the preconditioned iteration in Eq.~\eqref{eq:whitened_gradient_iteration},
          to which the \hyperref[rmk:whitening_high_noise]{whitening remark}
          transfers the guarantees of
          Theorem~\ref{thm:global_convergence_high_noise_full}, under anisotropic
          covariates, and compare exact and estimated covariance preconditioning
          of the regression gradient. \\
        \addlinespace[3pt]
        Pointwise high-SNR blindness across covariate laws
        & Figure~\ref{fig:design-blindness}
        & We examine the vanishing limits in
          Proposition~\ref{prop:pointwise_sm_blindness} across five covariate laws
          and estimate the decay rates of the imbalance gradient and
          latent-variance term. \\
        \bottomrule
    \end{tabular}
\end{table}
\FloatBarrier

\begin{table}[!htbp]
    \centering
    \caption{\textbf{Notation and diagnostics for the additional studies.}
    All parameter estimates in the statistical diagnostics are aligned to the
    ground truth via the joint sign \(s_n\).}
    \label{tab:additional-notation}
    \small
    \setlength{\tabcolsep}{5pt}
    \renewcommand{\arraystretch}{1.08}
    \begin{tabular}{@{}>{\raggedright\arraybackslash}p{0.27\linewidth}>{\raggedright\arraybackslash}p{0.67\linewidth}@{}}
        \toprule
        Quantity & Definition / interpretation \\
        \midrule
        \(\vartheta=(\theta,\nu)\), \(I\)
        & Joint parameter vector; \(I:=I(\theta^\ast,\nu^\ast)\) denotes the
          \((d+1)\)-dimensional Fisher information matrix at the ground truth. \\
        \(\bar\alpha_{T_n}\), \(a\)
        & Terminal retained-signal factor and its limit,
          \(\bar\alpha_{T_n}\to a\in[0,1)\).
          Fixed-level studies set \(\bar\alpha_{T_n}=a\in(0,1)\) for every \(n\). \\
        \(I_a\)
        & Fisher information of the response diffused to level \(a\)
          together with \(\bx\),
          with respect to the unscaled parameters \((\theta,\nu)\). \\
        \(V_a\)
        & Asymptotic covariance, defined in the
          \hyperref[rmk:fixed_terminal_covariance]{covariance comparison in Appendix~\ref{sup:setup}},
          of the \(\sqrt n\)-scaled error of the aligned exact score matching estimator
          when \(\bar\alpha_{T_n}\to a\in(0,1)\)
          (distinct from the scalar latent-variance terms \(V\) and \(V_t\)). \\
        Distance in standard errors (s.e.)
        & \(\sqrt n\,\|\vartheta-\vartheta'\|_I\), where
          \(\|v\|_I:=(v^\top Iv)^{1/2}\), measuring the Fisher-standardized
          distance between two parameter estimates. \\
        Mean Fisher-normalized error; coverage
        & Empirical mean of \(n(\widehat\vartheta-\vartheta^\ast)^\top
          I(\widehat\vartheta-\vartheta^\ast)/(d+1)\) across repetitions;
          coverage is the fraction of repetitions in which the numerator is at most
          the 95th percentile of \(\chi^2_{d+1}\). \\
        \(\widehat\Sigma\), \(\|v\|_{\widehat\Sigma}^2\)
        & Empirical covariate second-moment matrix \(\widehat\Sigma=n^{-1}\sum_{i=1}^n\bx^{(i)}(\bx^{(i)})^\top\)
          and induced seminorm \(\|v\|_{\widehat\Sigma}^2=v^\top\widehat\Sigma v\)
          (estimated from 500 independent covariate draws in the whitening study). \\
        Condition number of \(\Sigma\)
        & \(\lambda_{\max}(\Sigma)/\lambda_{\min}(\Sigma)\), the ratio of the largest
          and smallest eigenvalues of the covariate second-moment matrix. \\
        \(\ell_t^{\mathsf H}\), \(\ell_t^{\mathsf D}\)
        & Per-observation conditional Hyv\"arinen and denoising integrands defined in
          Eqs.~\eqref{eq:hyvarinen_integrand} and~\eqref{eq:denoising_integrand},
          whose expectations over \(\xi\) share the same parameter gradient. \\
        \bottomrule
    \end{tabular}
\end{table}

\begin{table}[!htbp]
    \centering
    \caption{\textbf{Experimental settings for the additional studies.}
    Unless stated otherwise, \(\nu^\ast=0.55\),
    \(\|\theta^\ast\|_\Sigma^2=1.21\), and the response-noise variance is one.
    GD denotes gradient descent.}
    \label{tab:additional-settings}
    \small
    \setlength{\tabcolsep}{4pt}
    \renewcommand{\arraystretch}{1.08}
    \begin{tabular}{@{}>{\raggedright\arraybackslash}p{0.17\linewidth}>{\raggedright\arraybackslash}p{0.33\linewidth}>{\raggedright\arraybackslash}p{0.44\linewidth}@{}}
        \toprule
        Study & Model / data & Schedule / algorithm \\
        \midrule
        Score matching versus MLE and EM
        & Gaussian \(\bx\sim\calN(0,I_d)\), \(d\in\{1,10\}\);
          \(n=250\cdot2^j\) for \(j=0,\ldots,6\); 500 repetitions per setting.
        & \(\theta^\ast=1.10\) for \(d=1\); fixed unit direction scaled by \(1.10\) for \(d=10\).
          \(\bar\alpha_{T_n}=n^{-3/2}\); paired fits across all 7000 datasets. \\
        \addlinespace[2pt]
        Statistical calibration across dimensions and covariate laws
        & Six covariate laws detailed in the text;
          \(d\in\{2,10,50\}\), \(n\in\{1000,4000,16000\}\).
        & \(\bar\alpha_{T_n}=n^{-3/2}\); 200 repetitions per setting
          (100 for \(d=50\)). \\
        \addlinespace[2pt]
        Terminal schedules and statistical efficiency
        & Same 7000 datasets as the comparison with MLE and EM.
        & \(\bar\alpha_{T_n}\in\{n^{-3/2},n^{-1},n^{-1/2},1/\ln n,0.05,0.3\}\). \\
        \addlinespace[2pt]
        Stochastic training accuracy and gradient variance
        & First 50 datasets per setting; \(d\in\{1,10\}\),
          \(n\in\{1000,4000,16000\}\); \(\bar\alpha_{T_n}=n^{-3/2}\).
        & Adam (initial learning rate \(0.05\) with cosine decay, batch size 256)
          for 2000, 5000, or 20000 steps; iterate averaging over the second half.
          Terminal sweep at \((d,n)=(1,4000)\):
          \(\bar\alpha_{T_n}\in\{0.1,0.01,10^{-4},4000^{-3/2}\}\). \\
        \addlinespace[2pt]
        EM and gradient descent
        & Population: \(d=1\), \(24\times16\) starts on
          \([-3,3]\times[-2,2]\).
          Finite-sample: \(d=10\), \(n=2000\), 8 starts on each of 10 datasets.
        & Path-integrated terminal factor \(10^{-3}\) (population) and \(n^{-3/2}\) (data).
          Fixed-scale levels \(\bar\alpha_t\in\{0.005,0.05,0.3,0.7,0.95\}\).
          With \(\Sigma=I_d\), the gradient methods use a scalar step size equal to the reciprocal of the largest Hessian eigenvalue at the reference minimizer. \\
        \addlinespace[2pt]
        Whitening for high-noise dynamics
        & \(d\in\{2,20\}\), \(c^\ast=0.8e_1\).
          The largest eigenvalue of \(\Sigma\) is 1, and the smallest is
          1, 0.1, or 0.01, with corresponding condition numbers 1, 10, or 100.
          The eigenvalues are geometrically spaced, with random orthogonal eigenvectors.
        & Step size \(\eta=1/50\), 20000 steps, 2000 shared random starts:
          \(\theta_0\sim\calN(0,0.25I_d)\), \(\nu_0\sim\mathrm{Unif}[-1,1]\).
          Empirical whitening uses \(\widehat\Sigma\) from 500 independent Gaussian draws. \\
        \addlinespace[2pt]
        Pointwise high-SNR blindness across covariate laws
        & \(d=1\). Gaussian, Student \(t_5\) and \(t_{10}\), standardized
          Gamma(2), and uniform covariates. We vary \(\|\theta^\ast\|\).
        & \(\bar\alpha_t=1/2\), \((\theta,\nu)=(0.80,-0.35)\).
          A geometric covariate grid and 80-point Gauss--Hermite quadrature
          in the response noise evaluate the expectations. \\
        \bottomrule
    \end{tabular}
\end{table}

\par\smallskip
\noindent\textbf{Implementation for estimator comparisons and calibration.}
In the statistical comparisons, we evaluate the path-integrated score matching
objective via the empirical bridge in Eq.~\eqref{eq:empirical_sm_mle_bridge}, which
expresses \(\widehat{\calR}_{n,T_n}^{\,\mathsf{SM}}\) as the empirical negative
log-likelihood \(\widehat{\calR}_n^{\,\mathsf{MLE}}\) minus the terminal
discrepancy \(\Delta_n\), up to a parameter-independent constant. Because only
\(\Delta_n\) depends on the forward diffusion noise \(\xi\sim\calN(0,1)\),
evaluating the bridge requires only a one-dimensional Gaussian expectation at
the terminal scale \(\bar{\alpha}_{T_n}\). We compute this expectation using
24-point Gauss--Hermite quadrature for \(\bar{\alpha}_{T_n}\le0.02\), 40-point
quadrature for \(0.02<\bar{\alpha}_{T_n}\le0.15\), and 64-point quadrature for
\(\bar{\alpha}_{T_n}>0.15\). We minimize \(\widehat{\calR}_{n,T_n}^{\,\mathsf{SM}}\)
using a damped Newton method with analytic gradients and Hessians until
the Euclidean gradient norm falls below \(10^{-10}\). The MLE applies the same
damped Newton solver directly to \(\widehat{\calR}_n^{\,\mathsf{MLE}}\).
For EM, we iterate the empirical map in
Proposition~\ref{prop:empirical_em_stationarity} until the \(\ell_\infty\)
parameter change between consecutive iterates is at most \(10^{-12}\),
with a cap of 20000 iterations. The implementation clips the argument of
\(\operatorname{arctanh}\) to \([-1+10^{-15},1-10^{-15}]\) to keep the
numerical mixing-weight update finite.

We benchmarked quadrature and solver tolerances against a 120-node
Gauss--Hermite rule. At \(\bar{\alpha}_{T_n}=n^{-3/2}\), the 24-node and
120-node minimizers agree to within \(10^{-14}\) standard errors across eight
datasets per sample size, and tightening the Newton tolerance to \(10^{-14}\)
with 120 nodes at \(n\in\{4000,16000\}\) alters the score matching--MLE
discrepancy by at most \(4\times10^{-9}\) standard errors (less than \(1\%\) of
its median value at \(n=16000\)).

To reduce sensitivity to initialization, we initialize the likelihood solver and EM at
the same two deterministic, data-based starts and select their outputs by
empirical negative log-likelihood. The first initialization generalizes the
data-moment estimate of Appendix~\ref{sup:experiment} to \(d\ge1\): its
regression vector \(\theta_0\) has squared \(\widehat\Sigma\)-norm
\(\max\{\E_{\Pr_n}[y_0^2]-1,\,8\times10^{-2}\}\) and, for \(d>1\), is aligned
with the leading generalized eigenvector of
\(\E_{\Pr_n}[(y_0^2-1)\bx\bx^\top]\) relative to \(\widehat\Sigma\) (the
empirical counterpart of the principal generalized eigenspace of \(\calC^\ast\)).
Its mixing imbalance is set to
\(\operatorname{arctanh}(\operatorname{clip}(\E_{\Pr_n}[y_0\langle\theta_0,\bx\rangle]/\|\theta_0\|_{\widehat\Sigma}^2,-0.92,0.92))\),
where \(\operatorname{clip}\) truncates to the stated interval and the joint
sign of \(\theta_0\) is chosen so that the initial imbalance is nonnegative.
The second initialization is \((0.55\,u,0.15)\), where
\(u=\theta_0/\|\theta_0\|\) denotes the Euclidean unit direction of \(\theta_0\).
The reported EM estimator is the finite-iteration output from the first
start unless the second lowers the empirical negative log-likelihood by
more than \(10^{-13}\). We initialize both EM runs from the data,
independently of the likelihood optimization.

Score matching starts at the selected MLE in the comparison,
terminal-schedule, and covariate-robustness studies.
Only the comparison at \(\bar\alpha_{T_n}=n^{-3/2}\) adds an independent run
from the data-moment initialization and retains the fit with the smaller
score matching objective. We assess recovery from the data-moment
initialization by reporting those fits separately below. These comparisons
concern the numerical fits selected from the specified starts;
the matched-start study compares their optimization trajectories.

For the calibration diagnostics, the population Fisher information matrix
\(I(\theta^\ast,\nu^\ast)\) is computed via a one-dimensional reduction along
\(\langle\theta^\ast,\bx\rangle\) using 200-point Gauss--Hermite quadrature for
Gaussian covariates, and via \(4\times10^5\) Monte Carlo covariate draws paired
with response-noise quadrature for non-Gaussian covariate laws. All standardized
errors and 95\% confidence-ellipse coverages are evaluated using this
ground-truth Fisher metric after joint sign alignment \(s_n\). The coverage
diagnostic therefore measures calibration with known Fisher information.

\par\smallskip
\noindent\textbf{Score matching versus MLE and EM.}
We first test whether the path-integrated score matching estimator matches the
finite-sample statistical accuracy of the MLE and EM on identical datasets.
Here MLE and EM denote the numerical fits defined above. The
\hyperref[rmk:consistent_em_fixed_point]{remark on global MLE uniqueness
and numerical EM output} distinguishes exact likelihood minimizers from
these computed fits. Main Theorem~\ref{thm:consistency_asymptotic_normality}
predicts the inverse-Fisher limit for score matching. We examine this
prediction through mean Fisher-normalized errors and joint 95\%
Fisher-ellipse coverage, whose asymptotic benchmarks are one and 95\%,
respectively. Paired distances measure how closely the fitted parameters
agree on each dataset.
Figure~\ref{fig:mle-em-comparison} and Table~\ref{tab:mle-em-comparison}
summarize the results across all 7000 datasets for \(d\in\{1,10\}\) and
\(n=250\cdot2^j\), \(j=0,\ldots,6\), using \(\bar\alpha_{T_n}=n^{-3/2}\).

Across all sample sizes and both dimensions, the three estimators achieve
virtually indistinguishable mean Fisher-normalized errors and identical joint
95\% Fisher-ellipse coverage. In 13 of the 14 \((d,n)\) configurations, their
mean Fisher-normalized errors agree to three decimal places, differing by less
than \(10^{-3}\) in the remaining small-sample configuration \((d,n)=(10,250)\).

The paired discrepancy between the score matching estimator and
the MLE closely follows the \(\calO_P(\bar\alpha_{T_n}/\sqrt n)\) rate
established in Proposition~\ref{prop:paired_sm_mle}. At \(n=16000\) (where
\(\bar\alpha_{T_n}=4.9\times10^{-7}\)), the median Fisher-standardized
distance between the two estimators is \(6.1\times10^{-7}\) standard errors in
\(d=1\) and \(1.4\times10^{-6}\) standard errors in \(d=10\). The empirical
root-mean-square discrepancy agrees within \(5\%\) with the fixed-level
asymptotic benchmark \(\sqrt{\Tr(I\,V_{\bar\alpha_{T_n}})-(d+1)}\) at every \(n\ge500\).

From the same two data-based starts, the independently computed EM fit is
within \(2.3\times10^{-8}\) standard errors of the selected numerical MLE
on 6999 of the 7000 datasets. This is an observed discrepancy, distinct from
the consecutive-iterate stopping tolerance. Likewise, when initialized
at the data-moment estimate alone, score matching
converges to the selected solution within \(2.3\times10^{-8}\) standard errors
on 6998 of the 7000 datasets.

The only two exceptions occur at the smallest
sample size in ten dimensions, \((d,n)=(10,250)\), where both the MLE and
score matching converge to a secondary local optimum from the data-moment
start; in both instances, the second initialization \((0.55\,u,0.15)\) locates
the deeper basin for the MLE, and warm-starting score matching from that
solution recovers the same basin (whereas EM remains trapped at the secondary
optimum from both starts on one of the two datasets).

From the data-moment start, our bridge-based quadrature implementation
of score matching takes the same number of Newton steps as the MLE on every
dataset. Each evaluation also averages over terminal diffusion noise.
On one CPU core, this implementation takes 15 to 68 times as long per fit
as likelihood fitting. We measure batch wall times divided by the number
of datasets, excluding initialization and selection across starts.

To check agreement with direct score matching, we average the conditional
Hyv\"arinen integrand in Eq.~\eqref{eq:hyvarinen_integrand} over observations and
diffusion noise, then minimize its path integral starting from the
bridge-based estimate.
On the first 20 datasets at each \(n\le4000\) in \(d=1\), we use 24
Gauss--Legendre scale nodes and 41 Gauss--Hermite noise nodes. The two fits
agree within \(5\times10^{-5}\) standard errors.

\begin{figure}[!htbp]
    \centering
    \includegraphics[width=0.72\linewidth]{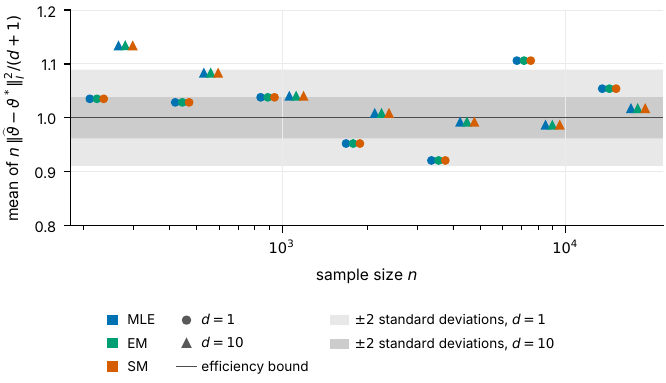}
    \caption{\textbf{Statistical accuracy of score matching, MLE, and EM.}
    Mean of \(n\) times the squared Fisher error divided by \(d+1\)
    for the three estimators, plotted side by side at each sample size \(n\)
    (circles for \(d=1\), triangles for \(d=10\)); the horizontal line at one
    marks the asymptotic efficiency benchmark, and the shaded bands show \(\pm2\)
    standard deviations of the plotted mean for an asymptotically efficient estimator
    (light gray for \(d=1\), dark gray for \(d=10\)).}
    \label{fig:mle-em-comparison}
\end{figure}

\begin{table}[!htbp]
    \centering
    \caption{\textbf{Accuracy, coverage, and paired differences.}
    Each row summarizes the 500 datasets of one setting of
    Figure~\ref{fig:mle-em-comparison} (\(\bx\sim\calN(\mathbf0,I_d)\),
    \(\bar\alpha_{T_n}=n^{-3/2}\)). We show two of its seven sample sizes.
    The \emph{Fisher-normalized error} is the mean of
    \(n(\widehat\vartheta-\vartheta^\ast)^\top I(\widehat\vartheta-\vartheta^\ast)/(d+1)\)
    for each estimator. For an asymptotically efficient estimator,
    \(n(\widehat\vartheta-\vartheta^\ast)^\top I(\widehat\vartheta-\vartheta^\ast)\)
    converges in distribution to \(\chi^2_{d+1}\), so the benchmark is one.
    The corresponding two-standard-deviation ranges for the mean over 500
    datasets are \(\pm0.089\) for \(d=1\) and \(\pm0.038\) for \(d=10\).
    \emph{Coverage} is the percentage of datasets in which
    \(n(\widehat\vartheta-\vartheta^\ast)^\top I(\widehat\vartheta-\vartheta^\ast)\)
    is at most the 95th percentile of \(\chi^2_{d+1}\), that is, in which the
    95\% Fisher ellipse around the estimate contains \(\vartheta^\ast\). The
    target is 95\%, with \(\pm2\) standard deviations of \(\pm1.9\) percentage
    points. The three estimators have identical coverage.
    The \emph{SM\(-\)MLE distance} is the median over datasets of
    \(\sqrt n\,\|\widehat\vartheta^{\,\mathsf{SM}}-\widehat\vartheta^{\,\mathsf{MLE}}\|_I\),
    the difference between the two estimates on the same dataset in units of the
    statistical standard error. Proposition~\ref{prop:paired_sm_mle} establishes
    the paired rate \(\calO_P(\bar\alpha_{T_n})\), with
    \(\bar\alpha_{T_n}\approx3.2\times10^{-5}\) at \(n=1000\)
    and \(\bar\alpha_{T_n}\approx4.9\times10^{-7}\) at \(n=16000\).}
    \label{tab:mle-em-comparison}
    \small
    \setlength{\tabcolsep}{5pt}
    \renewcommand{\arraystretch}{1.08}
    \begin{tabular}{@{}ccccc@{}}
        \toprule
        \(d\) & \(n\) & Fisher-normalized error, MLE / EM / SM & Coverage
            & SM\(-\)MLE distance (s.e.) \\
        \midrule
        1  & 1000  & 1.038 / 1.038 / 1.038 & 93.6\% & \(3.7\times10^{-5}\) \\
        1  & 16000 & 1.054 / 1.054 / 1.054 & 94.2\% & \(6.1\times10^{-7}\) \\
        10 & 1000  & 1.040 / 1.040 / 1.040 & 93.6\% & \(8.6\times10^{-5}\) \\
        10 & 16000 & 1.017 / 1.017 / 1.017 & 94.6\% & \(1.4\times10^{-6}\) \\
        \bottomrule
    \end{tabular}
\end{table}

\begin{figure}[!htb]
    \centering
    \includegraphics[width=0.9\linewidth]{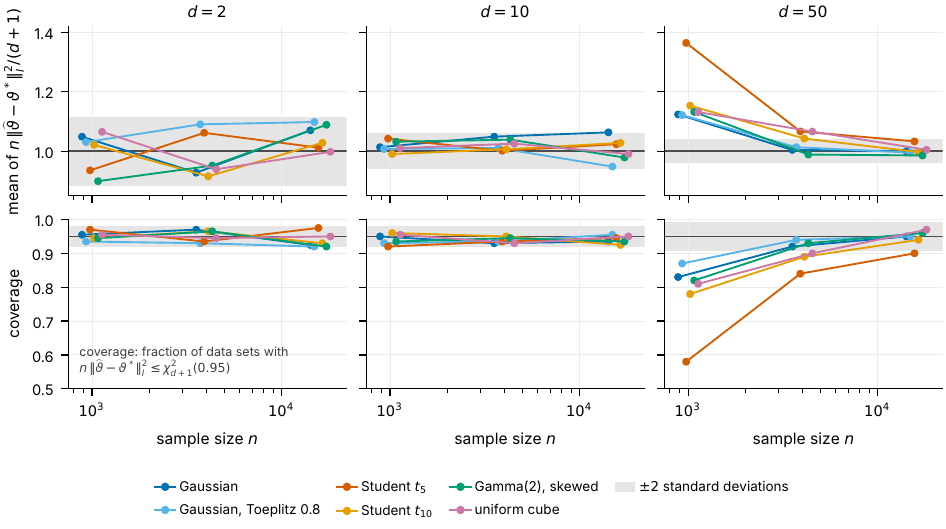}
    \caption{\textbf{Statistical calibration across dimensions and covariate laws.}
    Mean of \(n\) times the squared Fisher error divided by \(d+1\) (top row)
    and joint 95\% Fisher-ellipse coverage (bottom row) of the path-integrated
    score matching estimator across \(d\in\{2,10,50\}\) (columns) and
    \(n\in\{1000,4000,16000\}\), plotted against the sample
    size \(n\). Gray bands denote \(\pm2\) standard deviations for an
    asymptotically efficient estimator.}
    \label{fig:design-calibration}
\end{figure}

\par\smallskip
\noindent\textbf{Statistical calibration across dimensions and covariate laws.}
We assess statistical calibration in dimensions \(d\in\{2,10,50\}\) under
six covariate laws (Table~\ref{tab:additional-settings}): isotropic Gaussian;
Gaussian with Toeplitz covariance \(\Sigma_{ij}=0.8^{|i-j|}\); spherical
Student \(t_5\) and \(t_{10}\), normalized to unit covariance; independent
standardized Gamma(2) coordinates \((G-2)/\sqrt2\), with
\(G\sim\mathrm{Gamma}(2,1)\); and the uniform law on
\([-\sqrt3,\sqrt3]^d\). All six laws have zero means, unit marginal
variances, and finite fourth moments. For each dimension, one randomly drawn
unit direction \(u_d\) is shared across the laws, with
\(\theta^\ast=1.10\,u_d/\sqrt{u_d^\top\Sigma u_d}\), so that
\(\|\theta^\ast\|_\Sigma^2=1.21\).

The standardized Gamma and uniform densities vanish on subsets of
\(\R^d\) of positive volume. They therefore test calibration beyond the
positive-density assumption of Main Theorem~\ref{thm:consistency_asymptotic_normality} and
Proposition~\ref{prop:terminal_kl_pinsker_zero_set}.

Figure~\ref{fig:design-calibration} plots all 54 law--dimension--sample-size
settings against \(n\). At \(d=50\) and \(n=1000\),
the mean Fisher-normalized error exceeds one by \(12\%\) to \(15\%\), with
\(78\%\) to \(87\%\) coverage, for five laws. For Student \(t_5\), the
mean error is \(1.37\) and coverage is \(58\%\). At \(n=16000\), the mean
errors lie between \(0.948\) and \(1.098\), and coverage between \(90\%\)
and \(97.5\%\), across the 18 law--dimension pairs. The slowest-calibrating
case is \(t_5\) in \(d=50\), with \(90\%\) coverage. The MLE has the same
calibration pattern; the median score matching--MLE distance is at most
\(4.1\times10^{-6}\) standard errors across these 18 settings.
The
theorem keeps dimension fixed as sample size grows.

\par\smallskip
\noindent\textbf{Terminal schedules and statistical efficiency.}
This study examines whether slower vanishing terminal schedules retain
MLE efficiency and quantifies the variance increase at fixed positive
terminal levels.
We compare four vanishing terminal schedules,
\(\bar\alpha_{T_n}\in\{n^{-3/2},n^{-1},n^{-1/2},1/\ln n\}\), and two fixed
terminal levels, \(\bar\alpha_{T_n}=a\in\{0.05,0.3\}\).
Each objective integrates over retained-signal levels in \([\bar\alpha_{T_n},1]\);
the single-scale loss \(\calL_t\) is studied separately below.
All four vanishing schedules satisfy
Main Theorem~\ref{thm:consistency_asymptotic_normality}.
Only \(n^{-3/2}\) satisfies the stronger \(o(n^{-1})\) requirement of the
uncentered objective-value argument in the
\hyperref[rmk:terminal_schedule]{remark on terminal schedules}.
The other three schedules therefore probe the centered-gradient argument
that permits any deterministic vanishing terminal level.

At fixed \(a\in(0,1)\), the
\hyperref[rmk:fixed_terminal_covariance]{covariance comparison}
in Appendix~\ref{sup:setup} proves that the bridge has positive-definite
population Hessian \(I-I_a\) at the ground truth and that the
\(\sqrt n\)-scaled error of the aligned exact score matching estimator has
limiting covariance \(V_a\).
For the two estimators fitted to the same observations, it also shows that the influence functions of score
matching and the MLE have cross-covariance \(I^{-1}\). Thus
\[
    \sqrt n\bigl(\widetilde\vartheta_{n,T_n}^{\,\mathsf{SM}}
                 -\widetilde\vartheta_n^{\,\mathsf{MLE}}\bigr)
    \xrightarrow{\mathsf d} \calN(0,V_a-I^{-1}).
\]
The solid curves in Figure~\ref{fig:terminal-schedules}(b) plot the
following population predictions at each terminal level \(a\):
\[
\begin{aligned}
    \nu\text{ (orange):}\quad&
    \frac{[V_a-I^{-1}]_{d+1,d+1}}{[I^{-1}]_{d+1,d+1}},\\
    \theta\text{ (blue):}\quad&
    \frac1d\sum_{j=1}^d
    \frac{[V_a-I^{-1}]_{jj}}{[I^{-1}]_{jj}}.
\end{aligned}
\]
Each ratio is the relative asymptotic variance increase over the MLE,
so a value of \(0.01\) means a 1\% increase. The regression curve averages
the coordinatewise ratios. To evaluate these population covariance formulas
numerically, we sum exactly over the latent class. Conditioning on
\(\langle\theta^\ast,\bx\rangle\) reduces the Gaussian covariate
expectation to a one-dimensional integral, evaluated with 200-point
Gauss--Hermite quadrature. We evaluate the response and diffusion noise integrals
with 60 points each. Thus the solid curves depend on the model parameters
and terminal level, independently of the simulated estimator differences.

For vanishing schedules, we evaluate these fixed-level covariance formulas
at each actual terminal level \(\bar\alpha_{T_n}\) to obtain a benchmark with
\(V_{\bar\alpha_{T_n}}\to I^{-1}\). These values provide asymptotic
approximations for the finite-sample covariance comparison.

Empirically, the vanishing schedules \(n^{-1}\) and \(n^{-1/2}\) yield mean
Fisher-normalized errors within \(0.6\%\) of the MLE at every sample size and
within \(0.11\%\) for \(n\ge4000\); even the logarithmic schedule \(1/\ln n\)
(for which \(\bar{\alpha}_{T_n}\approx0.10\) at \(n=16000\)) stays within
\(4\%\) of the MLE across all \(n\) and within \(1\%\) at \(n=16000\).

Figure~\ref{fig:terminal-schedules}(a) examines the paired rate in
Proposition~\ref{prop:paired_sm_mle} using all 42 combinations of schedule
and sample size in \(d=10\), spanning nearly seven decades in
\(\bar{\alpha}_{T_n}/\sqrt n\). Its reference line follows from the
fixed-level limit displayed above, with \(V_a\) defined in the
\hyperref[rmk:fixed_terminal_covariance]{covariance comparison}
of Appendix~\ref{sup:setup}. The Gaussian limit has mean squared Fisher
norm \(\Tr\{I(V_a-I^{-1})\}=\Tr(I\,V_a)-(d+1)\). At level \(a\), the
corresponding Gaussian approximation to the unscaled estimator difference
has root-mean-square Fisher norm
\(\{\Tr(I\,V_a)-(d+1)\}^{1/2}/\sqrt n\).

For the plotted setting \(d=10\), \(\bx\sim\calN(\mathbf0,I_{10})\),
\(\|\theta^\ast\|=1.10\), and \(\nu^\ast=0.55\), the coefficient of the
reference line is the population RMS benchmark divided by
\(a/\sqrt n\), evaluated at the smallest terminal level in the study:
\[
    \left.
    \frac{\sqrt{\Tr\{I(V_a-I^{-1})\}}}{a}
    \right|_{a=16000^{-3/2}\,\approx\,4.94\times10^{-7}}
    \approx 2.93.
\]
The expanded covariance formula in the
\hyperref[rmk:fixed_terminal_covariance]{covariance comparison}
allows this coefficient to be evaluated without subtracting nearly equal
matrices. It depends on the stated ground truth parameters and dimension.

Across the tested levels up to \(2.5\times10^{-4}\),
\(\sqrt{\Tr\{I(V_a-I^{-1})\}}\) differs from \(2.93\,a\) by less than
\(0.2\%\). We therefore use the small-level reference line
\(2.93\,\bar\alpha_{T_n}/\sqrt n\), whose slope is one in the log--log plot.
For the tested terminal levels up to \(10^{-2}\), the empirical
root-mean-square falls within \(5\%\) of this line for all \(n\ge500\).

At larger terminal levels, the proportionality factor decreases to
\(1.77\) at \(a=0.3\). The fixed-level asymptotic benchmark
\(\{\Tr(I\,V_{\bar\alpha_{T_n}})-(d+1)\}^{1/2}/\sqrt n\) agrees with the
simulations within \(6\%\) for all \(n\ge500\). At \(n=250\) in \(d=10\),
the empirical root-mean-square is \(20\%\)--\(23\%\) above the benchmark
across the six schedules.

In standard errors, the paired discrepancy vanishes in probability under
all four vanishing schedules, as predicted by
Proposition~\ref{prop:paired_sm_mle}. At fixed \(a\), it instead converges in
distribution to the Fisher norm of the Gaussian limit displayed above.
For \(a=0.05\) and \(0.3\), the corresponding asymptotic root-mean-square
benchmarks are \(0.081\) and \(0.38\) in \(d=1\), and \(0.12\) and
\(0.53\) in \(d=10\).

Figure~\ref{fig:terminal-schedules}(b)
compares the predicted relative variance increase
\([V_a-I^{-1}]_{jj}/[I^{-1}]_{jj}\) with the empirical paired mean
\(n(\widehat\vartheta^{\,\mathsf{SM}}_j-\widehat\vartheta^{\,\mathsf{MLE}}_j)^2/[I^{-1}]_{jj}\).
The covariance identity makes this paired mean-square an asymptotic
diagnostic of the variance increase. Computing the difference on each
common dataset resolves small discrepancies without subtracting two
nearly equal empirical variance estimates.
The paired diagnostic agrees with the prediction
across more than eleven orders of magnitude, from \(7\times10^{-13}\) for
\(\nu\) at \(n=16000\), \(\bar\alpha_{T_n}=n^{-3/2}\), to \(0.14\) at
\(a=0.3\). This behavior is consistent with quadratic scaling in \(a\)
at small terminal levels.

The largest relative variance increase is in the mixing imbalance \(\nu\).
At \(a=0.3\) and \(n=16000\), its empirical variance increases by \(12.5\%\)
in \(d=1\) and \(13.0\%\) in \(d=10\), compared with the \(14.1\%\)
asymptotic prediction. The regression-coordinate variances change by at most
\(3.6\%\), with predicted increases of \(0.6\%\) to \(1.5\%\).
At \(a=0.05\), where \(a\|\theta^\ast\|^2=0.06\), the predicted variance
increase is below \(1\%\) in every coordinate.
The terminal level therefore controls the statistical cost of truncating
the diffusion path. In these experiments, the cost is small at \(a=0.05\),
whereas \(a=0.3\) mainly reduces the accuracy of the mixing-imbalance
estimate. Vanishing schedules remove this asymptotic variance increase.

\begin{figure}[!htbp]
    \centering
    \includegraphics[width=\linewidth]{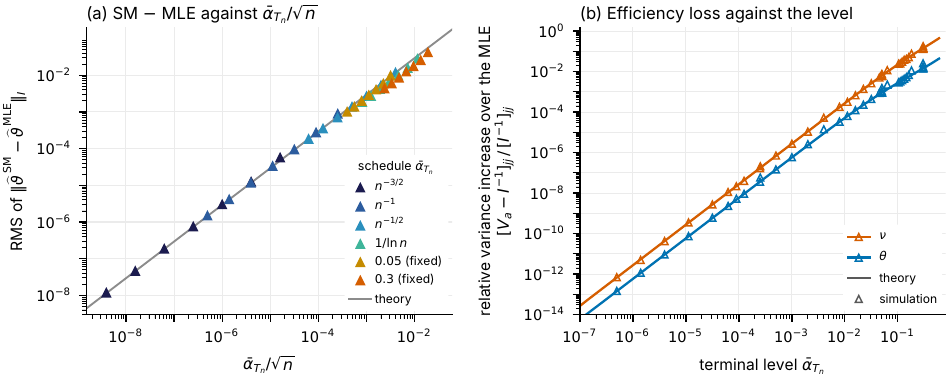}
    \caption{\textbf{Terminal schedules and statistical efficiency.}
    \textbf{(a)} Root-mean-square over 500 datasets of the Fisher-metric
    discrepancy \(\|\widehat\vartheta^{\,\mathsf{SM}}-\widehat\vartheta^{\,\mathsf{MLE}}\|_I\)
    plotted against \(\bar\alpha_{T_n}/\sqrt n\) for
    \(\bar\alpha_{T_n}\in\{n^{-3/2},n^{-1},n^{-1/2},1/\ln n,0.05,0.3\}\)
    across the seven sample sizes
    \(n\in\{250,500,1000,2000,4000,8000,16000\}\) in \(d=10\).
    The unit-slope reference line is
    \(2.93\,\bar\alpha_{T_n}/\sqrt n\), where \(2.93\) is the population
    value of \(\sqrt{\Tr\{I(V_a-I^{-1})\}}/a\), rounded to two decimal
    places, at \(a=16000^{-3/2}\).
    \textbf{(b)} Relative asymptotic variance increase over the MLE for \(\nu\)
    and for \(\theta\) (averaged across coordinates in \(d=10\)). Solid curves plot
    the theoretical ratio \([V_a-I^{-1}]_{jj}/[I^{-1}]_{jj}\) as a function of
    the terminal level \(a\); markers plot the empirical mean of
    \(n(\widehat\vartheta^{\,\mathsf{SM}}_j-\widehat\vartheta^{\,\mathsf{MLE}}_j)^2/[I^{-1}]_{jj}\)
    at \(\bar{\alpha}_{T_n}=a\).}
    \label{fig:terminal-schedules}
\end{figure}

\par\smallskip
\noindent\textbf{Stochastic training accuracy and gradient variance.}
To examine how terminal-scale selection and stochastic gradient variance affect
practical diffusion training when noise levels and Gaussian perturbations are
sampled online, we train \(s_{\theta_t,\nu}\) in Eq.~\eqref{eq:score} using Adam
under the protocol in Table~\ref{tab:additional-settings}. This is a single-unit
\(\tanh\) network with parameters tied across noise scales.
Working within the exact parametric score family isolates stochastic
optimization error from approximation error.

For a data point $(y_0,\bx)$, define $\mu_0=\langle\theta,\bx\rangle$,
$\mu_t=\sqrt{\bar\alpha_t}\mu_0$, and
$y_t=y_0\sqrt{\bar\alpha_t}+\xi\sqrt{1-\bar\alpha_t}$ as in
Section~\ref{subsec:2mlr}. To obtain the conditional Hyv\"arinen integrand,
apply \citet[Theorem~1, Eq.~(3), p.~697]{hyvarinen2005estimation}
to the scalar response, conditional on \(\bx\). One half of the mean squared
score discrepancy equals the expectation of
\(\tfrac12s_{\theta_t,\nu}^2+\partial_{y_t}s_{\theta_t,\nu}\), up to a
parameter-independent constant. To remove the ground truth score from
the cross term, we integrate by parts:
\[
    \E_{y_t\mid\bx;\theta_t^\ast,\pi^\ast}
      [s_{\theta_t,\nu}s_{\theta_t^\ast,\nu^\ast}]
    =-\E_{y_t\mid\bx;\theta_t^\ast,\pi^\ast}
      [\partial_{y_t}s_{\theta_t,\nu}],
\]
where both scores are evaluated at \((y_t,\bx)\). Expanding one half of
the squared score difference yields the stated objective. The boundary
term vanishes because the conditional laws are Gaussian mixtures. The
empirical form replaces the expectation by a sample average
\citep[Eq.~(4), p.~698]{hyvarinen2005estimation}. All derivatives act on
the response coordinate with the covariate vector held fixed.

Our loss in Eq.~\eqref{eq:loss} includes the additional factor
\(1/\bar\alpha_t\) for integration with respect to \(\bar\alpha_t\).
For numerical evaluation, we also subtract
\((y_t^2-2)/(2\bar\alpha_t)\), the weighted integrand for the
standard Gaussian score \(-y_t\). This term is independent of
\(\vartheta\), so its removal preserves parameter gradients and minimizers.
The model score in Eq.~\eqref{eq:score} and its response derivative are
\[
    s_{\theta_t,\nu}=-y_t+\mu_t\tanh(\mu_t y_t+\nu),
    \qquad
    \partial_{y_t}s_{\theta_t,\nu}
    =-1+\mu_t^2\{1-\tanh^2(\mu_t y_t+\nu)\}.
\]
We substitute these expressions and \(\mu_t=\sqrt{\bar\alpha_t}\mu_0\)
to obtain the per-observation integrand
\begin{equation}\label{eq:hyvarinen_integrand}
\begin{aligned}
\ell_t^{\mathsf H}(\vartheta;y_0,\bx,\xi)
&:=\frac{s_{\theta_t,\nu}(y_t,\bx)^2
  +2\partial_{y_t}s_{\theta_t,\nu}(y_t,\bx)-y_t^2+2}
  {2\bar\alpha_t}\\
&=-\frac{y_t\mu_0}{\sqrt{\bar\alpha_t}}\tanh(\mu_t y_t+\nu)
  +\frac{\mu_0^2}{2}\{2-\tanh^2(\mu_t y_t+\nu)\}.
\end{aligned}
\end{equation}
Thus, Eq.~\eqref{eq:hyvarinen_integrand} is Hyv\"arinen's scalar integrand
weighted by \(1/\bar\alpha_t\) and centered by subtracting the Gaussian baseline.

We compare the per-observation
squared denoising residual,
\begin{equation}\label{eq:denoising_integrand}
    \ell_t^{\mathsf D}(\vartheta;y_0,\bx,\xi)
    :=\frac{1}{2\bar\alpha_t}
      \left(s_{\theta_t,\nu}(y_t,\bx)
           +\frac{\xi}{\sqrt{1-\bar\alpha_t}}\right)^2,
    \qquad 0<\bar\alpha_t<1,
\end{equation}
against the conditional Hyv\"arinen integrand \(\ell_t^{\mathsf H}\)
in Eq.~\eqref{eq:hyvarinen_integrand}. The
squared transition-score term \(\xi^2/[2\bar\alpha_t(1-\bar\alpha_t)]\)
is independent of \(\vartheta\). Subtracting it from \(\ell_t^{\mathsf D}\)
recovers the integrand in Eq.~\eqref{eq:auxiliary_loss_finite} before averaging.
To compare the noise-averaged integrands, we fix an observation
\((y_0,\bx)\) and apply Gaussian integration by parts:
\[
    \E_\xi\!\left[
      s_{\theta_t,\nu}(y_t,\bx)\frac{\xi}{\sqrt{1-\bar\alpha_t}}
    \right]
    =\E_\xi[\partial_{y_t}s_{\theta_t,\nu}(y_t,\bx)],
\]
so expanding Eq.~\eqref{eq:denoising_integrand} and substituting
\(\E_\xi[y_t^2]=\bar\alpha_t y_0^2+1-\bar\alpha_t\) yields the relation
\[
    \E_\xi\ell_t^{\mathsf D}
    =\E_\xi\ell_t^{\mathsf H}
      +\frac{y_0^2}{2}
      +\frac{\bar\alpha_t}{2(1-\bar\alpha_t)}.
\]
We now average over the observed pairs to relate these integrands to the empirical
loss in Eq.~\eqref{eq:auxiliary_loss_finite}. For \(0<\bar\alpha_t<1\),
\begin{equation}\label{eq:empirical_integrand_equivalence}
\begin{aligned}
    \calJ_t^n(\theta,\nu)
    &=\E_{\Pr_n}\E_\xi\ell_t^{\mathsf D}
      -\frac{1}{2\bar\alpha_t(1-\bar\alpha_t)}\\
    &=\E_{\Pr_n}\E_\xi\ell_t^{\mathsf H}
      +\frac12\E_{\Pr_n}[y_0^2]
      -\frac{1+\bar\alpha_t}{2\bar\alpha_t}.
\end{aligned}
\end{equation}
All offsets are independent of \(\vartheta\), so the three noise-averaged
objectives have the same parameter gradients and minimizers at each fixed
level. For each observation,
\(\E_\xi\nabla_\vartheta\ell_t^{\mathsf D}=\E_\xi\nabla_\vartheta\ell_t^{\mathsf H}\),
while the sampled gradients can have different variances.
For path integration up to \(\bar\alpha_0=1\), subtract the squared
transition-score term from \(\ell_t^{\mathsf D}\) before integrating.
After averaging over the noise and observed pairs, this recovers
\(\widehat{\calR}_{n,T_n}^{\,\mathsf{SM}}\) and removes the parameter-independent
divergence of the raw squared denoising loss at the low-noise endpoint.

Each Adam step samples 256 observations with replacement. Each observation
receives an independent retained-signal level and Gaussian noise
\(\xi\sim\calN(0,1)\). Under ideal sampling from a level density \(q\) on
\([\bar\alpha_{T_n},1]\), we divide the per-level gradient by \(q(\bar\alpha_t)\)
to obtain an unbiased gradient of the path-integrated objective. The implementation
caps sampled levels at \(1-10^{-12}\) and clips \(\nu\) to \([-8,8]\)
for numerical stability. Every run starts from the data-moment estimate and
is compared with the quadrature reference on the same dataset.

We compare three stochastic gradients with the same expectation before endpoint
capping, all with uniform levels: \(\ell_t^{\mathsf D}\), \(\ell_t^{\mathsf H}\),
and a hybrid that evaluates \(\ell_t^{\mathsf D}\)
for \(\bar\alpha_t\le\tfrac12\) and \(\ell_t^{\mathsf H}\) otherwise.

\begin{figure}[!htbp]
    \centering
    \includegraphics[width=\linewidth]{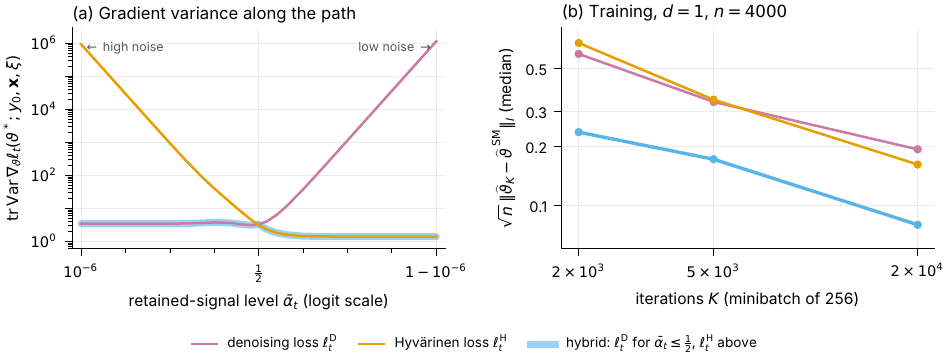}
    \caption{\textbf{Stochastic training accuracy and gradient variance.}
    \textbf{(a)} We evaluate the per-observation parameter-gradient variance by
    quadrature at \((\theta^\ast,\nu^\ast)=(1.10,0.55)\) across the diffusion
    path (on a logit scale in \(\bar\alpha_t\)) for the denoising integrand
    in Eq.~\eqref{eq:denoising_integrand}, the conditional Hyv\"arinen
    integrand in Eq.~\eqref{eq:hyvarinen_integrand}~\citep{hyvarinen2005estimation},
    and the single-draw hybrid integrand. Variance is over both the observation and
    diffusion noise at each fixed level, summed over parameter coordinates.
    \textbf{(b)} We plot the median Fisher-standardized distance from the
    trained iterate to the reference estimator over 50 datasets against the
    training budget \(K\), with \(d=1\) and \(n=4000\) fixed.
    Each budget uses a separate training run.}
    \label{fig:trained-score-model}
\end{figure}

To understand how diffusion-level sampling affects gradient noise, we
evaluate the variance of the per-observation stochastic gradient at the ground truth
parameters for each fixed diffusion level
(Figure~\ref{fig:trained-score-model}(a)). For the denoising form, this variance
grows as \((1-\bar\alpha_t)^{-1}\) at low noise. For the conditional
Hyv\"arinen form, it grows as \(\bar\alpha_t^{-1}\) at high noise.
To see these orders, we differentiate Eq.~\eqref{eq:denoising_integrand} and
use \(y_t\approx y_0\) at low noise. At high noise, we expand the gradient
of Eq.~\eqref{eq:hyvarinen_integrand} using \(y_t\approx\xi\) and
\(\tanh(\mu_t y_t+\nu)\approx\tanh\nu\). The leading random terms are
\[
\begin{aligned}
    \nabla_\vartheta\ell_t^{\mathsf D}
    &\approx\frac{\xi}{\sqrt{1-\bar\alpha_t}}\,
        \nabla_\vartheta s_{\theta,\nu}(y_0,\bx)
        \propto\frac{1}{\sqrt{1-\bar\alpha_t}},
        &&\bar\alpha_t\to1,\\
    \nabla_\vartheta\ell_t^{\mathsf H}
    &\approx-\frac{\xi}{\sqrt{\bar\alpha_t}}
        \begin{pmatrix}\bx\tanh\nu\\ \mu_0\sech^2\nu\end{pmatrix}
        \propto\frac{1}{\sqrt{\bar\alpha_t}},
        &&\bar\alpha_t\to0.
\end{aligned}
\]
Here \(\propto\) denotes the diffusion-level scaling of each leading random term.
The omitted terms have bounded second moments in this Gaussian setting.
The coefficient vectors are independent of \(\xi\) and have finite,
nonzero second moments at the plotted ground truth parameters.
Since \(\xi\) has mean zero and variance one, the two variance orders
follow by squaring the displayed prefactors.

Uniform diffusion-level sampling has density
\(q(\bar\alpha_t)=1/(1-\bar\alpha_{T_n})\). We therefore multiply each
sampled gradient by the interval length \(1-\bar\alpha_{T_n}\) to estimate
the gradient of the path-integrated loss. At the ground truth, the
population mean gradient is zero at every level. In the variance calculation,
the squared multiplier \((1-\bar\alpha_{T_n})^2\) combines with the sampling
density to leave one factor \(1-\bar\alpha_{T_n}\). The high-noise
contribution to the variance of the normalized stochastic gradient for the
conditional Hyv\"arinen form therefore has the same order as
\((1-\bar\alpha_{T_n})
\int_{\bar\alpha_{T_n}}^{1/2}\frac{\d\bar\alpha_t}{\bar\alpha_t}
=(1-\bar\alpha_{T_n})\ln\frac{1}{2\bar\alpha_{T_n}}\).
This contribution grows logarithmically as the terminal factor decreases.
The hybrid uses denoising at high noise and Hyv\"arinen at low noise,
keeping the variance bounded at both endpoints in this diagnostic.

After 20000 steps, the hybrid has median optimization errors between
\(0.05\) and \(0.17\) standard errors for \(d=1\) and between \(0.10\) and
\(0.41\) for \(d=10\) (Table~\ref{tab:trained-score-model}). Its mean
Fisher-normalized estimation error is within \(3\%\) of the reference in every
setting. The two plain variants have \(9\%\) to \(15\%\) excess error at
\(n=16000\).

In Figure~\ref{fig:trained-score-model}(b), with \(d=1\) and \(n=4000\),
the hybrid reaches a median distance of \(0.080\) standard errors after
20000 Adam steps, compared with \(0.193\) for plain denoising and \(0.162\)
for plain Hyv\"arinen. With a budget of 5000 steps, the hybrid reaches
\(0.172\), comparable to both plain forms at 20000 steps.
We train separately at each budget, with its own learning-rate
schedule and iterate averaging. At the fixed budget \(K=20000\) in
Table~\ref{tab:trained-score-model}, the remaining optimization error
occupies a larger fraction of the \(n^{-1/2}\) statistical error scale
as \(n\) increases.

For the terminal-level sweep, we set \((d,n)=(1,4000)\). Its minibatches, diffusion
levels, and Gaussian noise are drawn independently of those used for
Table~\ref{tab:trained-score-model}. As \(\bar\alpha_{T_n}\) decreases from
\(0.1\) to \(4000^{-3/2}\approx4\times10^{-6}\), the largest
per-observation gradient variance recorded while training the plain
Hyv\"arinen variant rises from 8 to 5100, and its median optimization error
rises from \(0.093\) to \(0.156\) standard errors. The hybrid and plain
denoising variants vary little over these levels, with errors of
\(0.076\)--\(0.084\) and \(0.204\)--\(0.210\) standard errors, respectively.

\begin{table}[!htbp]
    \centering
    \caption{\textbf{Stochastic training error relative to the quadrature reference.}
    We vary the sample size \(n\) at the fixed budget \(K=20000\) Adam steps,
    with 256 observations per minibatch.
    Each entry is the median over 50 datasets of
    \(\sqrt n\,\|\widehat\vartheta_K-\widehat\vartheta^{\,\mathsf{SM}}\|_I\), where
    \(\widehat\vartheta_K\) is the trained parameter estimate for the indicated
    per-level loss and \(\widehat\vartheta^{\,\mathsf{SM}}\) is the score matching
    reference computed by quadrature on the same dataset
    (\(\bar\alpha_{T_n}=n^{-3/2}\)). The factor \(\sqrt n\) expresses the
    remaining optimization error on the \(n^{-1/2}\) statistical scale,
    so \(0.1\) represents one tenth of a standard error.
    Figure~\ref{fig:trained-score-model}(b) varies \(K\) at fixed \(d=1\) and
    \(n=4000\). Its \(K=20000\) points match the corresponding column here.
    For comparison, the asymptotic root-mean-square sampling benchmark in
    these units is \(\sqrt{d+1}\), or \(1.4\) for \(d=1\) and
    \(3.3\) for \(d=10\).}
    \label{tab:trained-score-model}
    \small
    \setlength{\tabcolsep}{4.5pt}
    \renewcommand{\arraystretch}{1.08}
    \begin{tabular}{@{}lcccccc@{}}
        \toprule
        & \multicolumn{3}{c}{\(d=1\)} & \multicolumn{3}{c}{\(d=10\)} \\
        Training loss & \(n=1000\) & \(4000\) & \(16000\) & \(n=1000\) & \(4000\) & \(16000\) \\
        \midrule
        Denoising    & 0.093 & 0.193 & 0.315 & 0.310 & 0.676 & 1.208 \\
        Hyv\"arinen  & 0.083 & 0.162 & 0.456 & 0.198 & 0.425 & 0.875 \\
        Hybrid       & 0.051 & 0.080 & 0.168 & 0.096 & 0.189 & 0.408 \\
        \bottomrule
    \end{tabular}
\end{table}

\par\smallskip
\noindent\textbf{EM and gradient descent.}
To examine the optimization implications of the EM decomposition, we compare
EM and gradient descent on the negative log-likelihood with
\(\Sigma^{-1}\)-preconditioned gradient descent on the fixed-scale and
path-integrated score matching losses from identical initializations
(Table~\ref{tab:additional-settings}). The covariates are standard Gaussian,
so \(\Sigma=I_d\) and these preconditioned score matching steps coincide
with gradient descent. We evaluate population objectives by
quadrature. To account for differences in objective scale, we set each gradient
method's fixed step size to \(1/\lambda_{\max}\), where
\(\lambda_{\max}\) is the largest eigenvalue of its objective's Hessian at
its numerical reference minimizer. The same scalar step size applies to
both gradients in the original \((\theta,\nu)\) coordinates.
A run stops when its Euclidean distance
to that minimizer or its joint sign flip falls below \(10^{-8}\) for population objectives
(\(d=1\), 384 grid starts) or \(10^{-6}\) on data (\(d=10\), \(n=2000\),
80 runs). These iteration counts compare convergence with the reference
solutions and Hessians computed in advance. Total computational cost also
depends on the work per iteration.

All tested runs reach their reference solutions up to joint sign
(Figure~\ref{fig:matched-start-dynamics}(a) and
Table~\ref{tab:matched-start-iterations}). Gradient descent on the
path-integrated objective takes the same number of iterations as likelihood
gradient descent, to within one iteration from every start. This agreement is
consistent with the small terminal discrepancy in the likelihood bridge.
EM takes fewer iterations than likelihood gradient descent at the baseline
signal strength \(\|\theta^\ast\|^2=1.21\).

The joint population Hessian condition number in Table~\ref{tab:matched-start-iterations}
is \(\lambda_{\max}/\lambda_{\min}\), where \(\lambda_{\min}>0\) is the
smallest eigenvalue at the reference minimizer. With step size
\(1/\lambda_{\max}\), the slowest direction of the linearized gradient
iteration contracts by \(1-\lambda_{\min}/\lambda_{\max}\) per step.
A larger condition number therefore indicates slower local convergence.

The fixed-scale loss has different conditioning across noise levels. At
\(\bar\alpha_t\in\{0.05,0.3\}\), its Hessian condition numbers are
\(2.36\) and \(2.64\), below the likelihood's \(3.06\). The condition
number increases to \(6.79\) at \(\bar\alpha_t=0.95\) and \(18.0\) at
\(\bar\alpha_t=0.005\). The longest population run at the latter level takes
18102 iterations. The maximum count also depends on the trajectory before
it reaches this local neighborhood. The high-noise limit
\(\tfrac12\|\theta\tanh\nu-c^\ast\|_\Sigma^2\) identifies only the product
\(\theta\tanh\nu\), which explains the loss of local identifiability as
\(\bar\alpha_t\to0\).

These rows vary the single level \(\bar\alpha_t\) of the fixed-scale loss.
For the path-integrated row, we set the terminal factor to \(10^{-3}\) for the population
study and \(n^{-3/2}\) on data. Terminal-schedule selection is examined in
Figure~\ref{fig:terminal-schedules}, and stochastic training cost in
Figure~\ref{fig:trained-score-model} and Table~\ref{tab:trained-score-model}.

\begin{figure}[!htbp]
    \centering
    \includegraphics[width=\linewidth]{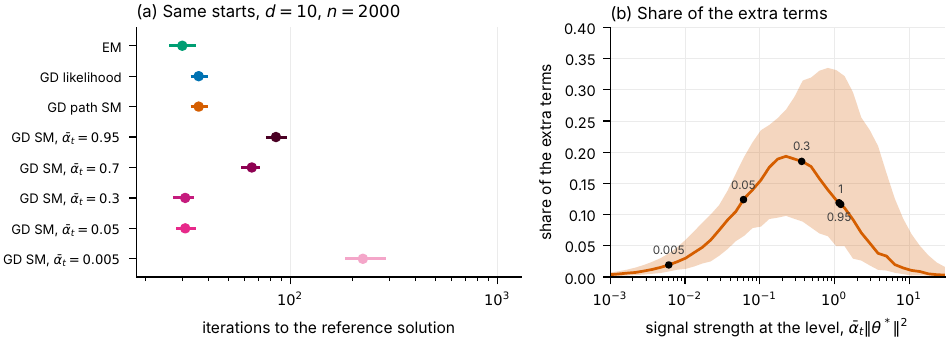}
    \caption{\textbf{EM and gradient descent: convergence and gradient corrections.}
    Panel \textbf{(a)} shows the median and 10--90\% range of iteration counts
    for EM and gradient descent to reach within \(10^{-6}\) of the reference solution across
    80 runs (\(d=10\), \(n=2000\)).
    Panel \textbf{(b)} shows the relative magnitude share
    of the combined operator-Jacobian and latent-variance corrections from
    Corollary~\ref{cor:sm_gradients_em_variance} in the \(\theta_t\)-gradient of
    \(\calL_t\) versus the effective signal strength
    \(\bar\alpha_t\|\theta^\ast\|^2\): median (curve) and 25--75\%
    quantile band across 117 candidate pairs, with dots marking
    \(\bar\alpha_t\in\{0.005,0.05,0.3,0.95,1\}\) at \(\|\theta^\ast\|^2=1.21\).}
    \label{fig:matched-start-dynamics}
\end{figure}

\begin{table}[!htbp]
    \centering
    \caption{\textbf{Iterations and joint Hessian conditioning.}
    We report the median and maximum numbers of iterations that EM and
    gradient descent require to reach
    within \(10^{-8}\) of the population reference solution (\(d=1\), 384 starts)
    and within \(10^{-6}\) on finite-sample data (\(d=10\), \(n=2000\), 80 runs),
    up to joint sign. The data medians match the points in
    Figure~\ref{fig:matched-start-dynamics}(a), whose intervals show the
    10--90\% range. Medians are rounded to integers.
    The Hessian condition number \(\lambda_{\max}/\lambda_{\min}\) is the ratio of the largest
    to the smallest eigenvalue of each population objective's Hessian with respect
    to both \(\theta\) and \(\nu\), evaluated at its reference minimizer
    (\(d=1\)). A dash marks the EM row.}
    \label{tab:matched-start-iterations}
    \small
    \setlength{\tabcolsep}{4.5pt}
    \renewcommand{\arraystretch}{1.08}
    \begin{tabular}{@{}lccccc@{}}
        \toprule
        Method & \multicolumn{3}{c}{Population, \(d=1\)} & \multicolumn{2}{c}{Data, \(d=10\)} \\
        \cmidrule(lr){2-4}\cmidrule(lr){5-6}
        & Median & Maximum & \shortstack{Joint Hessian\\condition number} & Median & Maximum \\
        \midrule
        EM                                              & 38  & 42    & --   & 30 & 39  \\
        Gradient descent, likelihood                    & 47  & 57    & 3.06 & 36 & 44  \\
        Gradient descent, path-integrated objective      & 47  & 57    & 3.07 & 36 & 44  \\
        Gradient descent, \(\calL_t\), \(\bar\alpha_t=0.95\)  & 116 & 127   & 6.79 & 85 & 104 \\
        Gradient descent, \(\calL_t\), \(\bar\alpha_t=0.7\)   & 86  & 95    & 5.13 & 65 & 76  \\
        Gradient descent, \(\calL_t\), \(\bar\alpha_t=0.3\)   & 39  & 48    & 2.64 & 31 & 42  \\
        Gradient descent, \(\calL_t\), \(\bar\alpha_t=0.05\)  & 37  & 81    & 2.36 & 31 & 39  \\
        Gradient descent, \(\calL_t\), \(\bar\alpha_t=0.005\) & 360 & 18102 & 18.0 & 224 & 967 \\
        \bottomrule
    \end{tabular}
\end{table}

To quantify the regression correction in
Corollary~\ref{cor:sm_gradients_em_variance}, let
\(g_t^{\mathsf{EM}}:=\Sigma(\theta_t-M(\theta_t,\nu))\),
\(g_t^{\mathsf{Jac}}:=-(\nabla_{\theta_t}M)^\top\Sigma\theta_t\), and
\(g_t^{\mathsf{var}}:=\nabla_{\theta_t}V\) denote the three components of
\(\nabla_{\theta_t}[\bar\alpha_t\calL_t]\) in \(d=1\). The two corrections
cancel at the ground truth. To assess their combined effect at other
candidates, the numerical comparisons below use the share
\[
    \frac{|g_t^{\mathsf{Jac}}+g_t^{\mathsf{var}}|}
    {|g_t^{\mathsf{EM}}|+|g_t^{\mathsf{Jac}}+g_t^{\mathsf{var}}|}.
\]
A zero share means that the corrections cancel in the regression block;
the mixing update and the scaling of the regression step can still differ
from EM.

The diffused model has signal \(\theta_t^\ast=\sqrt{\bar\alpha_t}\theta^\ast\).
For candidates scaled with \(\theta^\ast\), the share therefore depends on
\(\|\theta^\ast\|\) and \(\bar\alpha_t\) only through
\(\bar\alpha_t\|\theta^\ast\|^2\). In Figure~\ref{fig:matched-start-dynamics}(b),
we evaluate the share on a \(13\times9\) grid, equally spaced in
\((\theta/\theta^\ast,\nu)\in[0.1,2.5]\times[-2,2]\). Its median share is at
most \(0.03\) when \(\bar\alpha_t\|\theta^\ast\|^2\le10^{-2}\) or
\(\bar\alpha_t\|\theta^\ast\|^2\ge10\), peaks at \(0.19\) near \(0.2\), and
has a 75th percentile of up to \(0.34\). On this grid, the corrections thus
have their largest relative size at intermediate signal levels.

The corrections also affect the regression update near the optimum. Along
the near-solution portions of population trajectories initialized at
\((\theta_0,\nu_0)=(0.3,0)\), the ratio
\((g_t^{\mathsf{Jac}}+g_t^{\mathsf{var}})/g_t^{\mathsf{EM}}\) is approximately
\(+1.25\) at \(\bar\alpha_t=0.05\) and \(-1.27\) at
\(\bar\alpha_t=0.3\). The full regression gradient is then about \(2.25\)
and \(-0.27\) times the residual term, respectively. The corrections can
therefore amplify the residual-only update or reverse its direction.

\par\smallskip
\noindent\textbf{Whitening for high-noise dynamics.}
We examine how covariance anisotropy affects the high-noise limiting loss
\(\calL_\Sigma(\theta,\nu)=\tfrac12\|\theta\tanh\nu-c^\ast\|_\Sigma^2\)
from Eq.~\eqref{eq:whitened_high_noise_loss}. As noted in
Section~\ref{subsec:loss_landscape}, the
\(\Sigma^{-1}\)-preconditioned iteration in Eq.~\eqref{eq:whitened_gradient_iteration}
is gradient descent in the coordinates
\((\widetilde\theta,\nu):=(\Sigma^{1/2}\theta,\nu)\), with target
\(\widetilde c^\ast:=\Sigma^{1/2}c^\ast\). Thus
Theorem~\ref{thm:global_convergence_high_noise_full} applies when its
admissibility condition holds in these coordinates. Translating the theorem
back replaces \(J\) by
\(J_\Sigma(\vartheta):=\tfrac12(\|\theta\|_\Sigma^2-\sinh^2\nu)\) and
\(\|c^\ast\|\) by \(\|c^\ast\|_\Sigma\) in the admissibility condition, rate
constants, and saddle loss value. This theoretical guarantee uses the exact
covariance; below we evaluate whitening by \(\widehat\Sigma^{-1}\)
numerically.

For \(c^\ast\ne\mathbf0\), the critical points of \(\calL_\Sigma\) consist of
the global minima \(\{(\theta,\nu):\theta\tanh\nu=c^\ast\}\) and the saddle set
\(\calS_\Sigma:=\{(\theta,0):\langle\theta,\Sigma c^\ast\rangle=0\}\);
preconditioning does not change them.

\begin{figure}[!htbp]
    \centering
    \includegraphics[width=0.74\linewidth]{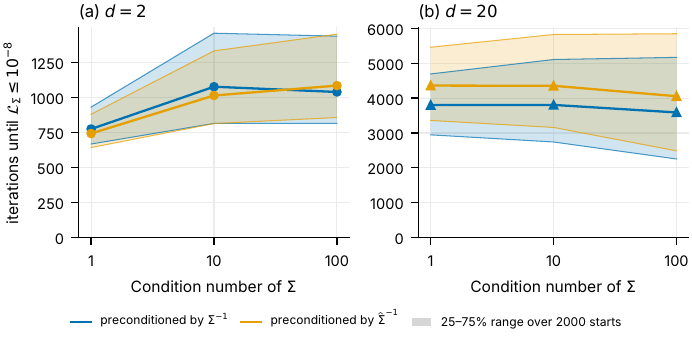}
    \caption{\textbf{Whitening for high-noise dynamics.}
    We report the number of iterations for gradient descent on
    \(\tfrac12\|\theta\tanh\nu-c^\ast\|_\Sigma^2\), with the regression gradient preconditioned by
    \(\Sigma^{-1}\) as in Eq.~\eqref{eq:whitened_gradient_iteration} or by the
    estimated \(\widehat\Sigma^{-1}\), to reach loss \(10^{-8}\), as a function of
    the condition number of \(\Sigma\) for \textbf{(a)} \(d=2\) and
    \textbf{(b)} \(d=20\). Curves show the median and shaded bands the
    25--75\% range over the same 2000 random starts.}
    \label{fig:anisotropic-whitening}
\end{figure}

Figure~\ref{fig:anisotropic-whitening} compares descent preconditioned by
\(\Sigma^{-1}\), the iteration covered by the theorem and the preconditioning
under which score matching relates to EM in Eq.~\eqref{eq:sm_em_scale_correction},
with descent
preconditioned by \(\widehat\Sigma^{-1}\), where \(\widehat\Sigma\) is
estimated from 500 independent Gaussian covariate draws. Within each setting,
we initialize both methods at the same 2000 random starts (Table~\ref{tab:additional-settings}).
Both methods precondition only the regression gradient and update \(\nu\) by ordinary
gradient descent.
The sampled starts include points both within and outside the theorem's
sufficient admissible set.
Under exact whitening, 1977 to 1997 of the 2000 starts are admissible in
\(d=2\); in \(d=20\), the counts are 37, 1655, and 1965 for
condition numbers of \(\Sigma\) equal to 1, 10, and 100, respectively. The experiment therefore also
examines behavior outside the theorem's guarantee.

The observed iteration counts under whitening vary little across the three
condition numbers. With exact whitening, the median count to reach loss
\(10^{-8}\) is 774 to 1076 in \(d=2\) and 3589 to 3806 in \(d=20\).
With estimated whitening, the corresponding counts are 742 to 1084 and
4054 to 4360, respectively. Under
exact whitening, \(\Sigma\) enters only through the transformed target
\(\widetilde c^\ast\) and starts \(\Sigma^{1/2}\theta_0\), so the remaining
variation reflects these rather than the conditioning. With either method, at
most one of the 2000 runs remains above the threshold after 20000 iterations
in any setting. 
\[
    |\nu_K|<10^{-4},\qquad
    |\theta_K^\top\Sigma c^\ast|<10^{-4},\qquad
    |\calL_\Sigma(\vartheta_K)-\tfrac12\|c^\ast\|_\Sigma^2|
    <10^{-3}\,\tfrac12\|c^\ast\|_\Sigma^2,
    \qquad K=20000.
\]
For \(c^\ast=\mathbf0\), we set \(\nu_0=0.70\) and
\(\theta_0=\sinh(0.70)e_1/\sqrt{\Sigma_{11}}\), so that
\(J_\Sigma(\vartheta_0)=0\). Log--log loss slopes of exact whitening fitted at
\(K=2000,2010,\ldots,20000\) are \(-2.01\) for all three condition numbers of \(\Sigma\),
consistent with the \(\calO(K^{-2})\) upper bound in
Theorem~\ref{thm:global_convergence_high_noise_full}(a).
These slopes describe decay over the finite iteration window.

\par\smallskip
\noindent\textbf{Pointwise high-SNR blindness across covariate laws.}
The blindness study extends the setting of
Figure~\ref{fig:high-snr-blindness} in \(d=1\), with
\(\bar\alpha_t=\tfrac12\), \((\theta,\nu)=(0.80,-0.35)\), and
\(\nu^\ast=0.55\), to Gaussian, Student \(t_5\), Student \(t_{10}\),
standardized Gamma(2), and uniform covariates. We integrate \(x\) on a
geometric grid with the trapezoidal rule and use 80-point Gauss--Hermite
quadrature for the response noise.

\begin{figure}[!htb]
    \centering
    \includegraphics[width=0.9\linewidth]{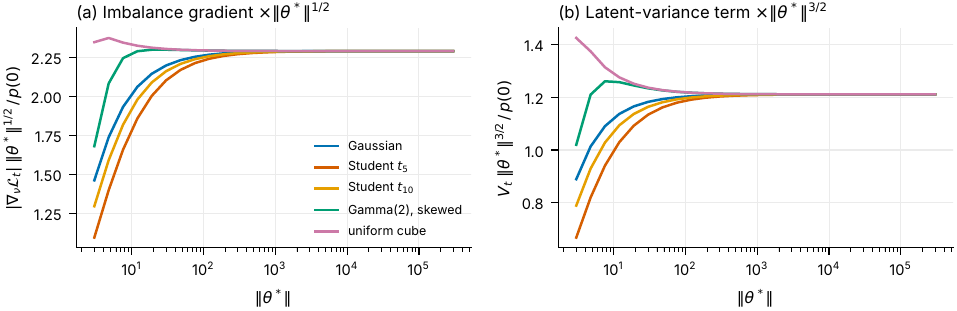}
    \caption{\textbf{Pointwise high-SNR blindness across covariate laws.}
    We use the setting of Figure~\ref{fig:high-snr-blindness}
    (\(d=1\), \(\bar\alpha_t=\tfrac12\), \((\theta,\nu)=(0.80,-0.35)\),
    \(\nu^\ast=0.55\)). Panels \textbf{(a)} and \textbf{(b)} show the rescaled
    imbalance gradient \(|\nabla_\nu\calL_t|\,\|\theta^\ast\|^{1/2}/p(0)\) and
    latent-variance term \(V_t\,\|\theta^\ast\|^{3/2}/p(0)\), respectively.
    Here \(p(0)\) is the density of
    \(\langle\theta^\ast,x\rangle/\|\theta^\ast\|\) at zero.}
    \label{fig:design-blindness}
\end{figure}

For every tested law, the fitted tail
slopes are \(-0.500\) for \(|\nabla_\nu\calL_t|\) and \(-1.500\) for
\(V_t\). These observations support decay as
\(\|\theta^\ast\|^{-1/2}\) and \(\|\theta^\ast\|^{-3/2}\), respectively,
and refine the numerical picture beyond the vanishing limits in
Proposition~\ref{prop:pointwise_sm_blindness}. For Gaussian covariates, the
values at \(\|\theta^\ast\|=3\times10^5\) reproduce those of
Figure~\ref{fig:high-snr-blindness}.
The rescaling in Figure~\ref{fig:design-blindness} follows from the integrands
in the proof of Proposition~\ref{prop:pointwise_sm_blindness}. In this
one-dimensional setting, the leading term of \(\mu_ty_t\) is
\((-1)^{z+1}\bar\alpha_t\theta\theta^\ast x^2\). The factor
\(\sech^2(\mu_ty_t+\nu)\) therefore concentrates the contribution in a
neighborhood of zero whose width is of order \(\|\theta^\ast\|^{-1/2}\).
The tested densities are continuous and positive at zero, so this
neighborhood has probability mass proportional to
\(p(0)\|\theta^\ast\|^{-1/2}\).
Within this neighborhood, \(x^2\) is of order \(\|\theta^\ast\|^{-1}\).
The \(\mu_ty_t\) term in the mixing gradient contains \(\theta^\ast x^2\),
so the growing ground truth coefficient \(\theta^\ast\) offsets the decrease
in \(x^2\). The leading contribution after averaging over this neighborhood
therefore has scale \(p(0)\|\theta^\ast\|^{-1/2}\). For \(V_t\), the prefactor is
\(\mu_t^2/(2\bar\alpha_t)=\theta^2x^2/2\), where the candidate \(\theta\)
remains fixed. The same \(x^2\) therefore contributes an additional factor
\(\|\theta^\ast\|^{-1}\), so the leading contribution has scale
\(p(0)\|\theta^\ast\|^{-1/2}\,\|\theta^\ast\|^{-1}
=p(0)\|\theta^\ast\|^{-3/2}\).
The two rescalings in Figure~\ref{fig:design-blindness} remove these leading
powers and the density factor \(p(0)\). The curves approach common levels
across the five covariate laws, as this calculation suggests.

Proposition~\ref{prop:pointwise_sm_blindness} establishes the vanishing limits
for arbitrary \(d\). The scalar calculation above also applies in higher
dimensions when \(\theta\) is collinear with the fixed ground truth direction
and the projected covariate has a density continuous and positive at zero.
Both \(\mu_t\) and \(\mu_t^\ast\) then depend only on the projection
\(\langle\theta^\ast,\bx\rangle/\|\theta^\ast\|\), whose density determines
\(p(0)\). For noncollinear candidates, the leading term involves the product
\(\langle\theta,\bx\rangle\langle\theta^\ast,\bx\rangle\), and the decay rates
depend on the joint distribution of these projections.


\end{document}